\documentclass[twoside,11pt]{article}

\usepackage{jmlr2e}
\usepackage{microtype}
\usepackage{graphicx}
\usepackage{booktabs} % for professional tables

\usepackage{hyperref}

\usepackage[linesnumbered,ruled,vlined]{algorithm2e}
\usepackage{fontawesome}

\usepackage{cite}
\usepackage{mathtools}
\usepackage{amssymb,amsmath,bm,bbold}
\usepackage{enumerate}
\usepackage{dsfont}
\usepackage{amsfonts}
\usepackage{wrapfig}
\usepackage{caption}
\usepackage{subcaption}
\usepackage[dvipsnames]{xcolor}
\usepackage{cleveref}
\newcommand\myshade{90}
\colorlet{mylinkcolor}{MidnightBlue}
\colorlet{mycitecolor}{MidnightBlue}
\colorlet{myurlcolor}{MidnightBlue}

\hypersetup{
  linkcolor  = mylinkcolor!\myshade!black,
  citecolor  = mycitecolor!\myshade!black,
  urlcolor   = myurlcolor!\myshade!black,
  colorlinks = true,
}
\usepackage[]{natbib}

\newcommand{\norm}[1]{\left\lVert#1\right\rVert}

\newcommand{\bT}{\boldsymbol{\Theta}}
\newcommand{\tr}{\mathsf{tr}}

\newcommand{\Diag}{\mathsf{Diag}}

\newcommand{\ww}{\bm{w}}

\newcommand{\sL}{\mathcal{L}}
\newcommand{\sA}{\mathcal{A}}
\newcommand{\sD}{\mathfrak{d}}

\ShortHeadings{Algorithms for Learning Graphs in Financial Markets}{Cardoso, Ying, and Palomar}
\firstpageno{1}

\begin{document}

\title{Algorithms for Learning Graphs in Financial Markets}

\author{\name Jos\'e Vin\'icius de Miranda Cardoso \email jvdmc@connect.ust.hk \\
       \addr Department of Electronic and Computer Engineering\\
       The Hong Kong University of Science and Technology\\
       Clear Water Bay, Hong Kong
       \AND
       \name Jiaxi Ying \email jx.ying@connect.ust.hk \\
       \addr Department of Electronic and Computer Engineering\\
       The Hong Kong University of Science and Technology\\
       Clear Water Bay, Hong Kong
       \AND
       \name Daniel P. Palomar \email palomar@ust.hk \\
       \addr Department of Electronic and Computer Engineering\\
       Department of Industrial Engineering and Decision Analytics\\
       The Hong Kong University of Science and Technology\\
       Clear Water Bay, Hong Kong
       }

\editor{}

\maketitle

\begin{abstract}%   <- trailing '%' for backward compatibility of .sty file
In the past two decades, the field of applied finance has tremendously benefited from graph theory.
As a result, novel methods ranging from asset network estimation to hierarchical asset selection and
portfolio allocation are now part of practitioners' toolboxes.
In this paper, we investigate the fundamental problem of learning undirected graphical models under Laplacian structural
constraints from the point of view of financial market times series data. In particular,
we present natural justifications, supported by empirical evidence, for the usage of the Laplacian matrix
as a model for the precision matrix of financial assets, while also establishing a direct link that
reveals how Laplacian constraints are coupled to meaningful physical interpretations related to the market index factor
and to conditional correlations between stocks. Those interpretations lead to a set of guidelines that
practitioners should be aware of when estimating graphs in financial markets. In addition, we design
numerical algorithms
based on the alternating direction method of multipliers to learn undirected, weighted graphs
that take into account stylized facts that are intrinsic to financial data such as heavy tails and modularity.
We illustrate how to leverage the learned graphs into practical scenarios
such as stock time series clustering and foreign exchange network estimation. The proposed graph learning algorithms outperform
the state-of-the-art, benchmark methods in an extensive set of practical experiments, evidencing the advantages of adopting
more principled assumptions into the learning framework. Furthermore,
we obtain theoretical and empirical convergence results for the proposed algorithms.
Along with the developed methodologies for graph learning in financial markets,
we release an $\mathsf{R}$ package, called \href{https://github.com/mirca/fingraph}{$\mathsf{fingraph}$ \faGithub}, accommodating the code and data to obtain all the
experimental results.
\end{abstract}


\section{Introduction}

Graph learning from data has been a problem of critical importance
for the statistical graph learning and graph signal processing
fields~\citep{friedman2008, lake2010, witten2011, kalofolias2016, egilmez2017, pavez2018, zhao2019},
with direct impact on applied areas such as unsupervised learning,
clustering \citep{hsieh2012, sun14, tan2015, feiping2016, hao2018, kumar2019, kumar20192}, 
applied finance~\citep{mantegna1999, deprado2016, marti2017}, network topology inference
~\citep{segarra2017, mateos2019, geert2019, mateos2020}, community detection~\citep{fortunato2010,li2018,chen2019},
and graph neural nets~\citep{wu2019, pal2019}.

The basic idea behind graph learning is to answer the following question: given a data matrix whose columns represent
signals (observations) measured at the graph nodes, how can one design a graph representation that ``best" fits such data
matrix without possibly any (or with at most partial) knowledge of the underlying graph structure? By
``graph representation" or ``graph structure'', it is often understood the Laplacian, adjacency, or incidence matrices of the graph, or
even a more general graph shift operator~\citep{marques2016}. In addition, the observed signals need not to live in
regular, ordered spaces and can take arbitrary values, such as categorical and numerical, hence the probability distribution
of the data can be highly unknown. Figure~\ref{fig:graph-signal-full} illustrates such setting.

\begin{figure}[!htb]
  \captionsetup[subfigure]{justification=centering}
\centering
\begin{subfigure}[t]{0.45\textwidth}
 \centering
  \includegraphics[scale=1]{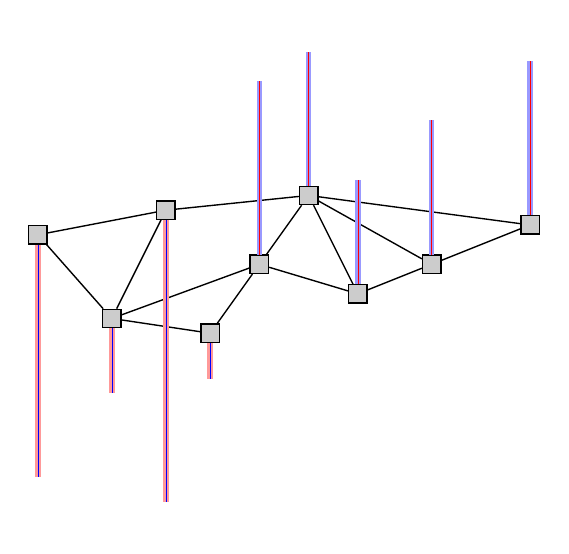}
 \caption{A graph signal.}
 \label{fig:graph-signal}
\end{subfigure}%
\begin{subfigure}[t]{0.45\textwidth}
  \centering
   \includegraphics[scale=0.65]{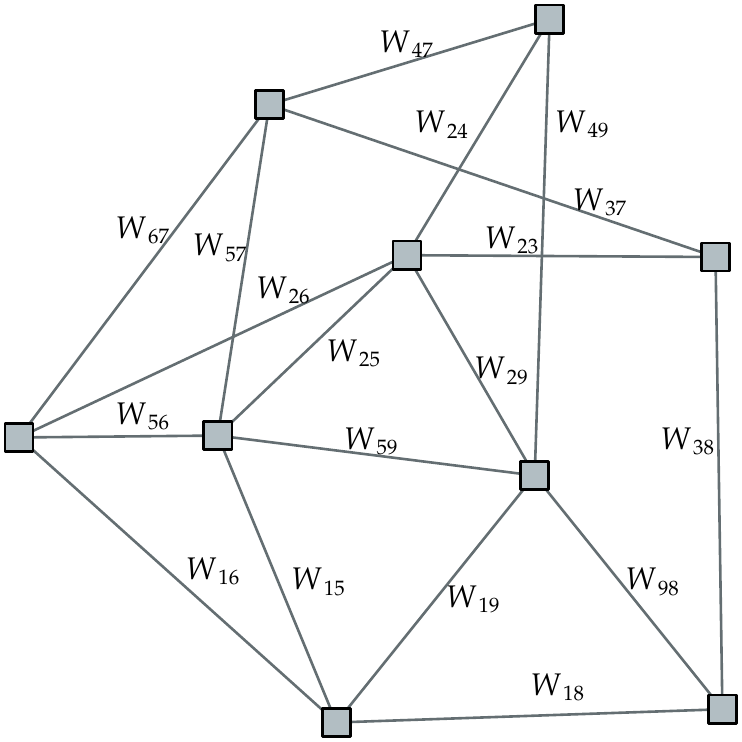}
 \caption{A graph to be estimated on the basis of its graph signals.}
  \label{fig:graph-structure}
 \end{subfigure}%
 \caption{Illustration of a hypothetical signal observed in a graph~(Figure~\ref{fig:graph-signal}). The gray squares represent the graph nodes, the thin black
 lines denote the graph edges, indicating the relationships among nodes, whereas the vertical red bars denote the signal intensities
 measured at each node. Graph learning techniques seek to estimate the underlying graph structure
 (edge weights $W_{ij}$ in Figure~\ref{fig:graph-structure}) through the graph signal measurements.} 
 \label{fig:graph-signal-full}
\end{figure}

Data derived from financial instruments such as equities and foreign exchanges, on the other hand,
are defined on the well-known, ordered time domain (equally-spaced
intraday\footnote{In high-frequency trading systems, in contrast, tick data may not necessarily be uniformly sampled.
This scenario is not contemplated in this work.}, daily, weekly, monthly, etc)
and take on real values whose returns are often modeled by Gaussian processes. Then, the question is: How can graph learning
be a useful tool for financial data analysis? In fact,~\citet{mantegna1999} in his pioneer work showed that
learning topological arrangements, such as graphs, in a stock market context, provides critical
information that reveals economic factors that affect the price data evolution.
The benefits that graph representations bring to applications on financial stocks are vastly discussed in the literature.
A non-exhaustive, yet representative list of examples include:
(i) identifying ``business as usual" and ``crash" periods via asset tree graphs~\citep{onnela2003,onnela2003b},
(ii) understanding portfolio dynamics via the topological properties of simulated minimum spanning trees~\citep{bonanno2003,bonanno2004, mantegna2004},
(iii) constructing networks of companies based on graphs~\citep{onnela2004},
(iv) understanding risks associated with a portfolio~\citep{malevergne2006},
(v) leveraging properties of the learned graph into follow-up tasks such as hierarchical portfolio
designs~\citep{deprado2016,raffinot2018,raffinot2018b},
(vi) mining the relationship structure among investors~\citep{yang2020},
(vii) exploring graph properties such as degree centrality and eigenvector centrality for
market crash detection and portfolio construction~\citep{millington2020}, and
(viii) community detection in financial stock markets~\citep{rama2020,cardoso2020}.

Despite the plethora of applications, learning the structure of general graphical
models is an NP-hard problem~\citep{anima2012} whose importance is critical towards visualizing,
understanding, and leveraging the full potential contained in the data that live in such structures.
Nonetheless, most existing techniques for learning graphs are often unable to impose a particular
graph structure due to their inability to incorporate prior information in the
learning process. More surprisingly, as it is shown in this work, state-of-the-art learning algorithms
fall short when it comes to estimate graphs that posses certain properties such as $k$-components.

Moreover, graph learning frameworks are designed with the operational assumption that
the observed graph signals are Gaussian
distributed~\citep{friedman2008, lake2010, dong2016, kalofolias2016, egilmez2017, zhao2019, kumar2019, ying2020nips},
inherently neglecting situations where there may exist outliers. As a consequence, those methods may
not succeed in fully capturing a meaningful representation of the underlying graph
especially in data from financial instruments, which are known to be heavy-tailed and
skewed~\citep{gourieroux1997, cont2001, tsay2010, harvey2013, feng2015, liu2019a}.

While estimators for connected graphs have been proposed~\citep{dong2016, kalofolias2016, egilmez2017, zhao2019},
some of its properties, such as sparsity, are yet being investigated~\citep{ying2020nips, ying2020}, and
only recently~\citet{kumar2019} have presented estimators for learning more general
graphical structures such as $k$-component, bipartite, and $k$-component bipartite. However,
one major shortcoming in~\citep{kumar2019}
is the lack of constraints on the degrees of the nodes. As we show in this work, the ability to control the degrees of the nodes
is key to avoid trivial solutions while learning $k$-component graphs. 

Recently, \citet{feiping2016,kumar2019,kumar20192} proposed optimization programs for learning the class of $k$-component graphs,
as such class is an appealing model for clustering tasks due to the spectral properties of the Laplacian matrix.
From a financial perspective, clustering financial time-series, such as stock return data, has been an active research
topic~\citep{mantegna1999, dose2005, gautier2016, marti2017, marti2017a}. However, these works rely primarily on
hierarchical clustering techniques and on the assumption that the underlying graph has a tree structure,
which does bring advantages due to its hierarchical clustering properties, but also have been shown to be
unstable~\citep{carlsson2010, lemieux2014, marti2015} and not suitable when the data is not Gaussian distributed~\citep{donnat2016}.
In this work, on the contrary, we tackle the problem of clustering stocks from a probabilistic perspective, similarly to the
approach layed out by~\citet{kumar20192, kumar2019}, where the Laplacian matrix of a $k$-component graph is assumed to model the pairwise
conditional correlations between stocks. A crucial advantage of this approach is that we can consider more realistic probabilistic
assumptions such as heavy tails.

In practice, prior information about clusters of stocks is available via sector classification systems such as
the Global Industry Classification Standard (GICS)~\citep{gics2006,msci2018} or the
Industry Classification Benchmark (ICB)~\citep{schreiner2019}.
However, more often than not, stocks have impacts on multiple industries, \textit{e.g.}, the evident case of technology
companies, such as Amazon, Apple, Google, and Facebook, whose influence on prices affect stocks not only in their own sector, but spans across multiple sectors.
One reason for this phenomena is the myriad of services offered by those companies,
resulting in challenges to precisely pin point which stock market sector they should belong to.

Motivated by practical applications in finance such as clustering of financial instruments and network estimation,
we investigate the problem of learning graph matrices whose structure follow that of a Laplacian
matrix of an undirected weighted graph.

The main contributions of our paper include:
\begin{enumerate}
  \item As far as the authors are aware of, we for the first time provide interpretations for Laplacian constraints
  of graphs from the perspective of stock market data. Those interpretations naturally lead to meaningful and intuitive guidelines
  on the data pre-processing required for learning graphs from financial data. 
  \item We show that rank constraints alone, a practice often used by state-of-the-art methods,
  are not sufficient to learn non-trivial $k$-component graphs. We achieve learning of $k$-component graphs without isolated nodes
  by leveraging linear constraints on the node degrees of the graph.
  \item We propose novel formulations to learn $k$-component graphs and heavy-tailed graphs, which are solved via carefully
  designed Alternating Direction Method of Multipliers (ADMM) algorithms. In addition, we establish theoretical convergence
  guarantees for the proposed algorithms along with experiments on their empirical convergence.
  The proposed algorithms can be easily extended to account for additional linear constraints on the graph weights.
  \item We present extensive practical results that showcase the advantage of the operational assumptions used in the proposed
  algorithms when compared to state-of-the-art methods. Along with the methods proposed in this paper,
  we release an $\mathsf{R}$ package, called \href{https://github.com/mirca/fingraph}{$\mathsf{fingraph}$ \faGithub}, 
  containing fast, unit-tested code that implements the proposed algorithms and it is publicly available at: \url{https://github.com/mirca/fingraph}.
\end{enumerate}

\subsection{Notation}
The reals, nonnegative reals, and positive reals fields are denoted as $\mathbb{R}$, $\mathbb{R}_{+}$, and $\mathbb{R}_{++}$, respectively.
We use the abbreviation iff to denote ``if and only if''.
Scalars and real-valued random variables are denoted by lowercase roman letters like $x$.
Matrices (vectors) are denoted by bold, italic, capital (lowercase) roman
letters like $\bm{X}$, $\bm{x}$. Vectors are assumed to be column vectors.
The $(i,j)$ element of a matrix $\bm{X} \in \mathbb{R}^{n\times p}$ is
denoted as $X_{ij}$. The $i$-th row (column) of $\bm{X}$ is denoted as $\bm{x}_{i,*} \in \mathbb{R}^{p\times 1}$
($\bm{x}_{*,i} \in \mathbb{R}^{n\times 1}$).
The $i$-th element of a vector $\bm{x}$ is denoted as $x_i$. The transpose of $\bm{X}$ is denoted as $\bm{X}^\top$.
The identity matrix of order $p$ is denoted as $\bm{I}_p$.
The Moore-Penrose inverse of a matrix $\bm{X}$ is denoted as $\bm{X}^\dagger$.
The trace of a square matrix, \textit{i.e.}, the sum of elements on the principal diagonal, is denoted as $\mathsf{tr}(\bm{X})$.
The inner product between two matrices $\bm{X}$, $\bm{Y}$ is denoted as
$\langle \bm{X}, \bm{Y} \rangle \triangleq \mathsf{tr}\left(\bm{X}^\top\bm{Y}\right)$. 
The element-wise sum of the absolute values and the Frobenius norm of a matrix
$\bm{X}$ are denoted as $\norm{\bm{X}}_{1} = \sum_{ij}\vert X_{ij}\vert$ and $\norm{\bm{X}}_{\text{F}} = \sqrt{\mathsf{tr}\left(\bm{X}^\top\bm{X}\right)}$, respectively.
For $\bm{x} \in \mathbb{R}^p$, $\Vert \bm{x}\Vert_2$ stands for the usual Euclidean norm of $\bm{x}$.
If $\bm{A}$ is a symmetric matrix, $\lambda_\mathsf{max}(\bm{A})$ denotes the maximum eigenvalue of $\bm{A}$.
Let $\bm{A}$, $\bm{B}$ be two self-adjoint matrices. We write $\bm{A} \succeq \bm{B}$ ($\bm{A} \succ \bm{B}$)
iff $\bm{A} - \bm{B}$ is nonnegative (positive) definite.
The symbols $\mathbf{1}$ and $\mathbf{0}$ denote the all ones and zeros vectors of appropriate dimension, respectively.
The operator $\mathsf{diag}: \mathbb{R}^{p\times p} \rightarrow \mathbb{R}^{p}$
extracts the diagonal of a square matrix.
The operator $\mathsf{Diag}: \mathbb{R}^{p} \rightarrow \mathbb{R}^{p\times p}$
creates a diagonal matrix with the elements of an input vector along its diagonal.
$(\bm{x})^+$ denotes the projection of $\bm{x}$ onto the nonnegative orthant, \textit{i.e.},
the elementwise maximum between $\mathbf{0}$ and $\bm{x}$.

\section{Background and Related Works}

An undirected, weighted graph is usually denoted as a triple $\mathcal{G} = \left(\mathcal{V}, \mathcal{E}, \bm{W}\right)$,
where $\mathcal{V} = \left\{1, 2, \dots, p\right\}$ is the vertex (or node) set,
$\mathcal{E} \displaystyle \subseteq \left\{\left\{u, v\right\}: u,v \in \mathcal{V}\right\}$ is the edge set, that
is, a subset of the set of all possible unordered pairs of $p$ nodes
such that $\{u, v\} \in \mathcal{E}$ iff nodes $u$ and $v$ are connected.
$\bm{W} \in \mathbb{R}_{+}^{p\times p}$ is the symmetric weighted adjacency matrix that satisfies
$W_{ii} = 0, W_{ij} > 0 ~\text{iff}~ \{i,j\} \in \mathcal{E} ~\text{and} ~W_{ij} = 0,~\text{otherwise}$.
We denote a graph as a 4-tuple $\mathcal{G} = \left(\mathcal{V}, \mathcal{E}, \bm{W}, f_t\right)$, where
$f_{t} : \mathcal{V} \rightarrow \left\{1, 2, \dots, t \right\}$ is a function that associates
a single type (label) to each vertex of the graph, where $t$ is the number of possible types.
This extension is necessary for computing certain graph properties of practical interest such as graph modularity.
We denote the number of elements in $\mathcal{E}$ by $\vert\mathcal{E}\vert$.
The combinatorial, unnormalized graph Laplacian matrix $\bm{L}$ is defined, as usual, as $\bm{L} \triangleq \bm{D} - \bm{W}$, where
$\bm{D} \triangleq \Diag(\bm{W}\mathbf{1})$ is the degree matrix.

An Improper Gaussian Markov Random Field (IGMRF)~\citep{rue2005, slawski2014}
of rank $p-k$, $k \geq 1$, is denoted as a $p$-dimensional, real-valued, Gaussian
random variable $\bm{x}$ with mean vector $\mathbb{E}\left[\bm{x}\right] \triangleq \boldsymbol{\mu}$ and rank-deficient precision
matrix $\mathbb{E}\left[(\bm{x} - \boldsymbol{\mu})(\bm{x} - \boldsymbol{\mu})^\top\right]^\dagger \triangleq \boldsymbol{\Xi}$.
The probability density function of $\bm{x}$ is then given as
\begin{equation}
  p(\bm{x}) \propto \sqrt{\mathrm{det}^*\left(\boldsymbol\Xi\right)}
  \exp\left\{-\dfrac{1}{2}(\bm{x} - \boldsymbol\mu)^\top\boldsymbol\Xi(\bm{x} - \boldsymbol\mu)\right\},
\end{equation}
where
$\mathrm{det}^*(\boldsymbol{\Xi})$ is the pseudo (also known as generalized) determinant of $\boldsymbol{\Xi}$, \textit{i.e.},
the product of its positive eigenvalues~\citep{knill2014}.

The data generating process is assumed to be a zero-mean, IGMRF $\bm{x} \in \mathbb{R}^{p}$, 
such that $x_i$ is the random variable generating a signal measured at node $i$,
whose rank-deficient precision matrix is modeled as a graph Laplacian matrix.
This model is also known as Laplacian constrained Gaussian Markov Random Field (LGMRF)~\citep{ying2020nips}.
Assume we are given $n$ observations of $\bm{x}$, \textit{i.e.}, 
$\bm{X} = \left[\bm{x}^\top_{1,*}, \bm{x}^\top_{2,*}, \dots, \bm{x}^\top_{n,*}\right]^\top$, $\bm{X} \in \mathbb{R}^{n\times p}$,
$\bm{x}_{i,*} \in \mathbb{R}^{p\times 1}$.
The goal of graph learning algorithms is to learn a Laplacian matrix, or equivalently an adjacency matrix,
given only the data matrix $\bm{X}$, \textit{i.e.}, often without any knowledge of $\mathcal{E}$ and $f_t$.

To that end, the classical penalized Maximum Likelihood Estimator (MLE) of the Laplacian-constrained precision matrix of $\bm{x}$,
on the basis of the observed data $\bm{X}$, may be formulated as the following optimization program:
\begin{equation}
\begin{array}{cl}
  \underset{\bm{L} \succeq \mathbf{0}}{\mathsf{minimize}} & \mathsf{tr}\left(\bm{L}\bm{S}\right)
  - \log\mathrm{det}^{*}\left(\bm{L}\right) + h_{\boldsymbol\alpha}(\bm{L}),\\
  \mathsf{subject~to} & \bm{L}\mathbf{1} = \mathbf{0},~L_{ij} = L_{ji} \leq 0,
\end{array}
\label{eq:laplacian-learning}
\end{equation}
where $\bm{S}$ is a similarity matrix, \textit{e.g.}, the sample covariance (or correlation) matrix $\bm{S} \propto \bm{X}^\top\bm{X}$,
and $h_{\boldsymbol\alpha}(\bm{L})$ is a regularization function, with hyperparameter vector $\boldsymbol\alpha$,
to promote certain properties on $\bm{L}$, such as sparsity or low-rankness.

Problem~\eqref{eq:laplacian-learning} is a fundamental problem in the graph signal processing field
that has served as a cornerstone for many extensions, primarily those involving the inclusion
of structure onto $\bm{L}$~\citep{egilmez2017,pavez2018,kumar20192,kumar20193, kumar2019}.
Even though Problem~\eqref{eq:laplacian-learning} is convex, provided we assume a convex choice for $h_{\boldsymbol\alpha}(\cdot)$, it is not
adequate to be solved by Disciplined Convex Programming (DCP) languages, such as $\mathsf{cvxpy}$~\citep{diamond2016cvxpy},
particularly due to
scalability issues related to the computation of the term $\log\mathrm{det}^*(\bm{L})$~\citep{egilmez2017, zhao2019}. Indeed, recently,
considerable efforts have been
directed towards the design of scalable, iterative algorithms based on Block Coordinate Descent (BCD)~\citep{wright2015},
Majorization-Minimization (MM)~\citep{sun2017}, and ADMM~\citep{boyd2011}
to solve Problem~\eqref{eq:laplacian-learning} in an efficient fashion, \textit{e.g.},
\citep{egilmez2017}, \citep{zhao2019}, \citep{kumar2019}, and \citep{ying2020nips}, just to name a few.

To circumvent some of those scalability issues
related to the computation of the term $\log\mathrm{det}^*(\bm{L})$,
~\citet{lake2010} proposed the following relaxed version with an $\ell_1$-norm penalization:
\begin{equation}
  \begin{array}{cl}
    \underset{\bm{\tilde{L}} \succ \mathbf{0}, \bm{W}, \sigma > 0}{\mathsf{maximize}} & - \mathsf{tr}\left(\bm{\tilde{L}}\bm{S}\right)
    + \log\mathrm{det}\left(\bm{\tilde{L}}\right) - \alpha\Vert \bm{W}\Vert_{1},\\
    \mathsf{subject~to} & \bm{\tilde{L}} = \Diag(\bm{W}\mathbf{1}) - \bm{W} + \sigma\bm{I}_{p}, \\
                        & \mathsf{diag}(\bm{W}) = \mathbf{0}, ~ W_{ij} = W_{ji} \geq 0, ~\forall~ i,j \in \{1, 2, ..., p\}.
  \end{array}
  \label{eq:lake2010}
\end{equation}

In words, Problem~\eqref{eq:lake2010} relaxes the original problem by forcing the precision matrix
to be positive definite through the introduction of the term $\sigma\bm{I}_{p}$, which bounds the
minimum eigenvalue of $\bm{\tilde{L}}$ to be at least $\sigma$, and thus the generalized determinant
can be replaced by the usual determinant. 
Although the technical issues related to the generalized determinant have been seemingly dealt with
(albeit in an indirect way), in this formulation, there are twice
as many variables to be estimated, which turns out to be prohibitive when designing practical, scalable algorithms,
and the applicability of $\mathsf{cvx}$, like it was done in ~\citep{lake2010}, is only possible in small scale ($p \approx 50$) scenarios.

In order to solve this scalability issue,~\citet{moghaddam2016} and \citet{egilmez2017} proposed
customized BCD algorithms~\citep{shwartz2011, ankan2013, wright2015}
to solve Problem~\eqref{eq:laplacian-learning}
assuming an $\ell_1$-norm regularization, \textit{i.e.},
$h_{\alpha}(\bm{L}) \triangleq \alpha \sum_{i\neq j}\vert\bm{L}_{ij}\vert = -\alpha \sum_{i\neq j}\bm{L}_{ij}$,
in order to promote sparsity on the resulting estimated Laplacian matrix.
However, as recently shown by~\citet{ying2020nips, ying2020}, in contrast to common practices,
the $\ell_1$-norm penalization surprisingly leads to denser graphs to the point that,
for a large value of $\alpha$, the resulting graph will be fully connected
with uniformly distributed graph weights.

On the other hand, due to such nuisances involved in dealing with the term $\log\mathrm{det}^*(\bm{L})$,
several works departed from the LGMRF formulation altogether. Instead, they focused on the assumption that the underlying signals in a graph are smooth~\citep{kalofolias2016, dong2016,chepuri2017}.
In its simplest form, learning a smooth graph from a data matrix $\bm{X} \in \mathbb{R}^{n \times p}$ is tantamount to finding an adjacency matrix $\bm{W}$ that minimizes the
Dirichlet energy, \textit{i.e.},
\begin{equation}
  \begin{array}{cl}
   \underset{\bm{W}}{\mathsf{minimize}} &
  \frac{1}{2}\sum_{i,j}W_{ij}\norm{\bm{x}_i - \bm{x}_j}^2_2, \\
  \mathsf{subject~to} & W_{ij} = W_{ji} \geq 0, ~ \mathsf{diag}(\bm{W}) = \mathbf{0}.
  \end{array}
  \label{eq:signal-smoothness1}
\end{equation}

Problem~\eqref{eq:signal-smoothness1} can also be equivalently expressed in terms of the Laplacian matrix:
\begin{equation}
  \begin{array}{cl}
   \underset{\bm{L} \succeq \mathbf{0}}{\mathsf{minimize}} &
   \mathsf{tr}\left(\bm{X}\bm{L}\bm{X}^\top\right),\\
  \mathsf{subject~to} & \bm{L}\mathbf{1} = \mathbf{0}, L_{ij} = L_{ji} \leq 0.
  \end{array}
  \label{eq:signal-smoothness2}
\end{equation}

In order for Problems~\eqref{eq:signal-smoothness1} and~\eqref{eq:signal-smoothness2} to be well-defined, \textit{i.e.},
to avoid the trivial solution $\bm{W} = \mathbf{0}$ $(\bm{L} = \mathbf{0})$,
several constraints have been proposed in the literature.
For instance,~\citet{dong2016} proposed one of the first estimators for
graph Laplacian as the following nonconvex optimization program:
\begin{equation}
  \begin{array}{cl}
    \underset{\bm{L} \succeq \mathbf{0}, \bm{Y} \in \mathbb{R}^{n\times p}}{\mathsf{minimize}} &
    \norm{\bm{X} - \bm{Y}}^{2}_{\mathrm{F}} + \alpha \mathsf{tr}\left(\bm{Y}\bm{L}\bm{Y}^\top\right) + \eta\norm{\bm{L}}^2_{\text{F}},\\
    \mathsf{subject~to} & \bm{L}\mathbf{1} = \mathbf{0},~L_{ij} = L_{ji} \leq 0,
  ~\mathsf{tr}\left(\bm{L}\right) = p,
  \end{array}
  \label{eq:signal-smoothness-dong}
\end{equation}
where the constraint $\mathsf{tr}\left(\bm{L}\right) = p$ is imposed to
fix the sum of the degrees of the graph, and $\alpha$ and $\eta$ are positive, real-valued hyperparameters that control
the amount of sparsity in the estimated graph. More precisely, for a given fixed $\alpha$,
increasing (decreasing) $\eta$ leads to sparser (denser) graphs. 
\citet{dong2016} adopted an iterative alternating minimization scheme in order to
find an optimal point of Problem~\eqref{eq:signal-smoothness-dong}, whereby at each iteration one of the variables is fixed
while the solution is found for the other variable. However,
the main shortcoming of formulation~\eqref{eq:signal-smoothness-dong}
is that it does not scale well for big data sets due to the
update of the variable $\bm{Y} \in \mathbb{R}^{n\times p}$, which scales as a function of the number of observations
$n$. In some graph learning problems in financial markets, the number of price recordings may be orders of magnitude larger
than the number of nodes (financial assets), particularly in high frequency trading scenarios~\citep{kirilenko2017}.

Yet following the smooth signal assumption, \citet{kalofolias2016} proposed a convex formulation as follows:
\begin{equation}
  \begin{array}{cl}
    \underset{\bm{W}}{\mathsf{minimize}}&
    \frac{1}{2}\tr\left({\bm{W}\bm{Z}}\right)
    -\alpha\mathbf{1}^\top\log(\bm{W}\mathbf{1}) +
    \frac{\eta}{2}\norm{\bm{W}}^{2}_{\mathrm{F}}, \\
    \mathsf{subject~to} & W_{ij} = W_{ji} \geq 0,~\mathsf{diag}(\bm{W}) = \mathbf{0},
  \end{array}
  \label{eq:smooth_static_graph}
\end{equation}
where $Z_{ij} \triangleq \norm{\bm{x}_{*, i} - \bm{x}_{*, j}}^2_2$, $\alpha$ and $\eta$ are positive, real-valued hyperparameters that control
the amount of sparsity in the estimated graph, with the same interpretation as in Problem~\eqref{eq:signal-smoothness-dong},
and $\log(\bm{W}\mathbf{1})$, assumed to be evaluated element-wise, is a regularization term added to avoid the degrees of the graph from becoming zero.

Problem~\eqref{eq:smooth_static_graph} is convex and can be solved via primal-dual, ADMM-like
algorithms~\citep{komodakis2015}.
It can be seen that the objective function in Problem~\eqref{eq:smooth_static_graph} is actually
an approximation of that of Problem~\eqref{eq:laplacian-learning}. From Hadamard's inequality~\citep{rozanski2017}, we have
\begin{equation}
\log \mathrm{det}^*(\bm{L}) \leq  \log\prod_{i=1}^{p}L_{ii} = \sum_{i=1}^p\log L_{ii} = \mathbf{1}^\top\log(\bm{W}\mathbf{1}).
\end{equation}
Therefore, Problem~\eqref{eq:smooth_static_graph}
can be thought of as an approximation of the penalized maximum likelihood estimator with a Frobenius
norm regularization in order to bound the graph weights.

The graph learning formulations previously discussed are only applicable to learn connected graphs.
Learning graphs with a \textit{prior} structure, \textit{e.g.}, $k$-component graphs,
poses a considerably higher challenge, as the dimension of the nullspace of the Laplacian matrix $\bm{L}$ is equal to the number
of components of the graph~\citep{chung1997}. Therefore, algorithms have to ensure that the algebraic multiplicity of the the
0 eigenvalue is equal to $k$. However, the latter condition is not sufficient to rule out the space of ``trivial" $k$-component graphs,
\textit{i.e.}, graphs with isolated nodes. In addition to the rank (or nullity) constraint on $\bm{L}$,
it is necessary to specify a constraint on the degrees of the graph.

Recent efforts have been made to introduce theoretical results from spectral graph theory~\citep{chung1997} into
practical optimization programs.
For instance, a formulation to estimate $k$-component graphs based on the smooth signal approach was proposed in~\citep{feiping2016}.
More precisely, they proposed the Constrained Laplacian-rank ($\mathsf{CLR}$) algorithm, which works in two-stages. On the first stage it estimates a connected graph using, \textit{e.g.},
the solution to Problem~\eqref{eq:smooth_static_graph}, and then on the second stage it heuristically projects the graph onto the set
of Laplacian matrices of dimension $p$ with rank $p-k$, where $k$ is the given number of graph components.
This approach is summarized in the following two stages:
\begin{enumerate}
  \item Obtain an initial affinity matrix $\bm{A}^\star$ as the optimal value of: %$\bm{L}_0 = \mathsf{Diag}(\bm{W}_0\mathbf{1}) - \bm{W}_0$ as the solution of Problem~\eqref{eq:stage1}.
\begin{equation}
  \begin{array}{cl}
    \underset{\bm{A}}{\mathsf{minimize}}&
    \frac{1}{2}\tr\left({\bm{A}\bm{Z}}\right) +
    \frac{\eta}{2}\norm{\bm{A}}^{2}_{\text{F}}, \\
    \mathsf{subject~to} & \mathsf{diag}\left(\bm{A}\right) = \bm{0},~\bm{A}\mathbf{1} = \mathbf{1},~{A}_{ij} \geq 0 ~ \forall i,j
  \end{array}
  \label{eq:stage1}
\end{equation}
\item Find a projection of $\bm{A}^\star$ such that
$\bm{L}^\star = \Diag(\frac{\bm{B}^{\star\top} + \bm{B}^\star}{2}) - \frac{\bm{B}^{\star\top} + \bm{B}^\star}{2}$
has rank $p-k$:
\begin{equation}
\begin{array}{cl}
  \underset{\bm{B}, \bm{L}\succeq\mathbf{0}}{\mathsf{minimize}} 
  & \Vert\bm{B} - \bm{A}^\star\Vert_\text{F}^2,\\
  \mathsf{subject~to} & \bm{B}\mathbf{1} = \mathbf{1},~\mathsf{rank}(\bm{L}) = p-k,\\
  & \bm{L} = \Diag(\frac{\bm{B}^{\top} + \bm{B}}{2}) - \frac{\bm{B}^{\top} + \bm{B}}{2}
  \end{array}
\end{equation}
where $k$ is the desired number of graph components.
  \label{eq:stage2}
\end{enumerate}

Spectral constraints on the Laplacian matrix are an intuitive way to recover $k$-component graphs as the multiplicity
of its zero eigenvalue, \textit{i.e.}, the nullity of $\bm{L}$, dictates the number of components of a graph. 
The first framework to impose structures on the estimated Laplacian matrix under the LGMRF model
was proposed by Kumar~\textit{et al.}~\citep{kumar20192, kumar2019}, through the use of spectral constraints, as follows:
\begin{equation}
  \begin{array}{cl}
    \underset{\bm{L}, \bm{U}, \boldsymbol\lambda}{\mathsf{minimize}} & \mathsf{tr}\left(\bm{L}\bm{S}\right)
    - \displaystyle\sum_{i=1}^{p-k}\log\left(\lambda_{i}\right) + \frac{\eta}{2} \norm{\bm{L} - \bm{U}\mathsf{Diag}(\boldsymbol\lambda)\bm{U}^\top}^{2}_{\mathrm{F}},\\
    \mathsf{subject~to} & \bm{L} \succeq \mathbf{0},~\bm{L}\mathbf{1} = \mathbf{0},~L_{ij} = L_{ji} \leq 0,\\
    & \bm{U}^\top\bm{U} = \bm{I}, ~\bm{U} \in \mathbb{R}^{p\times (p-k)}, \\
    & \boldsymbol{\lambda}\in\mathbb{R}^{p-k}_{+},~ c_1 < \lambda_{1} < \dots < \lambda_{p-k} < c_2.
  \end{array}
  \label{eq:gmrf}
\end{equation}
where the term $\frac{\eta}{2} \norm{\bm{L} - \bm{U}\mathsf{Diag}(\boldsymbol\lambda)\bm{U}^\top}^{2}_{\mathrm{F}}$, often called
spectral regularization, is added as a penalty term to indirectly promote $\bm{L}$ to have the same rank as
$\bm{U}\mathsf{Diag}(\boldsymbol\lambda)\bm{U}^\top$,
\textit{i.e.}, $p-k$, $k$ is the number of components of the graph to be chosen \textit{a priori}, and $\eta > 0$
is a hyperparameter that controls the penalization on the spectral factorization of $\bm{L}$, and $c_1$ and $c_2$
are positive, real-valued constants employed to promote bounds on the eigenvalues of $\bm{L}$.

Note that Problem~\eqref{eq:gmrf} learns a $k$-component graph without the need for a two-stage algorithm.
However, a clear caveat of this formulation is that it does not control the degrees of the nodes in the graph,
which may result in a trivial solution that contains isolated nodes, turning out not to be useful for
clustering tasks especially when applied to noisy data sets or to data sets that are not significantly Gaussian distributed.
In addition, choosing values for hyperparameters $\eta$, $c_1$, and $c_2$, is often an intricate task.

\section{Interpretations of Graph Laplacian Constraints for Financial Data}
\label{sec:interpretations}
In this section, we present novel interpretations and motivations for the Laplacian constraints
from the point of view of graphs learned from financial markets data. Those interpretations
lead to paramount guidelines that users may benefit from when applying graph learning algorithms
in financial problems. In addition, we provide sound justifications for the usage of the Laplacian
matrix as a model for the inverse correlation matrix of financial assets.

Graphical representations of data are increasingly important tools in financial signal processing
applied to uncover hidden relationships between variables~\citep{deprado2016, marti2017}.
In financial markets, one is generally interested
in learning quantifiable dependencies among assets and how to leverage them
into practical scenarios such as portfolio design and crisis forecasting.

Arguably, one of the most successful methods to estimate sparse graphs is the Graphical Lasso~\citep{friedman2008, banerjee2008},
modeled as the solution to the following convex optimization problem:
\begin{equation}
  \begin{array}{cl}
    \underset{\boldsymbol{\Sigma}^{-1} \succ \mathbf{0}}{\mathsf{minimize}} & \mathsf{tr}\left(\boldsymbol{\Sigma}^{-1}\bm{S}\right)
    - \log\mathrm{det}\left(\boldsymbol{\Sigma}^{-1}\right) + \alpha \Vert\boldsymbol{\Sigma}^{-1}\Vert_{1},
  \end{array}
  \label{eq:ggm}
\end{equation}
where $\bm{S} \propto \bm{X}^\top\bm{X}$ is an empirical covariance (or correlation) matrix.
The solution to Problem~\eqref{eq:ggm} can be efficiently computed via
the well-known glasso algorithm~\citep{friedman2008,sustik2012}.

While this model has been extremely successful in numerous fields, imposing a Laplacian structure onto $\boldsymbol{\Sigma}^{-1}$
brings significant benefits in financial data settings. To see that, we empirically evaluate
the out-of-sample log-likelihood using three models:
\begin{enumerate}
  \item Graphical Lasso as defined in~\eqref{eq:ggm}.
  \item Multivariate Totally Positive of Order 2 (MTP$_{2}$), given as 
  \begin{equation}
    \begin{array}{cl}
      \underset{\substack{\boldsymbol{\Sigma}^{-1} \succ \mathbf{0}, \\ \boldsymbol{\Sigma}^{-1}_{ij} \leq 0, \forall i \neq j}}{\mathsf{minimize}} & \mathsf{tr}\left(\boldsymbol{\Sigma}^{-1}\bm{S}\right)
      - \log\mathrm{det}\left(\boldsymbol{\Sigma}^{-1} \right) + \alpha \Vert\boldsymbol{\Sigma}^{-1}\Vert_{1}.
    \end{array}
    \label{eq:MTP_2}
  \end{equation}
  \item Laplacian GMRF (LGMRF) as defined in~\eqref{eq:laplacian-learning}, without regularization.\footnote{
    We do not regularize the LGMRF model with the $\ell_1$-norm as it leads to denser graphs~\citep{ying2020, ying2020nips}.}
\end{enumerate}
  
We collect log-returns data from $p = 50$ randomly chosen stocks from the S\&P500 index during the period
between Jan. 4th 2005 to Jul. 1st 2020 totalling $n = 3900$ observations. We subdivide the observations into
$26$ sequential datasets each of which containing $150$ observations. For the $i$-th dataset, we estimate the 
models~\eqref{eq:ggm} and \eqref{eq:MTP_2} for different values
of the hyperparameter $\alpha$, and compute their log-likelihood using the $(i+1)$-th dataset. We then average the log-likelihood
measurements over the datasets. In this fashion, we can infer how well these models generalize to unseen data.

Figure~\ref{fig:out-of-sample-ll} shows the log-likelihood measurements in this experiment.
We can readily notice that not only the LGMRF model has the higher explanatory power among the considered models,
but it is also the simplest of them, as it does not contain any hyperparameter.

\begin{figure}[!htb]
  \centering
  \includegraphics[scale=0.6]{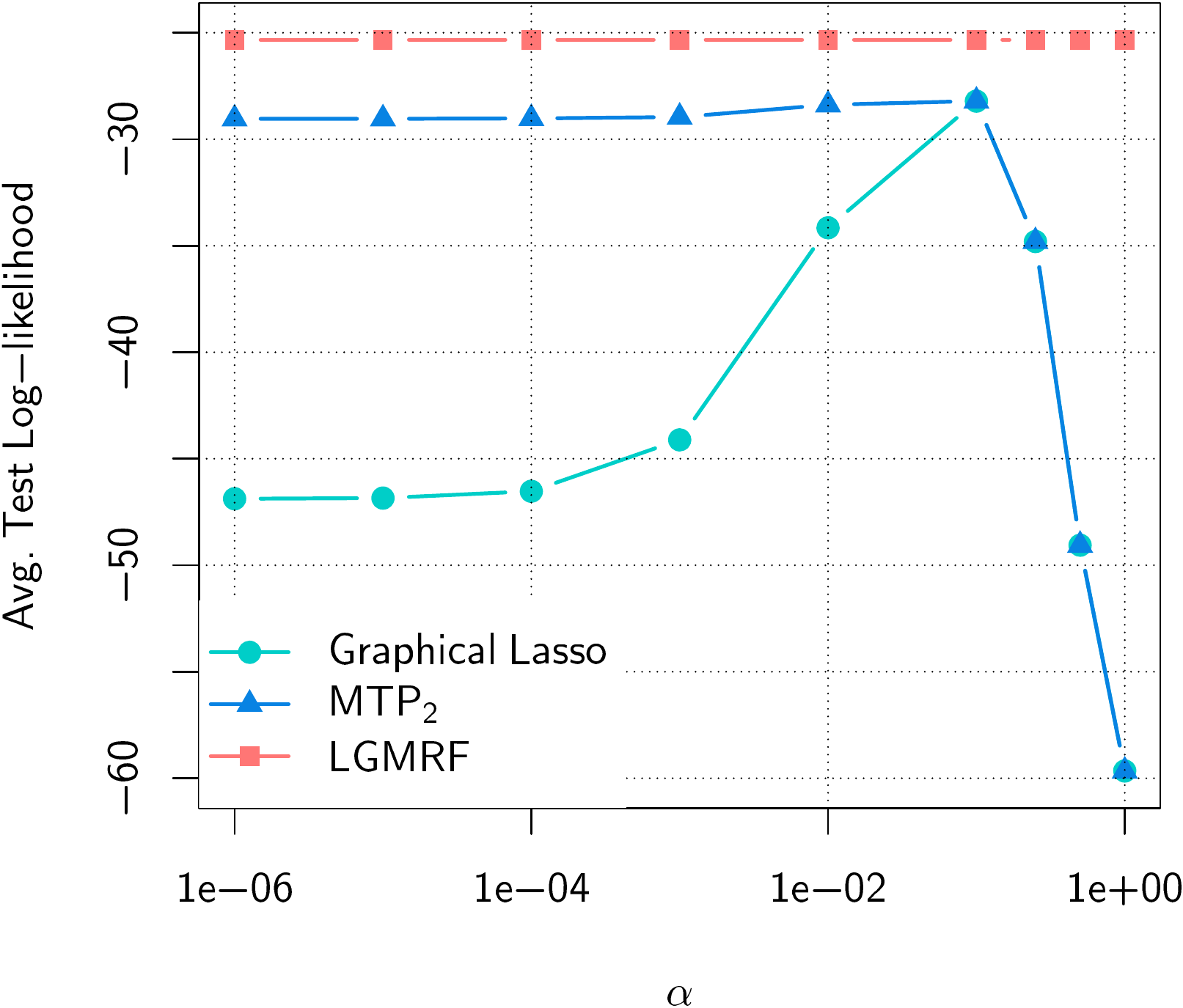}
  \caption{Average log-likelihood in out-of-sample data for different precision
           matrix estimation models as a function of the sparsity promoting hyperparameter $\alpha$.}
  \label{fig:out-of-sample-ll}
\end{figure}

In addition, we plot one instance of the estimated networks from each of the
models. Interestingly, Figure~\ref{fig:glasso-loglike} reveals that Graphical Lasso
estimates most conditional correlations as positive (blue edges),
with only a few negative ones (red edges). In addition, both Graphical Lasso and
$\mathrm{MTP}_2$ (Figure~\ref{fig:mtp2-loglike}) do not clearly uncover strong connections between clearly correlated stocks.
The LGMRF model (Figure~\ref{fig:lap-loglike}), on the other hand, displays vividly the interactions between evidently correlated nodes, \textit{e.g.},
\{$\mathsf{CAH}$ and $\mathsf{MCK}$\} and \{$\mathsf{SIVB}$ and $\mathsf{PBCT}$\},
which are companies in the health care industry and bank holdings, respectively.

\begin{figure}[!htb]
  \captionsetup[subfigure]{justification=centering}
  \centering
  \begin{subfigure}[t]{0.32\textwidth}
      \centering
      \includegraphics[scale=.38]{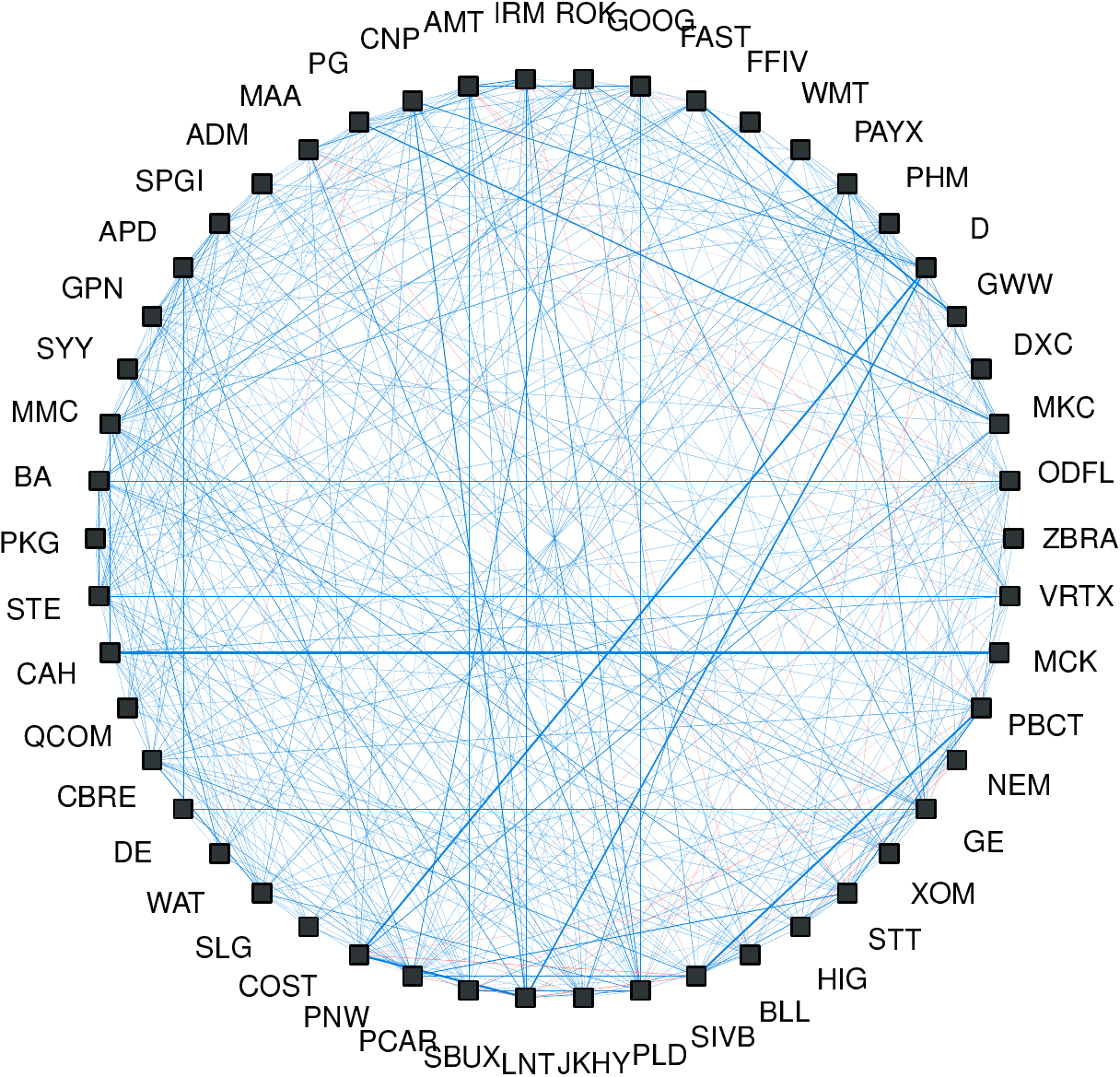}
      \caption{Graphical Lasso.}
      \label{fig:glasso-loglike}
  \end{subfigure}%
  ~
  \begin{subfigure}[t]{0.32\textwidth}
      \centering
      \includegraphics[scale=.38]{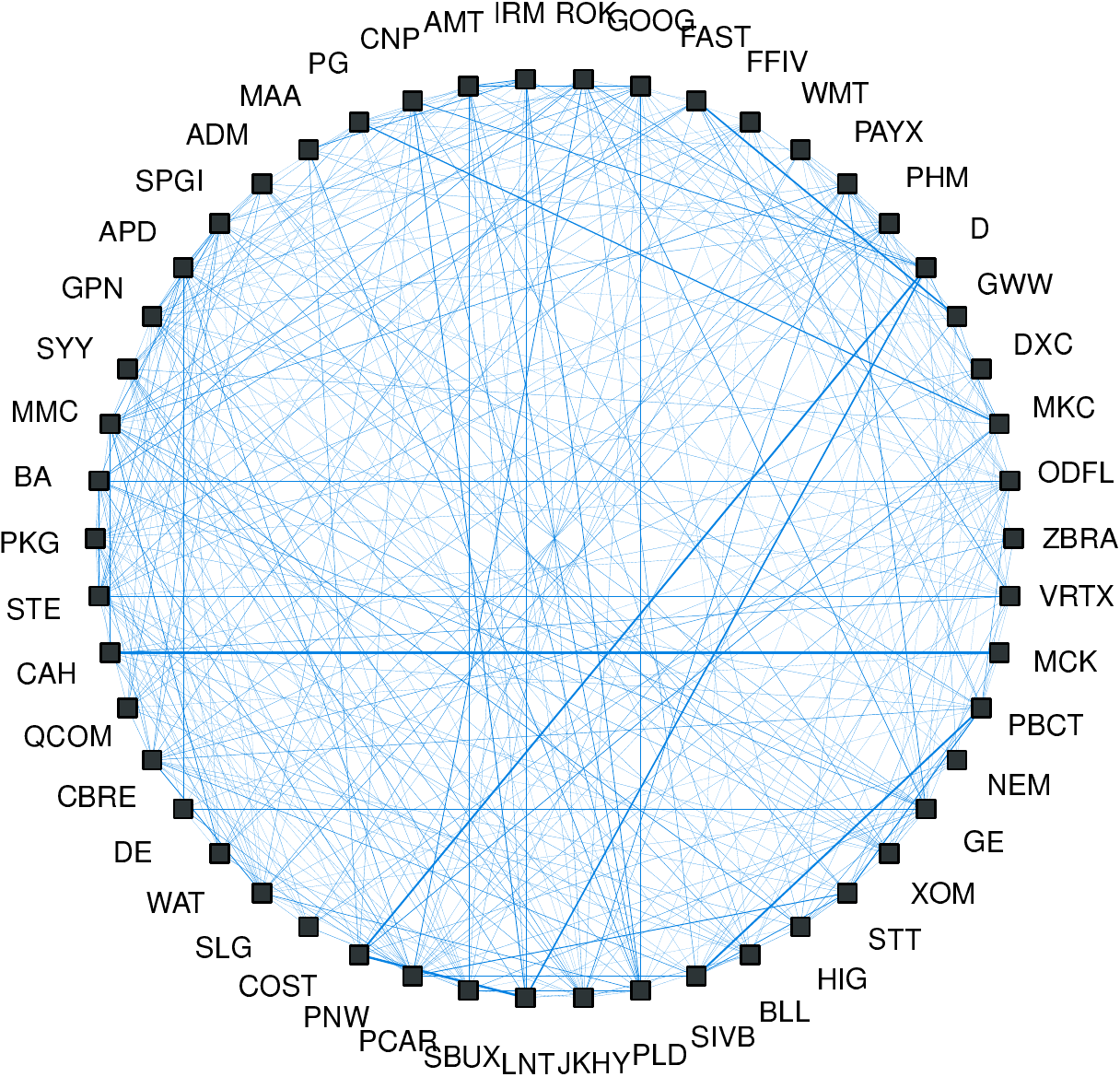}
      \caption{$\mathrm{MTP}_2$.}
      \label{fig:mtp2-loglike}
  \end{subfigure}%
  ~
  \begin{subfigure}[t]{0.32\textwidth}
    \centering
    \includegraphics[scale=.38]{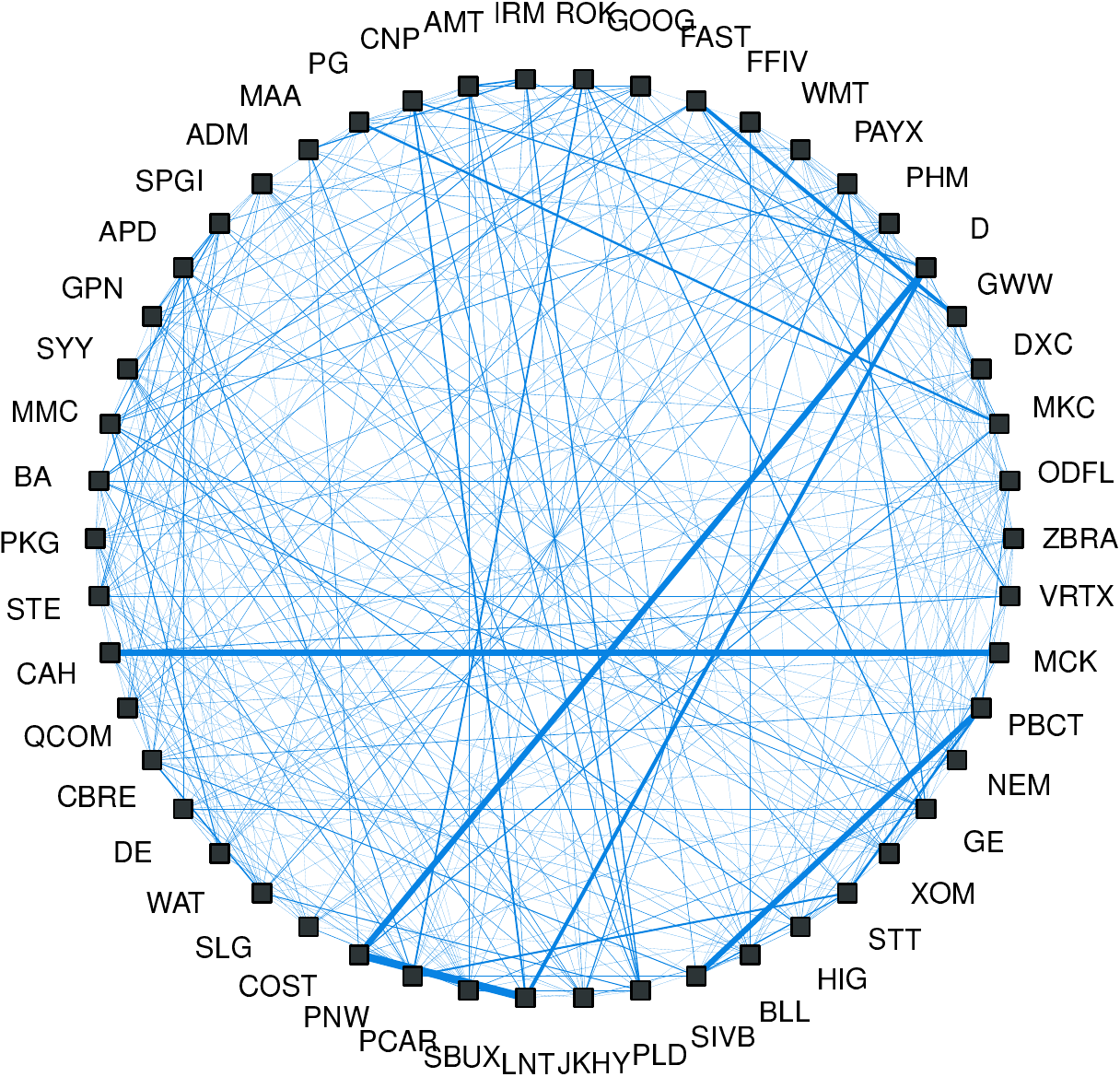}
    \caption{LGMRF.}
    \label{fig:lap-loglike}
  \end{subfigure}%
  \caption{Estimated networks of stocks from (a) Graphical Lasso (b) $\mathrm{MTP}_2$ and (c) LGMRF models
           at their highest average likelihood (Figure~\ref{fig:out-of-sample-ll}, $\alpha = 10^{-1}$).
           The widths of the edges are proportional to the absolute value of the graph weights.
           Blue edges represent positive conditional correlations, while red edges represent negative ones.}
  \label{fig:out-of-sample-graphs}
\end{figure}

From the perspective of the LGMRF model, we would like to estimate a matrix $\bm{L}$ that
enjoys the following two key properties:
\begin{description}
  \item[(P1)] $\bm{L}\mathbf{1} = \mathbf{0}$,
  \item[(P2)] ${L}_{ij} = {L}_{ji} \leq 0 ~\forall ~i \neq j$.
\end{description}

The first property states that the Laplacian matrix $\bm{L}$ is singular and the eigenvector
associated with its zero eigenvalue is given by $a\mathbf{1}, a \in \mathbb{R}$, \textit{i.e.},
the eigenvector is constant along all its components.

Based on the empirical and theoretical discussions about the spectrum of correlation matrices of stock time series~\citep{plerou1999},
we conduct an additional experiment to verify whether the sample inverse correlation matrix of stocks share the aforementioned properties:
we query data from $p = 414$ stocks belonging to the S\&P500 index from January 4th 2005 to June 18th 2020,
totalling $n = 3869$ observations. We then divide this dataset into 19 sequential overlapping datasets each of which containing
$n = 2070$ observations, such that $n/p = 5$. For each dataset, we compute two attributes of the inverse sample
correlation matrix: (i) its condition number, defined as the ratio between its maximum and minimum eigenvalues,
and (ii) the variance of each eigenvector. We observe that the smallest condition number across all datasets
is of the order of $10^{4}$, while the median of the condition numbers is of the order of $10^5$,
indicating that in fact the inverse sample correlation matrices are nearly singular.
In addition, the average variance of the eigenvector associated with the zero eigenvalue is $2\cdot 10^{-4}$,
which is around an order of magnitude smaller than the average variances of any the other eigenvectors, indicating
that it is, in fact, a constant eigenvector. Figure~\ref{fig:constant-eigenvector} illustrates this phenomenon for the
aforementioned dataset, where the constant nature of the market eigenvector ($\mathsf{eigenvector} \#1$)
is clearly observable when compared to the variability of the next two eigenvectors
($\mathsf{eigenvector} \#2$, $\mathsf{eigenvector} \#3$). 
\begin{figure}[!htb]
  \centering
  \includegraphics[scale=0.5]{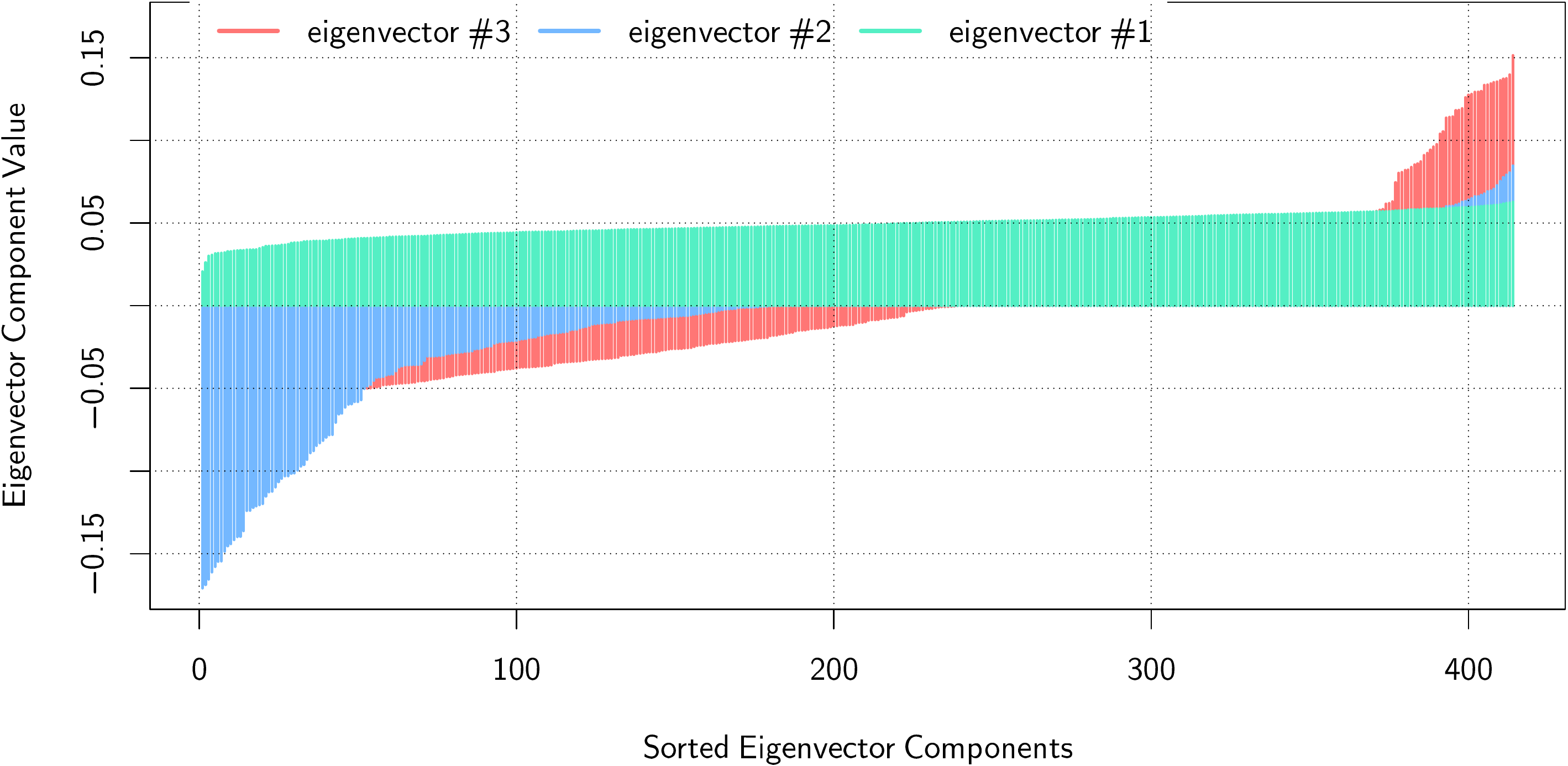}
  \caption{Eigenvectors of the sample correlation matrix of 414 S\&P500 stocks
           corresponding to the largest three eigenvalues over the period between Jan. 2005 to Jun. 2020.
           $\mathsf{eigenvector} \#1$ represents the market factor.  $\mathsf{eigenvector} \#2$ and $\#3$
           are displayed for reference on the expected variability.}
  \label{fig:constant-eigenvector}
\end{figure}

In practice, (P1) implies that signals living in a graph $\mathcal{G}$ have zero graph-mean, \textit{i.e.},
the sum of the graph signals, at a given time, is zero.
From a stock market perspective,
the vector of log-returns of a set of stocks, at a given time $i$, is often assumed to follow a
linear factor model~\citep{sharpe64,fama2004}, \textit{i.e.},
$\bm{x}_{*,i} = \boldsymbol{\beta}x_{\mathsf{mkt}, i} + \boldsymbol\epsilon_i$, where $\bm{x}_{*,i} \in \mathbb{R}^{p\times 1}$ contains the log-returns
of $p$ stocks, $x_{\mathsf{mkt}, i} \in \mathbb{R}$ is the log-return of the market factor,
$\boldsymbol\epsilon_i$ is the vector of idiosyncratic log-returns that is often assumed to be
a Gaussian process with zero mean vector and covariance matrix $\boldsymbol\Psi$, and $\boldsymbol\beta$ is the vector of market factor loadings.
Because $\mathcal{G}$ has zero graph-mean, this implies that a graph designed to accommodate
stock signals, assumed to follow a linear market factor model, will automatically remove the market component
from the learning process of the conditional dependencies among stocks. This is a crucial feature because the
market component would likely be a confounding factor in the estimation of the conditional correlations
due to its strong influence on all the stocks log-returns.

In addition, (P2) together with (P1) implies that $\bm{L}$ is positive semidefinite.
The fact that the off-diagonal entries are symmetric and non-positive means that the Laplacian matrix only represents non-negative conditional
dependencies\footnote{The correlation between any two pair of nodes conditioned on the rest of the graph is given as
$-\frac{L_{ij}}{\sqrt{L_{ii}L_{jj}}}$.}.
This assumption is often met for stock data, as assets are typically positively
dependent~\citep{plerou2002, kazakov2016, agrawal2019, soloff2020, wang2020}.

These two properties along with efficient learning frameworks make the Laplacian-based graphical model
a natural candidate for learning graphs of stock data. As a consequence of using the Laplacian model,
we propose the following guidelines when estimating Laplacian matrices with stock market data:

\begin{itemize}
  \item \textbf{Correlation vs Covariance}: Both the LGMRF and smooth signal approaches rely on the Dirichlet
  energy term $\tr{(\bm{S}\bm{L})} \propto \tr(\bm{W}\bm{Z})$, which quantifies the smoothness of the graph
  signals over the graph weights, where $\bm{S}$ is the sample covariance matrix.
  From the definition of $\bm{Z}$ in~\eqref{eq:smooth_static_graph}, we observe that two perfectly correlated stocks
  but with large Euclidean distances would be translated as largely far apart nodes on the graph. Hence, we advocate the
  use of the sample correlation matrix $\bar{\bm{S}} = \mathsf{Diag}(\bm{S})^{-1/2}\bm{S}\mathsf{Diag}(\bm{S})^{-1/2}$
  (or equivalently scaling the columns of $\bm{X}$ such that they have unit variance)
  in case we would like two highly correlated stocks to have a strong graph connection regardless of their
  individual variances.
  \item \textbf{Removing the market trend}: 
  A widely used and tested model for the returns of the stocks is the linear factor model,
  which explicitly includes the dependency on the market factor:
  $\bm{x}_{*,i} = \boldsymbol{\beta}x_{\mathsf{mkt},i} + \boldsymbol\epsilon_i$.
  Assuming that most of the stocks are heavily dominated by the market index $x_{\mathsf{mkt},i}$, it
  may be convenient to remove that component if we seek to explore the structure of the residual cross-dependency among the
  stocks, $\boldsymbol\epsilon_i$. Thus, an alternative to using the full covariance matrix $\boldsymbol{\Sigma}$ is to use the
  covariance matrix $\boldsymbol{\Psi}$ of the idiosyncratic component.
  However, if one first normalizes each stock, whose variances are $\mathbb{V}(\bm{x}_{*,i}) \approx \beta^2_{i}$,
  we have $\bar{\bm{x}}_{*,i} = \mathbf{1}\bar{x}_{\mathsf{mkt}, i} + \bar{\boldsymbol{\epsilon}}_i$, then
  it turns out that the market factor is automatically removed in the normalized squared distance matrix $\bar{\bm{Z}}$:
  \begin{equation}
  Z_{ij} = \Vert\bar{\bm{x}}_{*,i}-\bar{\bm{x}}_{*,j}\Vert^2_2 =
  \Vert\boldsymbol\bar{\boldsymbol{\epsilon}}_{i}-\boldsymbol\bar{\boldsymbol{\epsilon}}_{j}\Vert^2_2.
\end{equation}
  \item \textbf{Degree control}: Enforcing a rank smaller than $p-1$ on the Laplacian matrix will generate a $k$-component graph,
  which is one desired goal. However, one may get the undesired result of having isolated nodes. One possible strategy to avoid isolated
  nodes is via introducing constraints on the nodes degrees. The LGMRF formulation has the natural penalty term $\log\mathrm{det}^{*}(\bm{L})$
  in the objective, but that does not help in controlling the degrees of the nodes. Instead, some of the graph learning
  formulations from smooth signals include degree control via the constraint $\bm{W}\mathbf{1}=\mathbf{1}$,
  which fixes the degrees of all the nodes to $1$. The regularization term $\mathbf{1}^\top \log(\bm{W}\mathbf{1})$
  also avoids the trivial solution of any degree equals $0$.
  Hence, any graph learning formulation that enforces a $k$-component graph
  (or low-rank Laplacian matrix) should also control the degrees of the nodes to avoid a trivial solution with isolated nodes.
\end{itemize}

\section{Proposed Algorithms}

In this section, we design iterative algorithms for numerous graph learning formulations to
account for $k$-component structures and the heavy-tail nature of financial stock market data.

The proposed algorithms are based on the ADMM~\citep{boyd2011} 
and MM~\citep{ortega2000, sun2017} frameworks. We begin by briefly revisiting ADMM and MM.

\subsection{Alternating Direction Method of Multipliers (ADMM)}

ADMM is a primal-dual framework designed to solve the following class of optimization problems:
\begin{equation}
  \begin{array}{cl}
    \underset{\bm{x}, \bm{z}}{\mathsf{minimize}} & f(\bm{x}) + g(\bm{z}) \\
    \mathsf{subject~to} & \bm{A}\bm{x} + \bm{B}\bm{z} = \bm{c},
  \end{array}
\end{equation}
where $\bm{x} \in \mathbb{R}^n$ and $\bm{z} \in \mathbb{R}^m$ are the optimization variables;
$\bm{A} \in \mathbb{R}^{p \times n}$, $\bm{B} \in \mathbb{R}^{p \times m}$,
and, $\bm{c} \in \mathbb{R}^{p}$ are parameters; and $f$ and $g$ are convex, proper, closed, possibly non-differentiable functions.

The central object in the ADMM framework is the augmented Lagrangian function, which is given as
\begin{align}
  L_{\rho}(\bm{x}, \bm{z}, \bm{y}) = &~ f(\bm{x}) + g(\bm{z}) + \bm{y}^\top(\bm{A}\bm{x} + \bm{B}\bm{z} - \bm{c})
  + \dfrac{\rho}{2}\norm{\bm{A}\bm{x} + \bm{B}\bm{z} - \bm{c}}^2_{2},
\end{align}
where $\rho$ is a penalty parameter.

The basic workflow of the ADMM algorithm is summarized in Algorithm~\ref{alg:ADMM}.

\begin{algorithm}[!htb]
\caption{ADMM framework}
\label{alg:ADMM}
    \SetAlgoLined
    \KwData{$\bm{z}^0$, $\bm{y}^0$, $\bm{A}$, $\bm{B}$, $\bm{c}$, $\rho > 0$}
    \KwResult{$\bm{x}^\star, \bm{z}^\star, \bm{y}^\star$}
    $l \leftarrow 0$\\
    \While{not converged}{
      $\bm{x}^{l+1} \leftarrow \underset{\bm{x} \in \mathcal{X}}{\mathsf{arg min}}~ L_{\rho}\left(\bm{x}, \bm{z}^{l}, \bm{y}^{l}\right)$\\
      $\bm{z}^{l+1} \leftarrow \underset{\bm{z} \in \mathcal{Z}}{\mathsf{arg min}}~ L_{\rho}\left(\bm{x}^{l+1}, \bm{z}, \bm{y}^{l}\right)$\\
      $\bm{y}^{l+1} \leftarrow \bm{y}^{l} + \rho\left(\bm{A}\bm{x}^{l+1} +\bm{B}\bm{z}^{l+1} -\bm{c}\right)$\\
      $i \leftarrow l + 1$
    }
\end{algorithm}

The convergence of ADMM algorithms is attained provided that the following conditions are met:
\begin{enumerate}
  \item $\mathrm{epi}(f) = \{(\bm{x}, t) \in \mathbb{R}^n \times \mathbb{R} : f(\bm{x}) \leq t\}$ and
  $\mathrm{epi}(g) = \{(\bm{z}, s) \in \mathbb{R}^m\times\mathbb{R} : g(\bm{z}) \leq s\}$ are both closed
  nonempty convex sets;
  \item The unaugmented Lagrangian function $L_0$ has a saddle point.
\end{enumerate}
We refer readers to~\citep{boyd2011} where elaborate convergence results are discussed.

\subsection{Majorization-Minimization (MM)}

The MM framework seeks to solve the following general optimization problem:
\begin{equation}
  \begin{array}{cl}
    \underset{\bm{x}}{\mathsf{minimize}} & f(\bm{x}) \\
    \mathsf{subject~to} & \bm{x} \in \mathcal{X},
  \end{array}
\end{equation}
where here we consider $f$ a smooth, possibly non-convex function.

The general idea behind MM is to find a sequence of feasible points $\left\{\bm{x}^{i}\right\}_{i \in \mathbb{N}}$
by minimizing a sequence of carefully constructed global upper-bounds of $f$. The popular expectation-maximization (EM)
algorithm is a special case of MM~\citep{lange2010}.

At point $\bm{x}^i$, we design a continuous global upper-bound function
$g\left(\cdot, \bm{x}^{i}\right): \mathcal{X} \rightarrow \mathbb{R}$ such that
\begin{equation}
  g\left(\bm{x}, \bm{x}^i\right) \geq f(\bm{x}), ~\forall~ \bm{x} \in \mathcal{X}.
\end{equation}

Then, in the minimization step we update $\bm{x}$ as
\begin{equation}
  \bm{x}^{i+1} \in \underset{\bm{x} \in \mathcal{X}}{\mathsf{arg}~\mathsf{min}}~g(\bm{x}, \bm{x}^{i}).
\end{equation}

The global upper-bound function $g(\cdot, \bm{x}^{i})$ must satisfy the following conditions in order
to guarantee convergence:
\begin{enumerate}
  \item $g\left(\bm{x}, \bm{x}^i\right) \geq f(\bm{x}) ~\forall~\bm{x} \in \mathcal{X}$,
  \item $g\left(\bm{x}^i, \bm{x}^i\right) = f\left(\bm{x}^i\right)$,
  \item $\nabla g\left(\bm{x}^i, \bm{x}^i\right) = \nabla f\left(\bm{x}^i\right)$,
  \item $g(\bm{x}, \bm{x}^i)$ is continuous on both $\bm{x}$ and $\bm{x}^i$.
\end{enumerate}

A thorough discussion about MM, along with a significant number of its extensions,
with practical examples, can be found in~\citep{sun2017}.

\subsection{A Reformulation of the Graph Learning Problem}

We formulate the graph learning problem from the LGMRF perspective as the following general optimization program:
\begin{equation}
  \begin{array}{cl}
    \underset{\bm{L} \succeq \mathbf{0}}{\mathsf{minimize}} & \mathsf{tr}\left(\bm{S}\bm{L}\right)
    - \log\mathrm{det}^*\left(\bm{L}\right),\\
    \mathsf{subject~to} & \bm{L} \in \mathcal{C}_{\bm{L}},~ \bm{L}\mathbf{1} = \mathbf{0},~L_{ij} = L_{ji} \leq 0,
  \end{array}
  \label{eq:graph-learning}
\end{equation}
where $\mathcal{C}_{\bm{L}}$ is a set describing additional constraints onto the structure of the estimated Laplacian matrix,
\textit{e.g.}, $\mathcal{C}_{\bm{L}} = \left\{\bm{L} : \mathsf{diag}\left(\bm{L}\right) = d\mathbf{1},~d > 0\right\}$ specifies the
set of $d$-regular graphs.

Now, to split the constraints in Problem~\eqref{eq:graph-learning}, we introduce the following linear transformations:
(a) $\bm{L} = \sL\ww, \ww \in \mathbb{R}^{p(p-1)/2}_{+}$, where $\sL$ is the Laplacian operator
(\textit{cf.} Definition~\eqref{eq:lap-op}) and $\ww$ is the vector of edges weights; and (b) $\bT = \sL\ww$. With this,
we equivalently rewrite Problem~\eqref{eq:graph-learning}
as
\begin{equation}
  \begin{array}{cl}
    \underset{\bm{w} \geq \mathbf{0}, \boldsymbol{\Theta} \succeq \mathbf{0}}{\mathsf{minimize}} & \mathsf{tr}\left(\bm{S}\mathcal{L}\bm{w}\right)
    - \log\mathrm{det}^*\left(\boldsymbol{\Theta}\right),\\
    \mathsf{subject~to} & \boldsymbol{\Theta} = \mathcal{L}\bm{w}, ~\boldsymbol{\Theta} \in \mathcal{C}_{\bT}, ~\bm{w} \in \mathcal{C}_{\bm{w}},
  \end{array}
  \label{eq:lw-admm}
\end{equation}
where $\mathcal{C}_{\bT}$ and $\mathcal{C}_{\bm{w}}$ are sets describing additional constraints onto the structure
of the estimated Laplacian matrix. For example, to estimate connected $d$-regular graphs we can use
$\mathcal{C}_{\boldsymbol\Theta} = \left\{\boldsymbol{\Theta} \in \mathbb{R}^{p \times p} :
\mathsf{rank}(\boldsymbol\Theta) = p-1\right\}$
together with
$\mathcal{C}_{\bm{w}} = \left\{\bm{w} \in \mathbb{R}^{p(p-1)/2}_{+}: \mathfrak{d}\bm{w} = d\mathbf{1}\right\}$,
where $\mathfrak{d}$ is the degree operator (\textit{cf.} Definition~\eqref{eq:deg-op}).

While Problem~\eqref{eq:lw-admm} can be convex for a limited family of graph structures,
convex programming languages, such as $\mathsf{cvxpy}$, have shown to perform poorly even
for considerably small ($p \approx 50$) graphs~\citep{egilmez2017}. Hence, we develop scalable algorithms based on
the ADMM and MM frameworks.

\subsection{Connected Graphs}
\label{subsec:connected-graph}

We first specialize Problem~\eqref{eq:lw-admm} to the class of connected graphs. The rationale for that is twofold:
(1) while this problem has been well studied, we propose a significantly different algorithm than previous
works~\citep{egilmez2017, zhao2019,ying2020nips} by splitting the optimization variables whereby additional constraints
can be easily introduced and handled via ADMM; (2) in addition, the mathematical developments described for this simple class of graphs
will serve as building blocks when we tackle more elaborate classes of graphs such as $k$-component or heavy-tailed.

For connected graphs, we rely on the fact that
$\mathrm{det}^*(\bT) = \mathrm{det}\left(\bT + \bm{J}\right)$~\citep{egilmez2017},
where $\bm{J} = \frac{1}{p}\mathbf{1}\mathbf{1}^\top$,
to formulate the following convex optimization problem:
\begin{equation}
    \begin{array}{cl}
      \underset{\bm{w} \geq \mathbf{0}, \boldsymbol{\Theta} \succeq \mathbf{0}}{\mathsf{minimize}} & \mathsf{tr}\left(\bm{S}\mathcal{L}\bm{w}\right)
      - \log\det\left(\boldsymbol{\Theta} + \bm{J}\right),\\
      \mathsf{subject~to} & \boldsymbol{\Theta} = \mathcal{L}\bm{w}, ~\mathfrak{d}\bm{w} = \bm{d}.
    \end{array}
    \label{eq:lw-admm-regular}
\end{equation}

The partial augmented Lagrangian function of Problem~\eqref{eq:lw-admm-regular} can be written as
\begin{align}
    L_\rho(\bT, \ww, \bm{Y}, \bm{y}) = &~ \mathsf{tr}\left(\bm{S}\mathcal{L}\bm{w}\right)
     - \log\mathrm{det}\left(\boldsymbol{\Theta} + \bm{J}\right)
     + \langle \bm{y}, \mathfrak{d}\bm{w} - \bm{d} \rangle + \frac{\rho}{2}\norm{\mathfrak{d}\bm{w} - \bm{d}}^2_{2}
     \nonumber\\
    & + \langle \bm{Y}, \boldsymbol{\Theta} - \mathcal{L}\bm{w}\rangle
    +\frac{\rho}{2}\norm{\bT - \mathcal{L}\bm{w}}^2_{\mathrm{F}},
\end{align}
where $\bm{Y}$ and $\bm{y}$ are the dual variables associated with the constraints $\bT = \mathcal{L}\bm{w}$ and
$\mathfrak{d}\bm{w} = \bm{d}$, respectively. Note that we will deal with the
constraints $\bm{w} \geq \mathbf{0}$ and $\bT \succeq \mathbf{0}$ directly, hence there are no dual variables associated with them.

The subproblem for $\bT$ can be written as
\begin{align}
      \bT^{l+1} = \underset{\bT \succeq \mathbf{0}}{\mathsf{arg~min}} & - \log\mathrm{det}\left(\bT + \bm{J}\right) + \langle\bT, \bm{Y}^l\rangle
       + \frac{\rho}{2}\norm{\bT - \mathcal{L}\bm{w}^l}^2_{\mathrm{F}}.
\label{eq:lw-admm-regular-theta}
\end{align}
Now, making the simple affine transformation $\boldsymbol \Omega^{l+1} = \bT^{l+1} + \bm{J}$, we have
\begin{align}
  \boldsymbol\Omega^{l+1} = \underset{\boldsymbol \Omega \succ \mathbf{0}}{\mathsf{arg~min}} & - \log\mathrm{det}\left(\boldsymbol\Omega\right)
  + \langle\boldsymbol\Omega, \bm{Y}^l\rangle
   + \frac{\rho}{2}\norm{\boldsymbol\Omega - \mathcal{L}\bm{w}^l - \bm{J}}^2_{\mathrm{F}},
\label{eq:lw-admm-regular-omega}
\end{align}
which can be expressed as a proximal operator~\citep{parikh2014}, \textit{cf.} Definition~\eqref{eq:prox-def},
\begin{equation}
\boldsymbol\Omega^{l+1} = \mathsf{prox}_{\rho^{-1}\left(-\mathrm{log~det}(\cdot) + \langle \bm{Y}^l, \cdot \rangle \right)}\left(\sL\bm{w}^l + \bm{J}\right),
\label{eq:prox}
\end{equation}
whose closed-form solution is given by Lemma~\ref{eq:prox-lema}.
\begin{lemma}
The global minimizer of problem~\eqref{eq:prox} is \citep{witten2009, danaher2014}
\begin{equation}
  \boldsymbol\Omega^{l+1} = \frac{1}{2\rho}\bm{U}\left(\boldsymbol{\Gamma} + \sqrt{\boldsymbol{\Gamma}^2 + 4\rho \bm{I}}\right)\bm{U}^\top,
\end{equation}
where $\bm{U}\boldsymbol{\Gamma}\bm{U}^\top$ is the eigenvalue decomposition of $\rho\left(\mathcal{L}\bm{w}^l + \bm{J}\right) - \bm{Y}^l$.
\label{eq:prox-lema}
\end{lemma}
Hence the closed-form solution for~\eqref{eq:lw-admm-regular-theta} is
\begin{equation}
  \bT^{l+1} = \boldsymbol\Omega^{l+1} - \bm{J}.
  \label{eq:admm-Y}
\end{equation}

Now, using the linear properties of adjoint operators, we have that $\tr \left(\bm S\sL\ww\right) = \langle \bm w, \sL^* \bm S\rangle$
and $\norm{\sL\ww}^2_{\mathrm{F}} = \tr \left( \sL\ww \sL\ww \right) = \ww ^\top \sL^* \sL \ww$. Then,
the subproblem for $\ww$ can be written as
\begin{align}
    \ww^{l+1} =
    \underset{\ww \geq \mathbf{0}}{\mathsf{arg~min}} ~ \frac{\rho}{2}\ww^\top\left(\mathfrak{d}^*\mathfrak{d} + \mathcal{L}^*\mathcal{L}\right)\ww +
    \bigg\langle \ww, \sL^*\left(\bm{S} - \bm{Y}^l - \rho \bT^{l+1}\right) + \mathfrak{d}^*\left(\bm{y}^l - \rho\bm{d}\right)\bigg\rangle,
    \label{eq:w-update}
\end{align}
which is a nonnegative, convex quadratic program.
\begin{lemma}
  Problem~\eqref{eq:w-update} is strictly convex.
\end{lemma}
\begin{proof}
  It suffices to show that the matrix $\mathfrak{d}^*\mathfrak{d} + \mathcal{L}^*\mathcal{L}$ is positive definite.
  For any $\bm{x} \in \mathbb{R}^{p(p-1)/2}, \bm{x} \neq 0$, we have that
  $\norm{\sD\bm{x}}^2_{\mathrm{F}} = \langle\sD\bm{x}, \sD\bm{x}\rangle = \langle \bm{x}, \sD^*\sD\bm{x}\rangle \geq 0$.
  To see that $\mathcal{L}^*\mathcal{L}$ is positive definite, we refer the readers to~\citep[Lemma 5.3]{ying2020}.
\end{proof}

While the solution to Problem~\eqref{eq:w-update} might seem straightforward to obtain via quadratic programming solvers,
it actually poses an insurmountable scalability issue: the dimension of the matrices $\mathfrak{d}^*\mathfrak{d}$
and $\sL^*\sL$ is $p(p - 1)/2 \times p(p - 1)/2$, implying that the worst-case complexity of a convex QP solver for
this problem is $O(p^6)$~\citep{ye89}, which is impractical.
In addition, no closed-form solution is available.

Given these difficulties, we resort to the MM method, whereby we construct an upper-bound of the objective function
of~\eqref{eq:w-update} at point $\bm{w}^i = \bm{w}^l \in \mathbb{R}^{p(p-1)/2}_{+}$ as
\begin{align}
  g(\ww , \ww^i) = g(\ww^i , \ww^i) + \langle \ww - \ww^i, \nabla_{\ww} f(\ww^i) \rangle
   + \dfrac{\mu}{2}\norm{\ww - \ww^i}^{2}_{2},
  \label{eq:upper-bound-w}
\end{align}
where $f(\cdot)$ is the objective function in the minimization in~\eqref{eq:w-update},
$\mu = \rho\lambda_{\mathsf{max}}\left(\mathfrak{d}^*\mathfrak{d} + \mathcal{L}^*\mathcal{L}\right)$,
and the maximum eigenvalue of $\mathfrak{d}^*\mathfrak{d} + \mathcal{L}^*\mathcal{L}$
is given by Lemma~\ref{lemma:lipschitz}.
\begin{lemma}
  The maximum eigenvalue of the matrix $\mathfrak{d}^*\mathfrak{d} + \mathcal{L}^*\mathcal{L}$
  is given as
  \begin{equation} 
  \lambda_{\mathsf{max}}\left(\mathfrak{d}^*\mathfrak{d} + \mathcal{L}^*\mathcal{L}  \right)
   = 2(2p - 1).
  \end{equation}
  \label{lemma:lipschitz}
\end{lemma}
\begin{proof}
  The proof is deferred to Appendix \ref{sec:proof-max-eigenval}.
\end{proof}

Finally, we have that
$\nabla_{\ww}f(\ww^i) = \bm{a}^i + \bm{b}^i$, where
\begin{align}
  \bm{a}^i &= \sL^*\left(\bm{S} - \bm{Y}^l - \rho \left(\bT^{l+1} - \sL\ww^i\right)\right),\\
  \bm{b}^i &= \mathfrak{d}^*\left(\bm{y}^l - \rho\left(\bm{d} - \mathfrak{d}\ww^i\right)\right).
\end{align}

Thus, we have the following approximate strictly convex subproblem for $\ww$,
\begin{align}
    \ww^{i+1} = \underset{\ww \geq \mathbf{0}}{\mathsf{arg~min}} ~ \rho(2p - 1)\norm{\ww - \ww^i}^{2}_{2} +
    \langle \ww, \bm{a}^i + \bm{b}^i \rangle,
\end{align}
whose solution can be readily obtained via its KKT optimality conditions and its given as
\begin{align}
    \ww^{i+1} = \left(\ww^i - \frac{\bm{a}^i + \bm{b}^i}{2\rho(2p - 1)}\right)^{+},
     \label{eq:w-sol}
\end{align}
which is a projected gradient descent step with learning rate $2\rho(2p - 1)$.
Thus, we iterate~\eqref{eq:w-sol} in order to obtain the unique optimal point, $\ww^{l+1}$, of Problem~\eqref{eq:w-update}.
In practice, we observe that a few ($\approx 5$) iterations are sufficient for convergence.

The dual variables $\bm{Y}$ and $\bm{y}$ are updated as
\begin{equation}
  \bm{Y}^{l+1} = \bm{Y}^l + \rho\left(\bT^{l+1} - \sL\ww^{l+1} \right)
  \label{eq:Y-update}
\end{equation} and
\begin{equation}
  \bm{y}^{l+1} = \bm{y}^l + \rho\left(\mathfrak{d}\ww^{l+1} - \bm{d}\right).
  \label{eq:y-update}
\end{equation}

A practical implementation for the proposed ADMM estimation of connected graphs
is summarized in Algorithm~\ref{alg:connected}, whose convergence is stated in Theorem~\ref{connect-convergence}.

\begin{algorithm}[!htb]
  \caption{Connected graph learning}
    \label{alg:connected}
    \SetAlgoLined
    \KwData{Similarity matrix $\bm{S}$, initial estimate of the graph weights $\ww^{0}$, desired degree vector $\bm{d}$,
    penalty parameter $\rho > 0$, tolerance $\epsilon > 0$}
    \KwResult{Laplacian estimation: $\sL\ww^\star$}
    initialize $\bm{Y} = \mathbf{0}$, $\bm{y} = \mathbf{0}$\\
    $l \leftarrow 0$\\
    \While{$\mathsf{max}\left(\vert\bm{r}^l\vert\right) > \epsilon$ or $\mathsf{max}\left(\vert\bm{s}^l\vert\right) > \epsilon$}{
      $\triangleright$ update $\boldsymbol{\Theta}^{l+1}$ via \eqref{eq:admm-Y}\\
      $\triangleright$ iterate \eqref{eq:w-sol} until convergence so as to obtain $\ww^{l+1}$ \\
      $\triangleright$ update $\bm{Y}^{l+1}$ as in \eqref{eq:Y-update}\\
      $\triangleright$ update $\bm{y}^{l+1}$ as in \eqref{eq:y-update}\\
      $\triangleright$ compute residual $\bm{r}^{l+1} = \bT^{l+1} - \sL\ww^{l+1}$\\
      $\triangleright$ compute residual $\bm{s}^{l+1} = \mathfrak{d}\ww^{l+1} - \bm{d}$\\
      $l \leftarrow l + 1$
    }
\end{algorithm}

\begin{theorem}\label{connect-convergence}
  The sequence $\left\lbrace \left( \boldsymbol{\Theta}^{l}, \ww^{l}, \bm Y^{l}, \bm y^{l} \right)\right\rbrace$
  generated by Algorithm~\ref{alg:connected} converges to the optimal primal-dual solution of Problem~\eqref{eq:lw-admm-regular}.
\end{theorem} 
\begin{proof}
 The proof is deferred to Appendix \ref{sec:proof-connected}.
\end{proof}

\subsection{\texorpdfstring{$k$}~-component Graphs}

As discussed in Section~\ref{sec:interpretations}, in addition to a rank constraint,
some form of control of the node degrees is necessary to learn meaningful $k$-component graphs.
Here we choose \linebreak $\mathcal{C}_{\bm{L}} = \left\{\bm{L}: \mathsf{diag}\left(\bm{L}\right) = \bm{d},~\bm{d} \in \mathbb{R}^{p}_{++},
~\mathsf{rank}(\bm{L}) = p - k\right\}$, which is translated to the framework of Problem~\eqref{eq:lw-admm} as
\begin{align}
&\mathcal{C}_{\bT} = \{\bT \succeq \mathbf{0} : \mathsf{rank}(\bT) = p - k\},\\
&\mathcal{C}_{\ww} = \{\ww : \mathfrak{d}\ww = \bm{d}, ~\mathsf{rank}(\sL\ww) = p-k, \bm{d} \in \mathbb{R}^p_{++}\}.
\end{align}

\noindent \textit{Remark:} Although the rank constraints on both variables $\bT$ and $\ww$ may seem redundant, we have observed that it greatly improves
the empirical convergence of the algorithm. In addition, the rank constraint on $\bT$ does not incur any additional
computational cost, as will be shown in the numerical algorithmic derivations bellow.

Thus, Problem~\eqref{eq:lw-admm} can be specialized for the task of learning a $k$-component graph as the following non-convex
optimization program:
\begin{equation}
    \begin{array}{cl}
      \underset{\bm{w} \geq \mathbf{0}, \boldsymbol{\Theta} \succeq \mathbf{0}}{\mathsf{minimize}} & \mathsf{tr}\left(\bm{S}\mathcal{L}\bm{w}\right)
      - \log\mathrm{det}^*\left(\boldsymbol{\Theta}\right),\\
      \mathsf{subject~to} & \boldsymbol{\Theta} = \mathcal{L}\bm{w}, ~\mathsf{rank}(\bT) = p - k, ~\mathfrak{d}\bm{w} = \bm{d},~\mathsf{rank}(\sL\ww) = p - k.
    \end{array}
    \label{eq:lw-admm-regular-k-comp}
\end{equation}

However, unlike the rank constraint in the subproblem associated with $\bT$, the constraint $\mathsf{rank}(\sL\ww) = p-k$
cannot be directly dealt with. An alternative is to move this constraint to the objective function by approximating it
by noting that it is equivalent to having the sum of the $k$ smallest eigenvalues of $\sL\ww$ equals zero, \textit{i.e.},
$\sum_{i=1}^{k}\lambda_{i}\left(\sL\ww\right) = 0$~\citep{feiping2016}, assuming the sequence of eigenvalues
$\{\lambda_{i}(\sL\ww)\}_{i=1}^p$ in increasing order.
By Fan's theorem~\citep{fan1949}, we have
\begin{equation}\sum_{i=1}^{k}\lambda_{i}\left(\sL\ww\right) =
\underset{\bm{V} \in \mathbb{R}^{p\times k}, \bm{V}^\top \bm{V} = \bm{I}}{\mathsf{minimize}}
\mathsf{tr}\left(\bm{V}^\top \sL\ww \bm{V}\right).
\label{eq:ky-fan}
\end{equation}

Thus, moving~\eqref{eq:ky-fan} into the objective function of Problem~\eqref{eq:lw-admm-regular-k-comp}, we have the following relaxed
problem:
\begin{equation}
  \begin{array}{cl}
    \underset{\bm{w} \geq \mathbf{0}, \bT, \bm{V}}{\mathsf{minimize}} &
    \mathsf{tr}\left(\sL\ww(\bm{S} + \eta\bm{V}\bm{V}^\top)\right)
    - \log\mathrm{det}^*\left(\bT\right),\\
    \mathsf{subject~to} & \bT = \sL\ww,~\mathsf{rank}(\bT) = p-k, ~\mathfrak{d}\ww = \bm{d}, ~\bm{V}^\top\bm{V} = \bm{I},~ \bm{V} \in \mathbb{R}^{p\times k},
  \end{array}
  \label{eq:k-comp-proposed-relaxed}
\end{equation}
where $\eta > 0$ is a hyperparameter that controls how much importance is given to the term
$\mathsf{tr}\left(\bm{V}^\top \sL\ww \bm{V}\right)$, which indirectly promotes $\mathsf{rank}(\sL\ww) = p-k$.
Therefore, via~\eqref{eq:ky-fan}, we are able to incorporate the somewhat intractable constraint $\mathsf{rank}(\sL\ww) = p-k$
as a simple term in the optimization program.

The partial augmented Lagrangian function of Problem~\eqref{eq:k-comp-proposed-relaxed} can be written as
\begin{align}
    L_\rho(\bT, \ww, \bm{V}, \bm{Y}, \bm{y}) = & ~ \mathsf{tr}\left(\sL\ww(\bm{S} + \eta\bm{V}\bm{V}^\top)\right)
    - \log\mathrm{det}^*\left(\boldsymbol{\Theta}\right)
     + \langle \bm{y}, \mathfrak{d}\bm{w} - \bm{d}\rangle + \frac{\rho}{2}\norm{\mathfrak{d}\bm{w} - \bm{d}}^2_{2} \nonumber \\
    & + \langle \bm{Y}, \boldsymbol{\Theta} - \mathcal{L}\bm{w}\rangle  + \frac{\rho}{2}\norm{\bT - \mathcal{L}\bm{w}}^2_{\mathrm{F}}.
    \label{eq:k-component-Lagrangian}
\end{align}

The subproblem for $\bT$ can be written as
\begin{align}
      \bT^{l+1} = \underset{\substack{\mathsf{rank}(\bT) = p-k \\ \bT \succeq \mathbf{0}}}{\mathsf{arg~min}}  ~ - \log\mathrm{det}^*(\bT) + \langle\bT, \bm{Y}^l\rangle +
      \frac{\rho}{2}\norm{\bT - \mathcal{L}\bm{w}^{l}}^2_{\mathrm{F}},
    \label{eq:lw-admm-regular-theta-k-comp}
\end{align}
which is tantamount to that of~\eqref{eq:lw-admm-regular-omega}. Its solution is also given as
\begin{equation}
  \bT^{\star} = \frac{1}{2\rho}\bm{U}\left(\boldsymbol{\Gamma} + \sqrt{\boldsymbol{\Gamma}^2 + 4\rho \bm{I}}\right)\bm{U}^\top,
  \label{eq:bt-update-k-comp}
\end{equation}
except that now
$\bm{U}\boldsymbol{\Gamma}\bm{U}^\top$ is the eigenvalue decomposition of $\rho \mathcal{L}\bm{w}^l - \bm{Y}^l$,
with $\boldsymbol{\Gamma}$ having the largest $p - k$ eigenvalues along its diagonal and
$\bm{U} \in \mathbb{R}^{p\times (p-k)}$ contains the corresponding eigenvectors.

The update for $\ww$ is carried out similarly to that of~\eqref{eq:w-sol}, \textit{i.e.},
\begin{align}
  \ww^{i+1} = \left(\ww^i - \frac{\bm{a}^i + \bm{b}^i}{2\rho(2p - 1)}\right)^{+},
   \label{eq:w-sol-k-comp}
\end{align}
except that the coefficient $\bm{a}^i$ is given as
\begin{equation}
  \bm{a}^i = \sL^*\left(\bm{S} + \eta \bm{V}^l\bm{V}^{l\top} - \bm{Y}^l - \rho \left(\bT^{l+1} - \sL\ww^i\right)\right).
  \label{eq:al-update-k-comp}
\end{equation}

We have the following subproblem for $\bm{V}$:
\begin{equation}
  \begin{array}{cl}
    \underset{\bm{V} \in \mathbb{R}^{p\times k}}{\mathsf{minimize}} &
    \mathsf{tr}\left(\bm{V}^\top \sL\ww^{l+1} \bm{V}\right),\\
    \mathsf{subject~to} & \bm{V}^\top\bm{V} = \bm{I},
  \end{array}
  \label{eq:k-comp-proposed-relaxed-sub-V}
\end{equation}
whose closed-form solution is given by the $k$ eigenvectors associated with the $k$ smallest eigenvalues of
$\sL\ww^{l+1}$~\citep{horn1985,absil2008}. 

The updates for the dual variables $\bm{Y}$ and $\bm{y}$ are exactly the same as in~\eqref{eq:Y-update} and~\eqref{eq:y-update},
respectively.

A practical implementation for the proposed ADMM estimation of $k$-component graphs is summarized in
Algorithm~\ref{alg:k-component-ky-fan}, named $\mathsf{kGL}$. Its complexity is bounded by the complexity of the eigenvalue
decomposition in line~\ref{code:evd1}. Its convergence is stated in Theorem~\ref{k-component-convergence}.

\begin{algorithm}[!htb]
  \caption{$k$-component graph learning ($\mathsf{kGL}$)}
    \label{alg:k-component-ky-fan}
    \SetAlgoLined
    \KwData{Similarity matrix $\bm{S}$, initial estimate of the graph weights $\ww^{0}$,
    desired number of graph components $k$, desired degree vector $\bm{d}$,
    penalty parameter $\rho > 0$, tolerance $\epsilon > 0$}
    \KwResult{Laplacian estimation: $\sL\ww^\star$}
    initialize $\bm{Y} = \mathbf{0}$, $\bm{y} = \mathbf{0}$\\
    $l \leftarrow 0$\\
    \While{$\mathsf{max}\left(\vert\bm{r}^l\vert\right) > \epsilon$ or $\mathsf{max}\left(\vert\bm{s}^l\vert\right) > \epsilon$}{
      $\triangleright$ update $\boldsymbol{\Theta}^{l+1}$ via \eqref{eq:bt-update-k-comp}\label{code:evd1}\\
      $\triangleright$ iterate \eqref{eq:w-sol-k-comp} until convergence with $\bm{a}^i$ given as in \eqref{eq:al-update-k-comp} so as to obtain $\ww^{l+1}$\\
      $\triangleright$ update $\bm{V}^{l+1}$ as in \eqref{eq:k-comp-proposed-relaxed-sub-V}\\
      $\triangleright$ update $\bm{Y}^{l+1}$ as in \eqref{eq:Y-update}\\
      $\triangleright$ update $\bm{y}^{l+1}$ as in \eqref{eq:y-update}\\
      $\triangleright$ compute residual $\bm{r}^{l+1} = \bT^{l+1} - \sL\ww^{l+1}$\\
      $\triangleright$ compute residual $\bm{s}^{l+1} = \mathfrak{d}\ww^{l+1} - \bm{d}$\\
      $l \leftarrow l + 1$
    }
\end{algorithm}

\begin{theorem}\label{k-component-convergence}
  Algorithm~\ref{alg:k-component-ky-fan} subsequently converges for any sufficiently large $\rho$,
  that is, the sequence $\left\lbrace \left( \boldsymbol{\Theta}^{l}, \ww^{l}, \bm V^{l}, \bm Y^{l}, \bm y^{l} \right)\right\rbrace$
  generated by Algorithm~\ref{alg:k-component-ky-fan} has at least one limit point, and each limit point is a stationary point of
  \eqref{eq:k-component-Lagrangian}.
\end{theorem}

\begin{proof}
  The proof is deferred to Appendix \ref{sec:proof-k-component}.
\end{proof}

\subsection{Connected heavy-tailed graphs}

Following the LGMRF framework,~\citet{ying2020,ying2020nips} recently proposed non-convex regularizations so as to obtain 
sparse representations of the resulting estimated graphs. Enforcing sparsity is one possible way to remove spurious conditional
correlations  between nodes in the presence of data with outliers. However, we advocate that assuming a principled, heavy-tailed
statistical distribution has more benefits for the financial data setting, rather than simply imposing arbitrary, non-convex
regularizations, because they are often cumbersome to deal with from a theoretical perspective and, in practice,
they bring the additional task of tunning hyperparameters, which is often repetitive.

In order to address the inherently heavy-tailed nature of financial stock data~\citep{resnick2008},
we consider the Student-t distribution under the Improper Markov Random Field assumption~\citep{rue2005}
with Laplacian structural constraints, that is,
we assume the data generating process to be modeled a multivariate zero-mean Student-t distribution,
whose probability density function can be written as
\begin{equation}
p(\bm{x}) \propto \sqrt{\mathrm{det}^*(\bT)}\left(1 + \dfrac{\bm{x}^\top\bT\bm{x}}{\nu}\right)^{-\frac{\nu + p}{2}},~\nu > 2,
\end{equation}
where $\bT$ is a positive-semidefinite inverse scatter matrix modeled as a combinatorial graph Laplacian matrix.

This results
in a robustified version of the penalized MLE for connected graph learning, \textit{i.e.},
\begin{equation}
  \begin{array}{cl}
    \underset{\ww \geq \mathbf{0}, \bT \succ \mathbf{0}}{\mathsf{minimize}} &
    \dfrac{p+\nu}{n}\sum_{i=1}^{n}\log\left(1 + \dfrac{\bm{x}^\top_{i,*}\sL\ww{\bm{x}_{i,*}}}{\nu}\right)
    - \log\mathrm{det}\left(\bT + \bm{J}\right),\\
    \mathsf{subject ~to} & \bT = \sL\ww, ~\mathfrak{d}\ww = \bm{d}.
  \end{array}
  \label{eq:student-t}
\end{equation}

Problem~\eqref{eq:student-t} is non-convex due to the terms involving the $\log(\cdot)$ function
and hence it is difficult to be dealt with directly.
To tackle this issue, we leverage the MM framework whereby the concave terms in~\eqref{eq:student-t}
are linearized, which essentially results in a weighted Gaussian likelihood~\citep{sun2016, sun2017, wiesel2019}.

We start by following the exposition in the preceding sections, then the partial augmented Lagrangian function
of Problem~\eqref{eq:student-t} is given as
\begin{align}
  L_\rho(\bT, \ww, \bm{Y}, \bm{y}) = &~ \dfrac{p + \nu}{n}\sum_{i=1}^{n}\log\left(1 + \frac{\bm{x}^\top_{i,*}\sL\ww{\bm{x}_{i,*}}}{\nu} \right)
  - \log\mathrm{det}\left(\boldsymbol{\Theta} + \bm{J}\right)
   + \langle \bm{y}, \mathfrak{d}\bm{w} - \bm{d} \rangle 
   \nonumber\\
  & + \frac{\rho}{2}\norm{\mathfrak{d}\bm{w} - \bm{d}}^2_{2} + \langle \bm{Y}, \boldsymbol{\Theta} - \mathcal{L}\bm{w}\rangle
  +\frac{\rho}{2}\norm{\bT - \mathcal{L}\bm{w}}^2_{\mathrm{F}}.
\end{align}

The subproblem for $\bT$ is identical to that of \eqref{eq:lw-admm-regular-theta}.

The subproblem for $\ww$ can be written as
\begin{align}
    \underset{\ww \geq \mathbf{0}}{\mathsf{minimize}} & ~ \frac{\rho}{2}\ww^\top\left(\mathfrak{d}^*\mathfrak{d} + \mathcal{L}^*\mathcal{L}\right)\ww
    - \bigg\langle \ww, \sL^*\left(\bm{Y}^l + \rho \bT^{l+1}\right) - \mathfrak{d}^*\left(\bm{y}^l - \rho\bm{d}\right)\bigg\rangle
    \nonumber \\ & +
    \dfrac{p + \nu}{n}\sum_{i=1}^{n}\log\left(1 + \frac{\bm{x}^\top_{i,*}\sL\ww{\bm{x}_{i,*}}}{\nu} \right),
    \label{eq:w-update-heavy-tail}
\end{align}
which is similar to that of subproblem~\eqref{eq:w-update}, except it contains the additional concave term
$\dfrac{p + \nu}{n}\sum_{i=1}^{n}\log\left(1 + \frac{\bm{x}^\top_{i,*}\sL\ww{\bm{x}_{i,*}}}{\nu} \right)$
in place of the linear term $\langle \bm{S}, \sL\ww\rangle$.

Similarly to subproblem~\eqref{eq:w-update}, we employ the MM framework to formulate an efficient iterative algorithm to
obtain a stationary point of
Problem~\eqref{eq:w-update-heavy-tail}. 
We proceed by constructing a global upper bound of Problem~\eqref{eq:w-update-heavy-tail}.
Using the fact that the logarithm is globally
upper-bounded by its first-order Taylor expansion, we have
\begin{equation}
  \log\left(1 + \dfrac{t}{b}\right) \leq \log\left(1 + \dfrac{a}{b}\right) + \dfrac{t - a}{a + b}, \forall a \geq 0, t \geq 0, b > 2,
\end{equation}
which results in the following upper bound:
\begin{equation}
 \log\left(1 + \dfrac{\langle \ww, \sL^*{\bm{x}_{i,*}}\bm{x}^\top_{i,*}\rangle}{\nu}\right) \leq
 \dfrac{\langle \ww, \sL^*{\bm{x}_{i,*}}\bm{x}^\top_{i,*}\rangle}{\langle \ww^j, \sL^*{\bm{x}_{i,*}}\bm{x}^\top_{i,*}\rangle + \nu}  + c_1
 \label{eq:ub-student}
\end{equation}
where $c_1 = \log\left(1 + \dfrac{\langle \ww^j, \sL^*{\bm{x}_{i,*}}\bm{x}^\top_{i,*}\rangle}{\nu}\right)
- \dfrac{\langle \ww^j, \sL^*{\bm{x}_{i,*}}\bm{x}^\top_{i,*}\rangle}{\langle \ww^j, \sL^\star{\bm{x}_{i,*}}\bm{x}^\top_{i,*}\rangle + \nu}
$
is a constant.

By upper-bounding the objective function of Problem~\eqref{eq:w-update-heavy-tail}, at point $\bm{w}^j = \bm{w}^l$,
with \eqref{eq:ub-student},
the vector of graph weights $\ww$ can then be updated by solving the following nonnegative, quadratic-constrained, strictly convex problem:
\begin{align}
  \ww^{j+1} & =  
  \underset{\ww \geq \mathbf{0}}{\mathsf{arg~min}} ~ \frac{\rho}{2}\ww^\top\left(\mathfrak{d}^*\mathfrak{d} + \mathcal{L}^*\mathcal{L}\right)\ww
  - \bigg\langle \ww, \sL^*\left(\bm{Y}^l + \rho \bT^{l+1}\right) - \mathfrak{d}^*\left(\bm{y}^l - \rho\bm{d}\right)\bigg\rangle
  \nonumber\\ &\hspace{1.7cm}+ \dfrac{p + \nu}{n}
  \sum_{i=1}^{n}\dfrac{\langle \ww, \sL^*{\bm{x}_{i,*}}\bm{x}^\top_{i,*}\rangle}{\langle \ww^j, \sL^*{\bm{x}_{i,*}}\bm{x}^\top_{i,*}\rangle + \nu}
  \nonumber\\
  & = \underset{\ww \geq \mathbf{0}}{\mathsf{arg~min}} ~ \frac{\rho}{2}\ww^\top\left(\mathfrak{d}^*\mathfrak{d} + \mathcal{L}^*\mathcal{L}\right)\ww
  + \bigg\langle \ww, \sL^*\left(\bm{\tilde{S}}^j - \bm{Y}^l - \rho \bT^{l+1}\right)
  + \mathfrak{d}^*\left(\bm{y}^l - \rho\bm{d}\right)\bigg\rangle,
  \label{eq:w-update-heavy-tail-upper-bound}
\end{align}
where $\bm{\tilde{S}}^j \triangleq \displaystyle \dfrac{1}{n}\sum_{i=1}^{n} \dfrac{(p + \nu)}{{\langle \ww^j,
\sL^*({\bm{x}_{i,*}}\bm{x}^\top_{i,*})\rangle + \nu}}{\bm{x}_{i,*}}\bm{x}^\top_{i,*}$ is a weighted sample covariance matrix.

The objective function of Problem~\eqref{eq:w-update-heavy-tail-upper-bound} can be upper-bounded once again
following the same steps as the ones taken for Problem~\eqref{eq:upper-bound-w}, which results in a projected gradient descent step 
as in \eqref{eq:w-sol} with
\begin{equation}
  \bm{a}^j \triangleq \sL^*\left(
    \bm{\tilde{S}}^j
  - \bm{Y}^l - \rho \left(\bT^{l+1} - \sL\ww^j\right)\right).
  \label{eq:al-update}
\end{equation}

The dual variables $\bm{Y}$ and $\bm{y}$ are updated exactly as in~\eqref{eq:Y-update} and~\eqref{eq:y-update},
respectively.

Algorithm~\ref{alg:heavy-tail}, named $\mathsf{tGL}$, summarizes the implementation to solve Problem~\eqref{eq:student-t}.
The complexity of Algorithm~\ref{alg:heavy-tail} is bounded by the complexity of the eigenvalue decomposition in line~\ref{code:evd}
and its convergence is stated by Theorem~\ref{Connected heavy-tailed graphs}.
\begin{algorithm}
  \caption{Connected Student-$t$ graph learning ($\mathsf{tGL}$)}
    \label{alg:heavy-tail}
    \SetAlgoLined
    \KwData{Data matrix $\bm{X} \in \mathbb{R}^{n\times p}$, initial estimate of the graph weights $\ww^{0}$,
            desired degree vector $\bm{d}$, penalty parameter $\rho > 0$, degrees of freedom $\nu$, tolerance $\epsilon > 0$}
    \KwResult{Laplacian estimation: $\sL\ww^\star$}
    initialize $\bm{Y} = \mathbf{0}$, $\bm{y} = \mathbf{0}$\\
    $l \leftarrow 0$\\
    \While{$\mathsf{max}\left(\vert\bm{r}^l\vert\right) > \epsilon$ or $\mathsf{max}\left(\vert\bm{s}^l\vert\right) > \epsilon$}{
      $\triangleright$ update $\boldsymbol{\Theta}^{l+1}$ via \eqref{eq:admm-Y}\label{code:evd}\\
      $\triangleright$ iterate \eqref{eq:w-sol} with $\bm{a}^j$ given as in \eqref{eq:al-update} so as to obtain $\ww^{l+1}$\\
      $\triangleright$ update $\bm{Y}^{l+1}$ as in \eqref{eq:Y-update}\\
      $\triangleright$ update $\bm{y}^{l+1}$ as in \eqref{eq:y-update}\\
      $\triangleright$ compute residual $\bm{r}^{l+1} = \bT^{l+1} - \sL\ww^{l+1}$\\
      $\triangleright$ compute residual $\bm{s}^{l+1} = \mathfrak{d}\ww^{l+1} - \bm{d}$\\
      $l \leftarrow l + 1$
    }
\end{algorithm}

\noindent \textit{Remark:} in practical code implementations, the rank-1 data matrices
$\bm{x}_{i,*}\bm{x}^\top_{i,*}, i=1, 2, ..., n,$ involved in the computation of~\eqref{eq:al-update},
are only necessary through the terms $\sL^* \left(\bm{x}_{i,*}\bm{x}^\top_{i,*}\right)$, which can
be readily pre-computed before the starting of the iterative process.

\begin{theorem}\label{Connected heavy-tailed graphs}
  Algorithm~\ref{alg:heavy-tail} subsequently converges for any sufficiently large $\rho$, that is, the sequence
  $\left\lbrace \left( \boldsymbol{\Theta}^{l}, \ww^{l}, \bm Y^{l}, \bm y^{l} \right)\right\rbrace$ generated by
  Algorithm~\ref{alg:heavy-tail} has at least one limit point, and each limit point is a stationary point of
  \eqref{eq:student-t}.
\end{theorem}

\begin{proof}
  The proof is deferred to Appendix~\ref{sec:proof-connected-heavy-tail}.
\end{proof}

\subsection{\texorpdfstring{$k$-}~component heavy-tailed graphs}

Extending Problem~\eqref{eq:student-t} for $k$-component graphs follows the same strategy as in
Problem~\eqref{eq:k-comp-proposed-relaxed}, which results in the following optimization program
\begin{equation}
  \begin{array}{cl}
    \underset{\ww \geq \mathbf{0}, \bT \succeq 0, \bm{V}}{\mathsf{minimize}} &
    \dfrac{p + \nu}{n}\displaystyle\sum_{i=1}^{n}\log\left(1 + \dfrac{\bm{x}^\top_{i,*}\sL\ww{\bm{x}_{i,*}}}{\nu} \right)
    - \log\mathrm{det^*}\left(\bT\right) + \eta\mathsf{tr}(\sL\ww\bm{V}\bm{V}^\top),\\
    \mathsf{subject~to} & \bT = \sL\ww,~\mathsf{rank}(\bT) = p-k, ~\mathfrak{d}\ww = \bm{d}, ~\bm{V}^\top\bm{V} = \bm{I},~ \bm{V} \in \mathbb{R}^{p\times k}.
  \end{array}
  \label{eq:k-comp-heavy-tail}
\end{equation}

Following the exposition in the preceding sections, the partial augmented Lagrangian function
of Problem~\eqref{eq:k-comp-heavy-tail} is given as
\begin{align}\label{eq:student-t-k-comp}
  L_\rho(\bT, \ww, \bm{Y}, \bm{y}) = &~ \dfrac{p + \nu}{n}\sum_{i=1}^{n}\log\left(1 + \frac{\bm{x}^\top_{i,*}\sL\ww{\bm{x}_{i,*}}}{\nu} \right) 
  - \log\mathrm{det}^*\left(\boldsymbol{\Theta}\right)
   + \eta\mathsf{tr}\left(\sL\ww \bm{V}\bm{V}^\top\right)
   \nonumber\\
  &  + \langle \bm{y}, \mathfrak{d}\bm{w} - \bm{d} \rangle + \frac{\rho}{2}\norm{\mathfrak{d}\bm{w} - \bm{d}}^2_{2} + \langle \bm{Y}, \boldsymbol{\Theta} - \mathcal{L}\bm{w}\rangle
  +\frac{\rho}{2}\norm{\bT - \mathcal{L}\bm{w}}^2_{\mathrm{F}}.
\end{align}

The subproblems for the variables $\bT$ and $\bm{V}$ are identical to those of Problem~\eqref{eq:k-comp-proposed-relaxed}, hence they
follow the same update expressions.

The subproblem for $\bm{w}$ is virtually the same as in~\eqref{eq:w-update-heavy-tail-upper-bound}, except for the additional
term $\eta\mathsf{tr}(\sL\ww\bm{V}^l\bm{V}^{l\top}) = \eta \langle \ww, \sL^* \left(\bm{V}^l\bm{V}^{l\top}\right)\rangle$.
Hence, its update is also
a projected gradient descent step, alike~\eqref{eq:w-sol} where
\begin{equation}
  \bm{a}^j \triangleq \sL^*\left(\bm{\tilde{S}}^j + \eta\bm{V}^l\bm{V}^{l\top}
  - \bm{Y}^l - \rho \left(\bT^{l+1} - \sL\ww^j\right)\right).
  \label{eq:al-update-k-comp-heavy-tail}
\end{equation}

The dual variables $\bm{Y}$ and $\bm{y}$ are updated as in~\eqref{eq:Y-update} and~\eqref{eq:y-update},
respectively.

Algorithm~\ref{alg:heavy-tail-k-comp}, named $\mathsf{ktGL}$,
summarizes the implementation to solve Problem~\eqref{eq:k-comp-heavy-tail}.
\begin{algorithm}
  \caption{$k$-component Student-$t$ graph learning ($\mathsf{ktGL}$)}
    \label{alg:heavy-tail-k-comp}
    \SetAlgoLined
    \KwData{Data matrix $\bm{X} \in \mathbb{R}^{n\times p}$, initial estimate of the graph weights $\ww^{0}$,
            desired number of graph components $k$, desired degree vector $\bm{d}$, degrees of freedom $\nu$, penalty parameter $\rho > 0$,
            tolerance $\epsilon > 0$}
    \KwResult{Laplacian estimation: $\sL\ww^\star$}
    initialize $\bm{Y} = \mathbf{0}$, $\bm{y} = \mathbf{0}$\\
    $l \leftarrow 0$\\
    \While{$\mathsf{max}\left(\vert\bm{r}^l\vert\right) > \epsilon$ or $\mathsf{max}\left(\vert\bm{s}^l\vert\right) > \epsilon$}{
      $\triangleright$ update $\boldsymbol{\Theta}^{l+1}$ via \eqref{eq:admm-Y}\\
      $\triangleright$ update $\ww^{l+1}$ as in \eqref{eq:w-sol} with $\bm{a}^j$ given as in \eqref{eq:al-update-k-comp-heavy-tail}\\
      $\triangleright$ update $\bm{V}^{l+1}$ as in \eqref{eq:k-comp-proposed-relaxed-sub-V}\\
      $\triangleright$ update $\bm{Y}^{l+1}$ as in \eqref{eq:Y-update}\\
      $\triangleright$ update $\bm{y}^{l+1}$ as in \eqref{eq:y-update}\\
      $\triangleright$ compute residual $\bm{r}^{l+1} = \bT^{l+1} - \sL\ww^{l+1}$\\
      $\triangleright$ compute residual $\bm{s}^{l+1} = \mathfrak{d}\ww^{l+1} - \bm{d}$\\
      $l \leftarrow l + 1$
    }
\end{algorithm}

\begin{theorem}\label{heavy-tail-k-comp-convergence}
  Algorithm~\ref{alg:heavy-tail-k-comp} subsequently converges for any sufficiently large $\rho$, that is, the sequence
  $\left\lbrace \left( \boldsymbol{\Theta}^{l}, \ww^{l}, \bm V^{l}, \bm Y^{l}, \bm y^{l} \right)\right\rbrace$ generated by
  Algorithm~\ref{alg:heavy-tail-k-comp} has at least one limit point, and each limit point is a stationary point of
  \eqref{eq:k-comp-heavy-tail}.
\end{theorem}

\begin{proof}
  See Appendix~\ref{sec:proof-heavy-tail-k-comp}.
\end{proof}

\section{Experimental Results}

We perform experiments with price data queried from S\&P500 stocks.
In such real-world data experiments, where the a ground-truth graph cannot possibly be known, we evaluate the
performance of the learned graphs by visualizing the resulting estimated graph network and verifying whether
it is aligned with prior, expert knowledge available, \textit{e.g.}, the GICS sector information
of each stock\footnote{It is important to notice that the GICS sector classification system might itself be prone to
misclassifications specially for companies that serve many markets.}.
In addition, we employ measures such as graph modularity (\textit{cf.} Definition~\eqref{eq:modularity})
and density as an objective criterion to evaluate the quality of the estimated graphs.

\noindent\textit{Baseline algorithms}: We compare the proposed algorithms (Table~\ref{tab:proposed-methods}) with
state-of-the-art, baseline algorithms, depending on the specific graph structure that they are
suitable to estimate. In the existing literature, it is a common practice not to compare graph algorithms
that adopt distinct operational assumptions, \textit{i.e.}, the LGMRF approach and the smooth signal approach.
This separation is certainly useful from a theoretical perspective. In this work, however, we are mostly interested
in the applicability of graph learning algorithms in practical scenarios and whether the estimated graphs are aligned
with prior expert knowledge available irrespective of their underlying assumptions.
Therefore, in our experimental analysis, we consider algorithms from
both operational assumptions. A summary of the baseline algorithms along with their target graph
structure is illustrated in Table~\ref{tab:methods}. For a fair comparison among algorithms,
we set the degree vector $\bm{d}$ equal to $\bm{1}$ for the proposed algorithms,
\textit{i.e.}, we do not consider any prior information on the degree of nodes.

\noindent\textit{Initial graph}: Because the algorithms proposed in this paper work in an iterative fashion,
they naturally require an initial estimate of the graph. An appropriate initial
estimation is critical to obtain a meaningful solution, especially in cases when
the optimization problem is non-convex. However, obtaining an initial estimate inherently
involves  a trade-off between computational efficiency and quality. The latter being measured
by how far the initial point is from an actual optimal point. Since the computational complexity
of the proposed algorithms are bounded below by the eigenvalue decomposition $O(p^3)$,
we are interested in simple strategies. Here we consider the strategy used by \citet{kumar20192},
where the initial graph $\bm{w}^{0}$ is set as $(\tilde{\bm w})^+$, where $\tilde{\bm w}$ is the upper triangular
part of the pseudo sample inverse correlation matrix $\bm{S}^\dagger$.

\begin{table}[!htb]
  \centering
  \caption{Proposed algorithms, their target graph structure, operational assumption, and computational complexity.}
  \begin{tabular}{lllc}
  \hline
  Algorithm & Graph Structure & Assumption & Complexity\\ \hline
  $\mathsf{tGL}$ & connected & Laplacian Student-$t$ MRF & $O(p^3)$ \\
  $\mathsf{kGL}$ & $k$-component & LGMRF & $O(p^3)$ \\
  $\mathsf{ktGL}$ & $k$-component & Laplacian Student-$t$ MRF & $O(p^3)$ \\
  \end{tabular}
  \label{tab:proposed-methods}
\end{table}

\begin{table}[!htb]
  \centering
  \caption{Baseline algorithms, their target graph structure, operational assumption, and computational complexity.}
  \begin{tabular}{lllc}
  \hline
  Algorithm & Graph Structure & Assumption & Complexity\\ \hline
  $\mathsf{GL}$-$\mathsf{SigRep}$~\citep{dong2016} & connected & smooth signals & $O(np^2)$ \\
  $\mathsf{SSGL}$~\citep{kalofolias2016} & connected & smooth signals & $O(p^2)$\\
  $\mathsf{GLE}$-$\mathsf{ADMM}$~\citep{zhao2019} & connected & LGMRF & $O(p^3)$\\
  $\mathsf{NGL}$-$\mathsf{MCP}$~\citep{ying2020} & connected & LGMRF & $O(p^3)$\\
  $\mathsf{SGL}$~\citep{kumar20192, kumar2019} & $k$-component & LGMRF & $O(p^3)$\\
  $\mathsf{CLR}$~\citep{feiping2016} & $k$-component & smooth signals & $O(p^3)$ \\
  \end{tabular}
  \label{tab:methods}
\end{table}

In the experiments that follow, we use daily price time series data of stocks belonging to the
S\&P500 index. We start by constructing the log-returns data matrix, \textit{i.e.}, a matrix
$\bm{X} \in \mathbb{R}^{n\times p}$, where $n$ is the number of price observations and $p$ is the
number of stocks, such that the $j$-th column contains the time series of log-returns of the $j$-th stock,
which can be computed as
\begin{equation}
  X_{i,j} = \log P_{i,j} - \log P_{i-1,j} ,
\end{equation}
where $P_{i,j}$ is the closing price of the $j$-th stock on the $i$-th day.

\subsection{\texorpdfstring{$k$}~-component graphs: degree control is crucial}
\label{sec:degree-control}

To illustrate the importance of controlling the nodes degrees while learning $k$-component graphs, we conduct
a comparison between the spectral constraints algorithm proposed in~\citep{kumar2019}, denoted as $\mathsf{SGL}$,
and the proposed $k$-component graph learning (Algorithm~\ref{alg:k-component-ky-fan}) on the basis
of the sample correlation matrix. To that end, we set up experiments with two datasets: (i) 
stocks from four sectors, namely, $\mathsf{Health}~\mathsf{Care}$, $\mathsf{Consumer}~\mathsf{Staples}$,
$\mathsf{Energy}$, and $\mathsf{Financials}$, from the period starting from Jan. 1st 2014 to Jan. 1st 2018.
This datasets results in $n = 1006$ stock price observations of $p = 181$ stocks; (ii) we expand the dataset
by including two more sectors, namely, $\mathsf{Industrials}$ and $\mathsf{Information}~\mathsf{Technology}$.
In addition, we collect data from Jan. 1st 2010 to Jan. 1st 2018, resulting in $p = 292$ stocks and $n = 2012$
observations.

Figure~\ref{fig:isolated-nodes} shows the estimated financial stocks networks with
$k=4$ (Figures~\ref{fig:sgl-1} and~\ref{fig:dgl-1}) and $k=6$ (Figures~\ref{fig:sgl-2} and~\ref{fig:dgl-2}).
Clearly, the absence of degrees constraints in the learned graph by $\mathsf{SGL}$ (benchmark)~\citep{kumar2019}
shows evidence that the algorithm is unable to recover a non-trivial $k$-component
solution, \textit{i.e.}, a graph without isolated nodes. 
In addition, the learned graphs by $\mathsf{SGL}$ present a high number
of inter-cluster connections (grey-colored edges), which is not expected from prior expert knowledge of the sectors.
The proposed algorithm not only avoids isolated nodes via graph degree constraints, but
most importantly learns graphs with meaningful representations, \textit{i.e.}, they are aligned with
the available sector information.

\begin{figure}[!htb]
  \captionsetup[subfigure]{justification=centering}
  \centering
  \begin{subfigure}[t]{0.48\textwidth}
      \centering
      \includegraphics[scale=.58]{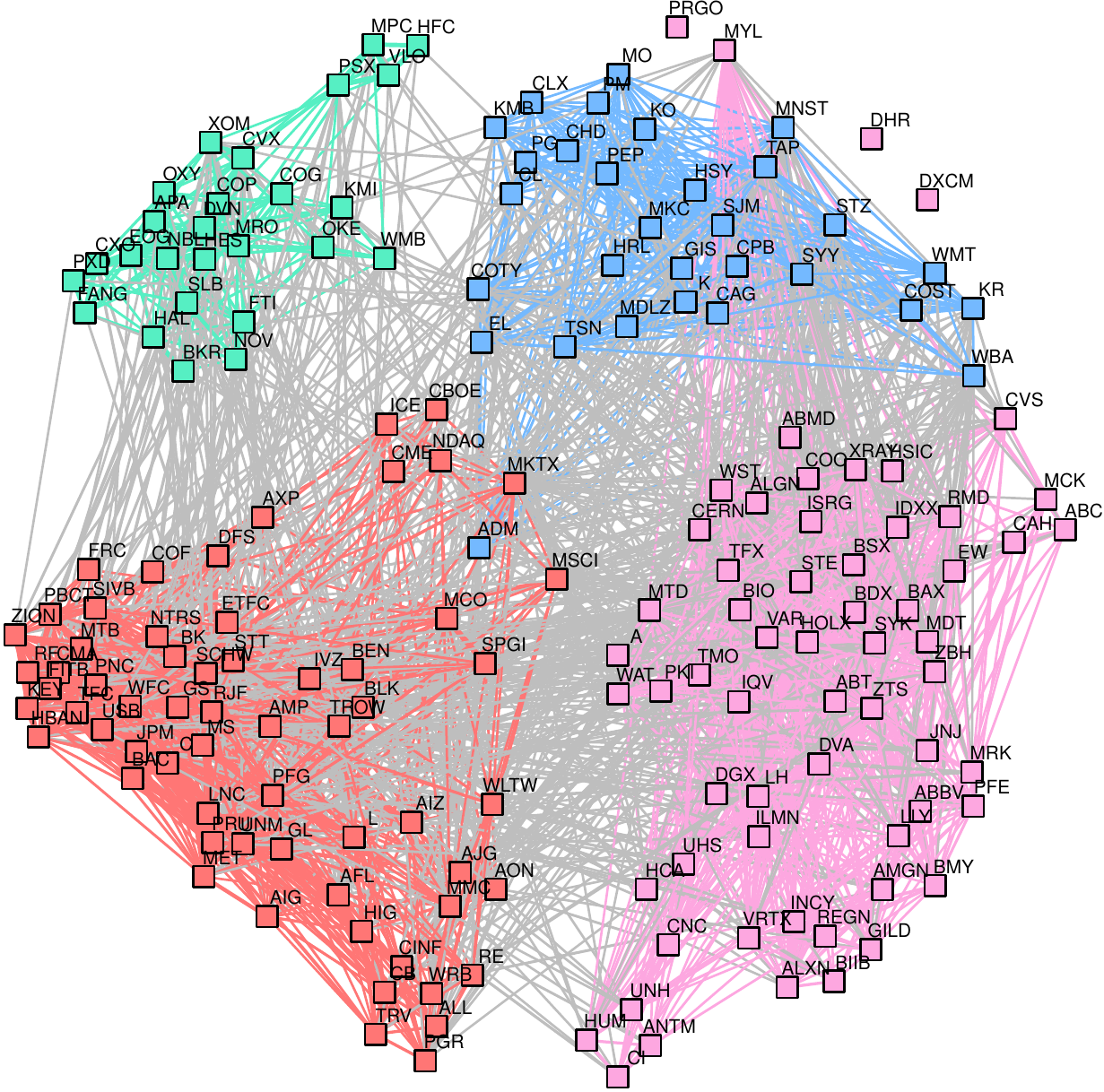}
      \caption{$4$-comp graph learned via $\mathsf{SGL}$. $\eta = 10$.}
      \label{fig:sgl-1}
  \end{subfigure}%
  ~
  \begin{subfigure}[t]{0.48\textwidth}
      \centering
      \includegraphics[scale=.58]{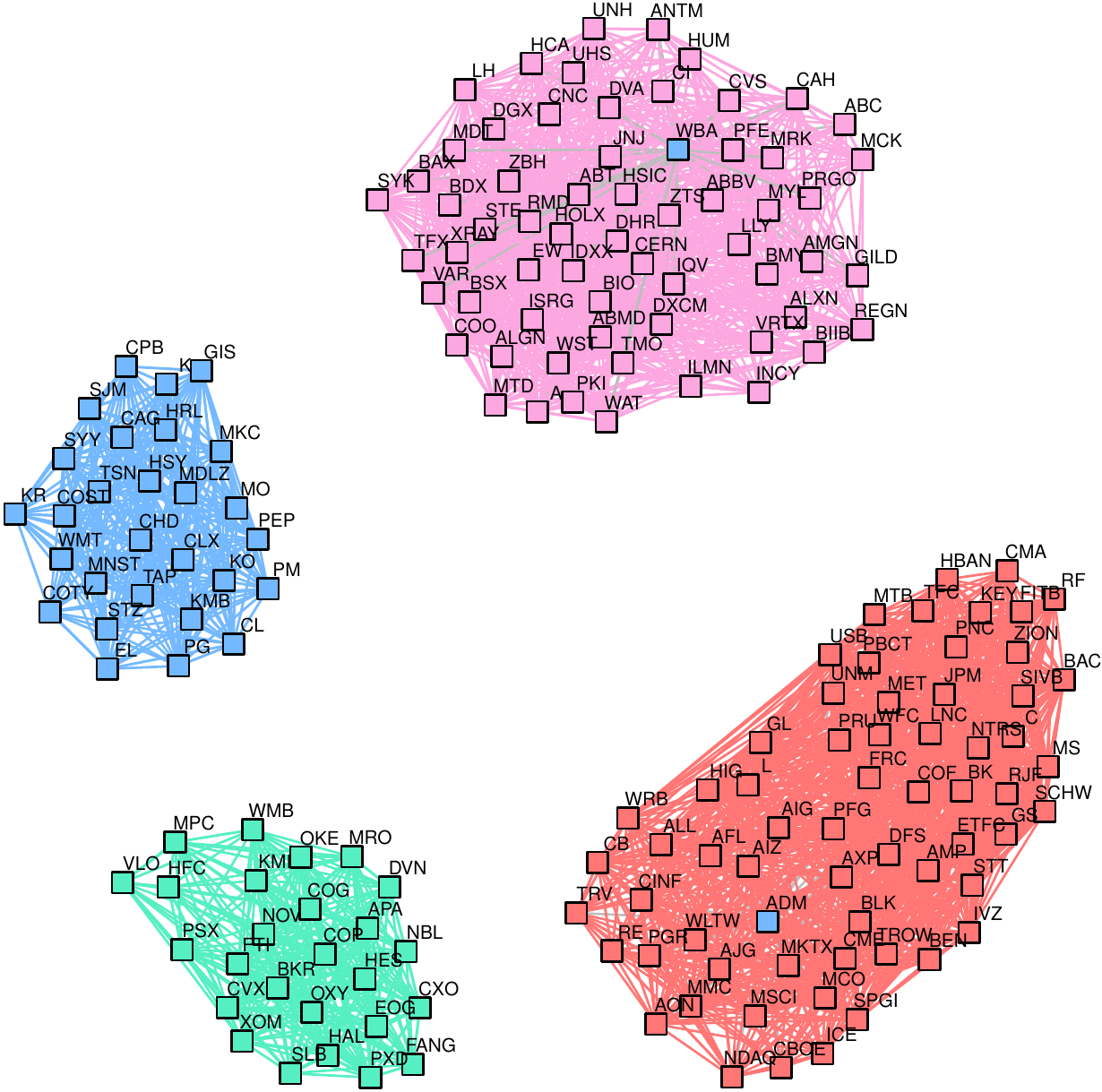}
      \caption{$4$-comp graph learned via the proposed $\mathsf{kGL}$ algorithm.}
      \label{fig:dgl-1}
  \end{subfigure}%
  \\
\begin{subfigure}[t]{0.48\textwidth}
    \centering
    \includegraphics[scale=.58]{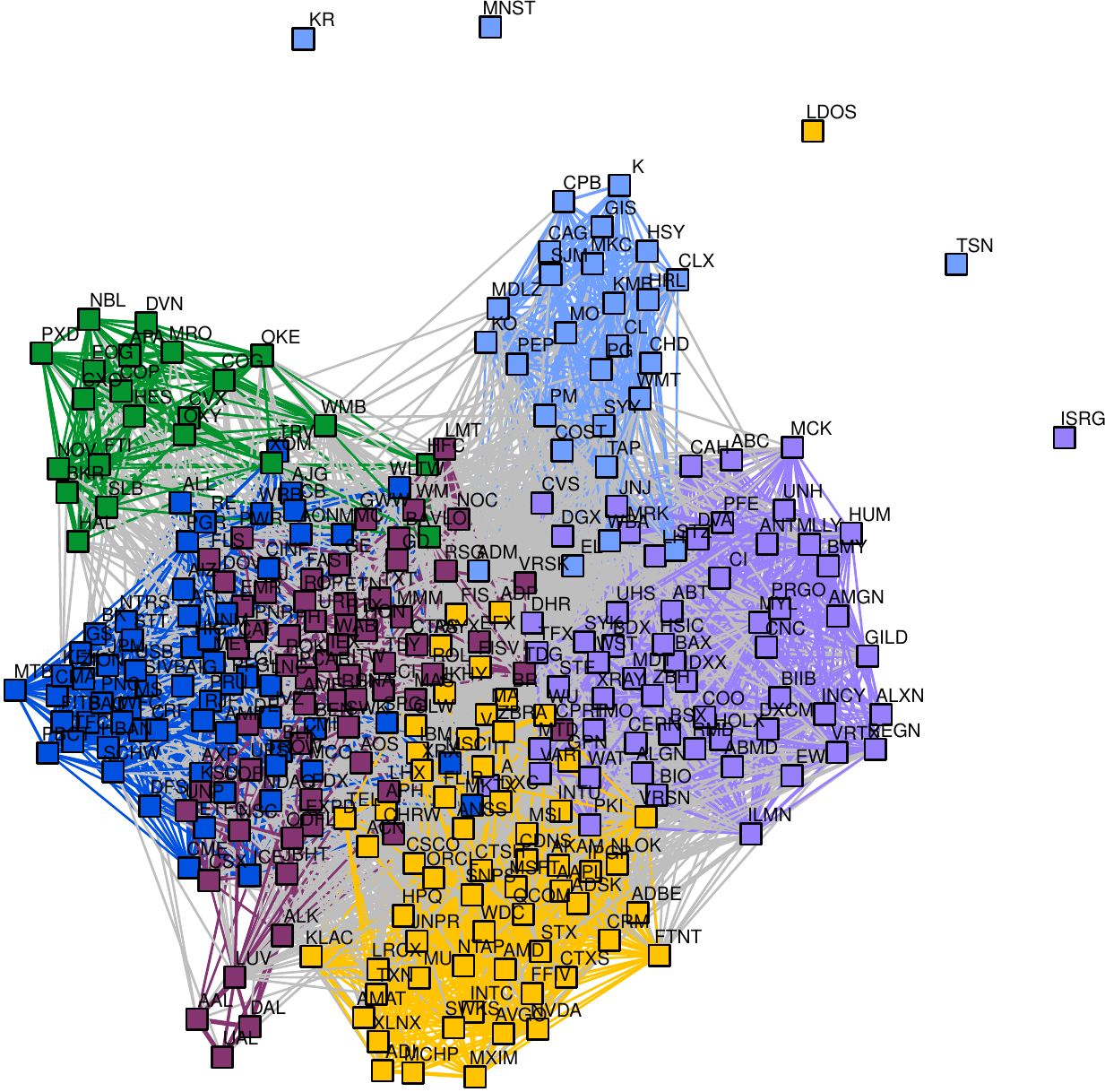}
    \caption{$6$-comp graph learned via $\mathsf{SGL}$. $\eta = 10$.}
    \label{fig:sgl-2}
\end{subfigure}%
~
\begin{subfigure}[t]{0.48\textwidth}
  \centering
  \includegraphics[scale=.58]{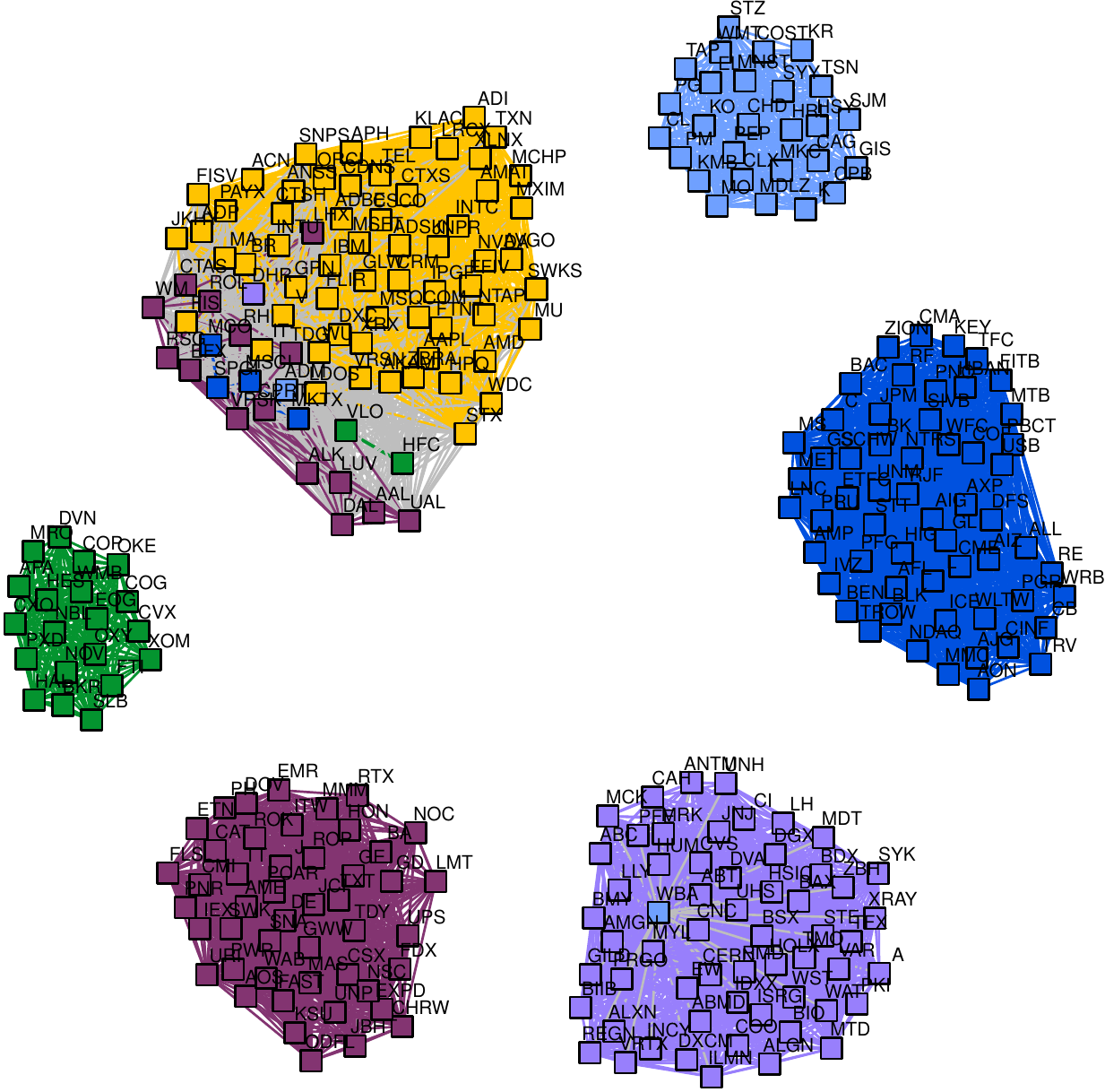}
  \caption{$6$-comp graph learned via the proposed $\mathsf{kGL}$ algorithm.}
      \label{fig:dgl-2}
\end{subfigure}
  \caption{Rank constraints are met by the $\mathsf{SGL}$ algorithm~\citep{kumar2019} (Figures~\ref{fig:sgl-1}, \ref{fig:sgl-2}), nonetheless the learned graph conveys little information due to the lack
           of control on the degrees, which allows the learning of trivial $k$-component graphs, \textit{i.e},
           those containing isolated nodes.}
  \label{fig:isolated-nodes}
\end{figure}

\subsection{Effects of market factor and data preprocessing}
\label{sec:market-factor-effect}

Removing the market factor prior to performing analysis on a set of stock prices is
a common practice~\citep{mantegna1999, laloux2000}. The market factor is the component of stock signals
associated with the strongest spectral coefficient.
As we have argued in Section~\ref{sec:interpretations}, removing the market when learning
graph matrices is implicitly done via the constraint $\bm{L}\mathbf{1} = \mathbf{0}$, for the
estimation of the Laplacian matrix, or via the construction of the $\bm{Z}$ matrix
for the estimation of the adjacency matrix. Therefore, it is not necessary to remove the market
factor when learning graphs. Another guideline presented in Section~\ref{sec:interpretations}
is that one should use the correlation matrix of the stock times series (or, equivalently, rescale the data
such that each stock time series has unit variance) so as to obtain a meaningful cluster representation.

In order to verify these claims in practice,
we set up an experiment were we collected price data from Jan. 3rd 2014 to Dec. 29th 2017
($n = 1006$ observations) of 82 selected stocks from three sectors: $28$ from
$\mathsf{Utilities}$, 31 from $\mathsf{Real~Estate}$, and $23$ from $\mathsf{Communication} \mathsf{Services}$.

We then proceed to learn four graphs, using the proposed $\mathsf{kGL}$ algorithm, with the following
settings for the input data:
\begin{enumerate}
  \item No data scaling and with market signal removed (Figure~\ref{fig:covariance-with-market}).
  \item No data scaling and with market signal included, \textit{i.e.}, no data preprocessing (Figure~\ref{fig:covariance-without-market}).
  \item Scaled data and with market signal removed (Figure~\ref{fig:correlation-with-market}).
  \item Scaled data and with market signal included (Figure~\ref{fig:correlation-without-market}).
\end{enumerate}

The market signal is removed via eigenvalue decomposition of the sample correlation (covariance) matrix,
where the largest eigenvalue is set to be zero.

\begin{figure}[!htb]
  \captionsetup[subfigure]{justification=centering}
  \centering
  \begin{subfigure}[t]{0.49\textwidth}
   %\centering
   \includegraphics[scale=0.5]{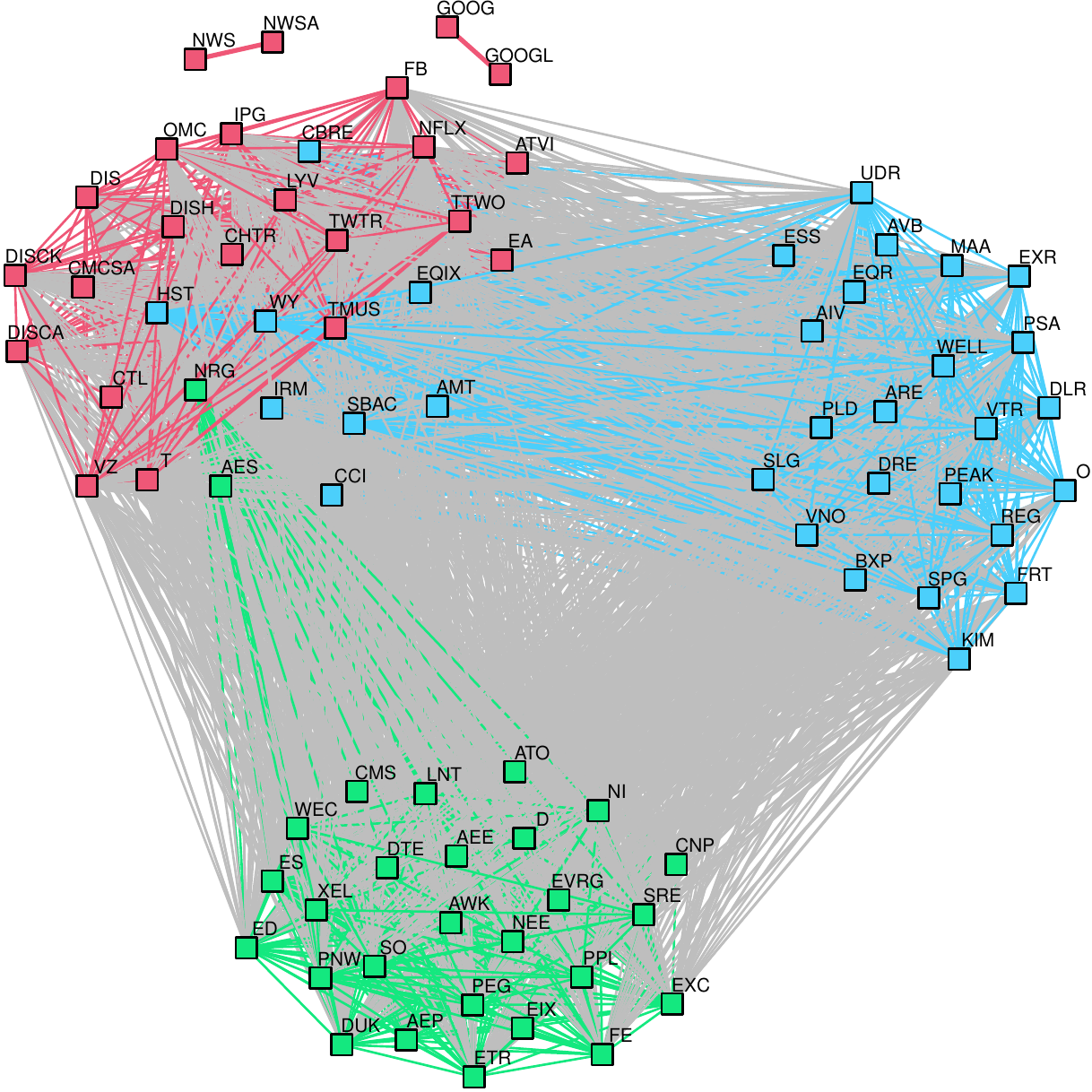}
   \caption{Learned graph without data scaling and removing the market signal.}
   \label{fig:covariance-with-market}
  \end{subfigure}%
  ~
  \begin{subfigure}[t]{0.49\textwidth}
    \centering
    \includegraphics[scale=0.5]{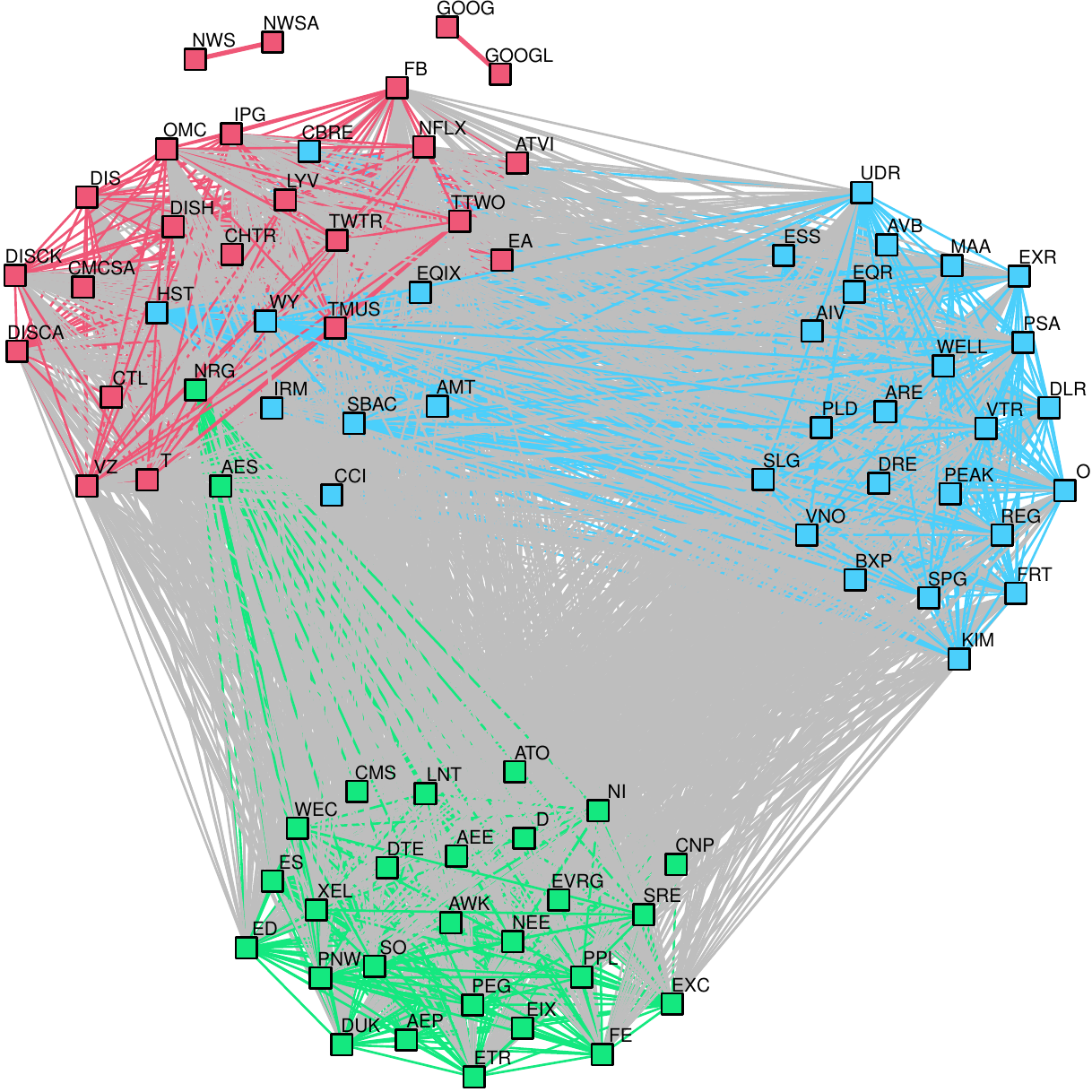}
   \caption{Learned graph without data scaling and without removing the market signal.}
    \label{fig:covariance-without-market}
  \end{subfigure}%
  \\
  \begin{subfigure}[t]{0.49\textwidth}
    %\centering
    \includegraphics[scale=0.5]{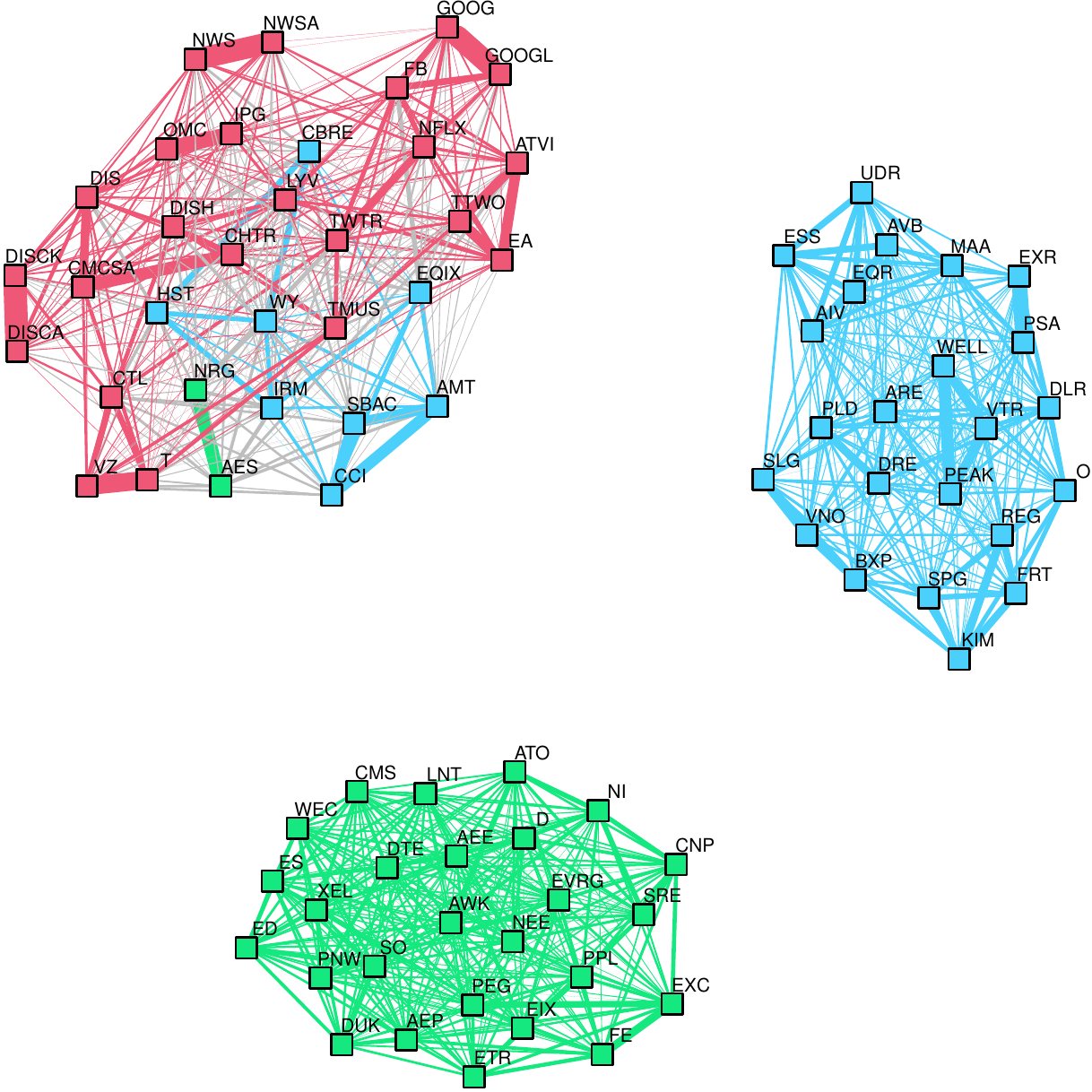}
    \caption{Learned graph with data scaling and without removing the market signal.}
    \label{fig:correlation-with-market}
   \end{subfigure}%
   ~
   \begin{subfigure}[t]{0.49\textwidth}
     \centering
     \includegraphics[scale=0.5]{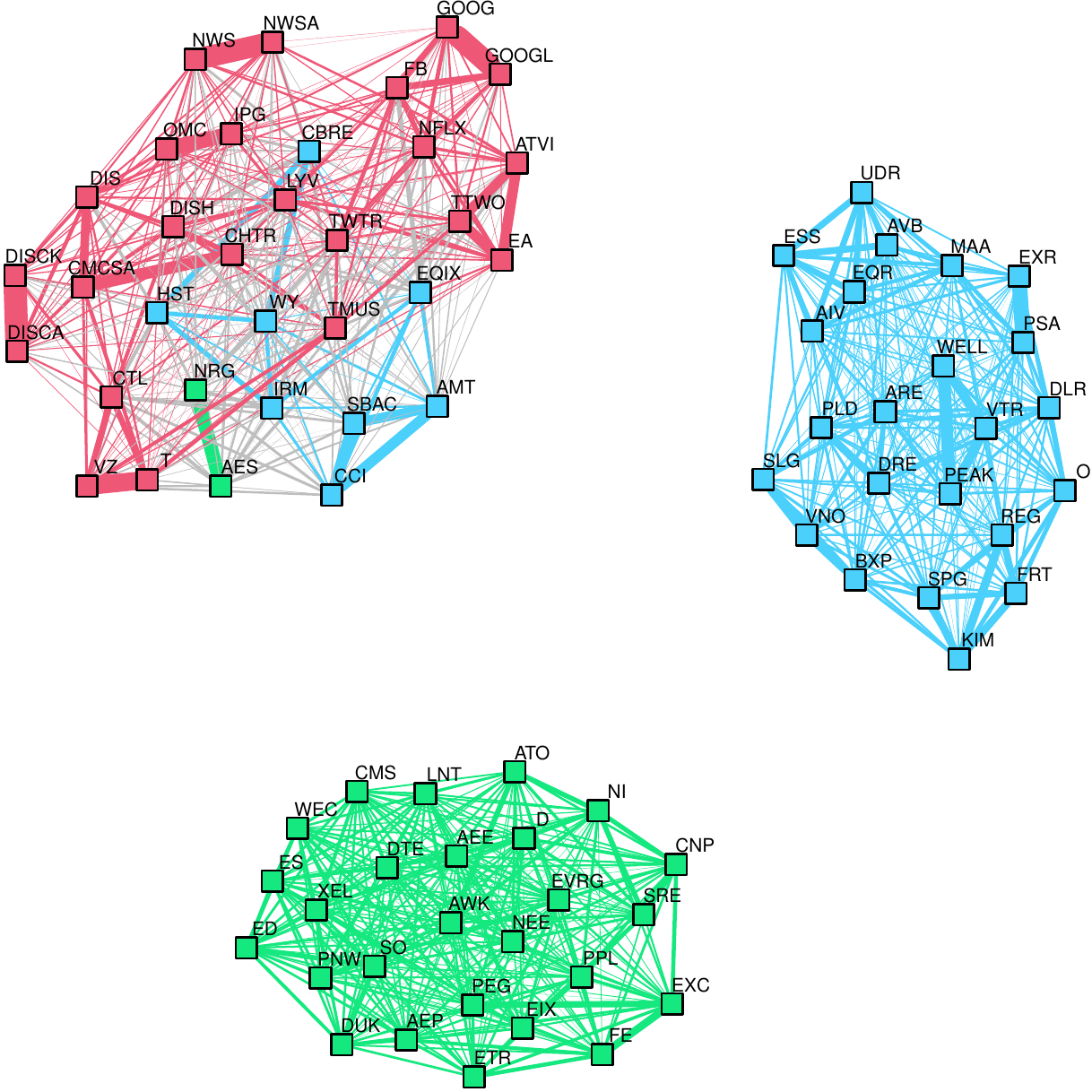}
    \caption{Learned graph with data scaling and removing the market signal.}
     \label{fig:correlation-without-market}
   \end{subfigure}
  \caption{Effects of data preprocessing on the learned graphs.}
  \label{fig:effects-of-data-preprocessing}
\end{figure}

Figures~\ref{fig:covariance-with-market}~and~\ref{fig:covariance-without-market} depict evidence that using the sample
covariance matrix (or equivalently, not scaling the input data), regardless of whether the market signal has been removed, 
leads to a graph with possibly many spurious connections (grey edges) that is not in agreement with the GICS sector classification.
Figures~\ref{fig:correlation-with-market}~and~\ref{fig:correlation-without-market} show that using the sample correlation matrix
(or equivalently, scaling the stock time series such that they have the same variance) prior to learning
the graph clearly shows meaningful graphical representations from stocks belonging to three distinct sectors, regardless of whether
the market signal has been removed. In addition, the relative error between the estimated graphs in
Figures~\ref{fig:covariance-with-market}~and~\ref{fig:covariance-without-market} is 0.09, whereas the relative error between the
estimated graphs in Figures~\ref{fig:correlation-with-market}~and~\ref{fig:correlation-without-market} is $4.6\cdot 10^{-5}$.
Those relative error measurements further confirm that removing the market has little effect on the estimated graph
due to the constraint $\bm{L}\mathbf{1} = \mathbf{0}$, as explained in Section~\ref{sec:interpretations}.

In addition, it has been argued that the sample correlation matrix may not always be a good measure of dependency
for highly noisy, often non-linear dependent signals such as log-returns of stocks~\citep{deprado2020}. 
In the proposed framework, other measures of similarities can be used in place
of the sample correlation matrix. For instance, we learn a graph under the same settings as the aforementioned experiment,
but using the normalized mutual information, $\bar{\bm{I}}$, between the log-return signals (assuming they follow a Gaussian distribution)
as the input similarity matrix, which may be computed as
\begin{equation}
  \bar{\bm{I}}_{ij} = 
  \left\{
  \begin{array}{ll}
    -\frac{1}{2}\log(1 - \bm{\bar{S}}^2_{ij}),  & \mbox{if } i \neq j, \\
    1, & \mbox{else},
  \end{array}
\right.
\end{equation}
where $\bm{\bar{S}}^2_{ij}$ is the sample correlation coefficient between the log-returns of stock $i$ and $j$.

Figure~\ref{fig:mutual-information-graph} depicts the graph structure learned using the normalized mutual information.
As it can be observed, the structure of Figure~\ref{fig:mutual-information-graph} is very similar to that of
Figure~\ref{fig:correlation-with-market}. Objectively, the $\mathsf{f}$-$\mathsf{score}$ between the learned graphs
is $0.91$ while the relative error is $0.35$, which may indicate that using either the normalized mutual information
or the sample correlation matrix are equally acceptable inputs for the learning algorithm.
\begin{figure}[!htb]
  \centering
  \includegraphics[scale=0.7]{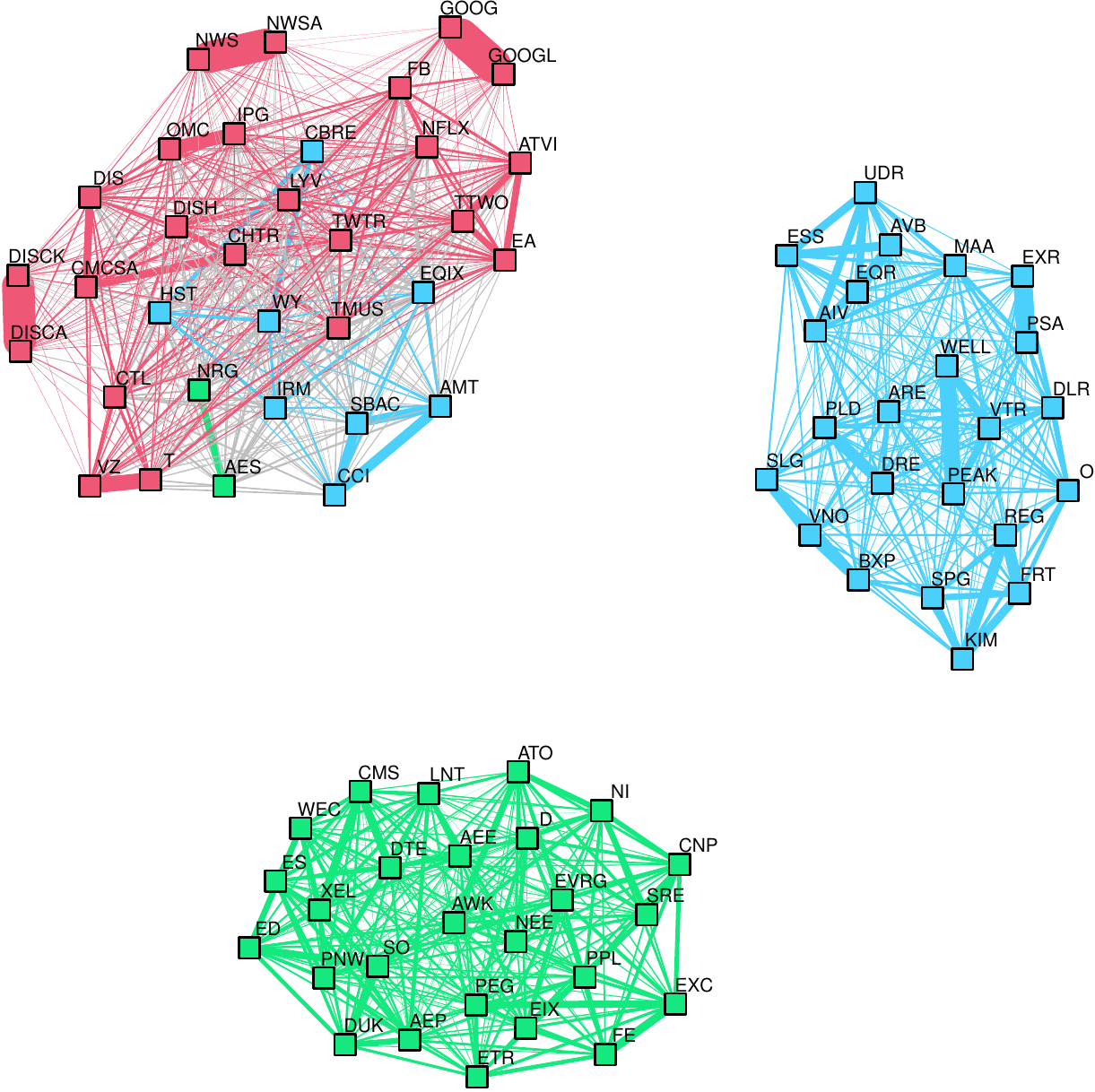}
  \caption{Learned graph with the proposed $\mathsf{kGL}$ algorithm~\ref{alg:k-component-ky-fan} using the normalized mutual information as input matrix.} 
  \label{fig:mutual-information-graph}
\end{figure} 

Finally, we use the state-of-the-art, two-stage $\mathsf{CLR}$ algorithm~\citep{feiping2016} to learn a $3$-component
graph for the selected stocks on the basis of the scaled input data matrix. Figure~\ref{fig:clr-graph} depicts the learned
graph network. As it can be observed, unlike the proposed algorithm, $\mathsf{CLR}$
clusters together most of the stocks belonging to the $\mathsf{Real~State}$ and $\mathsf{Communication~Services}$ sectors,
which is not expected from an expert \textit{prior} information such as GICS.
\begin{figure}[!htb]
  \centering
  \includegraphics[scale=0.7]{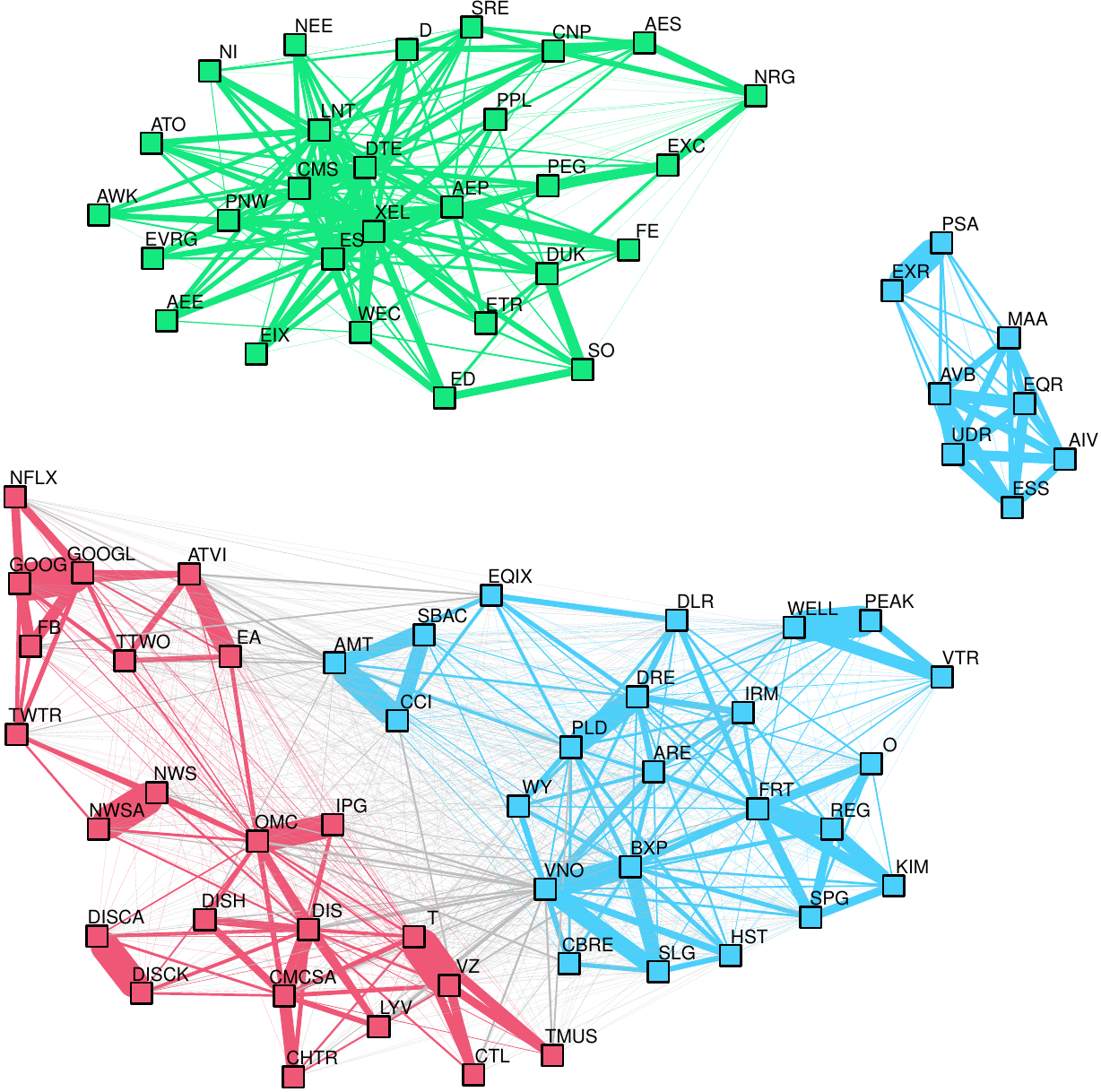} 
  \caption{Learned graph with the $\mathsf{CLR}$ algorithm~\citep{feiping2016} with unit-variance, scaled log-return data matrix as input.} 
  \label{fig:clr-graph}
\end{figure} 

\subsection{Heavy-tails effects: warm-up}

In this experiment, we would like to convey the advantages of using graph learning algorithms based on the assumptions
that the data is heavy-tailed. To that extent, we compare three Laplacian-constrained models: (1) Gaussian
($\mathsf{GLE}$-$\mathsf{ADMM}$~\citep{zhao2019}) (2)
Gaussian with minimax concave sparsity penalty ($\mathsf{NGL}$-$\mathsf{MCP}$ \citep{ying2020nips}), and (3) Student-$t$
($\mathsf{tGL}$ Algorithm~\ref{alg:heavy-tail}).
These models are investigated under two scenarios: (i) strong ($\nu \approx 4$)
and (ii) weak ($\nu \approx 10$) presence of heavy-tails. On both scenarios, we selected stocks belonging
to five different sectors, namely: $\mathsf{Consumer~Staples}$,
$\mathsf{Consumer~Discretionary}$, $\mathsf{Industrials}$, $\mathsf{Energy}$, and $\mathsf{Information~Technology}$.

\noindent\textit{Remark on hyperparameters}: For the Gaussian model with minimax concave penalty, we tune the sparsity hyperparameter
so that the estimated graph obtains the highest modularity. While tunning an one-dimensional hyperparameter
may not pose issues while performing post-event analysis, it does compromise the performance of real-world online
systems where the value of such hyperparameter is often unknown and data-dependent. For the Student-$t$ model,
the degrees of freedom $\nu$ can be computed in a prior stage directly from the data using, \textit{e.g.}, the methods
in~\citep{liu2019a}, or in a sliding-window fashion for the case of real-time systems.

\noindent \textbf{Strong heavy-tails}: for this experiment, we queried data from 222 stocks from Jan. 3rd 2008
to Dec. 31st 2009, which represents 504 data observations per stock, resulting in a sample-parameter size ratio of $n/p \approx 2.27$.
This particular time-frame presents a high amount of volatility due to the 2008 US depression. To quantify the extent of
heavy-tails in the data, we fit a multivariate Student-$t$ distribution using the matrix of
log-returns $\bm{X}$, where we obtain $\nu \approx 4.06$, which indeed indicates a high presence of heavy-tailed data points.
In addition, we measured the average annualized volatility across all stocks and obtained $\mathsf{volatility} \approx 0.53$.
Figure~\ref{fig:strong-ht-exp-1} provides a summary of this market scenario.

\begin{figure}[!htb]
  \captionsetup[subfigure]{justification=centering}
  \centering
  \begin{subfigure}[t]{0.67\textwidth}
   %\centering
   \includegraphics[scale=0.56]{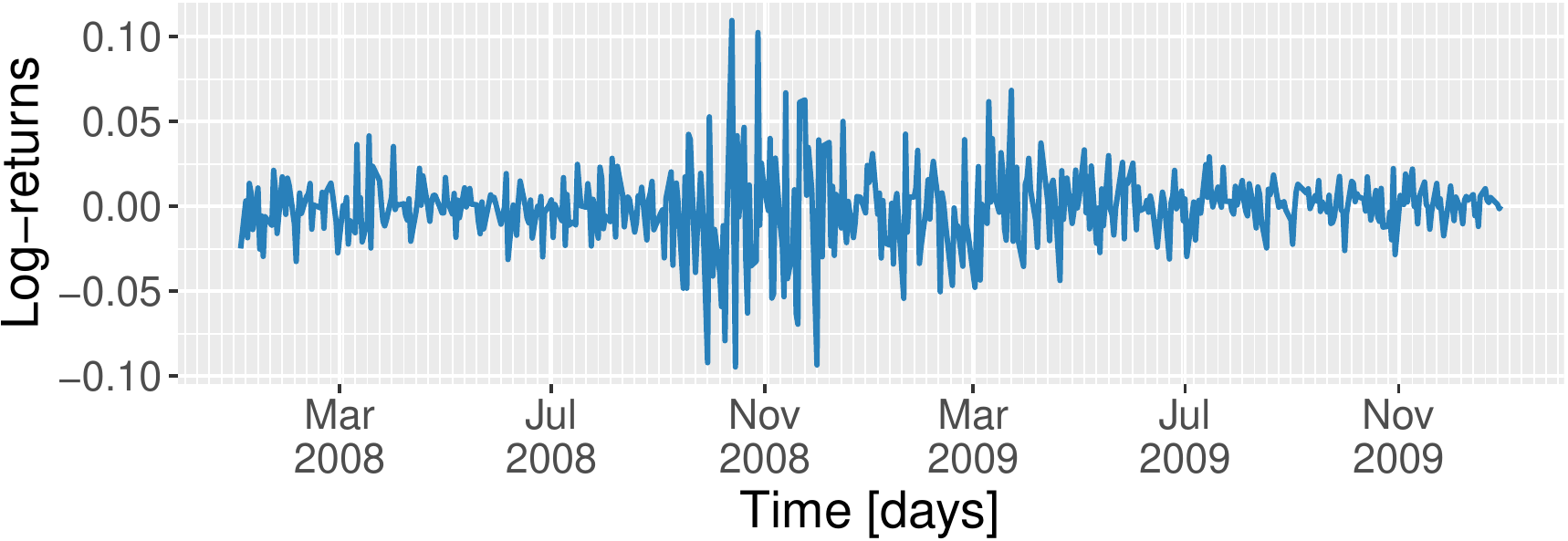}
   \caption{S\&P500 log-returns.}
   \label{fig:strong-returns-exp-1}
  \end{subfigure}%
  ~
  \begin{subfigure}[t]{0.3\textwidth}
    \centering
    \includegraphics[scale=0.56]{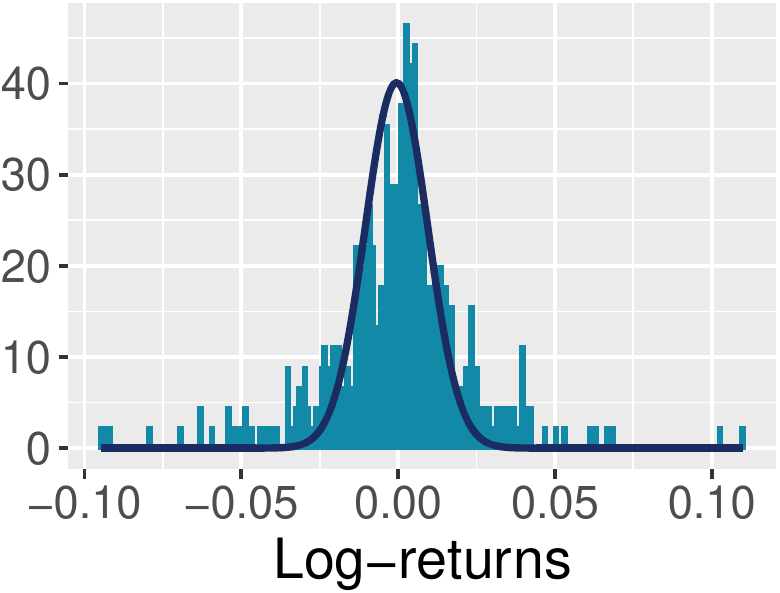}
    \caption{Histogram of S\&P500 log-returns.}
    \label{fig:strong-hist-exp-1}
  \end{subfigure}%
  \caption{State of the US stock market, as captured by the S\&P500 index, on the \textbf{strong heavy-tails} scenario,
  which starts from Jan. 3rd 2008 until Dec. 31st 2009.
  Figure~\ref{fig:strong-returns-exp-1} shows the S\&P500 log-returns time series, where the increase in volatility due to the global
  financial crisis in 2008 is clearly noticeable. Figure~\ref{fig:strong-hist-exp-1} shows a histogram of the S\&P500
  log-returns during the aforementioned time period, where the solid curve represents a Gaussian fit. It can be noticed that the tails
  of the Gaussian decays much faster than the tails of the empirical histogram, indicating the presence of heavy-tails or outliers.}
  \label{fig:strong-ht-exp-1}
\end{figure}

\noindent \textbf{Weak heavy-tails}: in this scenario, we collected data from 204 stocks from Jan. 5th 2004
to Dec. 30th 2006, which represents 503 data points per stock, resulting in a sample-parameter size ratio of $n/p \approx 2.47$.
During this time-window, the market was operating relatively nominal. By fitting a multivariate Student-$t$ distribution
to the matrix of log-returns, we obtain $\nu \approx 10.11$, which indicates little presence of outliers, and that the data is
nearly Gaussian. The average
annualized volatility measured across all stocks is $\mathsf{volatility} \approx 0.27$, which is half of the annualized
volatility in the strong heavy-tails case. Figure~\ref{fig:moderate-ht-exp-1} provides a summary of this market scenario.

\begin{figure}[!htb]
  \captionsetup[subfigure]{justification=centering}
  \centering
  \begin{subfigure}[t]{0.67\textwidth}
   %\centering
   \includegraphics[scale=0.56]{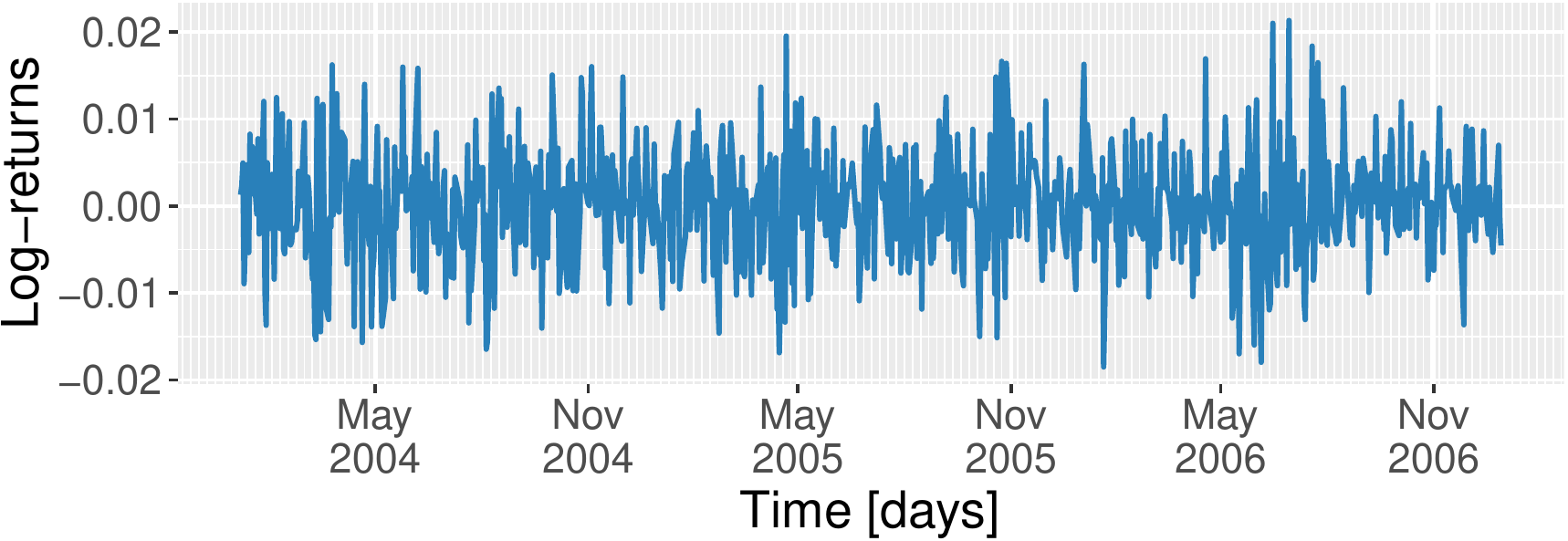}
   \caption{S\&P500 log-returns.}
   \label{fig:moderate-returns-exp-1}
  \end{subfigure}%
  ~
  \begin{subfigure}[t]{0.3\textwidth}
    \centering
    \includegraphics[scale=0.56]{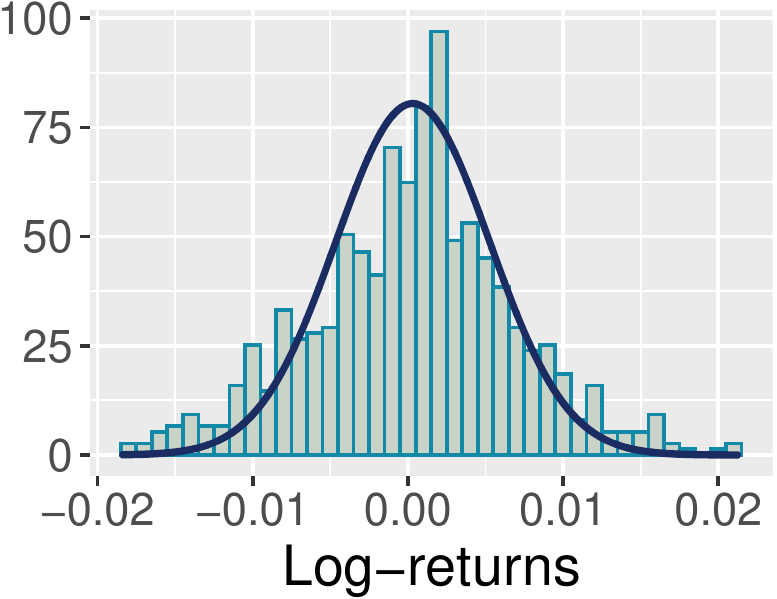}
    \caption{Histogram of S\&P500 log-returns.}
    \label{fig:moderate-hist-exp-1}
  \end{subfigure}%
  \caption{State of the US stock market, as captured by the S\&P500 index, on the \textbf{weak heavy-tails} scenario,
  which starts from Jan. 5th 2004 until Dec. 30th 2006.
  Figure~\ref{fig:moderate-returns-exp-1} shows the S\&P500 log-returns time series, where no noticeable volatility
  clustering event is present, while Figure~\ref{fig:moderate-ht-exp-1} depicts its histogram along with a Gaussian
  fit that closely matches the empirical distribution.}
  \label{fig:moderate-ht-exp-1}
\end{figure}

Figure~\ref{fig:five-sectors-heavy-tail} depicts the learned stock graphs on these scenarios.
In either scenario, it can be readily noticed that the graphs learned with the Student-$t$ distribution are
sparser than those learned with the Gaussian assumption, which results from the fact that the Gaussian distribution
is more sensitive to outliers. As for the Gaussian graphs with sparsity, they present a significant improvement
when compared to the non-sparse counterpart.
The Student-$t$ graphs, on the other hand, present the highest degree of interpretability
as measured by their higher modularity value and ratio between the number of intra-sector edges
and inter-sector edges (\textit{cf.} Tables~\ref{tab:edge-distribution-strong}~and~\ref{tab:edge-distribution-moderate}),
which is the expected behavior from stock sector classification systems such as GICS. 

Among the learned Gaussian graphs
(Figures~\ref{fig:ht-graph-gauss-strong}~and~\ref{fig:ht-graph-gauss-moderate}),
it can be seen that the learned graph in the weak heavy-tailed scenario presents a cleaner graphical
representation, by having less inter-sector and intra-sector edges, while also having a higher graph modularity 
(\textit{cf.} Tables~\ref{tab:edge-distribution-strong}~and~\ref{tab:edge-distribution-moderate}),
than that of the Gaussian graph in the strong heavy-tailed scenario.

\begin{figure}[!htb]
  \captionsetup[subfigure]{justification=centering}
  \centering
  \begin{subfigure}[t]{0.32\textwidth}
      \centering
      \includegraphics[scale=.37]{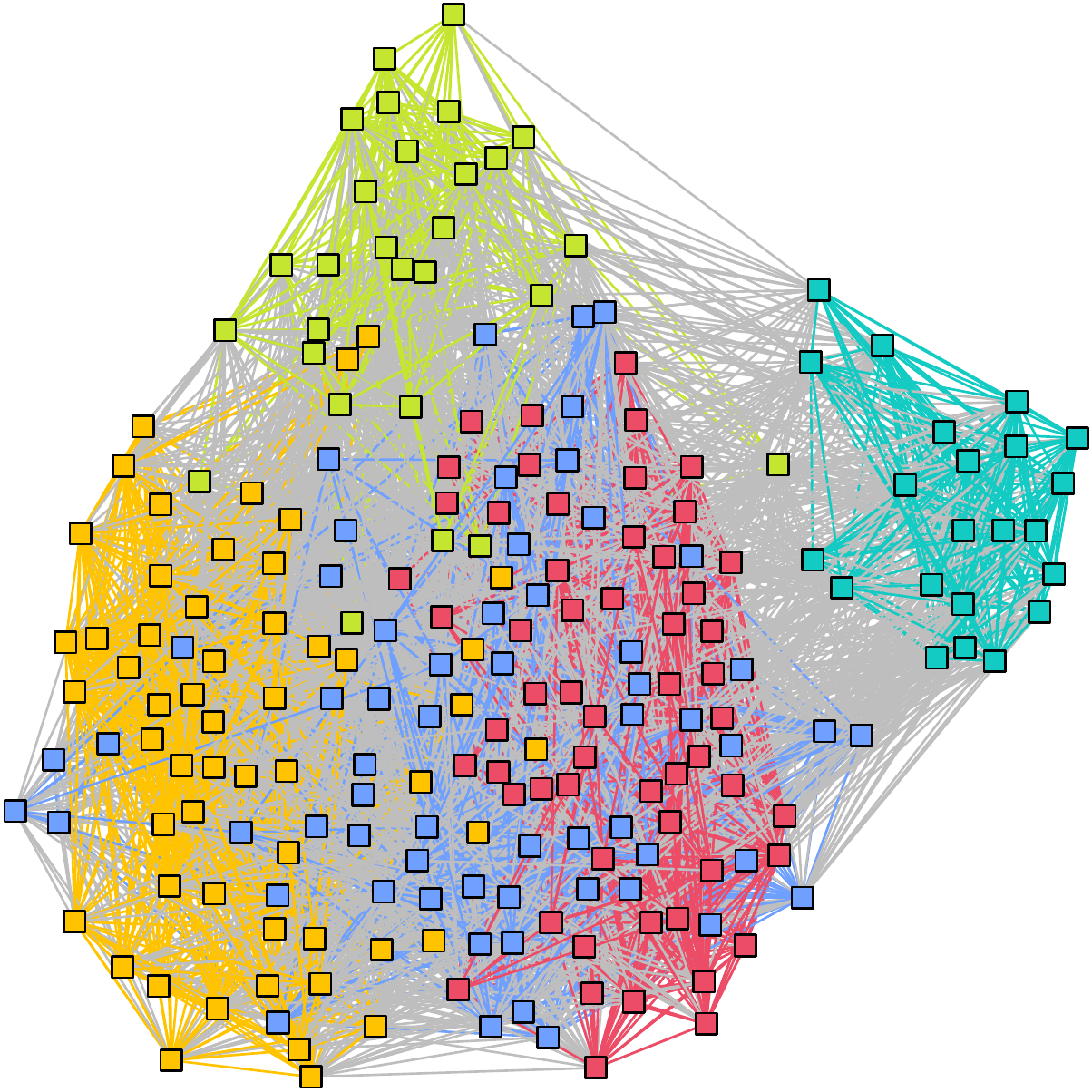}
      \caption{Learned Gaussian graph on strong heavy-tails. $Q = 0.23$.}
      \label{fig:ht-graph-gauss-strong}
  \end{subfigure}%
  ~
  \begin{subfigure}[t]{0.32\textwidth}
    \centering
    \includegraphics[scale=.37]{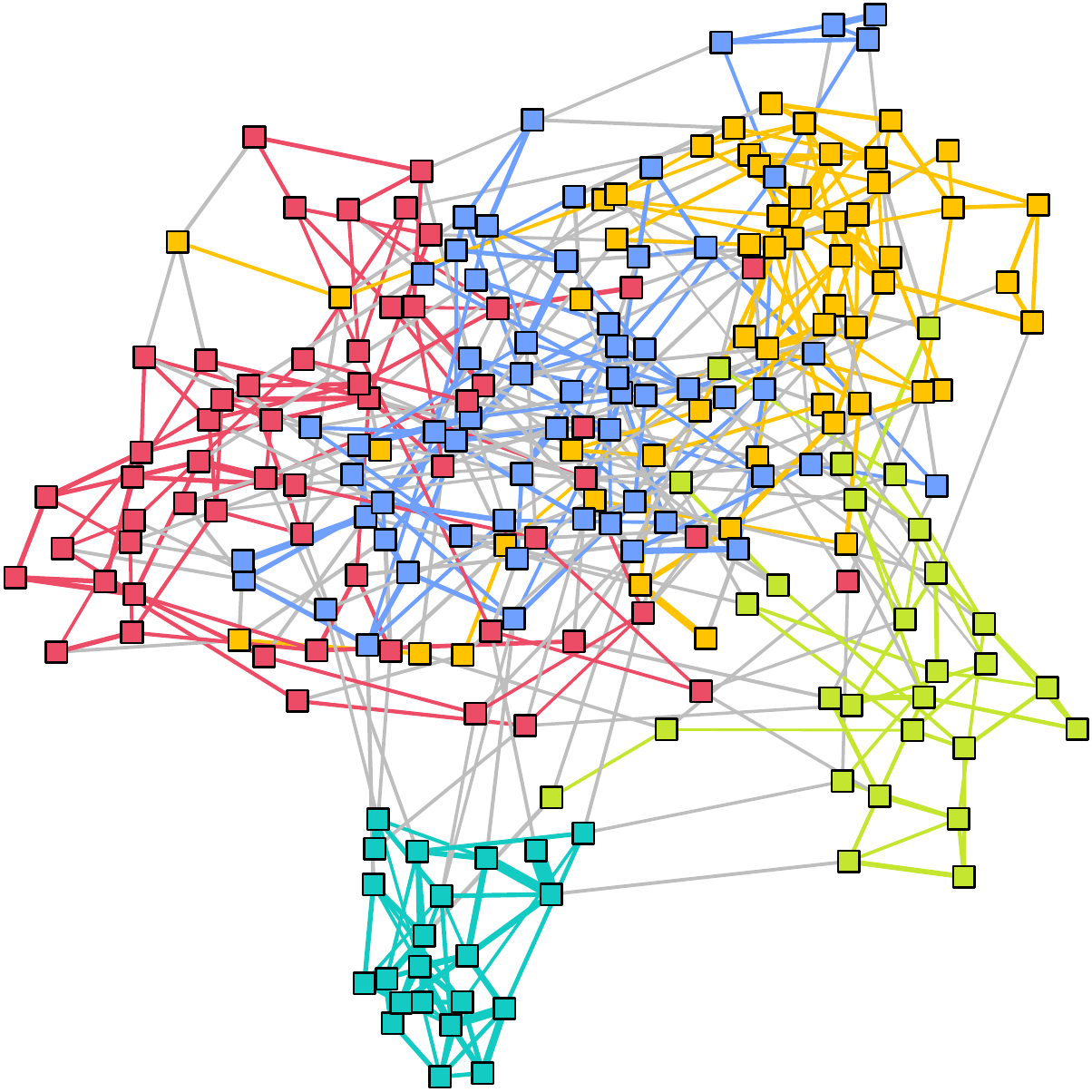}
    \caption{Learned Gaussian graph with sparsity on strong heavy-tails. $Q = 0.46$.}
    \label{fig:ht-graph-mcp-strong}
  \end{subfigure}%
  ~
  \begin{subfigure}[t]{0.32\textwidth}
      \centering
      \includegraphics[scale=.37]{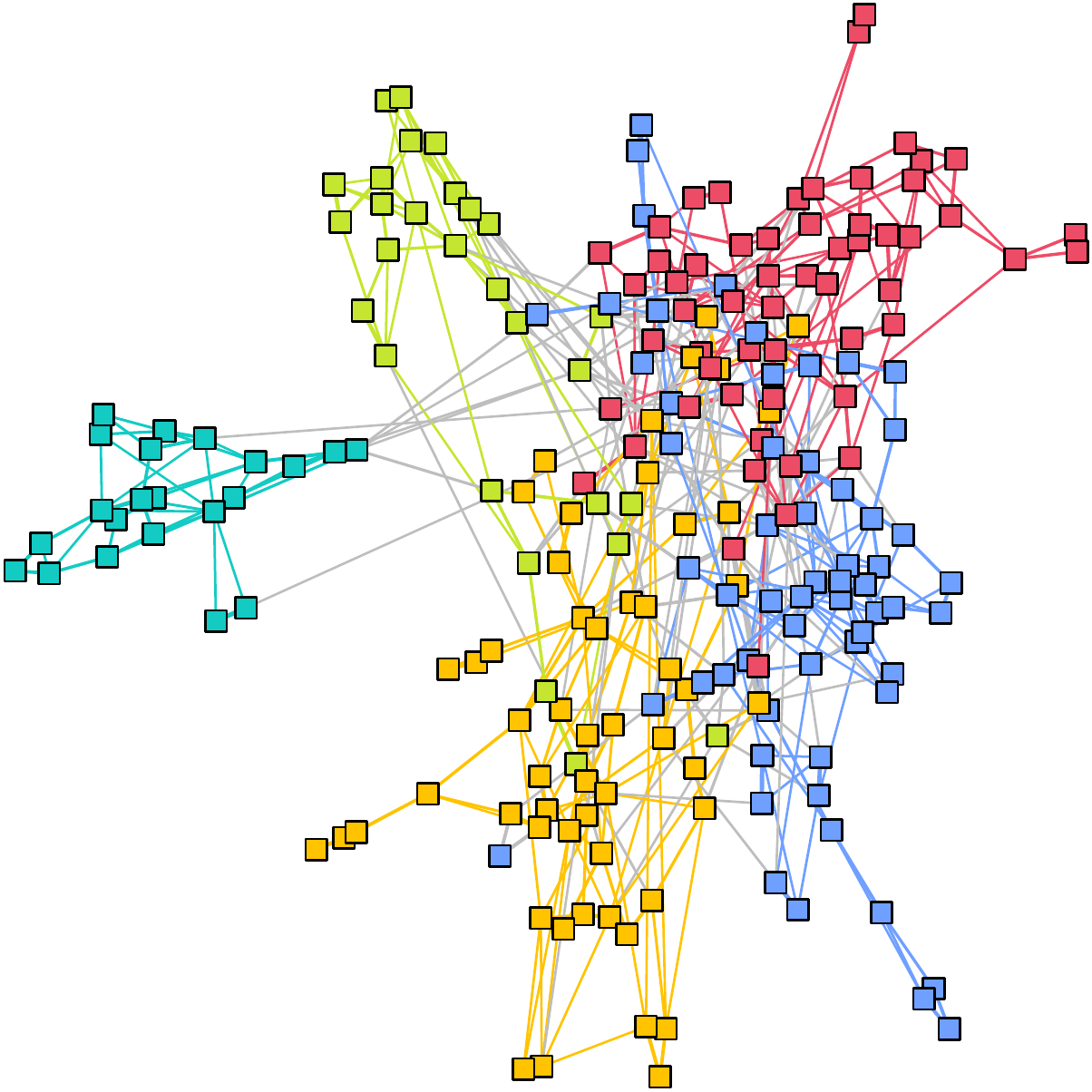}
      \caption{Learned Student-$t$ graph on strong heavy-tails. $Q = 0.51$.}
      \label{fig:ht-graph-t-strong}
  \end{subfigure}%
  \\
  \begin{subfigure}[t]{0.32\textwidth}
      \centering
      \includegraphics[scale=.37]{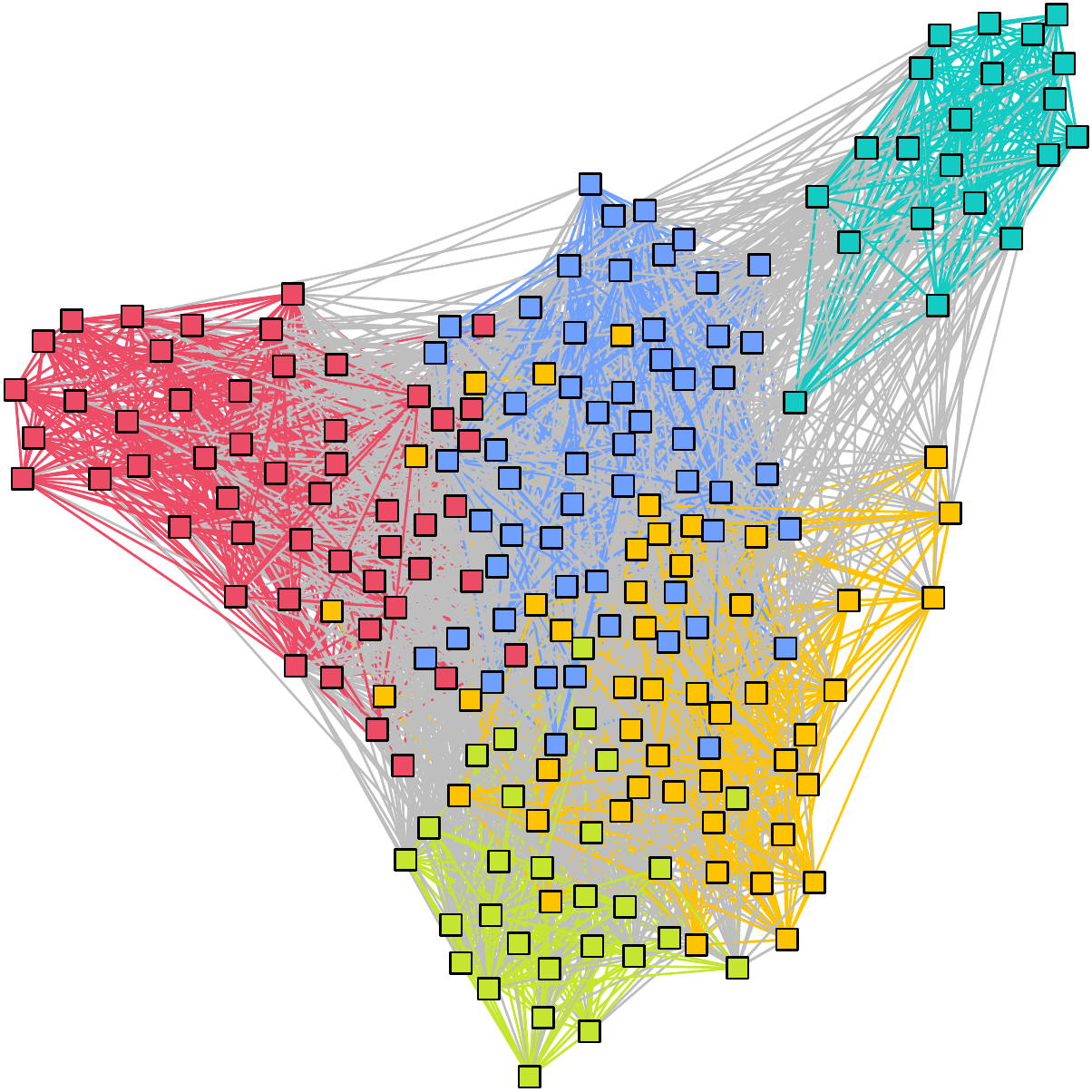}
      \caption{Learned Gaussian graph on weak heavy-tails. $Q = 0.26$.}
      \label{fig:ht-graph-gauss-moderate}
  \end{subfigure}%
  ~
  \begin{subfigure}[t]{0.32\textwidth}
      \centering
      \includegraphics[scale=.37]{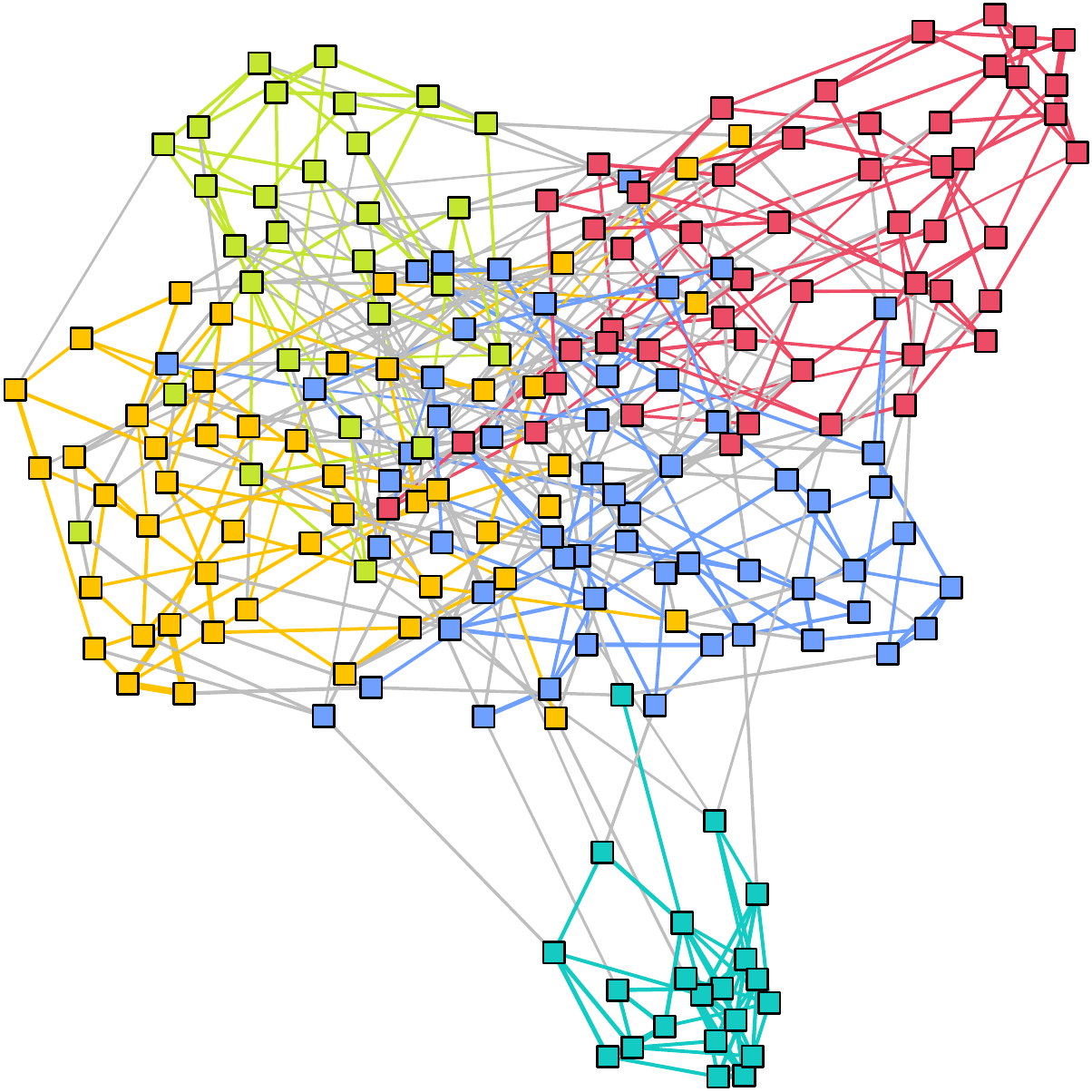}
      \caption{Learned Gaussian graph with sparsity on weak heavy-tails. $Q = 0.44$.}
      \label{fig:ht-graph-mcp-moderate}
  \end{subfigure}%
  ~
  \begin{subfigure}[t]{0.32\textwidth}
      \centering
      \includegraphics[scale=.37]{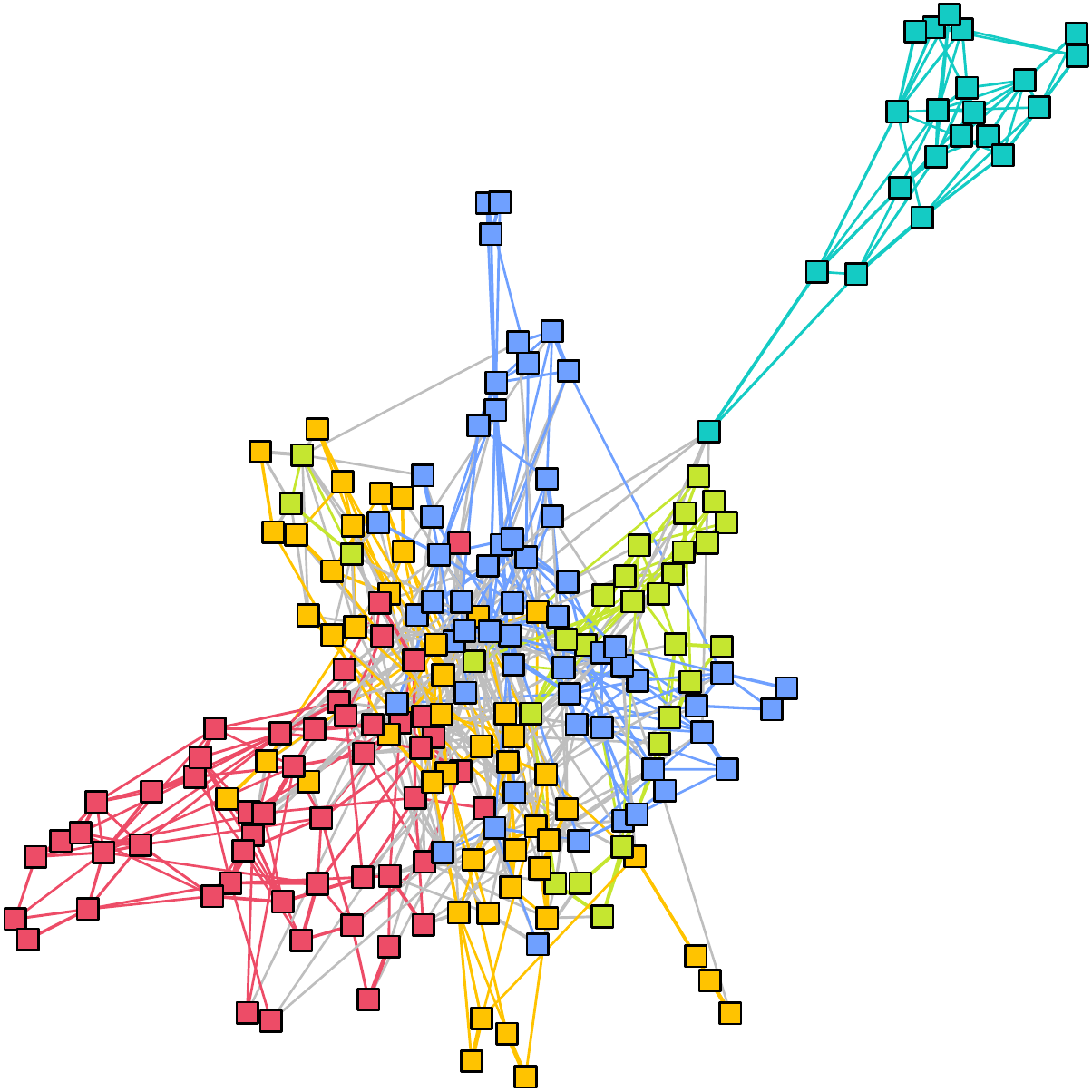}
      \caption{Learned Student-$t$ graph on weak heavy-tails. $Q = 0.46$.}
      \label{fig:ht-graph-t-moderate}
  \end{subfigure}%
  \caption{Learned graph networks with
  Gaussian (Figures~\ref{fig:ht-graph-gauss-strong}~and~\ref{fig:ht-graph-gauss-moderate}),
  Gaussian with sparsity (Figures~\ref{fig:ht-graph-mcp-strong}~and~\ref{fig:ht-graph-mcp-moderate}),
  and Student-$t$ (Figures~\ref{fig:ht-graph-t-strong}~and~\ref{fig:ht-graph-t-moderate}),
  for contrasting heavy-tail scenarios.}
  \label{fig:five-sectors-heavy-tail}
\end{figure}

\begin{table}[!htb]
  \centering
  \caption{Edge distribution for the \textbf{strong} heavy-tails case.}
 \begin{tabular}{llll}
  \specialrule{.05em}{.05em}{.2em} 
  model & inter-sector edges & intra-sector edges & modularity ($Q$) \\
  Gaussian & 2918 & 2615 & 0.23\\
  Gaussian w/ sparsity & 158 & 339 & 0.46\\
  Student-$t$ & 137 & 384 & \textbf{0.51}
 \end{tabular}
 \label{tab:edge-distribution-strong}
\end{table}

\begin{table}[!htb]
  \centering
  \caption{Edge distribution for the \textbf{weak} heavy-tails case.}
 \begin{tabular}{llll}
  \specialrule{.05em}{.05em}{.2em} 
  model & inter-sector edges & intra-sector edges & modularity ($Q$) \\
  Gaussian & 2028 & 1966 & 0.26\\  
  Gaussian w/ sparsity & 173 & 325 & 0.44\\
  Student-$t$ & 197 & 438 & \textbf{0.46}
 \end{tabular}
 \label{tab:edge-distribution-moderate}
\end{table}

\subsection{Heavy-tails effects: additional analysis}

In this section, we perform a similar analysis as in the previous experiment,
except that we consider stocks from the sectors $\mathsf{Industrials}$,
$\mathsf{Consumer~Staples}$, $\mathsf{Consumer~Discretionary}$,
$\mathsf{Information~Technology}$, $\mathsf{Energy}$, $\mathsf{Health~Care}$, and $\mathsf{Real~State}$.

\noindent \textbf{Strong heavy-tails}: for this experiment, we queried data from 347 stocks from Jan. 5th 2016
to Dec. 23rd 2020, which represents 1253 data observations per stock, resulting in a sample-parameter ratio of
$n/p \approx 3.61$. This particular time-frame presents an extreme high amount of volatility around the beginning of 2020
due to the financial crisis caused by the COVID-19 pandemic. To quantify the amount of outliers in this time frame,
we fit a multivariate Student-$t$ distribution using the matrix of
log-returns $\bm{X}$, where we obtain $\nu \approx 4.15$, which indeed indicates a high presence of heavy-tailed data points.
In addition, we measured the average annualized volatility across all stocks and obtained $\mathsf{volatility} \approx 0.34$.
Figure~\ref{fig:strong-ht-exp-2} provides a summary of this market scenario.

\begin{figure}[!htb]
  \captionsetup[subfigure]{justification=centering}
  \centering
  \begin{subfigure}[t]{0.67\textwidth}
   %\centering
   \includegraphics[scale=0.56]{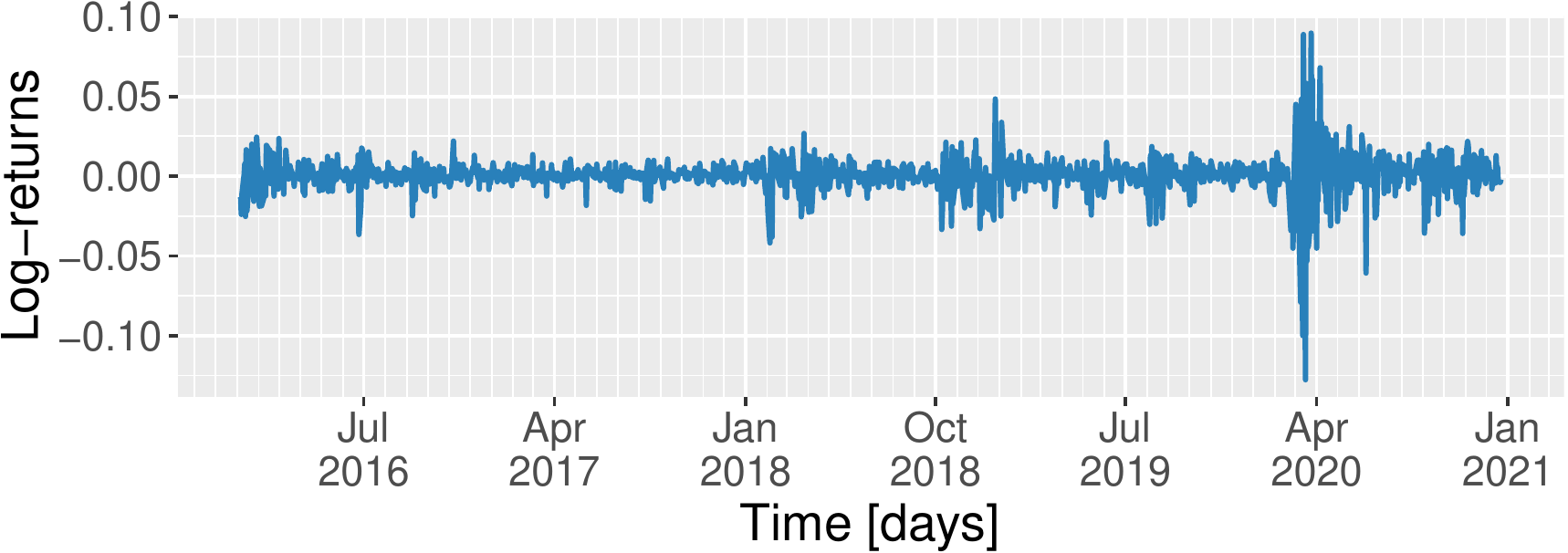}
   \caption{S\&P500 log-returns.}
   \label{fig:strong-returns-exp-2}
  \end{subfigure}%
  ~
  \begin{subfigure}[t]{0.3\textwidth}
    \centering
    \includegraphics[scale=0.57]{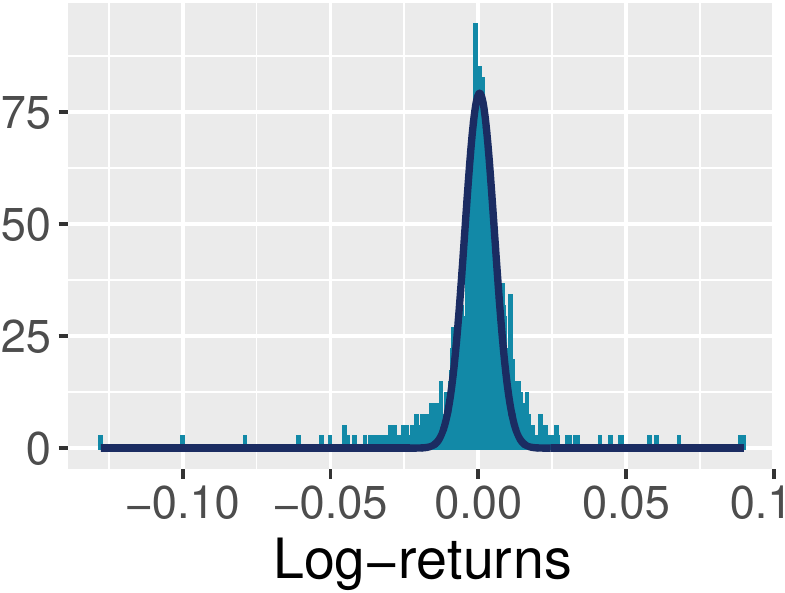}
    \caption{Histogram of S\&P500 log-returns.}
    \label{fig:strong-hist-exp-2}
  \end{subfigure}%
  \caption{State of the US stock market, as captured by the S\&P500 index, on the \textbf{strong heavy-tails} scenario,
  which starts from Jan. 5th 2016 until Jul. 20th 2020.
  Figure~\ref{fig:strong-returns-exp-2} shows the S\&P500 log-returns time series, where the increase in volatility due to the COVID-19
  pandemic is prominent. Figure~\ref{fig:strong-hist-exp-2} illustrates the empirical distribution of the S\&P500 log-returns along with
  its Gaussian fit. It can be noticed that events far beyond the tails decay are present.}
  \label{fig:strong-ht-exp-2}
\end{figure}

\noindent \textbf{Moderate heavy-tails}: for this setting, we queried data from 332 stocks from Jan. 2nd 2013
to Jun. 29th 2018, which represents 1383 data observations per stock, resulting in a sample-parameter ratio of
$n/p \approx 4.17$.
We fit a multivariate Student-$t$ distribution using the matrix of log-returns $\bm{X}$,
where we obtain $\nu \approx 7.11$.
In addition, we measured the annualized average volatility across all stocks and obtained $\mathsf{volatility} \approx 0.25$.
Figure~\ref{fig:moderate-ht-exp-2} provides a summary of this market scenario.

\begin{figure}[!htb]
  \captionsetup[subfigure]{justification=centering}
  \centering
  \begin{subfigure}[t]{0.67\textwidth}
   %\centering
   \includegraphics[scale=0.56]{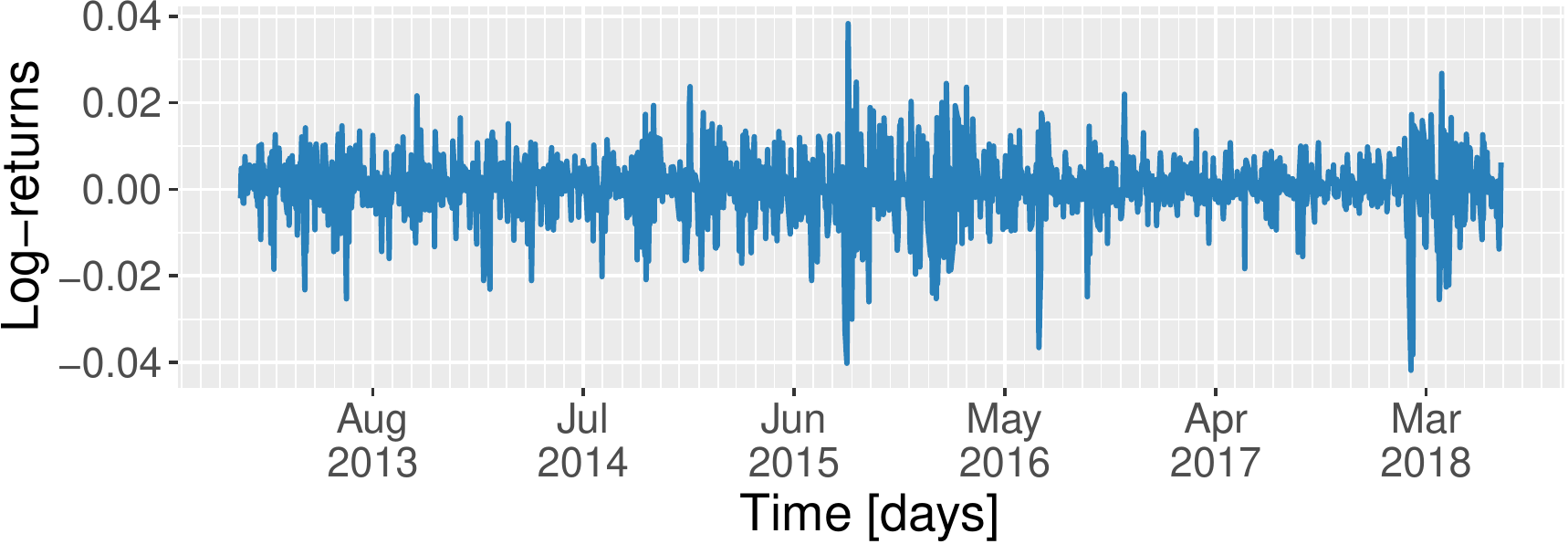}
   \caption{S\&P500 log-returns.}
   \label{fig:moderate-returns-exp-2}
  \end{subfigure}%
  ~
  \begin{subfigure}[t]{0.3\textwidth}
    \centering
    \includegraphics[scale=0.56]{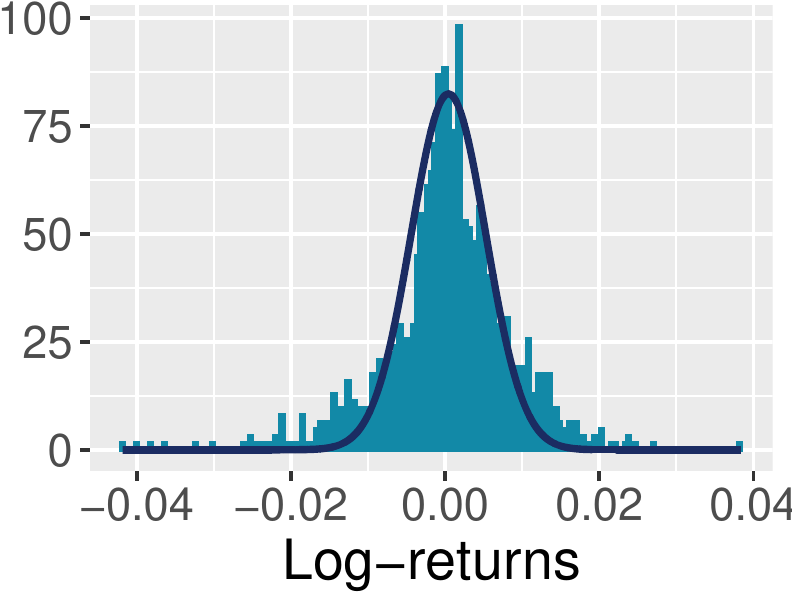}
    \caption{Histogram of S\&P500 log-returns.}
    \label{fig:moderate-hist-exp-2}
  \end{subfigure}%
  \caption{State of the US stock market, as captured by the S\&P500 index, on the \textbf{moderate heavy-tails} scenario,
  which starts from Jan. 2nd 2013 until Jun. 29th 2018.
  Figure~\ref{fig:moderate-returns-exp-2} shows the S\&P500 log-returns time series, where a few significant heavy-tailed
  observations are noticeable, along with its histogram (Figure~\ref{fig:moderate-hist-exp-2}) whose Gaussian fit
  indicates that the presence of outlier data points are mostly concentrated around the turning points of the tails.}
  \label{fig:moderate-ht-exp-2}
\end{figure}

Figure~\ref{fig:all-sectors-heavy-tail} depicts the learned graphs on the aforementioned scenarios.
Similarly from the previous experiment, it can be noticed that in either scenario the graphs learned with
Student-$t$ are indeed sparser, more modular, and hence, more interpretable, than those learned with the Gaussian assumption
(\textit{cf.} Tables~\ref{tab:edge-distribution-strong-full}~and~\ref{tab:edge-distribution-moderate-full}).

The learned Gaussian graphs
(Figures~\ref{fig:ht-graph-gauss-strong-full}~and~\ref{fig:ht-graph-gauss-moderate-full}),
are very dense and present a high number of spurious connections (grey edges) among stocks that are arguably not related in practice.
However, it can be noticed that the learned graph in the heavy-tailed scenario presents a cleaner graphical
representation, by having less inter-sector edges, while also having a higher graph modularity 
(\textit{cf.} Tables~\ref{tab:edge-distribution-moderate-full}~and~\ref{tab:edge-distribution-strong-full})
than that of the Gaussian graph in the strong heavy-tailed scenario, which is consistent with the previous experiment.
The usage of sparsity in does improve the Gaussian graphs
(Figures~\ref{fig:ht-graph-mcp-strong-full}~and~\ref{fig:ht-graph-mcp-moderate-full}), but only to limited extent when compared to the
graphs learned using the Student-$t$ assumption.

The Student-$t$ graphs, on the other hand, present a high degree of modularity, where most of the sectors can easily be
identified. We also measured that the Student-$t$ distribution outputs graphs
(Figure~\ref{fig:ht-graph-t-strong-full}~and~\ref{fig:ht-graph-t-moderate-full}) with higher graph modularity.

\begin{figure}[!htb]
  \captionsetup[subfigure]{justification=centering}
  \centering
  \begin{subfigure}[t]{0.32\textwidth}
      \centering
      \includegraphics[scale=.37]{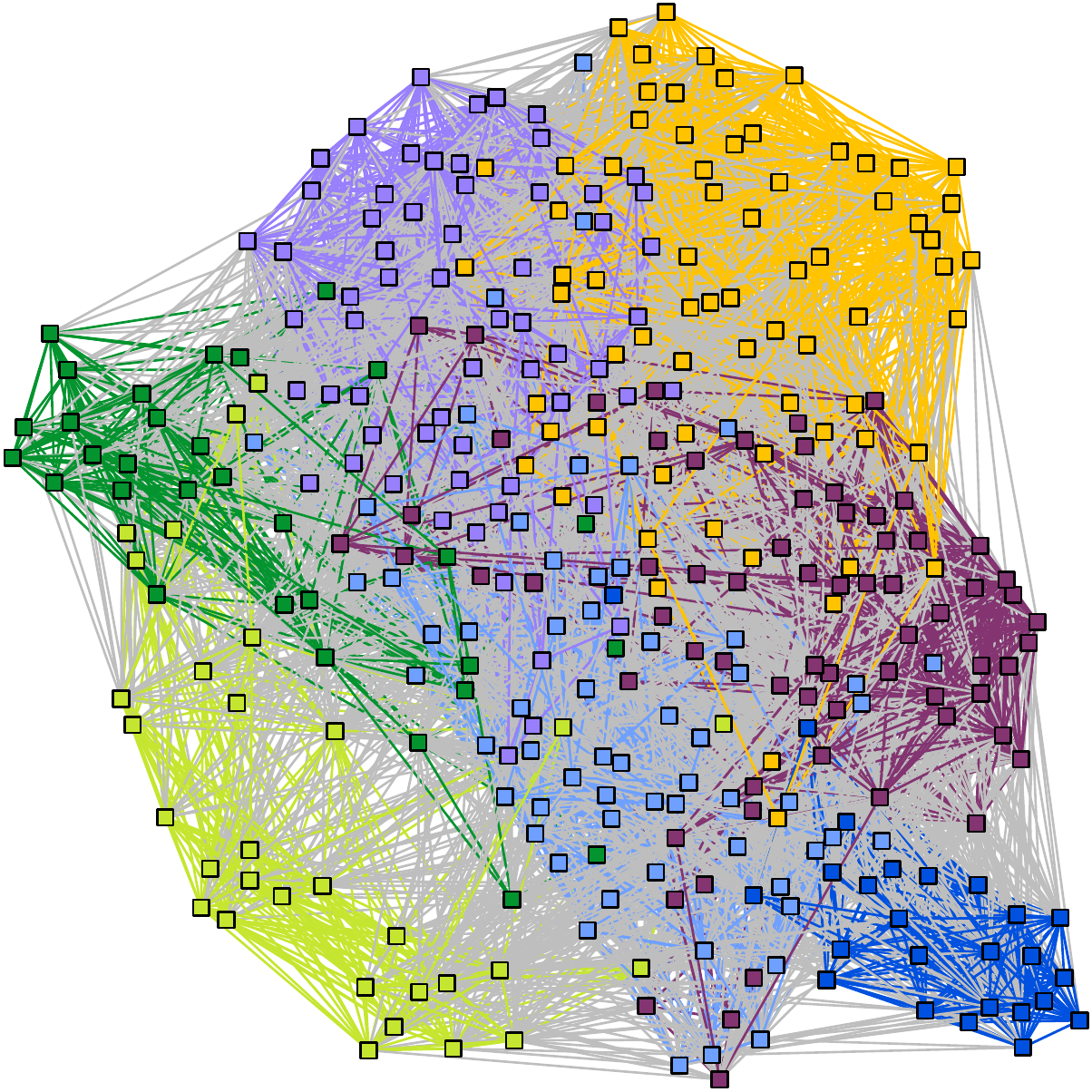}
      \caption{Learned Gaussian graph on strong heavy-tails. $Q = 0.31$.}
      \label{fig:ht-graph-gauss-strong-full}
  \end{subfigure}%
  ~
  \begin{subfigure}[t]{0.32\textwidth}
    \centering
    \includegraphics[scale=.37]{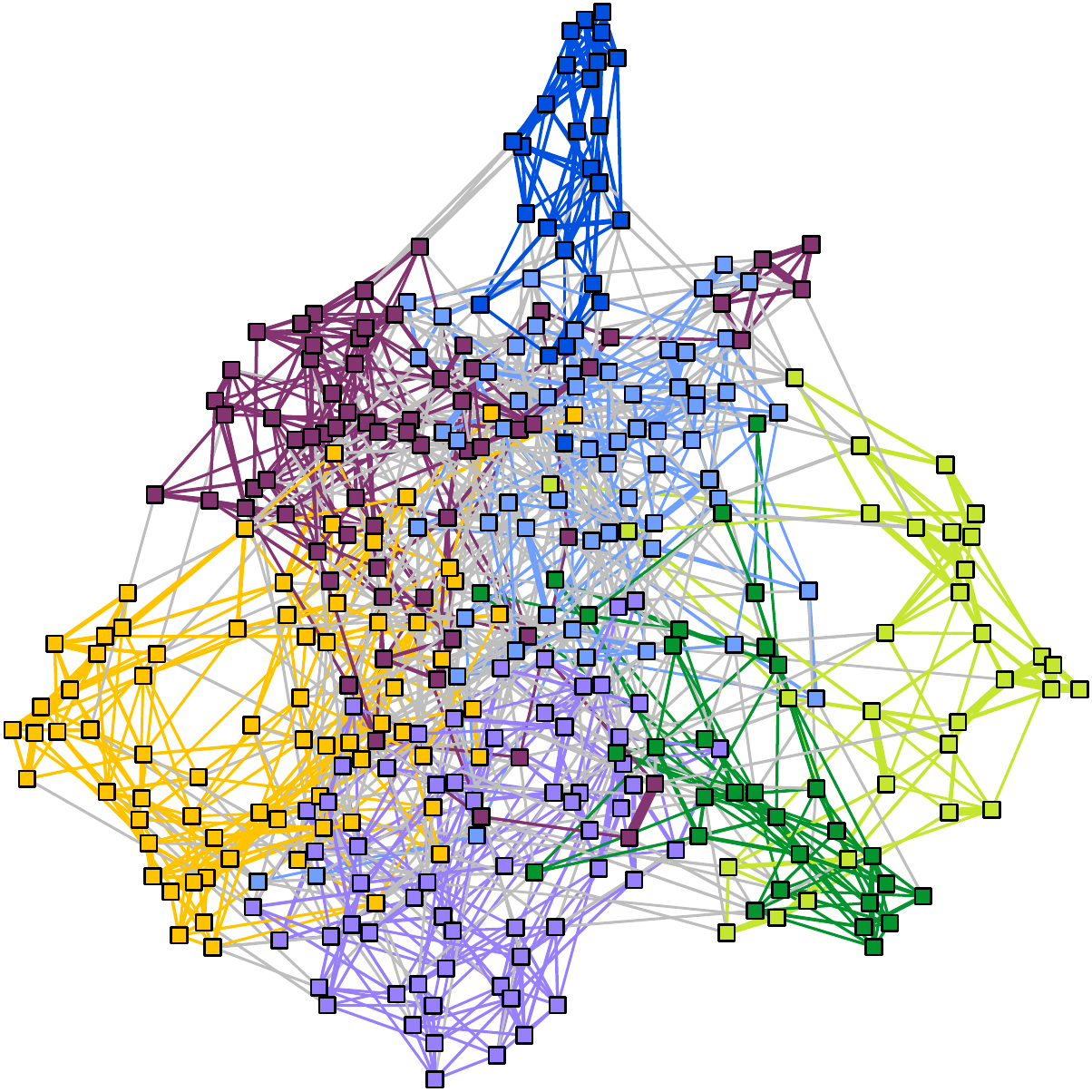}
    \caption{Learned Gaussian graph with sparsity on strong heavy-tails. $Q = 0.54$.}
    \label{fig:ht-graph-mcp-strong-full}
  \end{subfigure}%
  ~
  \begin{subfigure}[t]{0.32\textwidth}
      \centering
      \includegraphics[scale=.37]{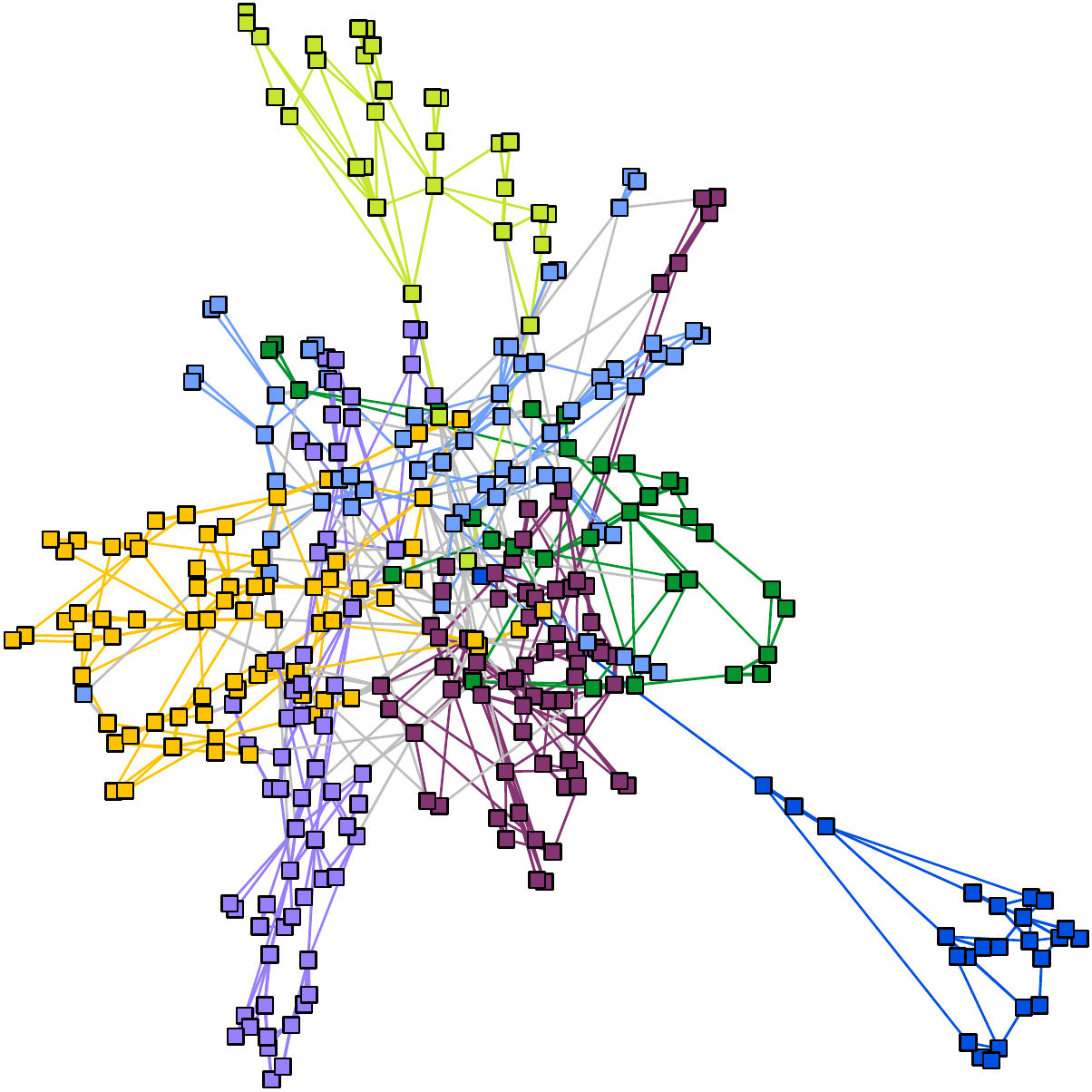}
      \caption{Learned Student-$t$ graph on strong heavy-tails. $Q = 0.66$.}
      \label{fig:ht-graph-t-strong-full}
  \end{subfigure}%
  \\
  \begin{subfigure}[t]{0.32\textwidth}
      \centering
      \includegraphics[scale=.37]{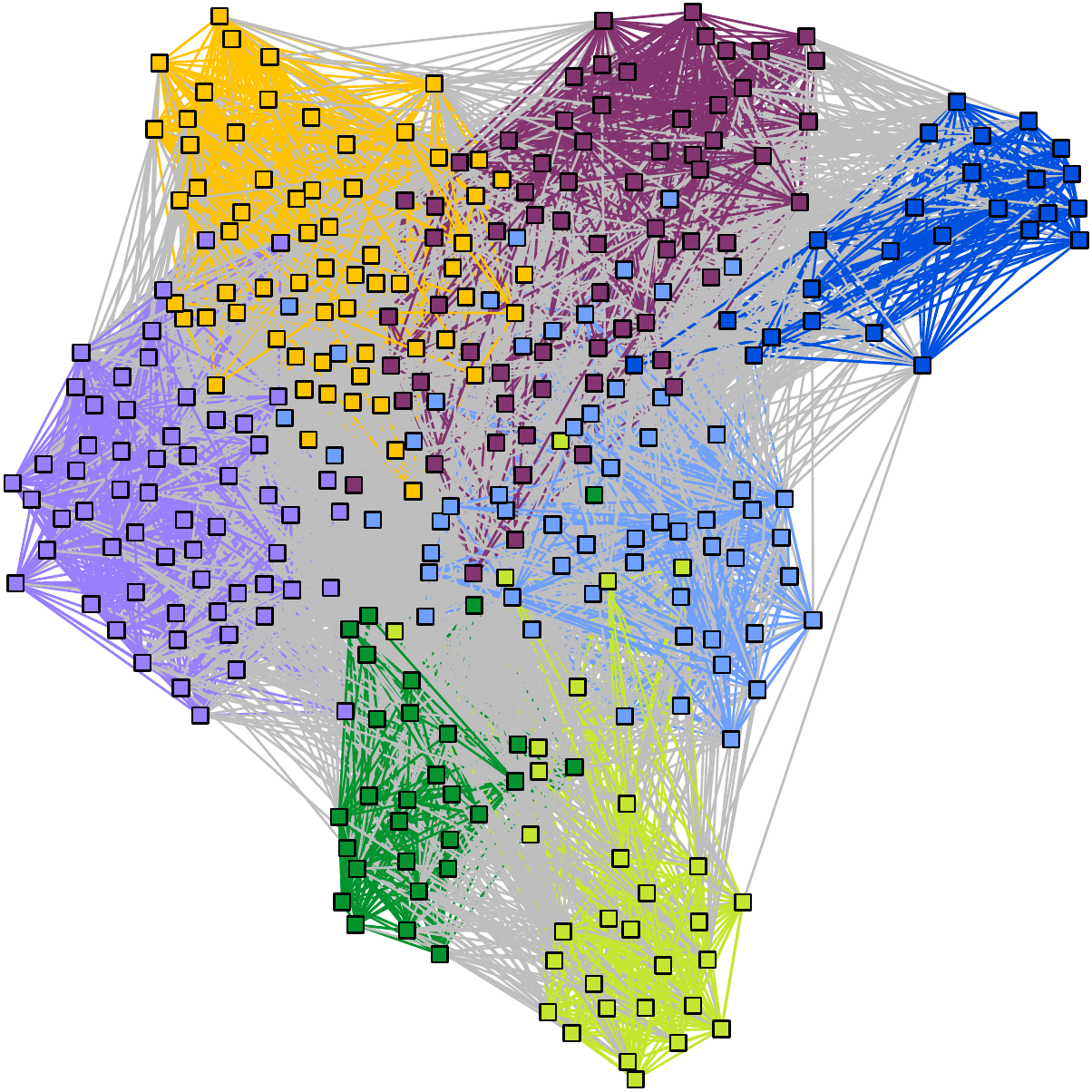}
      \caption{Learned Gaussian graph on moderate heavy-tails. $Q = 0.32$.}
      \label{fig:ht-graph-gauss-moderate-full}
  \end{subfigure}%
  ~
  \begin{subfigure}[t]{0.32\textwidth}
    \centering
    \includegraphics[scale=.37]{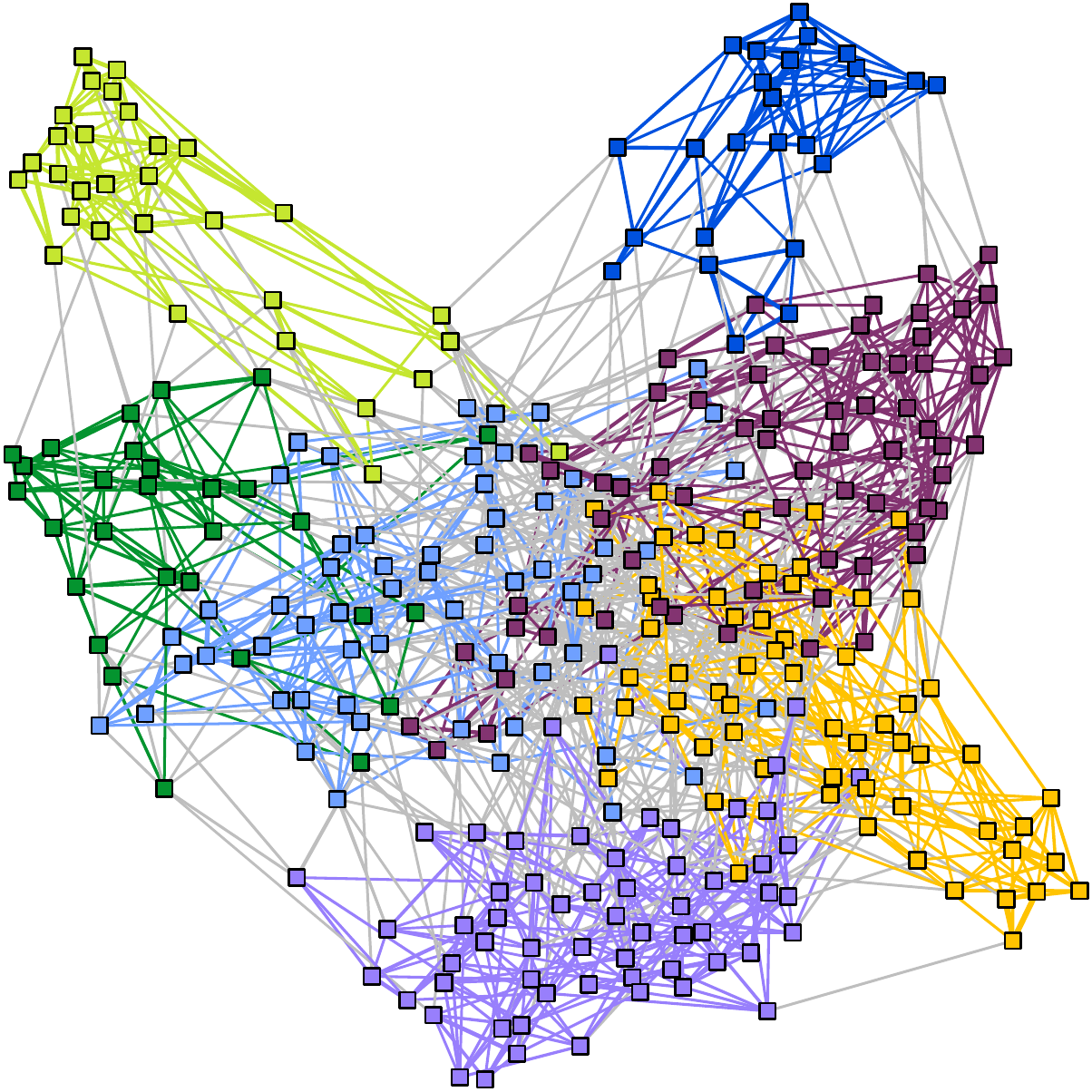}
    \caption{Learned Gaussian graph with sparsity on moderate heavy-tails. $Q = 0.57$.}
    \label{fig:ht-graph-mcp-moderate-full}
  \end{subfigure}%
  ~
  \begin{subfigure}[t]{0.32\textwidth}
      \centering
      \includegraphics[scale=.37]{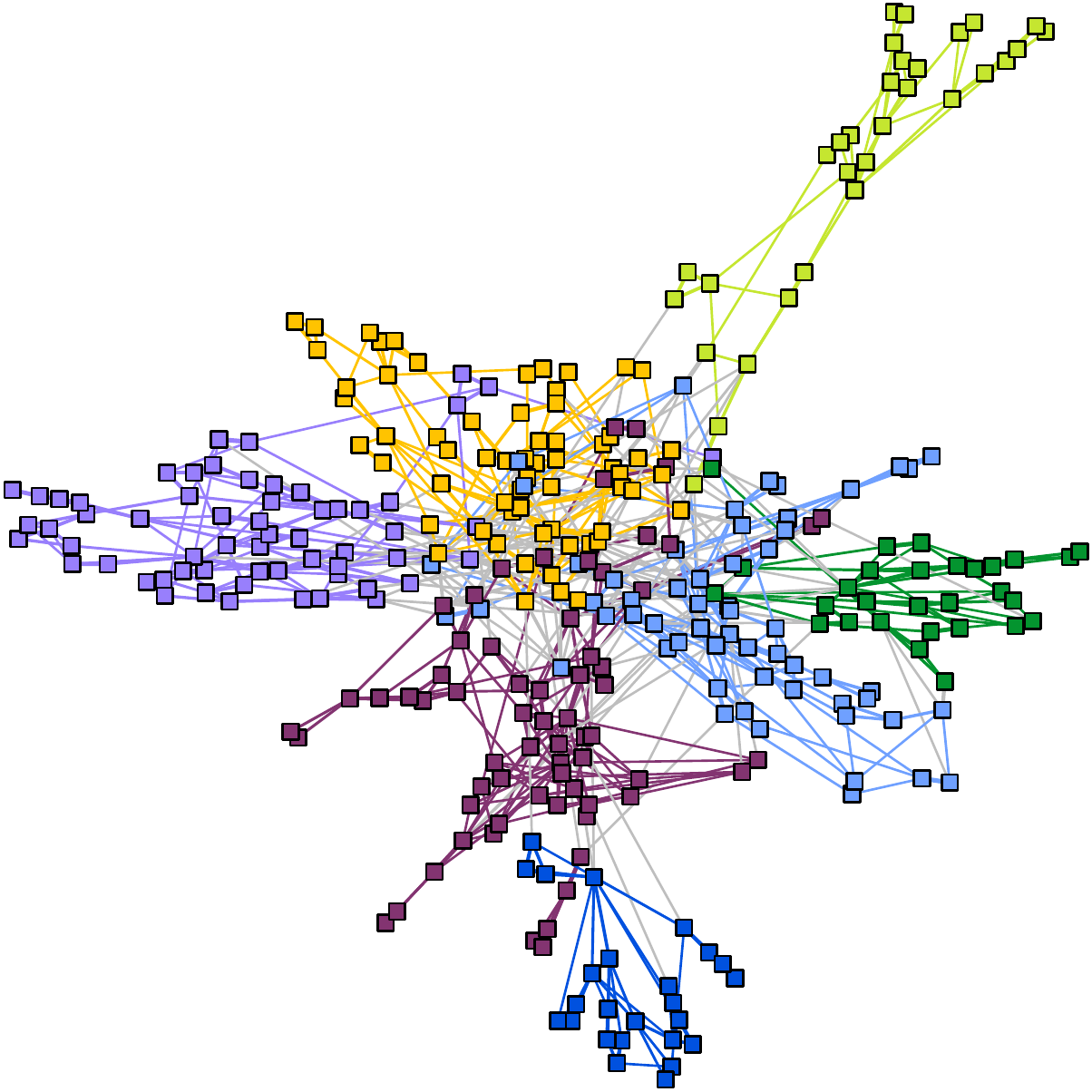}
      \caption{Learned Student-$t$ graph on moderate heavy-tails. $Q = 0.66$.}
      \label{fig:ht-graph-t-moderate-full}
  \end{subfigure}%
  ~
  \caption{Learned graph networks with
  Gaussian (Figures~\ref{fig:ht-graph-gauss-strong-full}~and~\ref{fig:ht-graph-gauss-moderate-full}),
  Gaussian with sparsity (Figures~\ref{fig:ht-graph-mcp-strong-full}~and~\ref{fig:ht-graph-mcp-moderate-full}),
  and Student-$t$ (Figures~\ref{fig:ht-graph-t-strong-full}~and~\ref{fig:ht-graph-t-moderate-full}),
  for contrasting heavy-tail scenarios.
  }
  \label{fig:all-sectors-heavy-tail}
\end{figure}

\begin{table}[!htb]
  \centering
  \caption{Edge distribution for the \textbf{strong} heavy-tails case.}
 \begin{tabular}{llll}
  \specialrule{.05em}{.05em}{.2em} 
  model & inter-sector edges & intra-sector edges & modularity $(Q)$ \\
  Gaussian & 4262 & 3947 & 0.31  \\
  Gaussian w/ sparsity & 405 & 986 & 0.54  \\
  Student-$t$ & 124 & 579 & \textbf{0.66}
 \end{tabular}
 \label{tab:edge-distribution-strong-full}
\end{table}

\begin{table}[!htb]
  \centering
  \caption{Edge distribution for the \textbf{moderate} heavy-tails case.}
 \begin{tabular}{llll}
  \specialrule{.05em}{.05em}{.2em} 
  model & inter-sector edges & intra-sector edges & modularity ($Q$) \\
  Gaussian & 4198 & 4063 & 0.32  \\
  Gaussian w/ sparsity & 375 & 1021 & 0.57  \\
  Student-$t$ & 131 & 627 & \textbf{0.66}
 \end{tabular}
 \label{tab:edge-distribution-moderate-full}
\end{table}

\subsection{Heavy-tails and \texorpdfstring{$k$}~-component graphs}

In order to verify the learning of heavy-tail and $k$-component graphs jointly, we estimate
a graphs of stocks using the datasets described in subsections~\ref{sec:degree-control}~\ref{sec:market-factor-effect}
via the $\mathsf{ktGL}$ algorithm.
Figure~\ref{fig:ktGL-1} depicts the learned graphs where we can observe sparse characteristics
that agree with the connected graphs estimated in the preceding section. In addition,
when compared to the Gaussian case (Figures~\ref{fig:dgl-1} and~\ref{fig:correlation-with-market}),
the graphs estimated in Figures~\ref{fig:ktGL-1} and~\ref{fig:ktGL-4-comp}
reveal a finer, possibly more accurate description of the actual underlying stock market scenario.

\begin{figure}[!htb]
  \begin{subfigure}[t]{0.48\textwidth}
    \centering
     \includegraphics[scale=.45]{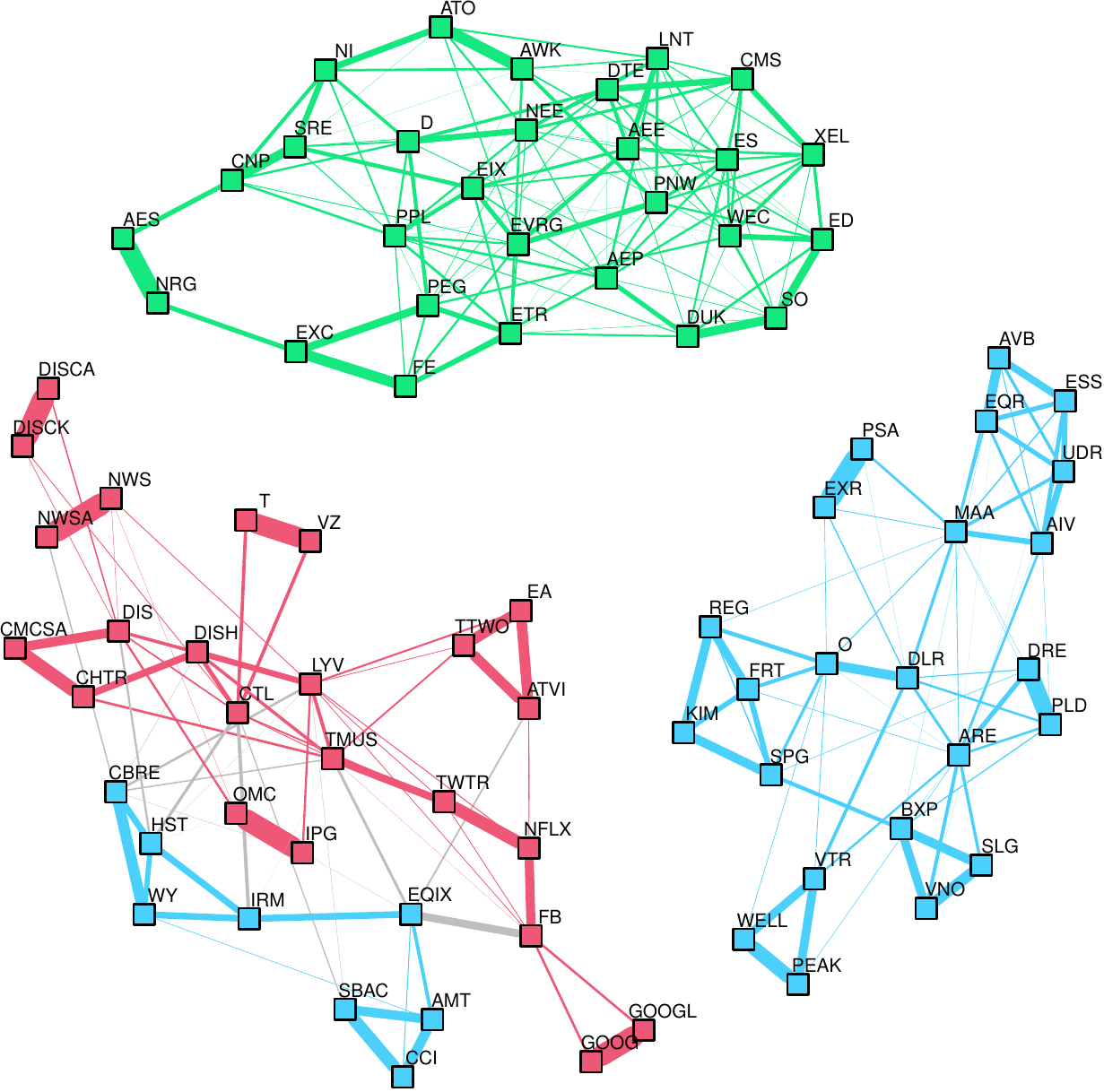}
    \caption{$3$-component graph learned via $\mathsf{ktGL}$.}
      \label{fig:ktGL-3-comp}
  \end{subfigure}%
  ~
  \begin{subfigure}[t]{0.48\textwidth}
      \centering
      \includegraphics[scale=.45]{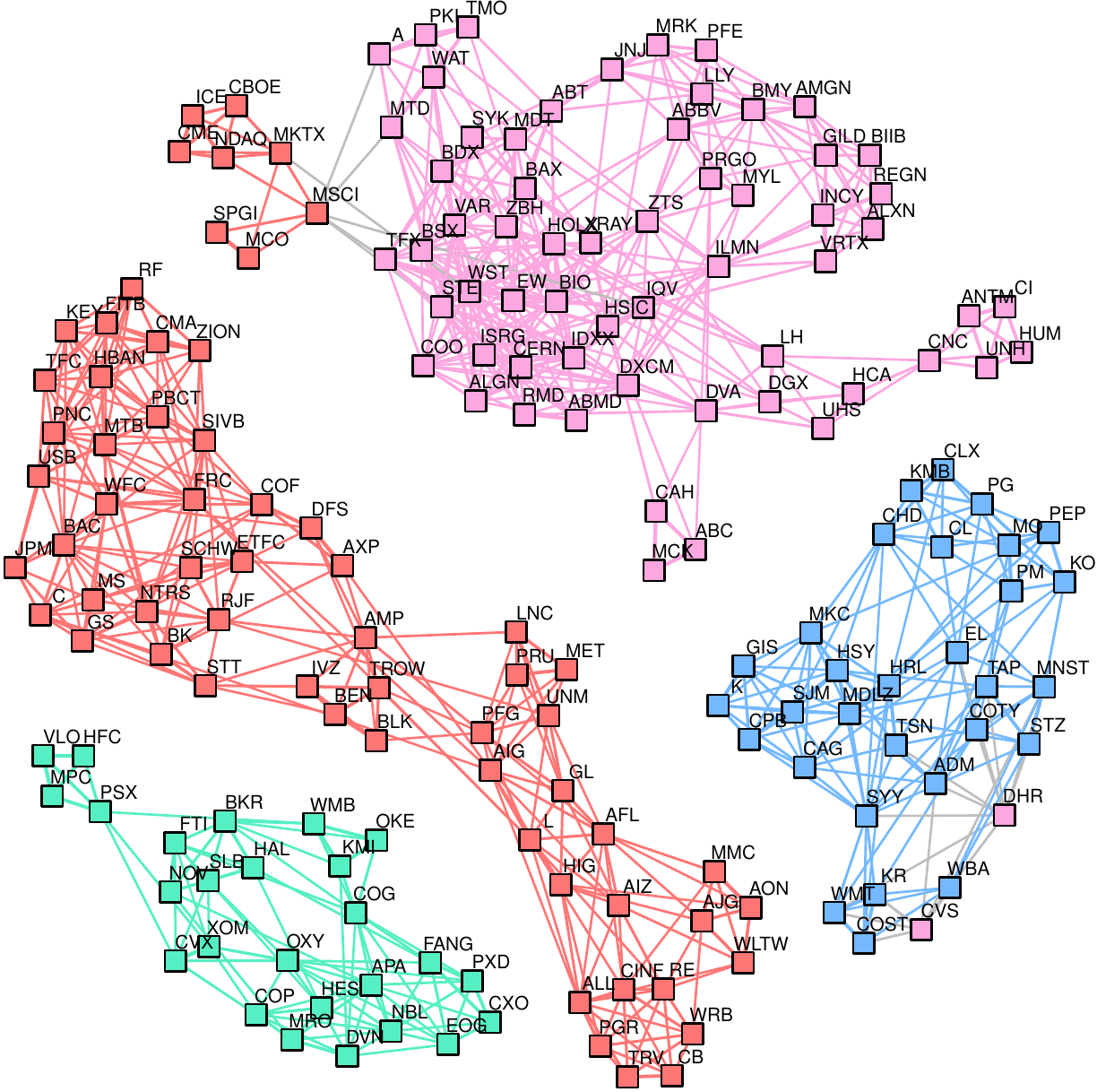}
      \caption{$4$-component graph learned via $\mathsf{ktGL}$.}
      \label{fig:ktGL-4-comp}
  \end{subfigure}%
  \caption{Learned graphs from different stock market scenarios taking into account both $k$-component requirements
  and heavy-tail assumptions.}
  \label{fig:ktGL-1}
\end{figure}

\subsection{Effect of crisis on the learned graphs: COVID-19 case of study}

In the experiments that follow, we focus on illustrating the impact of the COVID-19 economic crisis
on the learned graphs from the S\&P500 stock market and the foreign exchange market.
The COVID-19 pandemic affected the US stock market quite significantly especially throughout the month of March 2020.

\subsubsection{Stocks}

In this experiment, we investigate the effects of the financial crisis caused by the COVID-19 pandemic
on the learned graphs of stocks.

We start by selecting 97 stocks across all 11 sectors of the S\&P500
and computing their log-returns during two time frames:
(i) from Apr. 22nd 2019 to Dec. 31st 2019 and
(ii) from Jan. 2nd 2020 to Jul. 31st 2020.

Out of those stocks, nine of them showed growth in returns
over the period of 24 days starting from Feb. 18th 2020 to March 20th 2020.
Their symbols along with their monthly return during this period is
summarized in Table~\ref{tab:stocks-doing-well-during-covid}.
Figure~\ref{fig:linear-returns-stocks-covid} shows the log-returns of the selected stocks
over the considered time period. In Figure~\ref{fig:returns-covid}, the economic crisis
is noticeable from the increase in the spread of the log-returns, throughout the month of March 2020,
caused by the COVID-19 pandemic.

\begin{figure}[!htb]
  \captionsetup[subfigure]{justification=centering}
  \centering
  \begin{subfigure}[t]{0.98\textwidth}
      \centering
      \includegraphics[scale=.55]{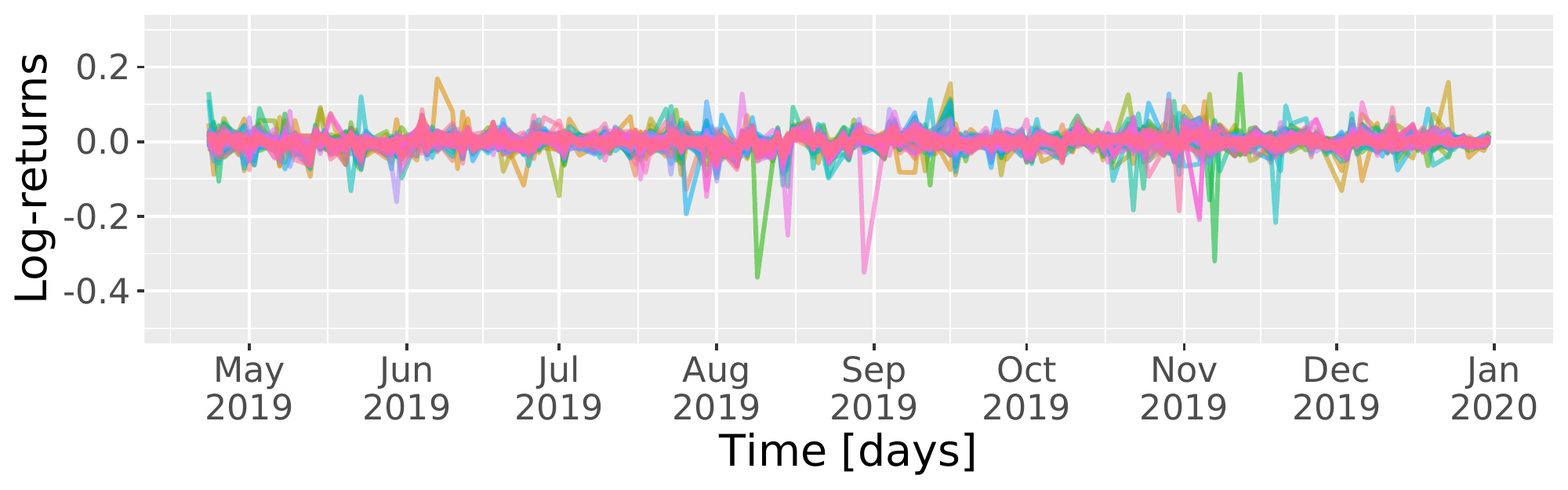}
      \caption{Log-returns of the selected stocks prior to the COVID-19 pandemic.}
      \label{fig:returns-prior-covid}
  \end{subfigure}%
  \\
  \vspace{.2in}
  \begin{subfigure}[t]{0.98\textwidth}
      \centering 
      \includegraphics[scale=.55]{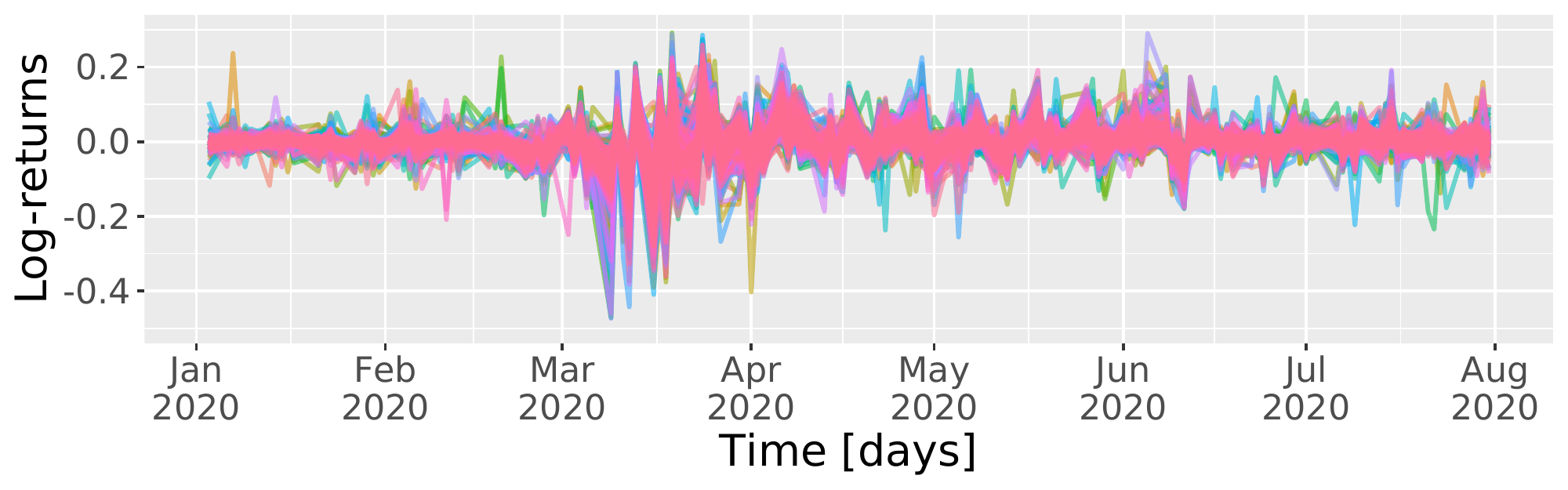}
      \caption{Log-returns of the selected stocks during the COVID-19 pandemic.}
      \label{fig:returns-covid}
  \end{subfigure}% 
  \caption{Log-returns of 97 selected stocks from April 22nd 2019 to July 31st 2020.
  The vertical axis is fixed on panels in order to better illustrate the change in volatility
  as a result of the COVID-19 pandemic.}
  \label{fig:linear-returns-stocks-covid}
\end{figure}

\begin{table}[!htb]
  \centering
  \caption{Stocks with positive monthly return from 2020-02-15 to 2020-03-20.}
 \begin{tabular}{lllc}
  \specialrule{.05em}{.05em}{.2em} 
   & Symbol & GICS Sector & monthly return \\
   & $\mathsf{CLX}$ & $\mathsf{Consumer~Staples}$ & 4.8\%  \\
   & $\mathsf{COG}$ & $\mathsf{Energy}$ & 2.5\%\\
   & $\mathsf{CTXS}$ & $\mathsf{Information~Technology}$ & 0.4\%\\
   & $\mathsf{DPZ}$ & $\mathsf{Consumer~Discretionary}$ & 3.9\% \\
   & $\mathsf{GILD}$ & $\mathsf{Health~Care}$ & 5.8\% \\
   & $\mathsf{GIS}$ & $\mathsf{Consumer~Staples}$ & 0.9\% \\
   & $\mathsf{KR}$ & $\mathsf{Consumer~Staples}$ & 6.9\% \\
   & $\mathsf{REGN}$ & $\mathsf{Health~Care}$ & 6.1\% \\
   & $\mathsf{ZM}$ & $\mathsf{-}$ & 19.4\%
 \end{tabular}
 \label{tab:stocks-doing-well-during-covid}
\end{table}

Figure~\ref{fig:covid-stocks} shows the learned networks using the proposed $\mathsf{tGL}$ algorithm~\ref{alg:heavy-tail}
with Student-$t$ assumption.
Figure~\ref{fig:graph-prior-covid} shows that the stocks with positive average linear return (blue squares)
during the COVID-19 pandemic are not particularly correlated on the period before the pandemic.
Figure~\ref{fig:graph-during-covid}, on the other hand, shows that the learned graph is able to correctly cluster
those stocks. In particular, it can be observed that the network of Figure~\ref{fig:graph-during-covid}
is objectively more modular than that of Figure~\ref{fig:graph-prior-covid}.
This experiment shows evidence that the graph models can be employed as a tool to identify events of interest
in the network.

\begin{figure}[!htb]
  \captionsetup[subfigure]{justification=centering}
  \centering
  \begin{subfigure}[t]{0.49\textwidth}
      \centering
      \includegraphics[scale=.45]{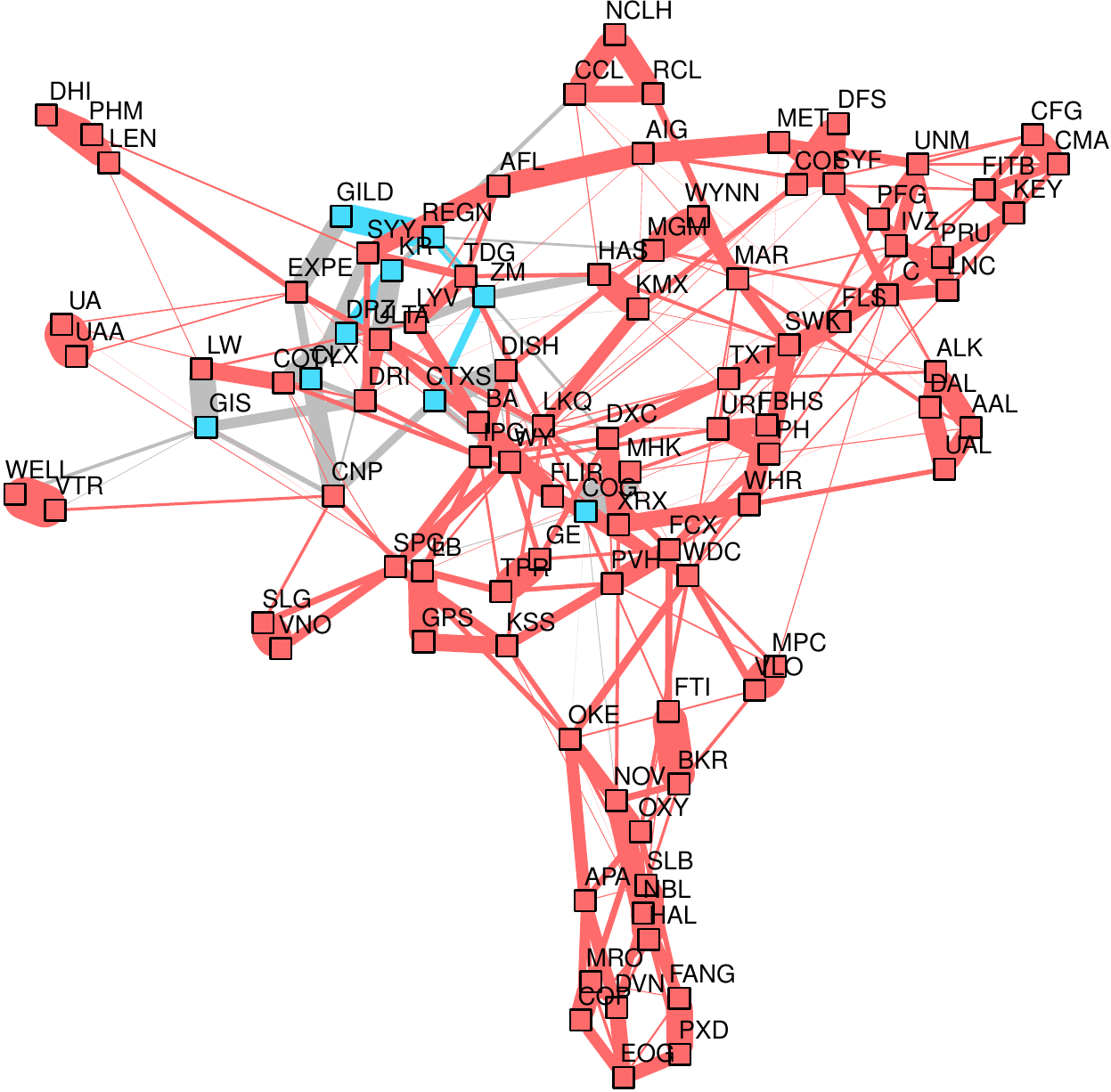}
      \caption{Learned Student-$t$ graph from data comprising the time window
               from Apr. 22nd 2019 to Dec. 31st 2019. $Q = 0.024$.}
      \label{fig:graph-prior-covid}
  \end{subfigure}%
  ~
  \begin{subfigure}[t]{0.49\textwidth}
      \centering
      \includegraphics[scale=.45]{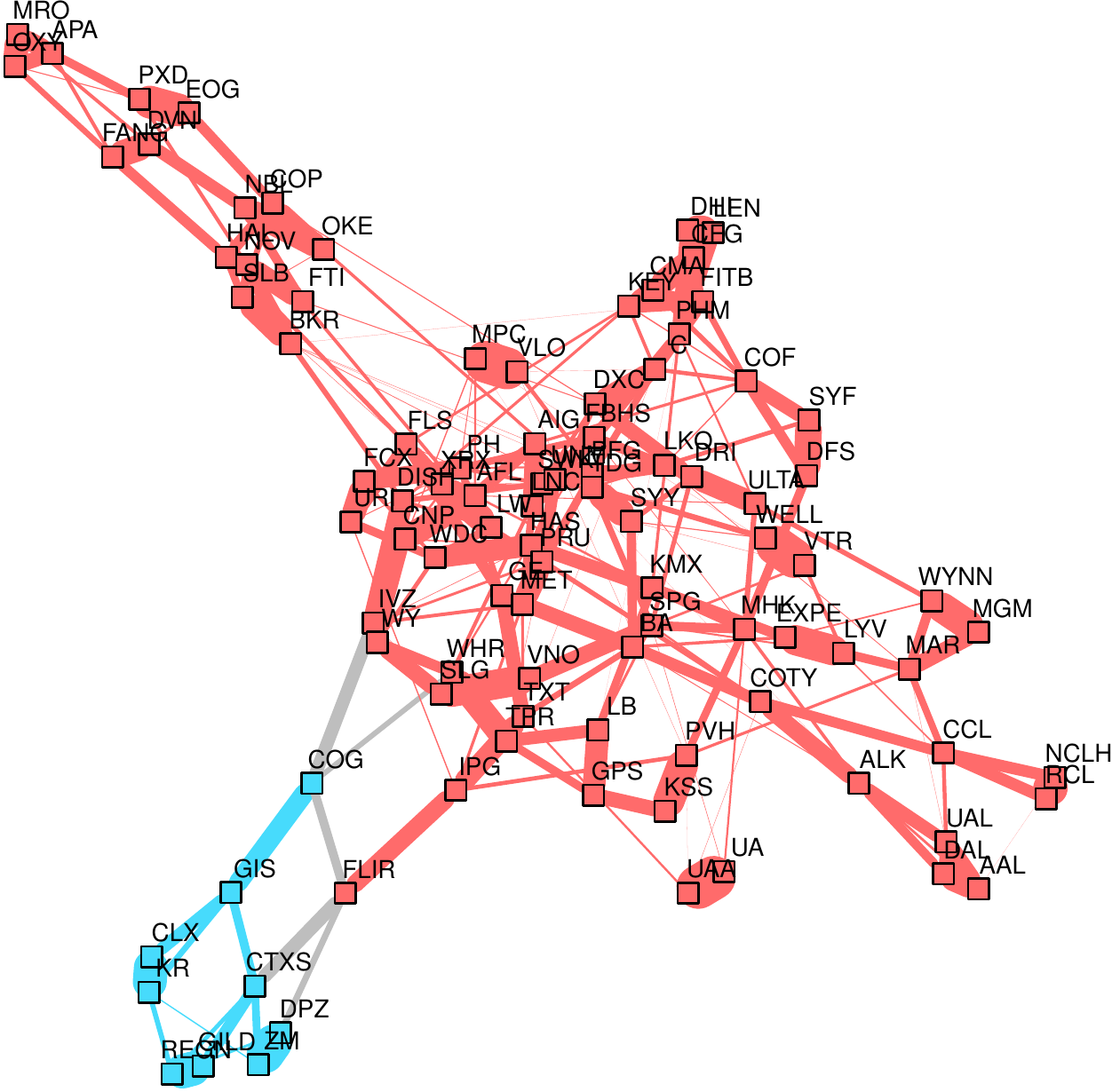}
      \caption{Learned Student-$t$ graph from stock data during the financial
               crisis caused by COVID-19 from Jan. 2nd 2020 to Jul. 31st 2020. $Q = 0.10$.}
      \label{fig:graph-during-covid}
  \end{subfigure}%
  \caption{Graphs of stocks learned with data prior and during the financial crisis caused by the COVID-19 pandemic in 2020.
  Figure~\ref{fig:graph-prior-covid} shows that before the COVID-19 pandemic the nodes in blue are somewhat independent among
  themselves. Figure~\ref{fig:graph-during-covid} shows that during the COVID-19 pandemic the stocks highlighted in blue
  are strongly connected and separated from the rest of the graph, as in fact they showed a positive monthly return during
  the severe economic period between Feb. 18th and Mar. 20th, 2020.}
  \label{fig:covid-stocks}
\end{figure}

In addition, we compare the proposed learned graphs in Figure~\ref{fig:covid-stocks} with graphs learned from algorithms that
employ the smooth signal approach. Figure~\ref{fig:smooth-signal-covid-stocks} shows the learned graphs from
$\mathsf{SSGL}$~\citep{kalofolias2016} and $\mathsf{GL}$-$\mathsf{SigRep}$~\citep{dong2016}. As we can observe from both
Figure~\ref{fig:kalo-covid} and~\ref{fig:dong-covid}, the learned networks do not present a meaningful graph representation in
this setting. While the clustering property of the stocks with positive monthly return is preserved
in the network learned with $\mathsf{GL}$-$\mathsf{SigRep}$, it does not capture the fine dependencies between pairs of stocks
like the ones shown by the proposed $\mathsf{kGL}$ algorithm in Figure~\ref{fig:covid-stocks}. In addition, tunning the hyperparameters in
the smooth signal algorithms is an involved task.

\begin{figure}[!htb]
  \captionsetup[subfigure]{justification=centering}
  \centering
  \begin{subfigure}[t]{0.49\textwidth}
      \centering
      \includegraphics[scale=.45]{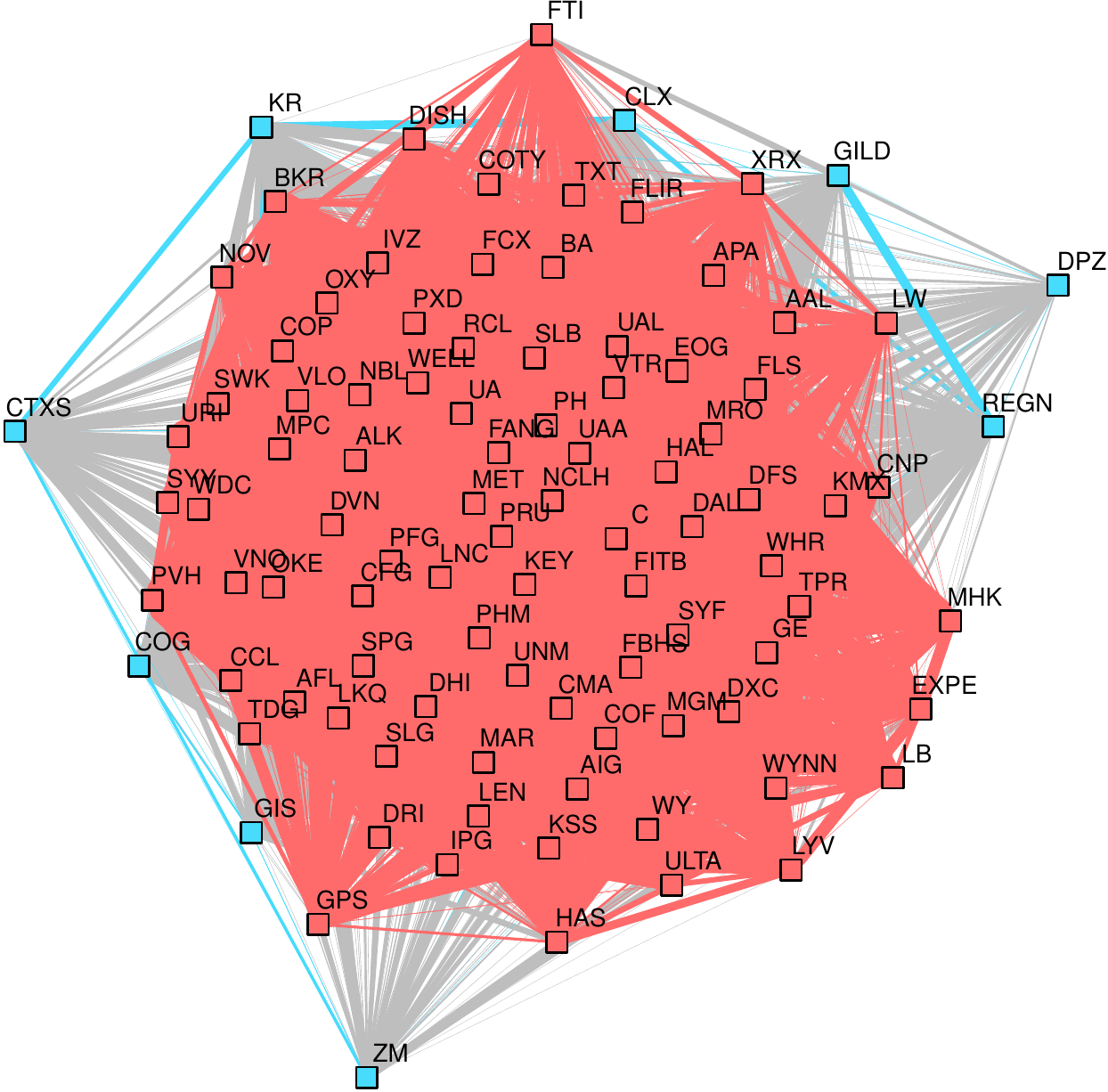}
      \caption{$\mathsf{SSGL}$ algorithm~\eqref{eq:smooth_static_graph}~\citep{kalofolias2016} with $\alpha = 10^{-2}$ and $\gamma = 10^{-4}$.}
      \label{fig:kalo-covid}
  \end{subfigure}%
  ~
  \begin{subfigure}[t]{0.49\textwidth}
      \centering
      \includegraphics[scale=.45]{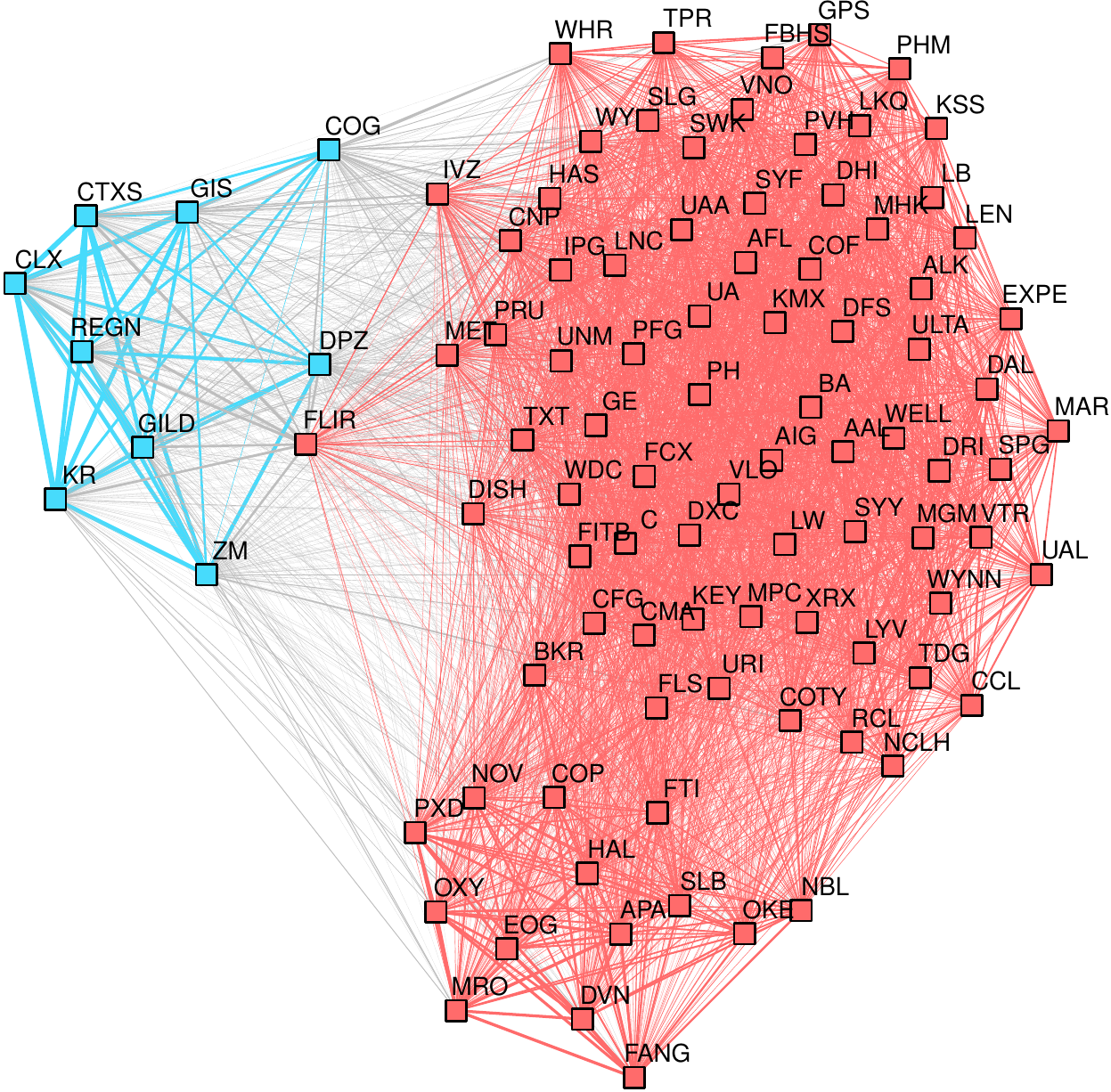}
      \caption{$\mathsf{GL}$-$\mathsf{SigRep}$ algorithm~\eqref{eq:signal-smoothness-dong}~\citep{dong2016} with $\alpha = 10^{-3}$, $\gamma = 0.5$.}
      \label{fig:dong-covid}
  \end{subfigure}%
  \caption{Learned graphs with existing smooth signal-based algorithms from stock data during the financial crisis caused by COVID-19.}
  \label{fig:smooth-signal-covid-stocks}
\end{figure}

\subsubsection{Foreign exchange}

We set up an experiment with data from the foreign exchange (FX) market, where, similarly to the previous experiment,
we would like to investigate whether the learned graph is able to identify currencies that became more valuable with respect to the
US dollar during the COVID-19 pandemic. To that end, we query FX data of the 34 most traded currencies as of 2019
in two time windows:
(i) from Feb. 1st 2019 to May 1st 2019 and (ii) from Feb. 3rd 2020 to May 1st 2020.

We then obtain the list of currencies for which the US dollar became less valuable during the period
from Feb. 15th to Apr. 15th 2020. Table~\ref{tab:currencies-return-usd} shows the list of such currencies
along with the annualized return of the ratio USD/CUR, where CUR is a given currency.

\begin{table}[!htb]
  \centering
  \caption{Currencies for which the US dollar had negative annualized return from 2020-02-15 to 2020-04-15.}
 \begin{tabular}{cllr}
  \specialrule{.05em}{.05em}{.2em} 
   & Symbol & Name & Annualized return \\
   & $\mathsf{EUR}$ & $\mathsf{Euro}$ & $-7.2$\%  \\
   & $\mathsf{JPY}$ & $\mathsf{Japanese~Yen}$ & $-13.4$\% \\
   & $\mathsf{CHF}$ & $\mathsf{Swiss~Franc}$ & $-10.8\%$ \\
   & $\mathsf{HKD}$ & $\mathsf{Hong~Kong~Dollar}$ & $-1.1$\% \\
   & $\mathsf{DKK}$ & $\mathsf{Danish~Krone}$ & $-8.0$\% \\
 \end{tabular}
 \label{tab:currencies-return-usd}
\end{table}

We then use the proposed heavy-tail graph learning framework ($\mathsf{tGL}$ Algorithm~\ref{alg:heavy-tail})
to learn the graph networks of foreign exchange data for the two aforementioned time frames.
Figure~\ref{fig:covid-fx-data} depicts the learned graphs, where in Figure~\ref{fig:fx-covid}
clearly shows that the currencies that had an increase in value during the pandemic
are clustered together (red edges), except for the HKD, which was the currency that presented
the smallest increase. On the other hand, prior to the pandemic, the FX market behaved in a
somewhat random fashion, which is seen by the smaller value in graph modularity.

\begin{figure}[!htb]
  \captionsetup[subfigure]{justification=centering}
  \centering
  \begin{subfigure}[t]{0.49\textwidth}
      \centering
      \includegraphics[scale=.5]{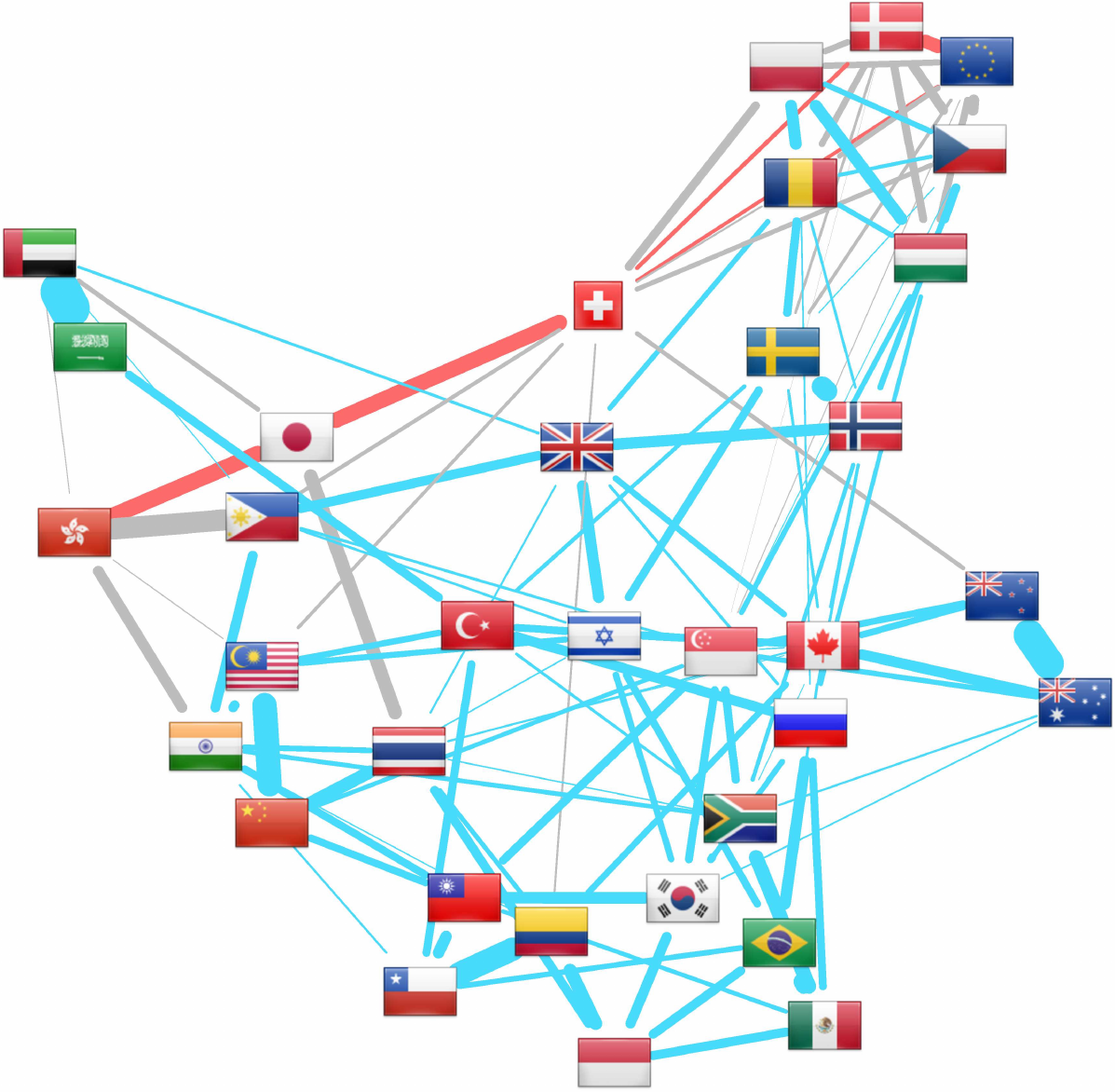}
      \caption{Learned Student-$t$ graph of FX data one-year prior COVID-19 financial crisis. $Q = 0.039$.}
      \label{fig:fx-before-covid}
  \end{subfigure}%
  ~
  \begin{subfigure}[t]{0.49\textwidth}
      \centering
      \includegraphics[scale=.5]{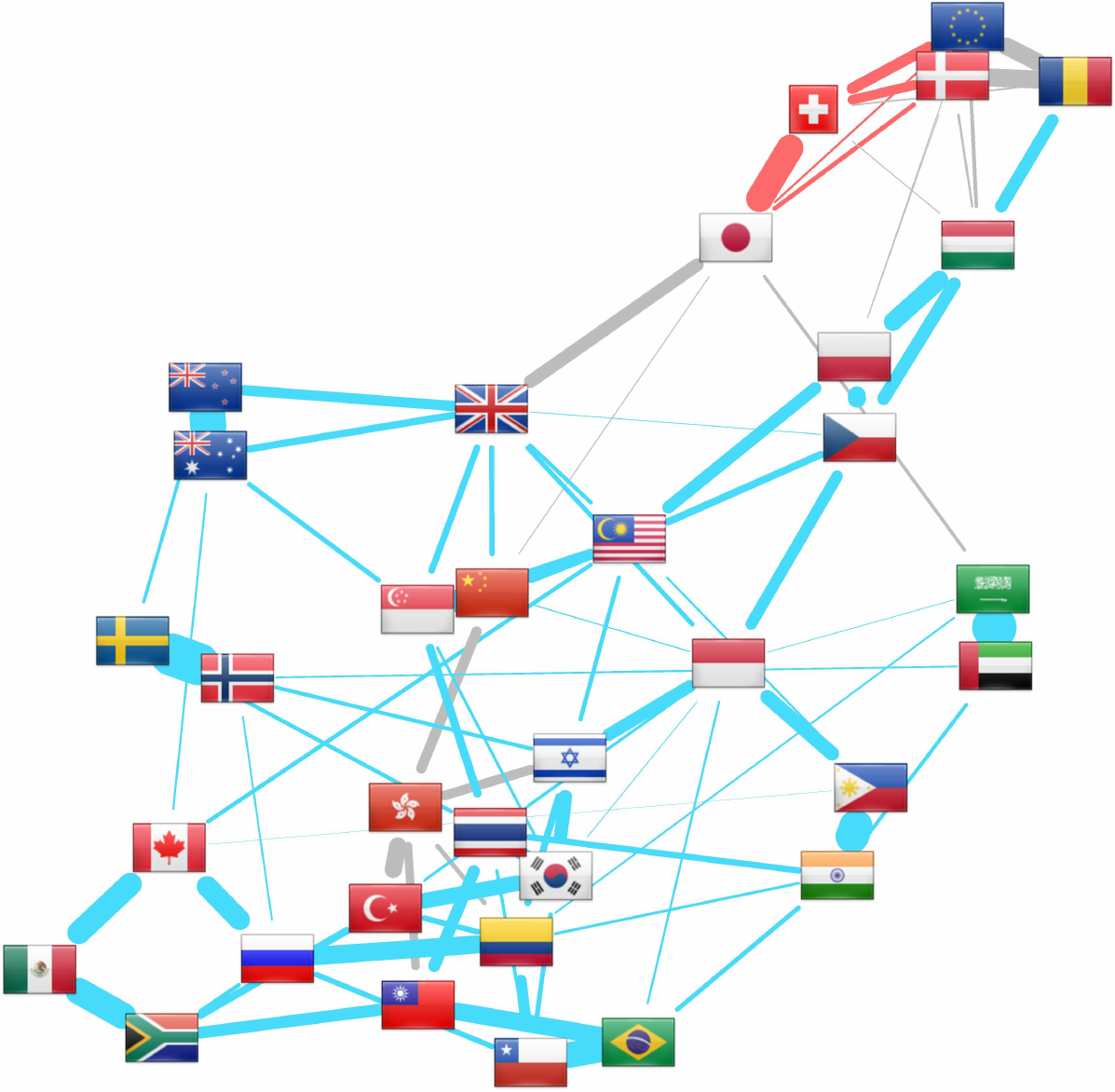}
      \caption{Learned Student-$t$ graph of FX data during COVID-19 financial crisis. $Q = 0.085$.}
      \label{fig:fx-covid}
  \end{subfigure}%
  \caption{Learned graphs from FX data.}
  \label{fig:covid-fx-data}
\end{figure}

% Acknowledgements should go at the end, before appendices and references

\section{Conclusions}

This paper has presented novel interpretations for Laplacian constraints of graphs from the perspective of financial data.
Those interpretations serve as guidelines for users when it comes to apply graph learning algorithms to estimate networks
of financial instruments such as stocks and currencies. We have also proposed novel algorithms based on the ADMM and MM frameworks
for learning graphs from data. Those algorithms fill major gaps in the literature, especially on what concerns learning
heavy-tailed and $k$-component graphs. In particular, the heavy-tail graph learning framework is paramount for financial data,
which exceptionally outperforms conventional state-of-the-art algorithms derived on the assumption that the input data is Gaussian.
Another feature of heavy-tail graphs is that they are naturally sparse. State-of-the-art sparse graph learning frameworks,
while useful in many contexts beyond finance, are cumbersome to tune due to their hyperparameters.
We, on the other hand, advocate that sparsity provided from heavy-tail distributions is a more principled
way to obtain interpretable graph representations. In the case of $k$-component graphs, we proposed a principled, versatile framework
that avoids isolated nodes via a simple linear constraint on the degree of the nodes. This extension allows, for instance, the
estimation of particular types of graphs such as regular graphs. Moreover, the proposed graph algorithms have shown significant
potential to capture nuances in the data caused by, \textit{e.g.}, a financial crisis event. Finally, it is worth noting that
the methods developed in this paper may be applicable to scenarios beyond financial markets, in particular,
we envision benefits for practical applications where the data distributions significantly departs from that of Gaussian.

\acks{The numerical algorithms proposed in this work were implemented in the
$\mathsf{R}$ language and made use of softwares such as $\mathsf{CVXR}$~\citep{fu2020},
$\mathsf{Rcpp}$~\citep{eddelbuettel2011},
$\mathsf{RcppEigen}$~\citep{bates2013},
$\mathsf{RcppArmadillo}$~\citep{eddelbuettel2014},
and $\mathsf{igraph}$~\citep{csardi2019}.
This work was supported by the Hong Kong GRF 16207019 research grant.}


\newpage

\appendix

\section{Empirical Convergence}

In this supplementary section, we illustrate the empirical convergence performance of the proposed algorithms for
the experimental settings considered.
All the experiments were carried out in a MacBook Pro 13in. 2019 with Intel Core i7 2.8GHz,
16GB of RAM.

The quantities $\bm{r}^l$ and $\bm{s}^l$, which are defined as
  $\bm r^l  = \bT^{l} - \sL\ww^l$,
  $\bm s^l  = \sD\ww^{l} - \bm d$,
are the primal residuals and
  $\bm v^l = \rho\sL^*\left(\bT^{l} - \bT^{l-1}\right)$
is the dual residual.

From Figures~\ref{fig:convergence-warm-up-heavy-tail}--\ref{fig:convergence-forex},
we can observe that the norm of the residuals quantities quickly approach zero after
a transient phase typical of ADMM-like algorithms.

\begin{figure}[!htb]
  \captionsetup[subfigure]{justification=centering}
  \centering
  \begin{subfigure}[t]{0.25\textwidth}
      \centering
      \includegraphics[scale=.3]{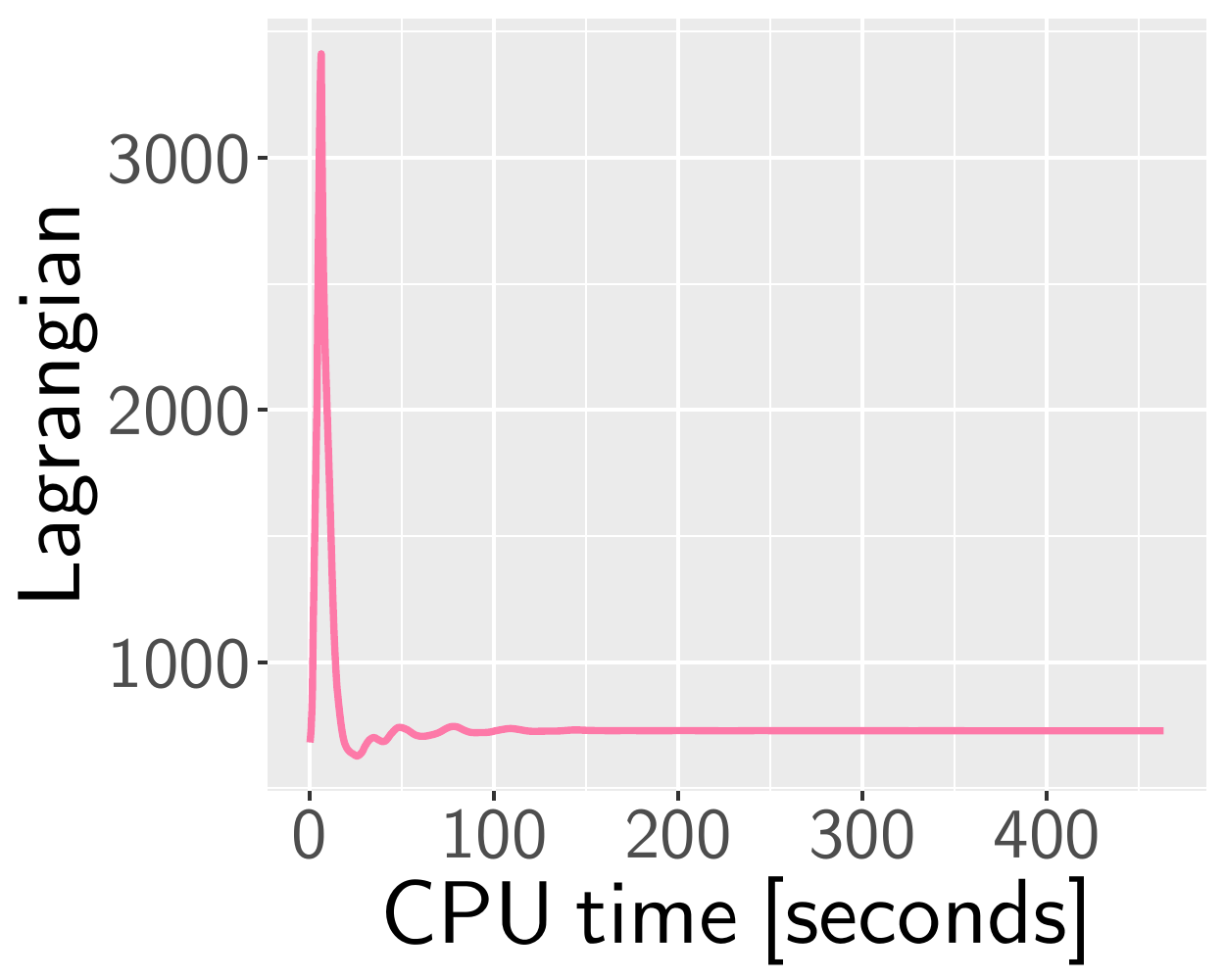}
  \end{subfigure}%
  \begin{subfigure}[t]{0.25\textwidth}
      \centering
      \includegraphics[scale=.3]{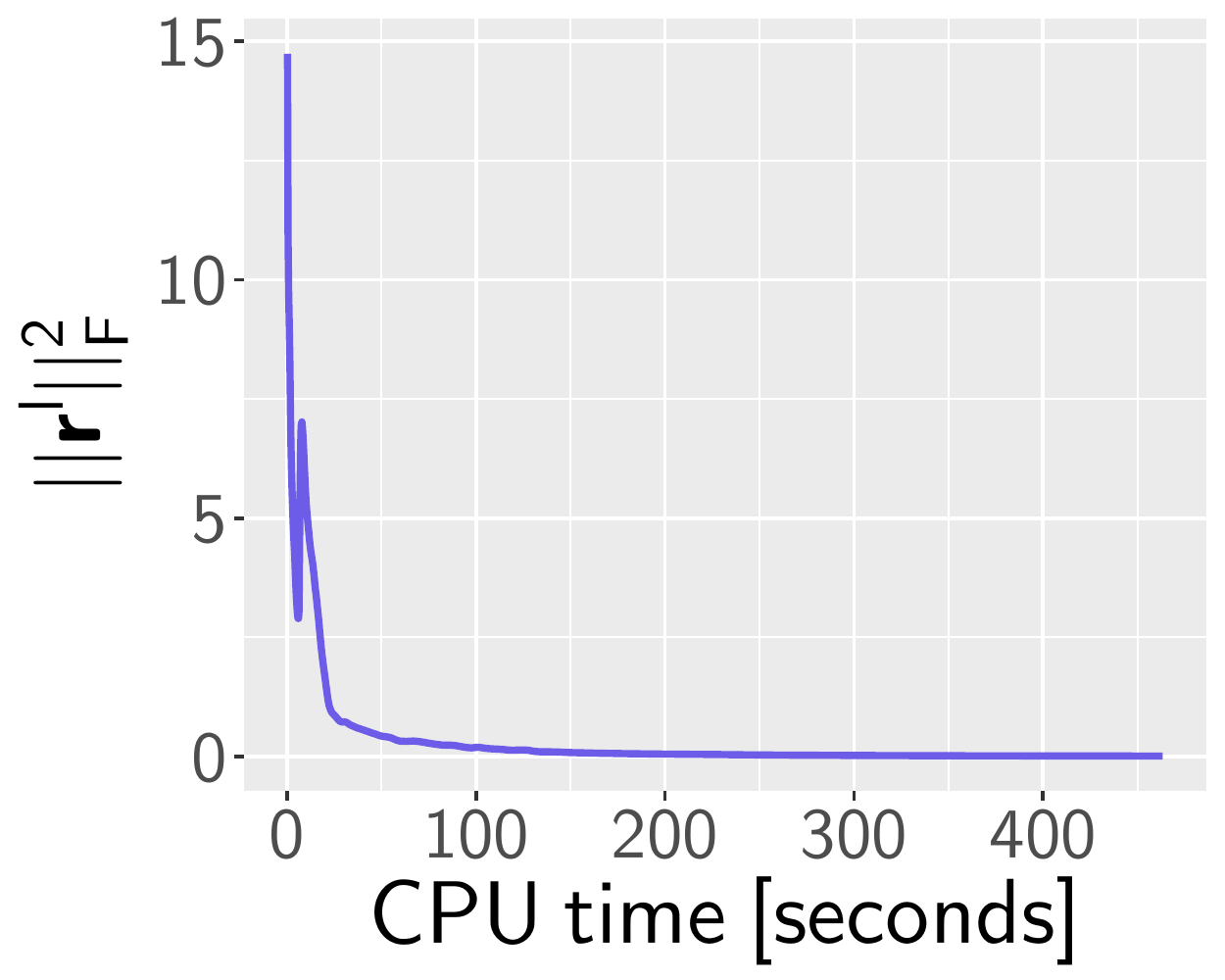}
  \end{subfigure}%
  \begin{subfigure}[t]{0.25\textwidth}
    \centering
    \includegraphics[scale=.3]{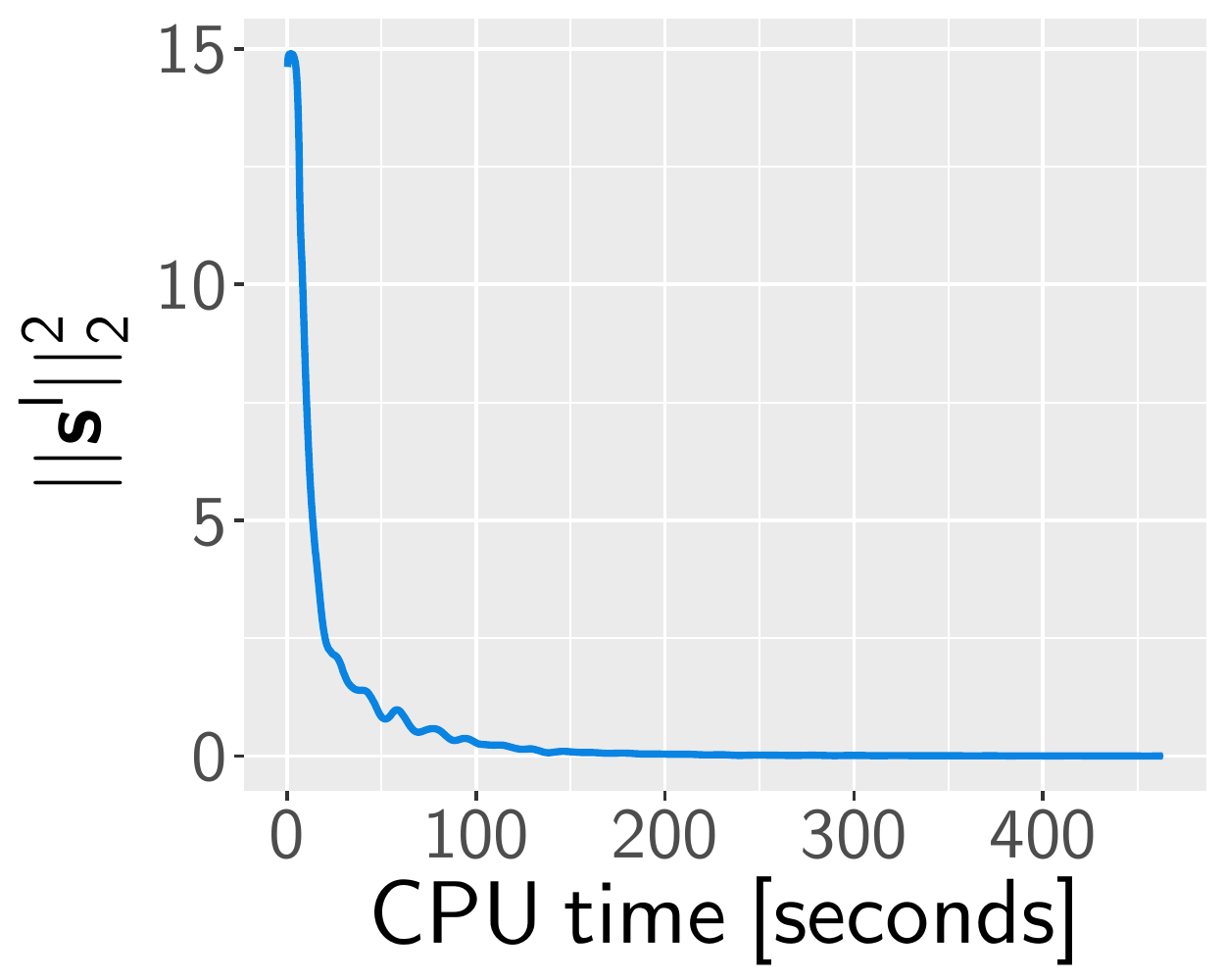}
  \end{subfigure}%
  \begin{subfigure}[t]{0.25\textwidth}
    \centering
    \includegraphics[scale=.3]{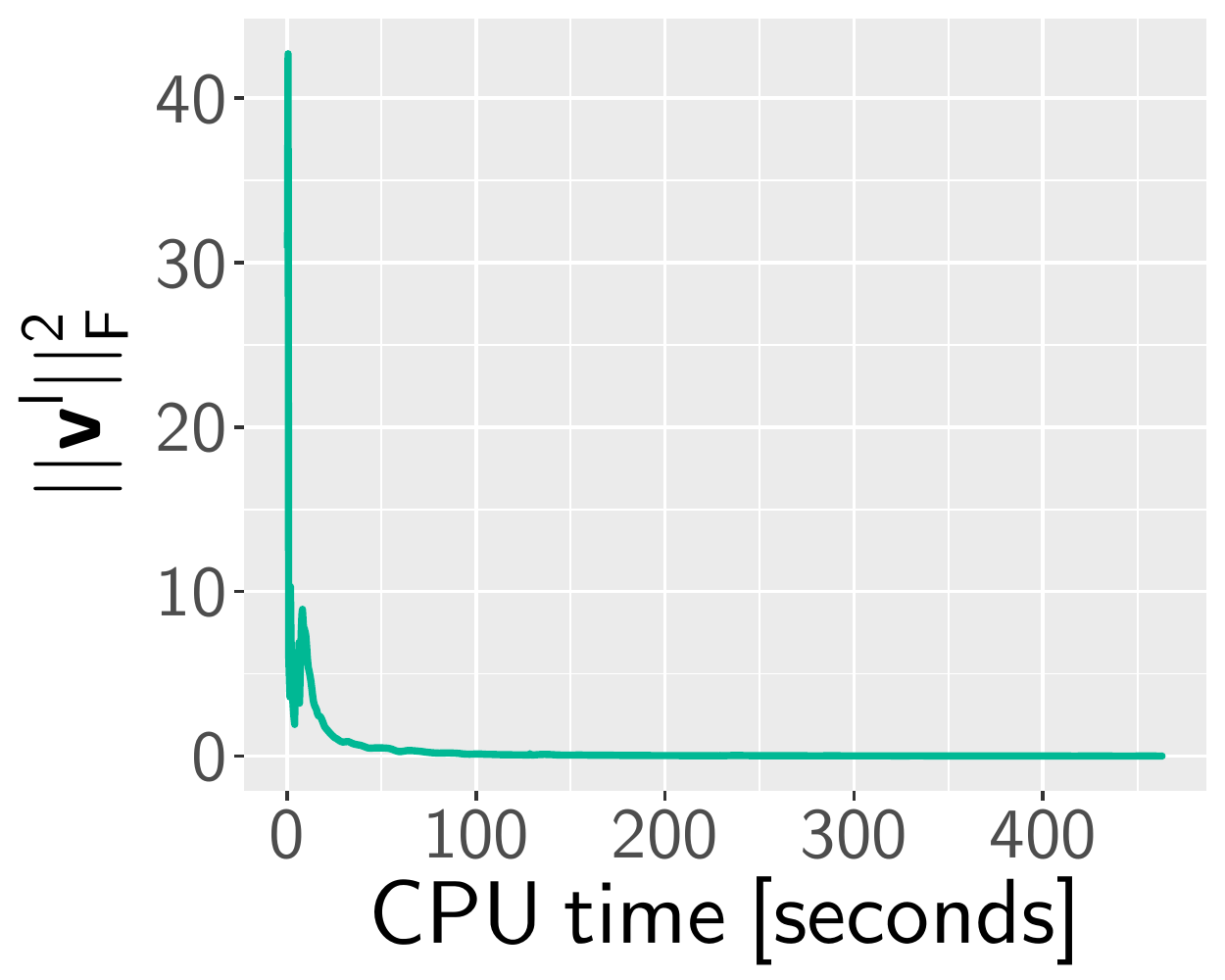}
\end{subfigure}%
  \caption{Empirical convergence for ``warm-up'' heavy-tail experiment data with Student-$t$ model.}
  \label{fig:convergence-warm-up-heavy-tail}
\end{figure}

\begin{figure}[!htb]
  \captionsetup[subfigure]{justification=centering}
  \centering
  \begin{subfigure}[t]{0.25\textwidth}
      \centering
      \includegraphics[scale=.3]{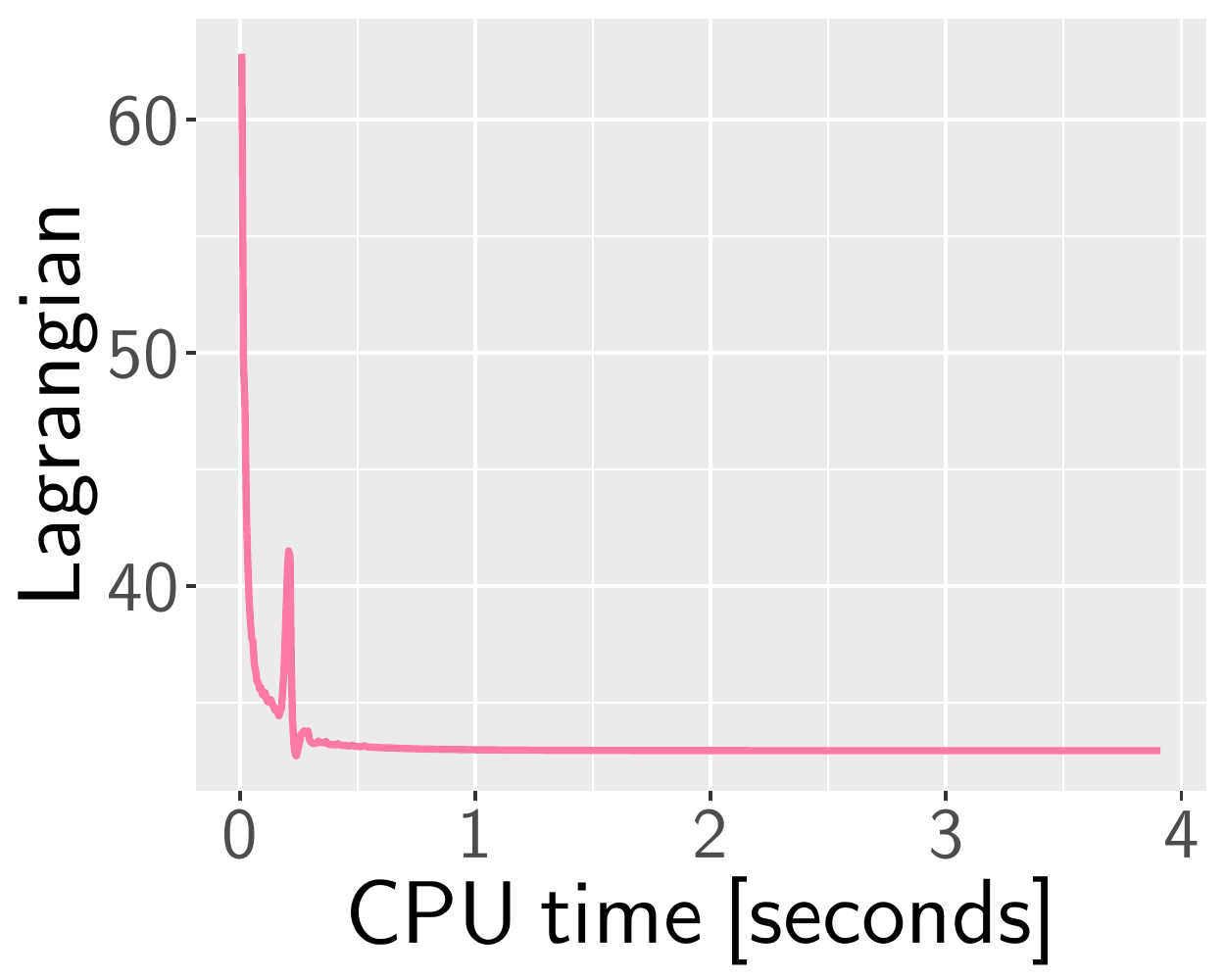}
  \end{subfigure}%
  \begin{subfigure}[t]{0.25\textwidth}
      \centering
      \includegraphics[scale=.3]{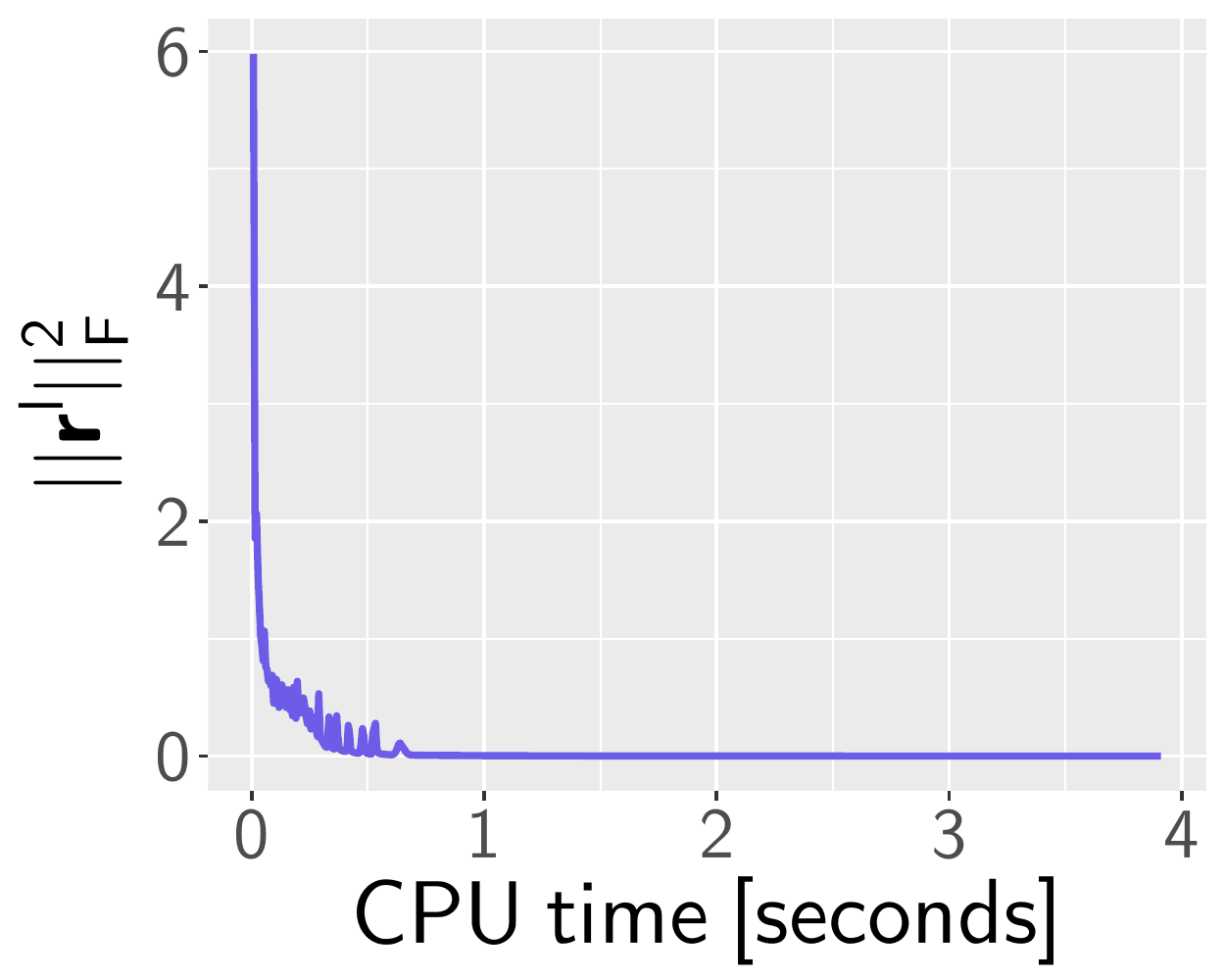}
  \end{subfigure}%
  \begin{subfigure}[t]{0.25\textwidth}
    \centering
    \includegraphics[scale=.3]{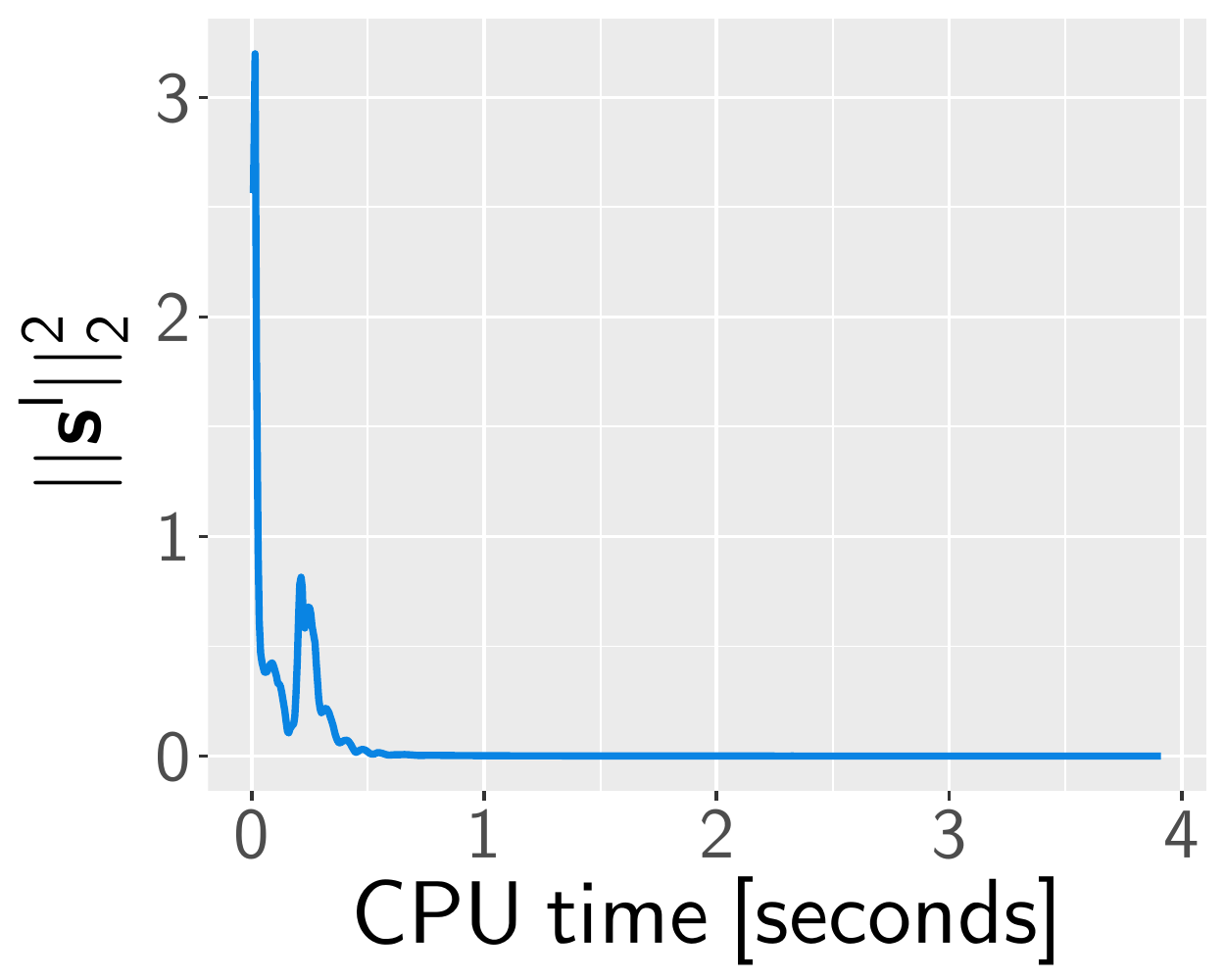}
  \end{subfigure}%
  \begin{subfigure}[t]{0.25\textwidth}
    \centering
    \includegraphics[scale=.3]{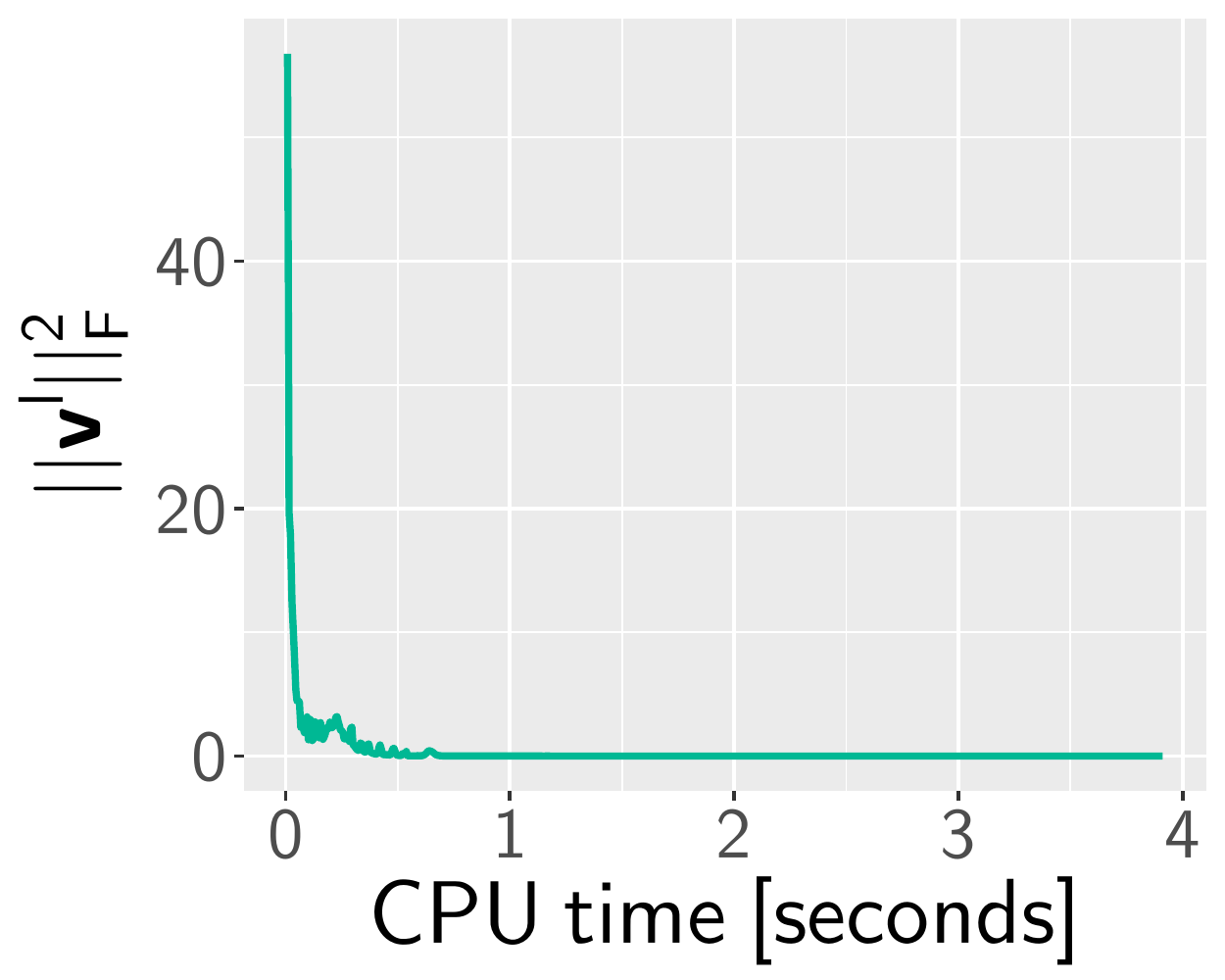}
\end{subfigure}%
  \caption{Empirical convergence for stock clustering with three sectors.}
  \label{fig:convergence-stock-clustering-three-sectors}
\end{figure}

\begin{figure}[!htb]
  \captionsetup[subfigure]{justification=centering}
  \centering
  \begin{subfigure}[t]{0.25\textwidth}
      \centering
      \includegraphics[scale=.3]{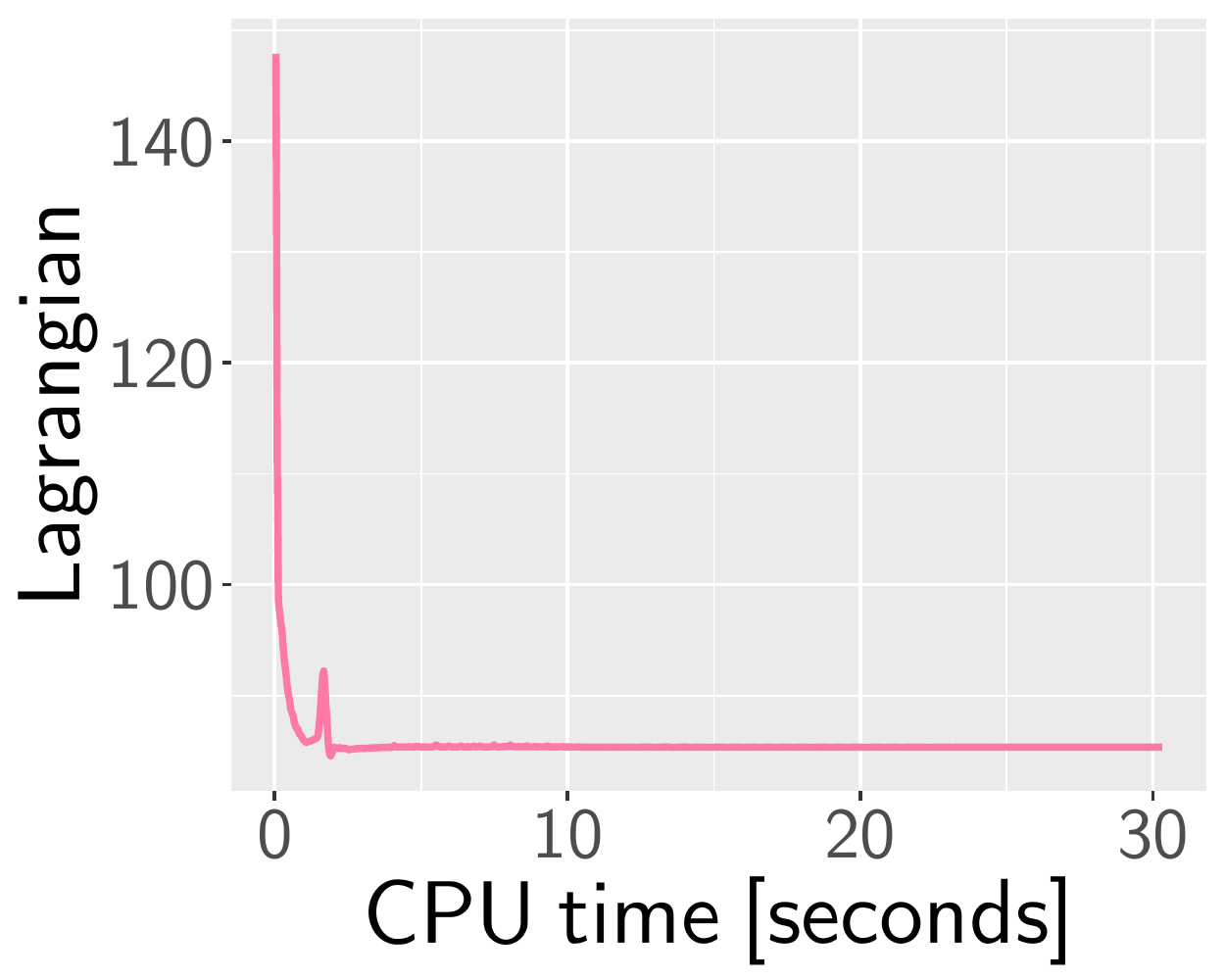}
  \end{subfigure}%
  \begin{subfigure}[t]{0.25\textwidth}
      \centering
      \includegraphics[scale=.3]{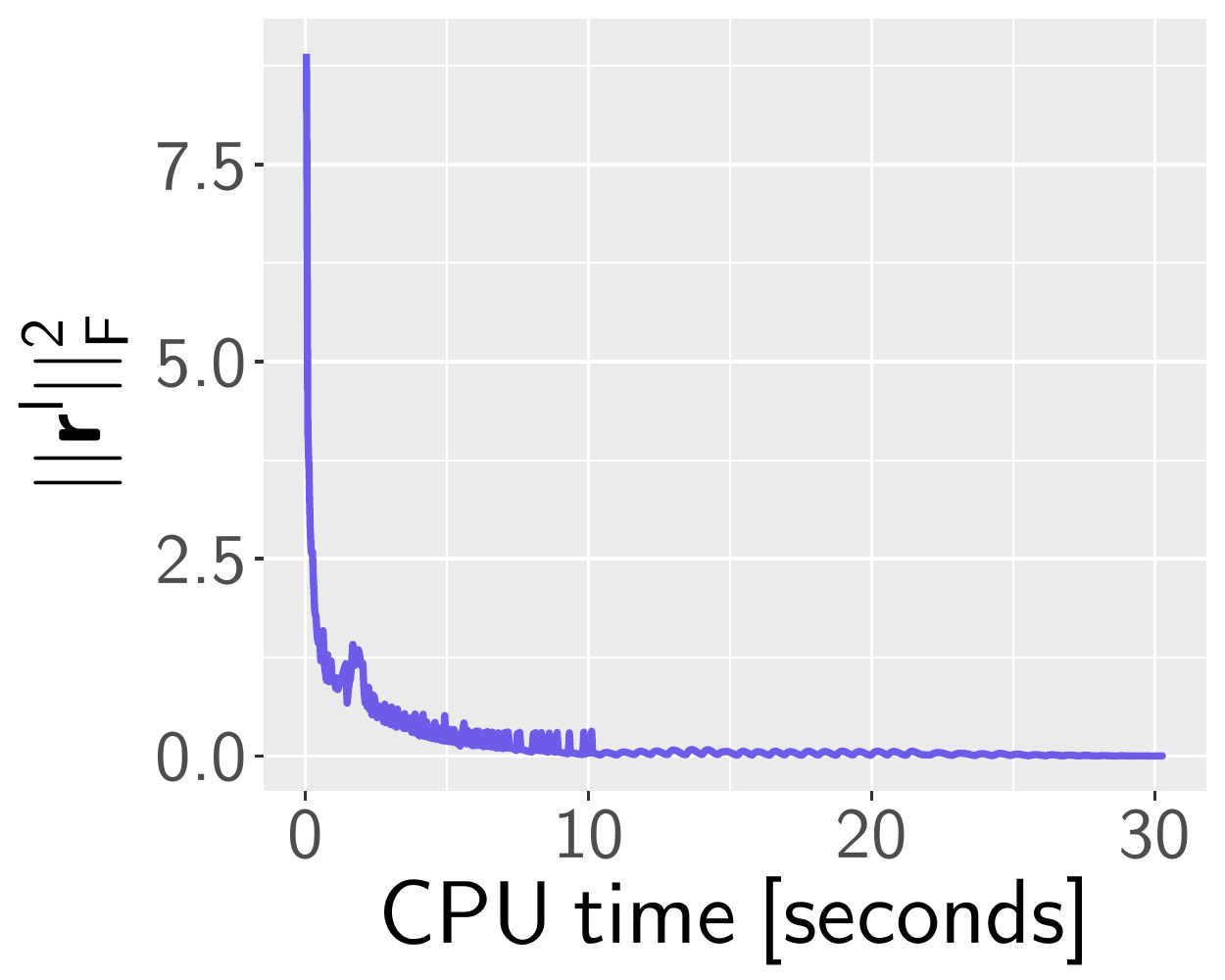}
  \end{subfigure}%
  \begin{subfigure}[t]{0.25\textwidth}
    \centering
    \includegraphics[scale=.3]{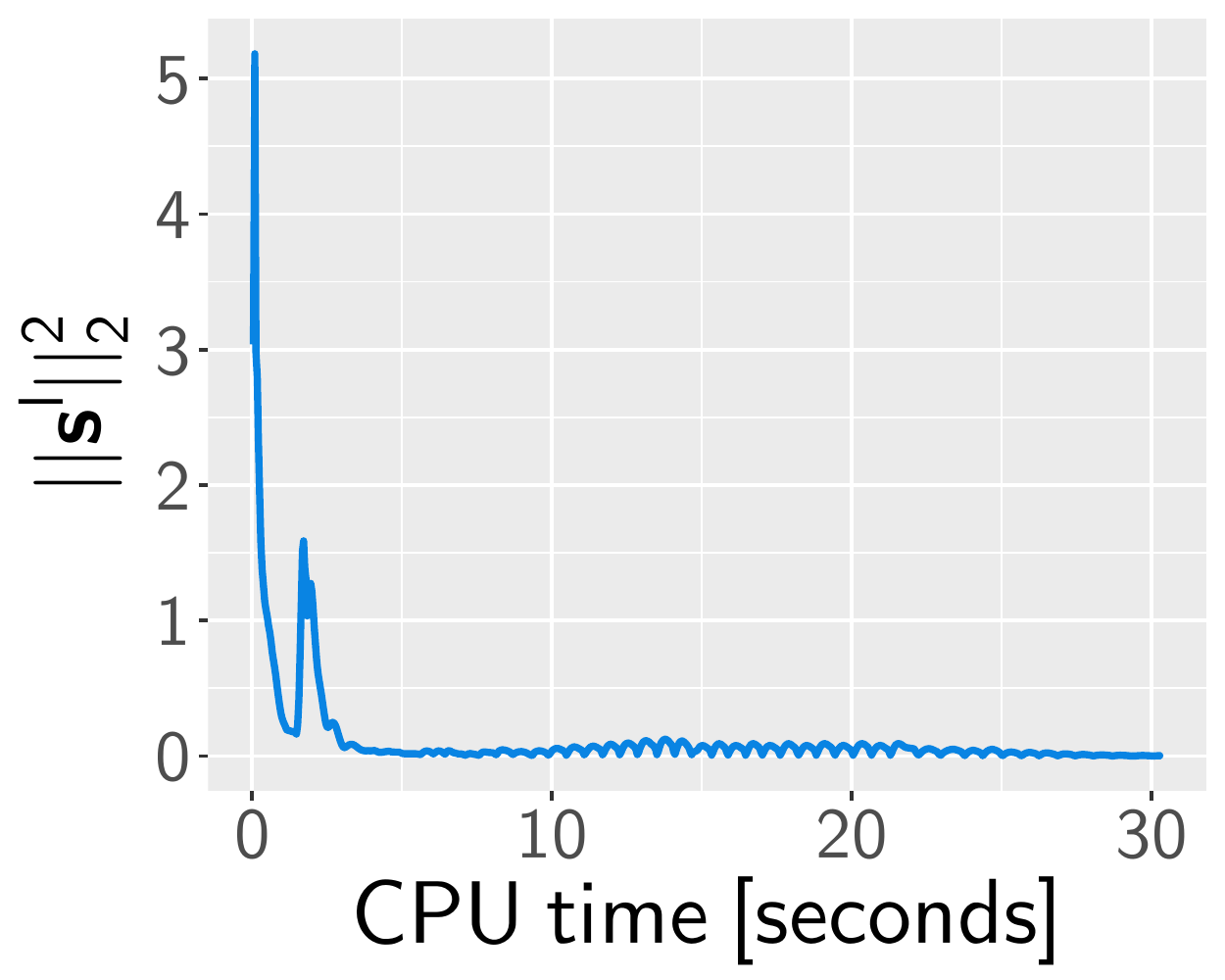}
  \end{subfigure}%
  \begin{subfigure}[t]{0.25\textwidth}
    \centering
    \includegraphics[scale=.3]{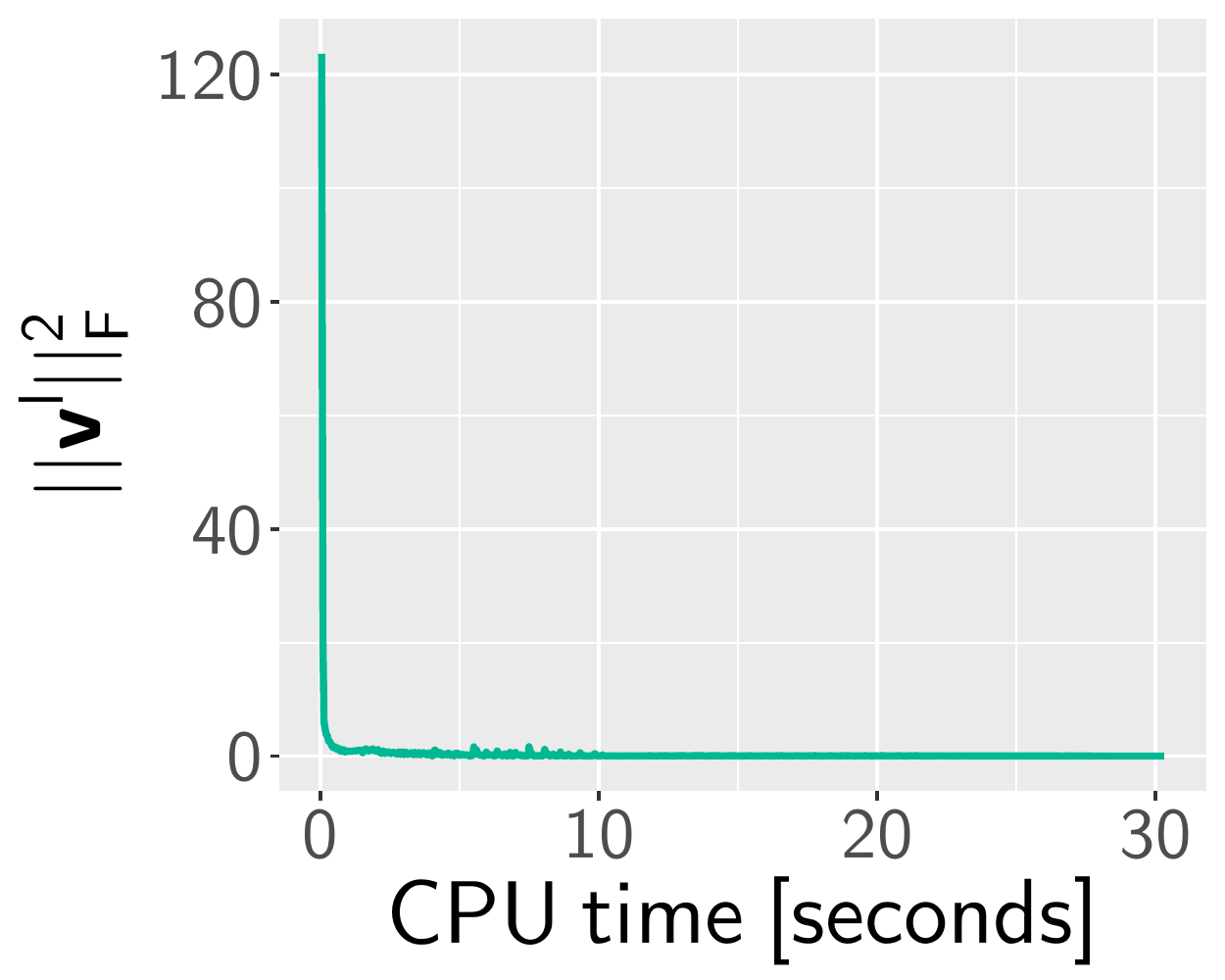}
\end{subfigure}%
  \caption{Empirical convergence for stock clustering with four sectors.}
  \label{fig:convergence-stock-clustering-four-sectors}
\end{figure}

\begin{figure}[!htb]
  \captionsetup[subfigure]{justification=centering}
  \centering
  \begin{subfigure}[t]{0.25\textwidth}
      \centering
      \includegraphics[scale=.3]{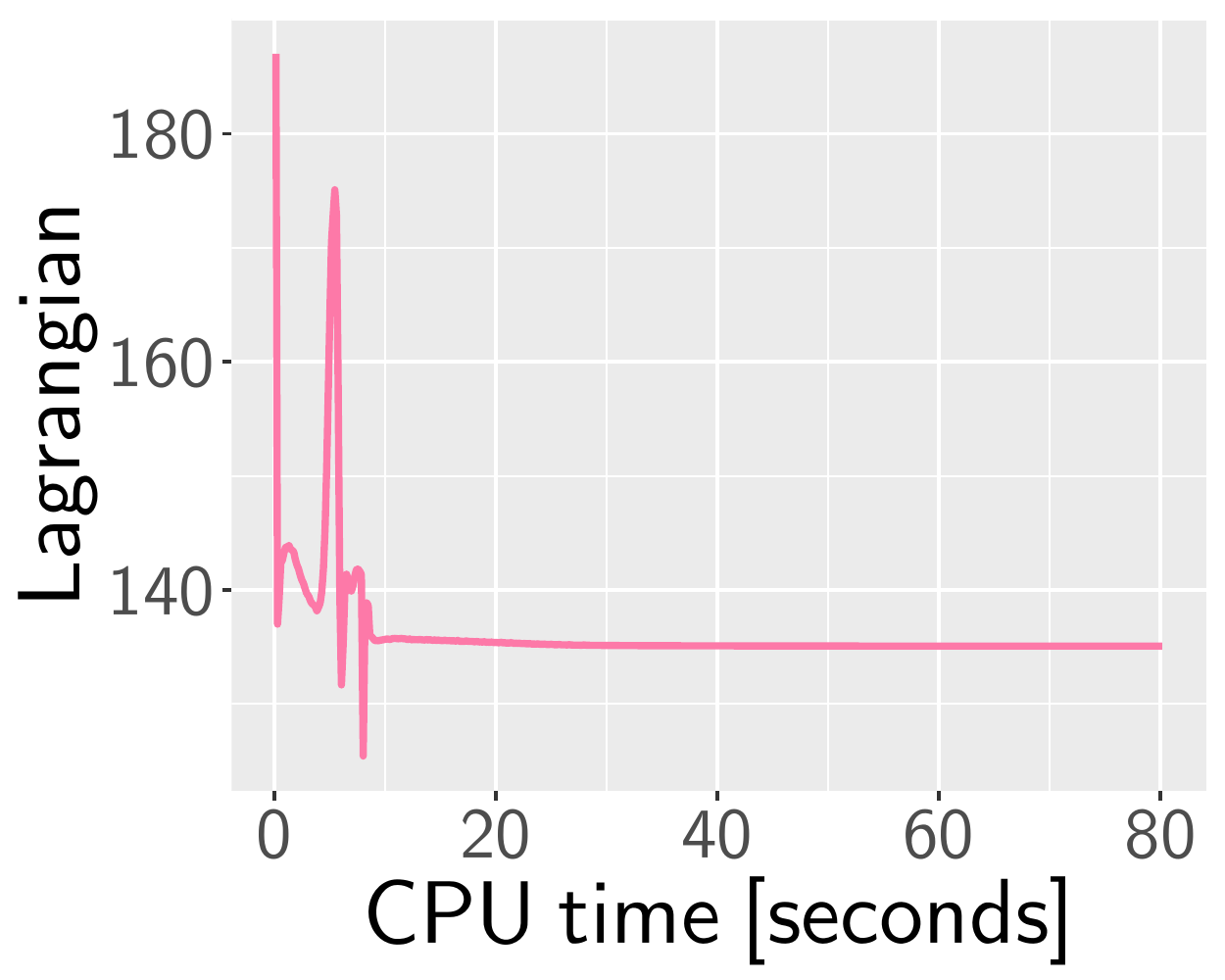}
  \end{subfigure}%
  \begin{subfigure}[t]{0.25\textwidth}
      \centering
      \includegraphics[scale=.3]{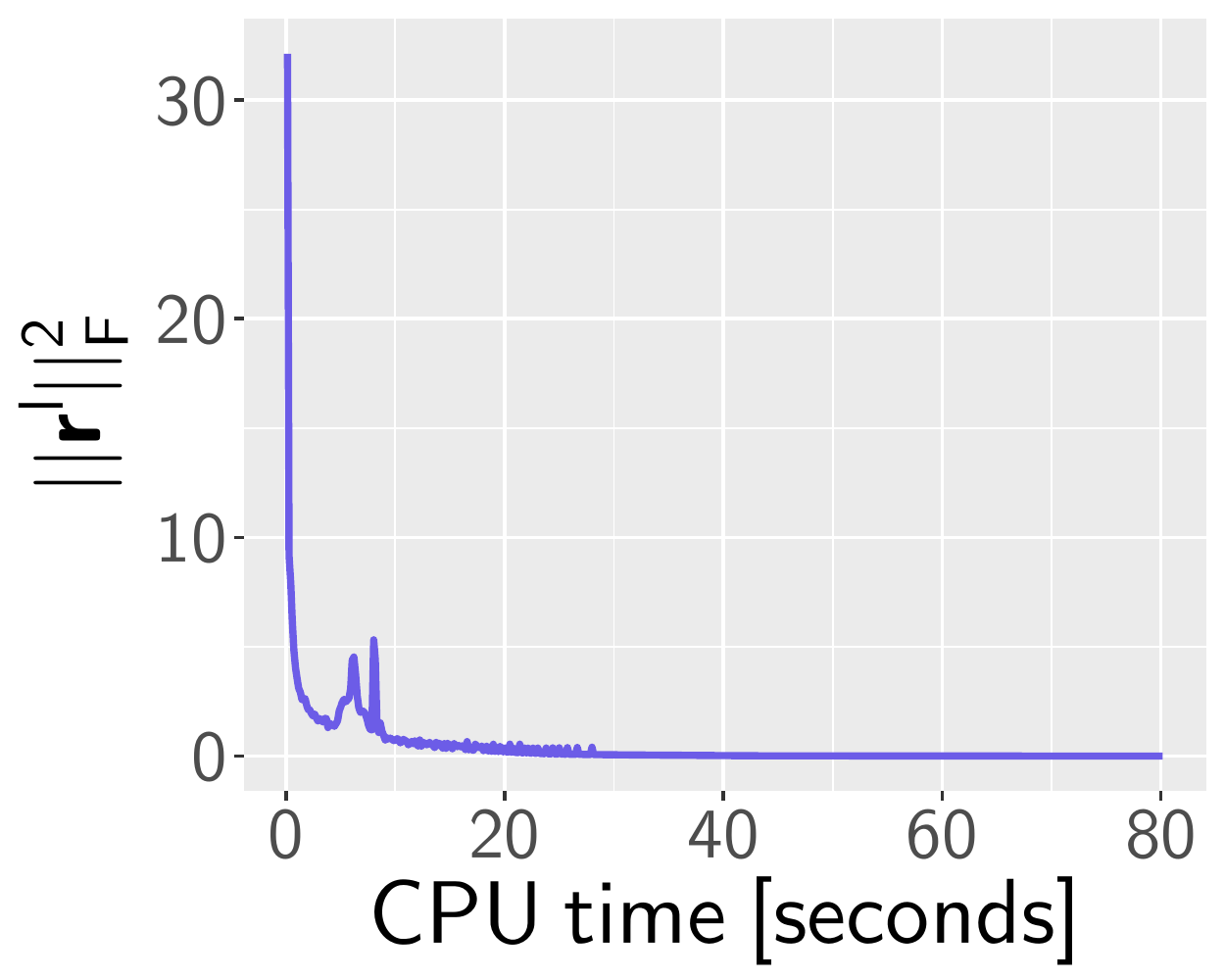}
  \end{subfigure}%
  \begin{subfigure}[t]{0.25\textwidth}
    \centering
    \includegraphics[scale=.3]{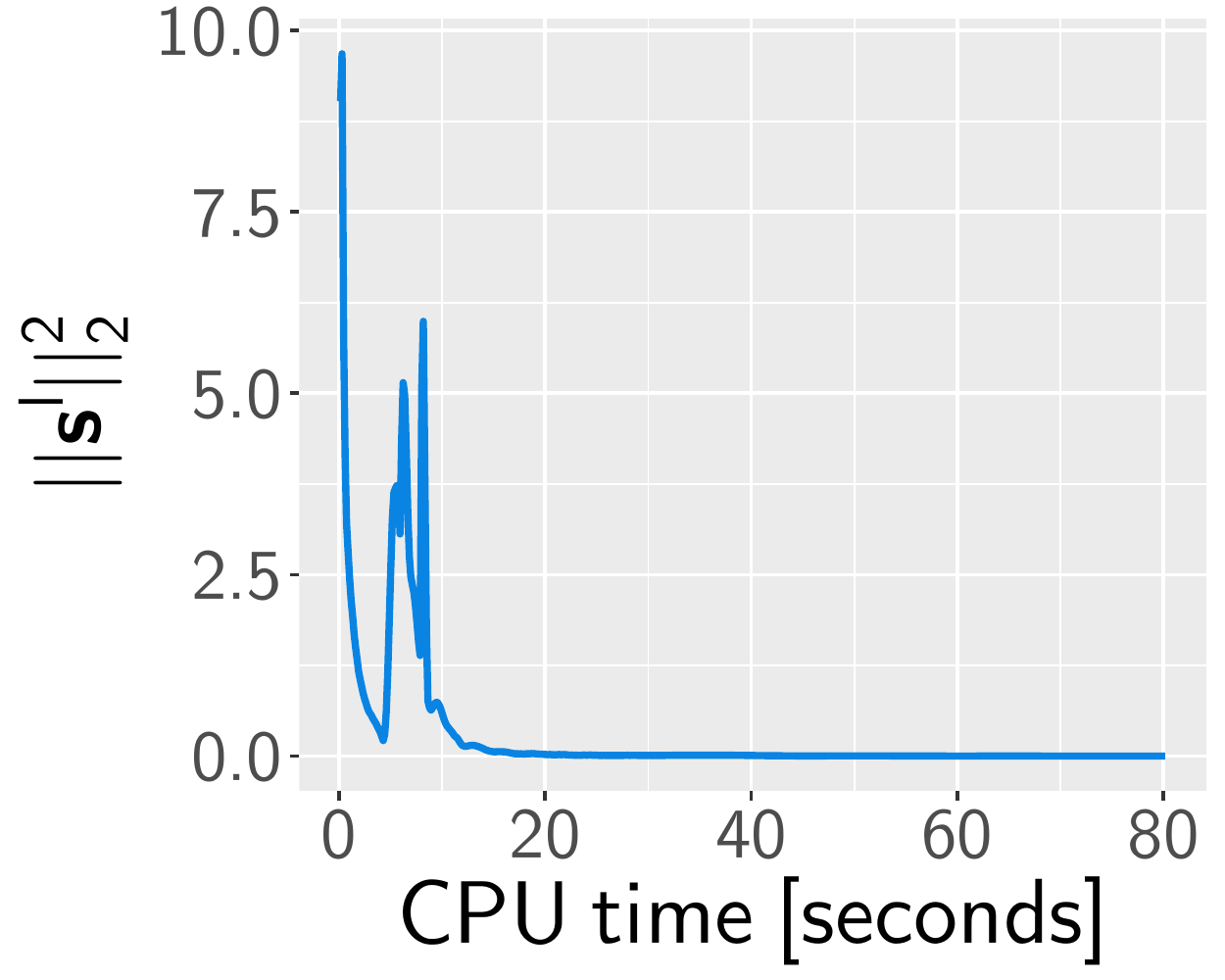}
  \end{subfigure}%
  \begin{subfigure}[t]{0.25\textwidth}
    \centering
    \includegraphics[scale=.3]{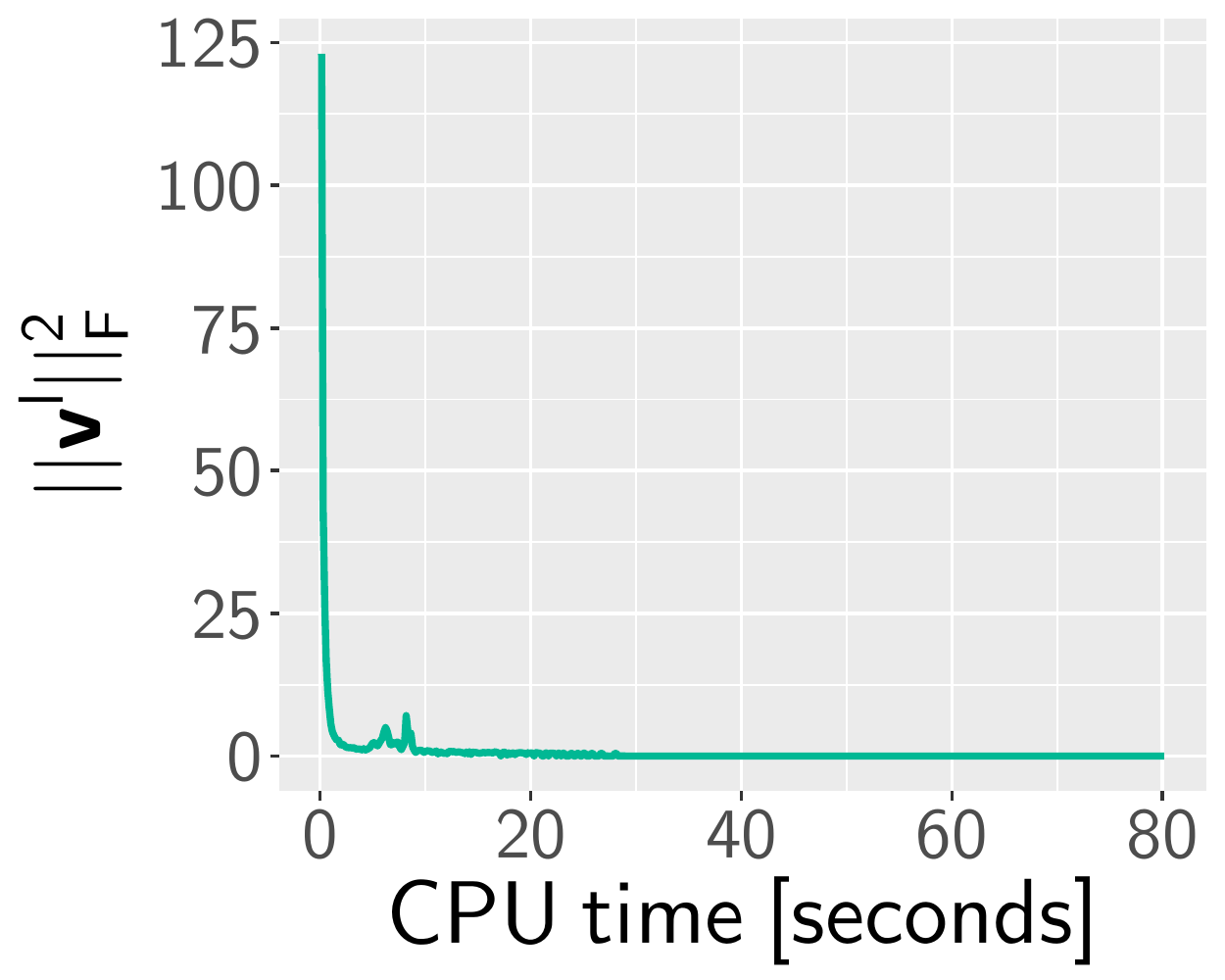}
\end{subfigure}%
  \caption{Empirical convergence for stock clustering with six sectors.}
  \label{fig:convergence-stock-clustering-six-sectors}
\end{figure}

\begin{figure}[!htb]
  \captionsetup[subfigure]{justification=centering}
  \centering
  \begin{subfigure}[t]{0.25\textwidth}
      \centering
      \includegraphics[scale=.3]{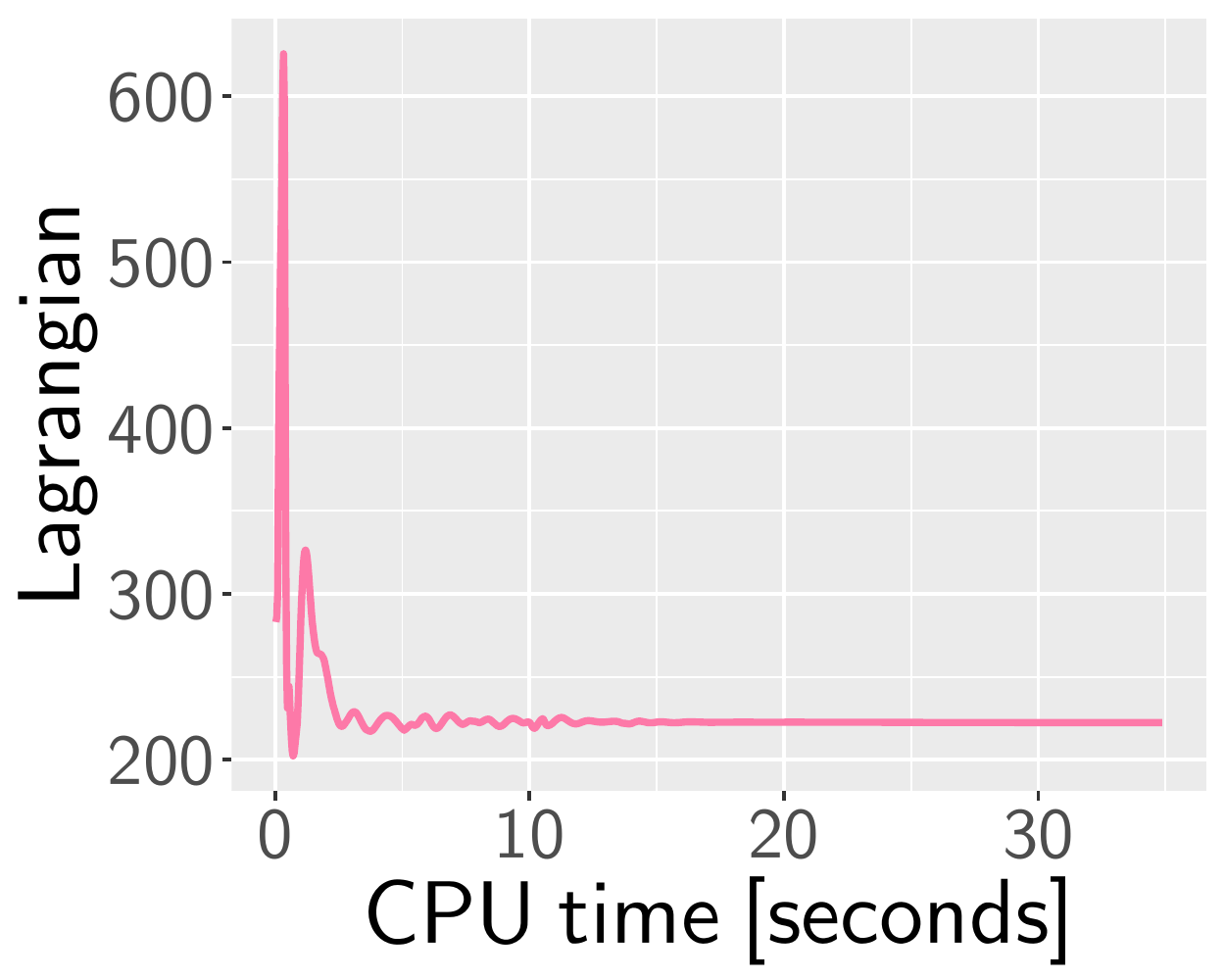}
  \end{subfigure}%
  \begin{subfigure}[t]{0.25\textwidth}
      \centering
      \includegraphics[scale=.3]{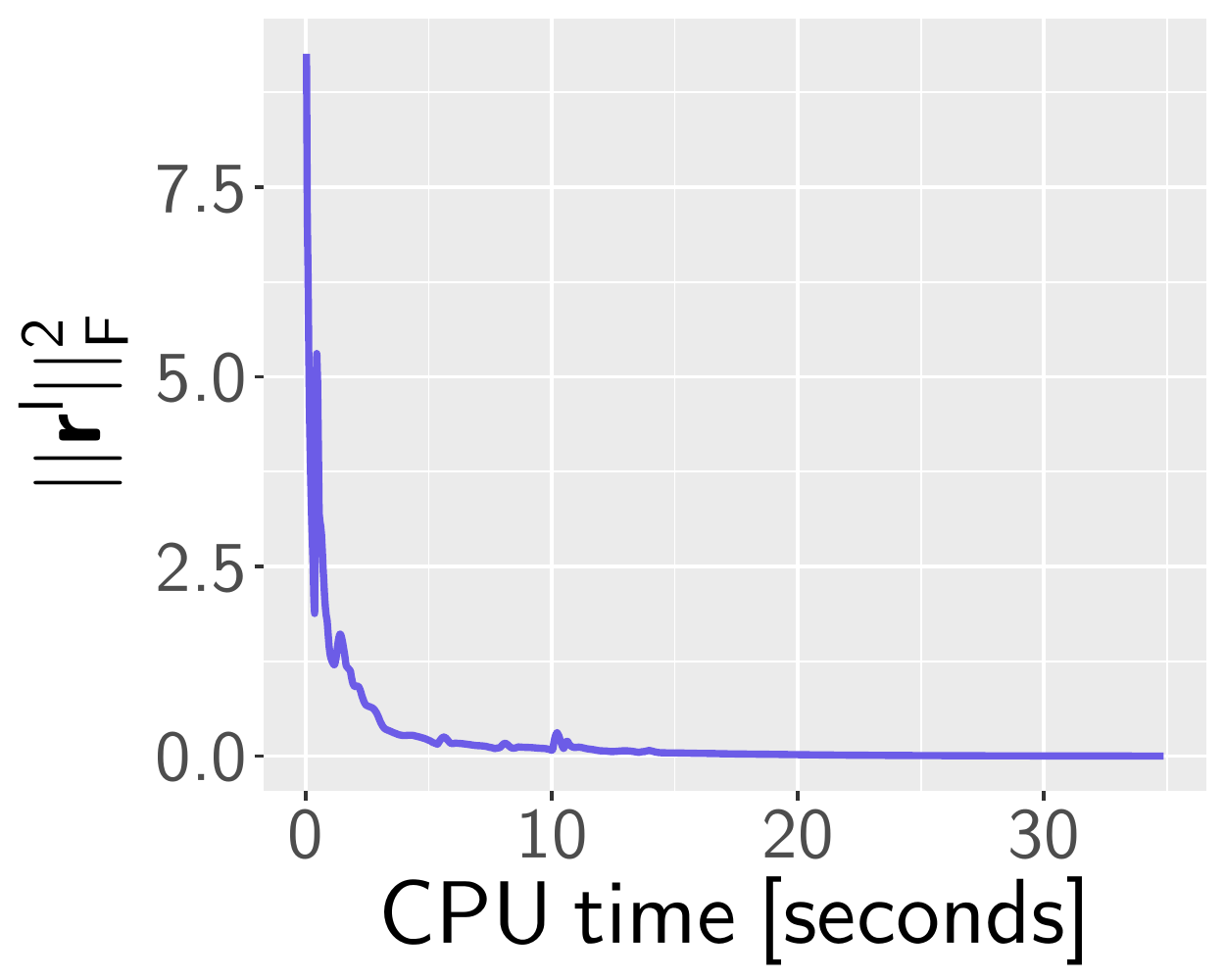}
  \end{subfigure}%
  \begin{subfigure}[t]{0.25\textwidth}
    \centering
    \includegraphics[scale=.3]{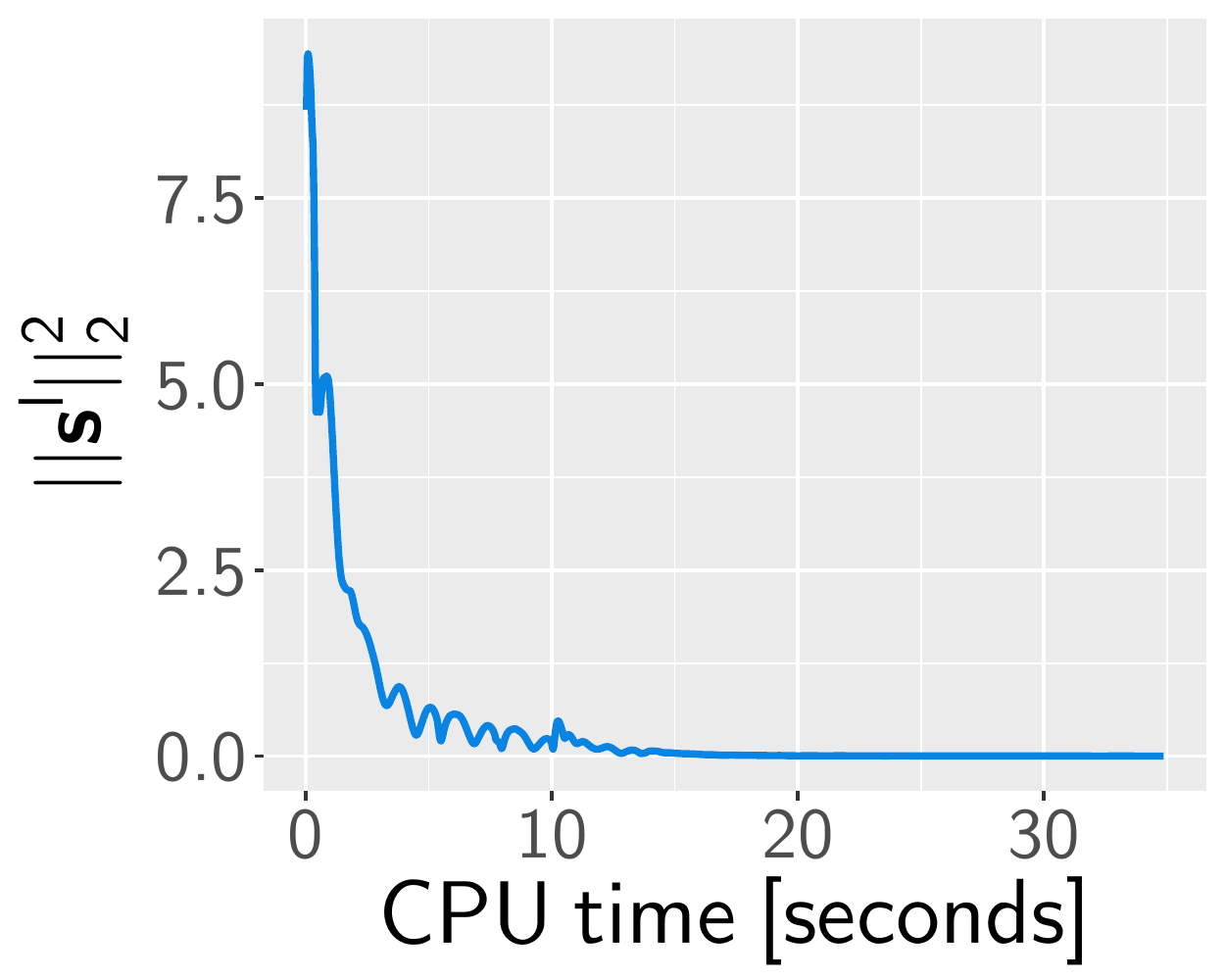}
  \end{subfigure}%
  \begin{subfigure}[t]{0.25\textwidth}
    \centering
    \includegraphics[scale=.3]{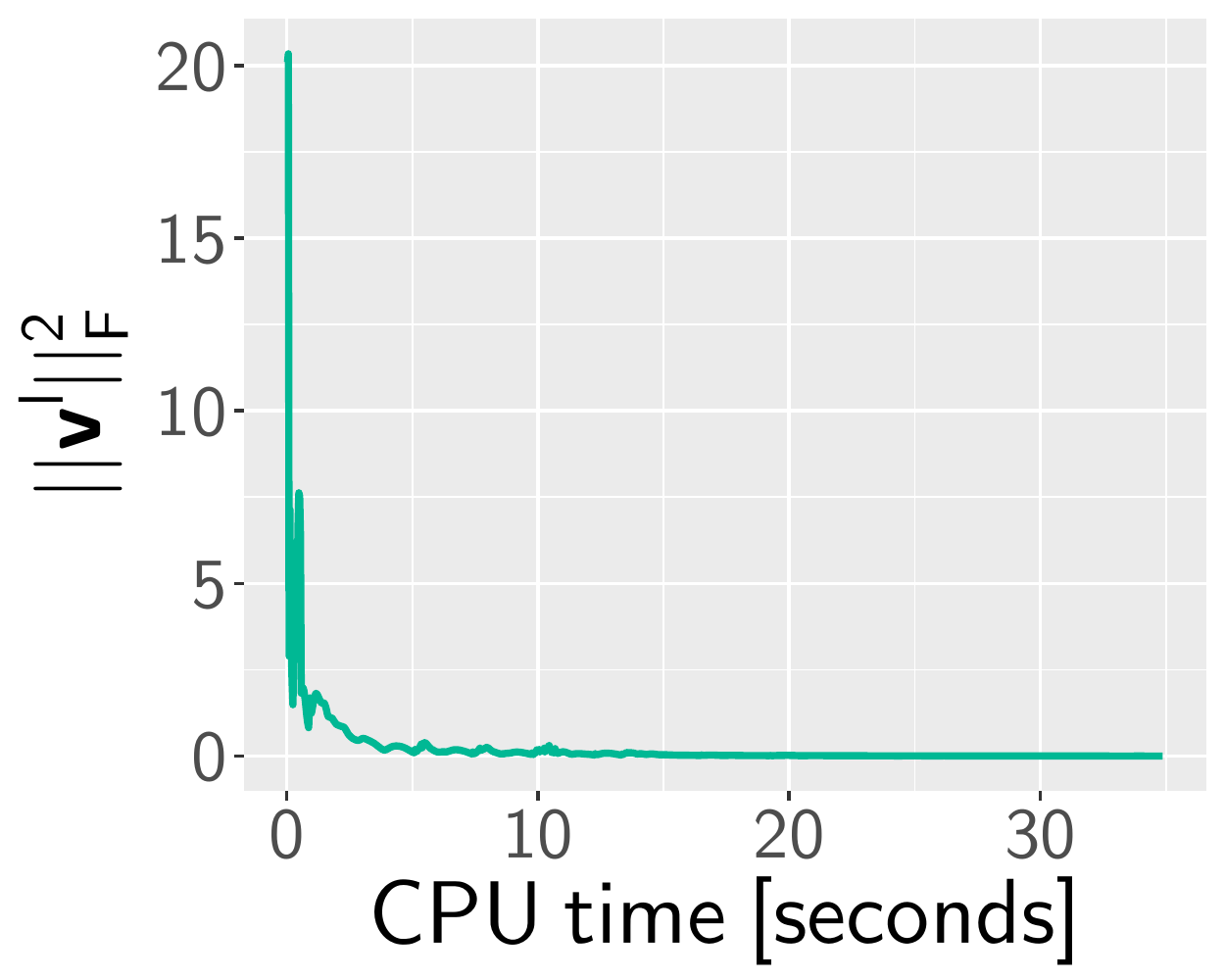}
\end{subfigure}%
  \caption{Empirical convergence for COVID-19 data experiment.}
  \label{fig:convergence-covid}
\end{figure}

\begin{figure}[!htb]
  \captionsetup[subfigure]{justification=centering}
  \centering
  \begin{subfigure}[t]{0.25\textwidth}
      \centering
      \includegraphics[scale=.3]{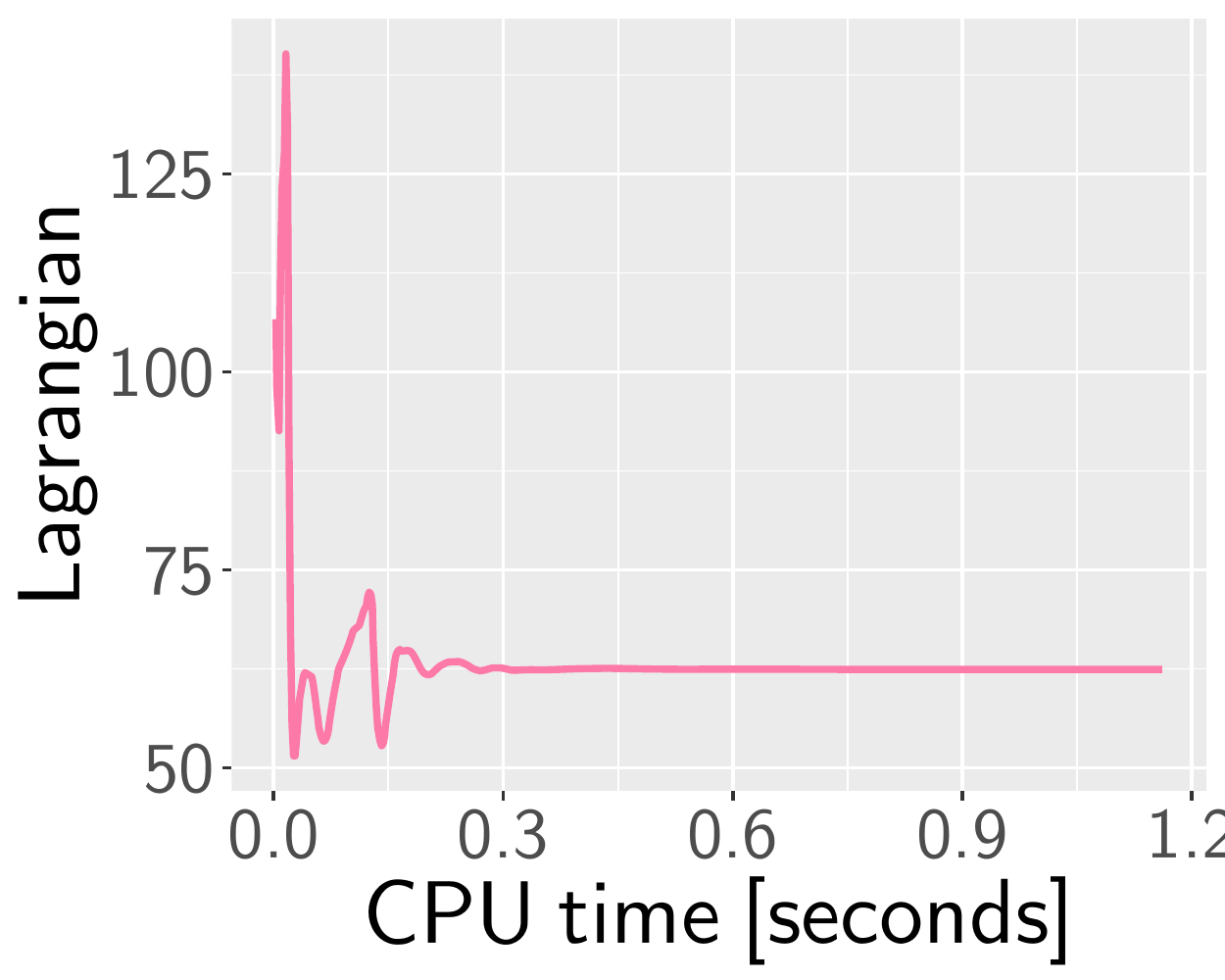} 
  \end{subfigure}%
  \begin{subfigure}[t]{0.25\textwidth}
      \centering
      \includegraphics[scale=.3]{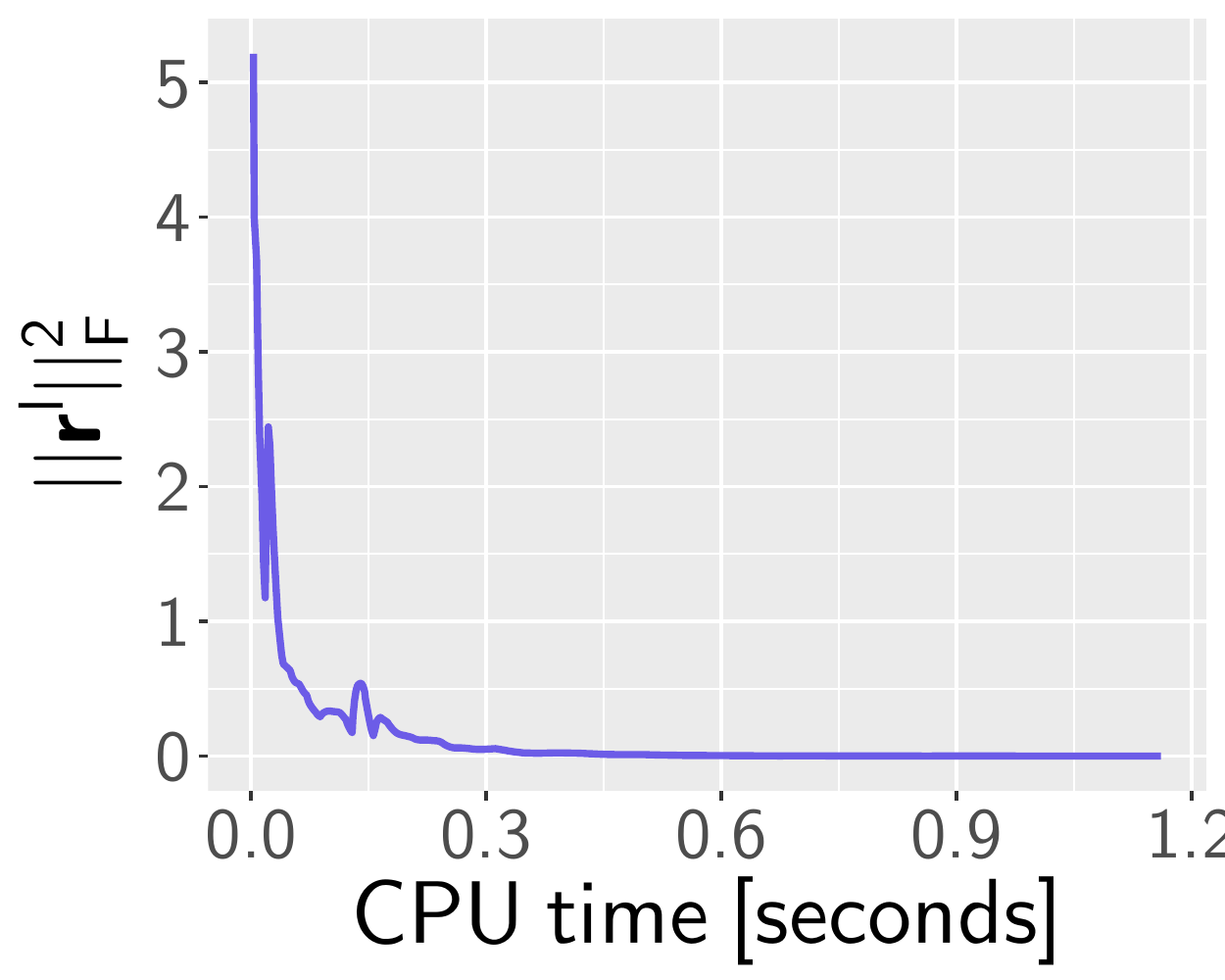}
  \end{subfigure}%
  \begin{subfigure}[t]{0.25\textwidth}
    \centering
    \includegraphics[scale=.3]{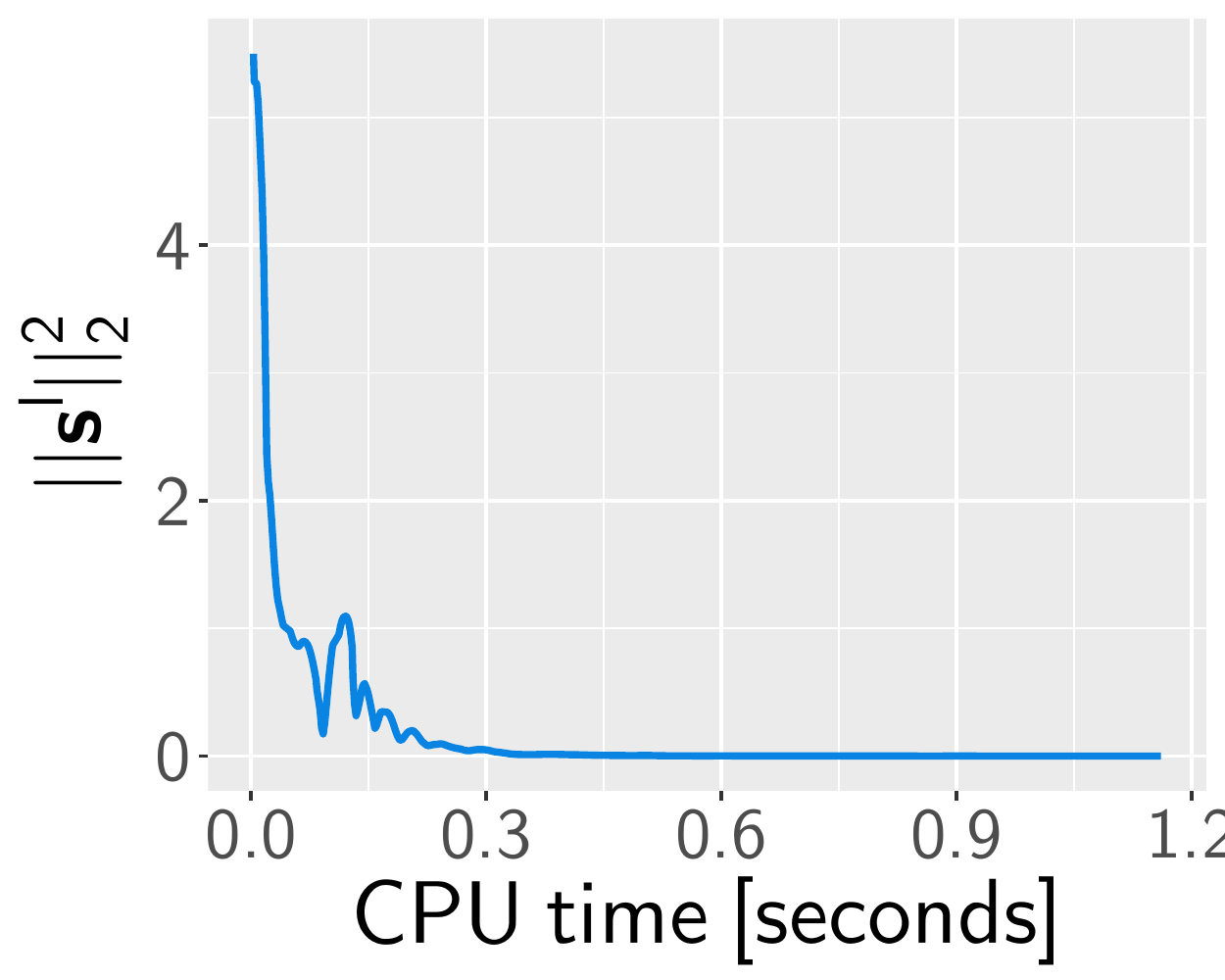}
  \end{subfigure}%
  \begin{subfigure}[t]{0.25\textwidth}
    \centering
    \includegraphics[scale=.3]{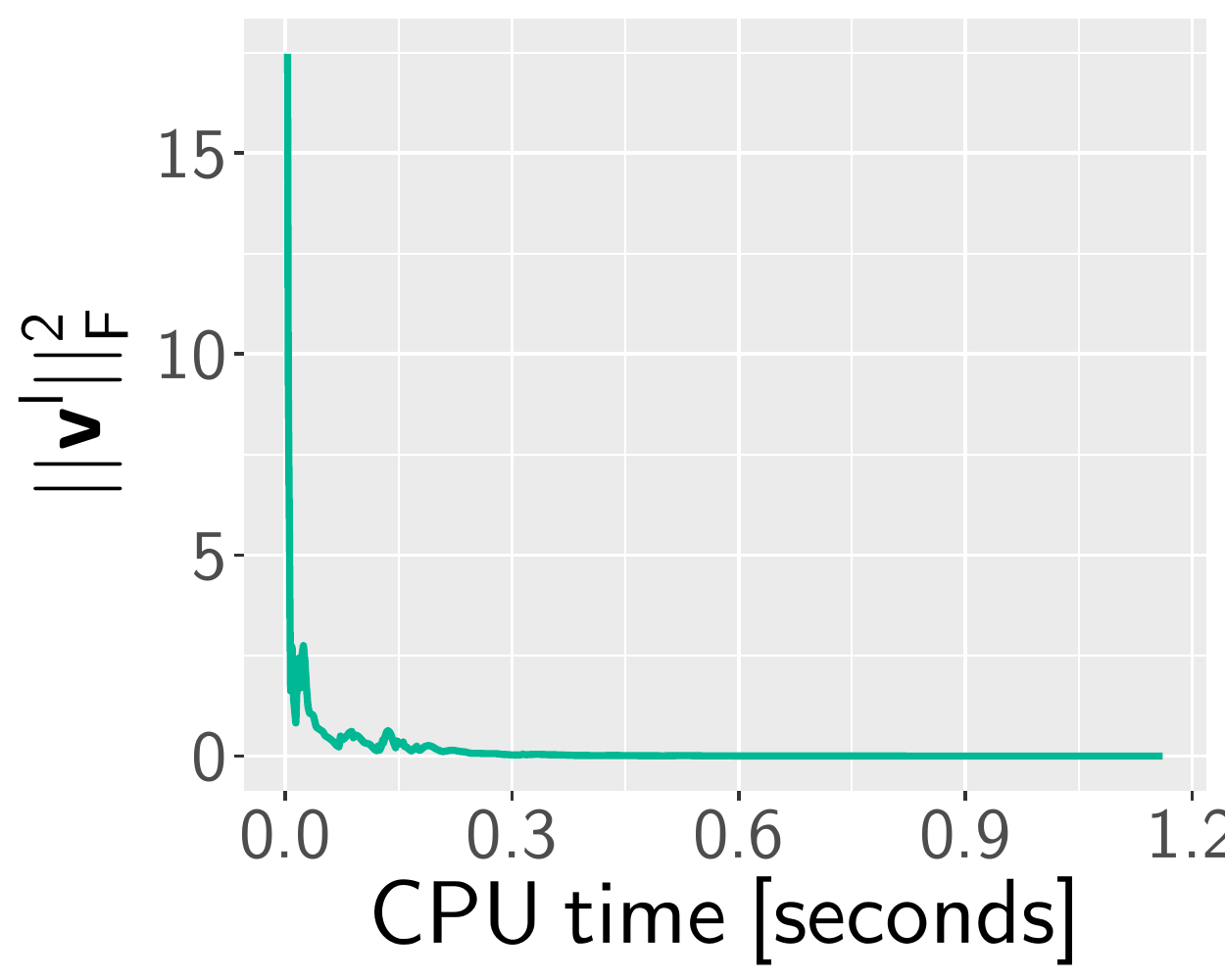}
\end{subfigure}%
  \caption{Empirical convergence for FX data experiment.}
  \label{fig:convergence-forex}
\end{figure}

\newpage
\section{Definitions}

\begin{definition}{\textbf{(Laplacian operator)}}
  The Laplacian operator~\citep{kumar20192} $\mathcal{L}: \mathbb{R}_{+}^{p(p-1)/2} \rightarrow \mathbb{R}^{p\times p}$,
  which takes a nonnegative vector $\bm{w}$ and outputs a Laplacian matrix $\sL\ww$, is defined as
  \begin{equation}
    [\sL\bm{w}]_{ij} =
    \begin{cases}
      - w_{i+s(j)}, & \textrm{if}\quad i > j,\\
      [\sL\bm{w}]_{ji}, & \textrm{if}\quad i < j,\\
      -\sum_{i\neq j} [\sL\bm{w}]_{ij}, & \textrm{if}\quad i = j,
    \end{cases}
    \label{eq:lap-op}
  \end{equation}
  where $s(j) = \frac{j-1}{2}(2p - j) - j$.
\end{definition}

\begin{definition}{\textbf{(Adjacency operator)}}
  The adjacency operator~\citep{kumar2019} $\mathcal{A}: \mathbb{R}_{+}^{p(p-1)/2} \rightarrow \mathbb{R}^{p}$,
  which takes a nonnegative vector $\bm{w}$ and outputs an Adjacency matrix $\mathcal{A}\ww$, is defined as
  \begin{equation}
    [\sA\bm{w}]_{ij} =
    \begin{cases}
      w_{i+s(j)}, & \textrm{if}\quad i > j,\\
      [\sA\ww]_{ji}, & \textrm{if}\quad i < j,\\
      0, & \textrm{if}\quad i = j,
    \end{cases}
    \label{eq:adj-op}
  \end{equation}
  where $s(j) = \frac{j-1}{2}(2p - j) - j$.
\end{definition}

\begin{definition}{\textbf{(Degree operator)}}
  The degree operator $\mathfrak{d}: \mathbb{R}^{p(p-1)/2} \rightarrow \mathbb{R}^{p}$, which takes a nonnegative
  vector $\bm{w}$ and outputs the diagonal of a Degree matrix, is defined as
  \begin{equation}
    \mathfrak{d}\bm{w} = \left(\mathcal{A}\bm{w}\right)\mathbf{1}.
    \label{eq:deg-op}
  \end{equation}
\end{definition}

\begin{definition}{\textbf{(Adjoint of Laplacian operator)}}
  The adjoint of Laplacian operator~\citep{kumar20192} $\sL^*: \mathbb{R}^{p\times p} \rightarrow \mathbb{R}^{p(p-1)/2}$
  is defined as
  \begin{equation}
    \left(\sL^*\bm{P}\right)_{s(i,j)} = \bm{P}_{i,i} - \bm{P}_{i,j} - \bm{P}_{j,i} + \bm{P}_{j,j}
  \end{equation}
  where $s(i,j) = i - j + \frac{j-1}{2}(2p - j), i > j.$
\end{definition}

\begin{definition}{\textbf{(Adjoint of adjacency operator)}}
  The adjoint of adjacency operator~\citep{kumar20192} $\sA^*: \mathbb{R}^{p\times p} \rightarrow \mathbb{R}^{p(p-1)/2}$
  is defined as
  \begin{equation}
    \left(\sA^*\bm{P}\right)_{s(i,j)} = \bm{P}_{i,i} + \bm{P}_{j,j}
  \end{equation}
  where $s(i,j) = i - j + \frac{j-1}{2}(2p - j), i > j.$
\end{definition}

\begin{definition}{\textbf{(Adjoint of degree operator)}}
  The adjoint of degree operator
  $\mathfrak{d}^*: \mathbb{R}^{p} \rightarrow \mathbb{R}^{p(p-1)/2}$ is given as
  \begin{equation}
  \left(\mathfrak{d}^{*}\bm{y}\right)_{s(i,j)} = \bm{y}_i + \bm{y}_j,
  \end{equation}
where $s(i,j) = i - j + \frac{j - 1}{2}(2p - j), i > j$.
\end{definition}

An alternative expression for $\mathfrak{d}^*$ is
$\mathfrak{d}^*\bm{y} = \mathcal{L}^*\Diag\left(\bm{y}\right)$,
where $\mathcal{L}^*$ is the adjoint of the Laplacian operator.

\begin{definition}{\textbf{(Modularity)}}
  The modularity of a graph~\citep{newman2006} is defined as $Q : \mathbb{R}^{p\times p} \rightarrow [-1/2, 1]$:
  \begin{equation}
    Q(\bm{W}) \triangleq \dfrac{1}{p(p-1)} \sum_{i,j}\left(\bm{W}_{ij} - \dfrac{d_id_j}{p(p - 1)}\right)\mathbb{1}(t_i = t_j),
    \label{eq:modularity}
  \end{equation}
\end{definition}
where $d_i$ is the weighted degree of the $i$-th node, \textit{i.e.} $d_i \triangleq \left[\mathfrak{d}\left(\ww\right)\right]_i$,
$t_i \triangleq f_t(i), i \in \mathcal{V},$ is the type of the $i$-th node, and $\mathbb{1}(\cdot)$
is the indicator function.

\begin{definition}{\textbf{(Proximal Operator)}}
  The proximal operator of a function $f$,
  $f : \mathbb{R}^{p\times p} \rightarrow \mathbb{R}$, with parameter $\rho$, $\rho \in \mathbb{R}_{++}$, is defined as~\citep{parikh2014}
  \begin{equation}
    \mathsf{prox}_{\rho^{-1} f}\left(\bm{V}\right) \triangleq \underset{\bm{U} \in \mathbb{R}^{p\times p}}{\mathsf{arg~min}}~~ f\left(\bm{U}\right) +
    \dfrac{\rho}{2}\norm{\bm{U} - \bm{V}}^{2}_{\mathrm{F}}.
    \label{eq:prox-def}
  \end{equation}
\end{definition}

\section{Proofs}

\subsection{Proof of Lemma \ref{lemma:lipschitz}}\label{sec:proof-max-eigenval}
\begin{proof}
We define an index set $\Omega_t$:
\begin{align}\label{omega}
\Omega_t := \left \{ l \ | \left[\mathcal{L} \bm w \right]_{tt} = \sum_{l \in \Omega_t} x_l \right \}, \quad t \in [1, p].
\end{align} 
Then we have
\begin{equation*}
\begin{split}
\lambda_{\max} \left( \mathcal{L}^\ast \mathcal{L}\right) &= \sup_{\norm{\bm x}=1} \bm x^\top \mathcal{L}^\ast \mathcal{L} \bm x =\sup_{\norm{\bm x}=1} \norm{\mathcal{L} \bm x}_F^2 = \sup_{\norm{\bm x}=1}2\sum_{k=1}^{p(p-1)/2}x_k^2 + \sum_{i=1}^{p}([\mathcal{L} \bm w]_{ii})^2 \\
&= \sup_{\norm{\bm x}=1} 4\sum_{k=1}^{p(p-1)/2}x_k^2 + \sum_{t=1}^p \sum_{i, j \in \Omega_t, \ i \neq j} x_i x_j  \leq 4+ \sup_{\norm{\bm x}=1} \frac{1}{2}\sum_{t=1}^p \sum_{i, j \in \Omega_t, \ i \neq j} x_i^2 + x_j^2 \\
& = (4 + 2(|\Omega_t| -1)) \sup_{\norm{\bm x}=1}\sum_{k=1}^{p(p-1)/2}x_k^2 = 2p,
\end{split}
\end{equation*}
with equality if and only if $x_1 = \cdots = x_{p(p-1)/2} = \sqrt{\frac{2}{p(p-1)/2}}$ or $x_1 = \cdots = x_{p(p-1)/2} = -\sqrt{\frac{2}{p(p-1)/2}}$. The last equality follows the fact that $|\Omega_t |=p-1$.

Similarly, we can obtain
\begin{equation*}
\begin{split}
\lambda_{\max} \left( \mathfrak{d}^\ast \mathfrak{d} \right) &= \sup_{\norm{\bm x}=1} \bm x^\top \mathfrak{d}^\ast \mathfrak{d} \bm x =\sup_{\norm{\bm x}=1} \norm{\mathfrak{d} \bm x}^2 = 
 \sup_{\norm{\bm x}=1} 2\sum_{k=1}^{p(p-1)/2}x_k^2 + \sum_{t=1}^p \sum_{i, j \in \Omega_t, \ i \neq j} x_i x_j \\
& \leq 2+ \sup_{\norm{\bm x}=1} \frac{1}{2}\sum_{t=1}^p \sum_{i, j \in \Omega_t, \ i \neq j} x_i^2 + x_j^2 = (2 + 2(|\Omega_t| -1)) \sup_{\norm{\bm x}=1}\sum_{k=1}^{p(p-1)/2}x_k^2 = 2p-2,
\end{split}
\end{equation*}
with equality if and only if $x_1 = \cdots = x_{p(p-1)/2} = \sqrt{\frac{2}{p(p-1)/2}}$ or $x_1 = \cdots = x_{p(p-1)/2} = -\sqrt{\frac{2}{p(p-1)/2}}$.

Finally, we have
\begin{equation}\label{eig_max_Ld}
\begin{split}
\lambda_{\max} \left( \mathcal{L}^\ast \mathcal{L} + \mathfrak{d}^\ast \mathfrak{d}\right) &= \sup_{\norm{\bm x}=1} \bm x^\top \left( \mathcal{L}^\ast \mathcal{L} + \mathfrak{d}^\ast \mathfrak{d}\right)\bm x \\
&\leq \sup_{\norm{\bm x}=1} \bm x^\top \left( \mathcal{L}^\ast \mathcal{L}\right)\bm x + \sup_{\norm{\bm y}=1} \bm y^\top \left(\mathfrak{d}^\ast \mathfrak{d}\right)\bm y \\
&=4p-2.
\end{split}
\end{equation}
Note that the equality in \eqref{eig_max_Ld} can be achieved because the eigenvectors of $\mathcal{L}^\ast \mathcal{L}$ and $\mathfrak{d}^\ast \mathfrak{d}$ associated with the maximum eigenvalue are the same. Therefore, we conclude that $\lambda_{\max} \left( \mathcal{L}^\ast \mathcal{L} + \mathfrak{d}^\ast \mathfrak{d}\right) = 4p-2$, completing the proof.
\end{proof}

\subsection{Proof of Theorem \ref{connect-convergence}}\label{sec:proof-connected}
We can rewrite the updating of $\boldsymbol{\Theta}$, $\ww$, $\bm Y$ and $\bm y$ in a compact form:
\begin{align}
\boldsymbol{\Theta}^{l+1} &= \underset{\bT \succeq \mathbf{0}}{\mathsf{arg~min}} ~ L_\rho(\bT, \ww^l, \bm{Y}^l, \bm{y}^l), \\
\ww^{l+1} & = \underset{ \ww \geq \bm 0}{\mathsf{arg~min}} ~ L_\rho(\bT^{l+1}, \ww, \bm{Y}^l, \bm{y}^l), \\
\begin{pmatrix}
\bm Y^{l+1} \\ \bm y^{l+1}
\end{pmatrix}&=
\begin{pmatrix}
\bm Y^{l} \\ \bm y^{l} 
\end{pmatrix} + \rho
\begin{pmatrix}
\boldsymbol{\Theta}^{l+1} -\mathcal{L} \ww^{l+1} \\ \mathfrak{d} \ww^{l+1} - \bm d
\end{pmatrix}.
\end{align}
Now we can see that our ADMM algorithm satisfies the standard form with two blocks of primal
variables $\boldsymbol{\Theta}$ and $\ww$, and one block of dual variable $\left( \bm Y, \bm y \right)$.
Our ADMM approach splits the original problem into two blocks in \eqref{eq:lw-admm-regular}.
According to the existing convergence results of ADMM in \citep{boyd2011}, we can conclude
that Algorithm~\ref{alg:connected} will converge to the optimal primal-dual solution for \eqref{eq:lw-admm-regular}.

\subsection{Proof of Theorem \ref{k-component-convergence}}\label{sec:proof-k-component}
\begin{proof}
To prove Theorem \ref{k-component-convergence}, we first establish the boundedness of the sequence $ \left\lbrace \left( \bm \Theta^l, \bm w^l, \bm V^l, \bm Y^l, \bm y^l \right) \right\rbrace$ generated by Algorithm \ref{alg:k-component-ky-fan} in Lemma \ref{Lem: bounded-T}, and the monotonicity of $ L_\rho \left( \bm \Theta^l, \bm w^l, \bm V^l, \bm Y^l, \bm y^l \right)$ in Lemma \ref{Lem: bounded-L}.
\begin{lemma}\label{Lem: bounded-T}
The sequence $ \left\lbrace \left( \bm \Theta^l, \bm w^l, \bm V^l, \bm Y^l, \bm y^l \right) \right\rbrace$ generated by Algorithm \ref{alg:k-component-ky-fan} is bounded.
\end{lemma}
\begin{proof}
Let $\bm w^0$, $\bm V^0$, $\bm Y^0$ and $\bm y^0$ be the initialization of the sequences $\left\lbrace \bm w^l \right\rbrace$, $\left\lbrace \bm V^l \right\rbrace$, $\left\lbrace \bm Y^l\right\rbrace$ and $\left\lbrace \bm y^l \right\rbrace$, respectively, and $\norm{\bm w^0}$, $\norm{\bm V^0}_{\mathrm{F}}$, $\norm{\bm Y^0}_{\mathrm{F}}$ and $\norm{\bm y^0}$ are bounded. 

We prove the lemma by induction. Recall that the sequence $ \left\lbrace \bm \Theta^l \right\rbrace$ is established by 
\begin{equation}
  \bT^l = \frac{1}{2\rho}\bm{U}^{l-1}\left(\boldsymbol{\Gamma}^{l-1} + \sqrt{ \left( \boldsymbol{\Gamma}^{l-1} \right)^2 + 4\rho \bm{I}}\right)\bm{U}^{l-1\top},
  \label{eq:bt-update-k-comp-l}
\end{equation}
where $\boldsymbol{\Gamma}^{l-1}$ contains the largest $p-k$ eigenvalues of $ \rho \mathcal{L}\bm{w}^{l-1} - \bm{Y}^{l-1}$, and $\bm U^{l-1}$ contains the corresponding eigenvectors. When $l=1$, $\norm{\boldsymbol{\Gamma}^{0}}_{\mathrm{F}}$ is bounded since both $\norm{\bm w^0}$ and $\norm{\bm Y^0}_{\mathrm{F}}$ are bounded. Therefore, we can conclude that $\norm{\bT^1}_{\mathrm{F}}$ is bounded.
The sequence $\left\lbrace \ww^l \right\rbrace$ is established by solving the subproblems
\begin{equation}
\begin{split}
    \ww^l &=
    \underset{\ww \geq \mathbf{0}}{\mathsf{arg~min}} ~ \frac{\rho}{2}\ww^\top \left(\mathfrak{d}^*\mathfrak{d} + \mathcal{L}^*\mathcal{L}\right) \ww  \\
    &\quad \quad + \left\langle \ww, \sL^*\left(\bm{S} + \eta\bm{V}^{l} \bm{V}^{l\top}- \bm{Y}^{l-1} - \rho \bT^l \right) + \mathfrak{d}^*\left(\bm{y}^{l-1} - \rho\bm{d}\right) \right\rangle.
    \label{eq:w-update-new-l}
\end{split}
\end{equation}
Let
\begin{equation*}
 f_l(\ww) = \frac{\rho}{2}\ww^\top\left(\mathfrak{d}^*\mathfrak{d} + \mathcal{L}^*\mathcal{L}\right)\ww + \left\langle \ww, \bm a^l \right\rangle,
\end{equation*}
where $\bm a^l =\sL^*\left(\bm{S} - \bm{Y}^{l-1} - \rho \bT^l \right) + \mathfrak{d}^*\left(\bm{y}^{l-1} - \rho\bm{d}\right)$. We get that $\norm{\bm a^1}$ is bounded because of the boundedness of $\norm{\bm{Y}^{0}}_{\mathrm{F}}$, $\norm{\bm{y}^{0}}$, and $\norm{\bT^1}_{\mathrm{F}}$. By \citep{ying2020}, we know that $\mathcal{L}^*\mathcal{L}$ is a positive definite matrix and the minimum eigenvalue $\lambda_{\min} \left( \mathcal{L}^*\mathcal{L}\right) = 2$. On the other hand, $\mathfrak{d}^*\mathfrak{d}$ is a positive semi-definite matrix as follows,
\begin{equation}
\lambda_{\min} \left( \mathfrak{d}^*\mathfrak{d} \right) = \sup_{\bm x \neq \bm 0} \frac{\bm x^\top \mathfrak{d}^*\mathfrak{d} \bm x}{\bm x^\top \bm x} = \sup_{\bm x \neq \bm 0} \frac{ \left\langle \mathfrak{d} \bm x, \ \mathfrak{d} \bm x \right\rangle}{ \bm x^\top \bm x} \geq 0.
\end{equation} 
Therefore, we obtain that
\begin{equation*}
\lim_{\norm{\ww} \to +\infty} f_1(\ww) \geq \lim_{\norm{\ww} \to + \infty} \rho \ww^\top \ww + \langle \ww, \bm a^1\rangle = + \infty,
\end{equation*}
implying that $f_1(\ww)$ is coercive, \textit{i.e.}, $f_1(\ww) \to + \infty$ for any $\norm{\ww} \to + \infty$. Thus $\norm{ \ww^1}$ is bounded. According to \eqref{eq:k-comp-proposed-relaxed-sub-V}, we can see that $\bm V$ is in a compact set, and thus $\norm{\bm V^l}_{\mathrm{F}}$ is bounded for any $l \geq 1$. Finally, according to \eqref{eq:Y-update} and \eqref{eq:y-update}, one has 
\begin{equation}
  \bm{Y}^1 = \bm{Y}^{0} + \rho\left(\bT^1 - \sL\ww^1 \right),
  \label{eq:Y-update-1}
\end{equation} and
\begin{equation}
  \bm{y}^1 = \bm{y}^0 + \rho\left(\mathfrak{d}\ww^1 - \bm{d}\right).
  \label{eq:y-update-1}
\end{equation}
It is obvious that both $\norm{\bm{Y}^1}_{\mathrm{F}}$ and $\norm{\bm{y}^1}$ are bounded. Therefore, it holds for $l=1$ that $ \left\lbrace \left( \bm \Theta^l, \bm w^l, \bm V^l, \bm Y^l, \bm y^l \right) \right\rbrace$ is bounded. 

Assume that $ \left\lbrace \left( \bm \Theta^{l-1}, \bm w^{l-1}, \bm V^{l-1}, \bm Y^{l-1}, \bm y^{l-1} \right) \right\rbrace$
is bounded for some $l \geq 1$, \textit{i.e.}, each term in $ \left\lbrace \left( \bm \Theta^{l-1}, \bm w^{l-1}, \bm V^{l-1}, \bm Y^{l-1}, \bm y^{l-1} \right) \right\rbrace$ is bounded under $\ell_2$-norm or Frobenius norm. Following from \eqref{eq:bt-update-k-comp-l}, we can obtain that $\norm{\bm \Theta^l}_{\mathrm{F}}$ is bounded. By \eqref{eq:w-update-new-l}, similarly, we can get that $\norm{\bm w^l}$ is bounded. Similar to \eqref{eq:Y-update-1} and \eqref{eq:y-update-1}, we can also obtain that $\norm{\bm Y^l}$ and $\norm{\bm y}^l$ are bounded, because of the boundedness of $\norm{\bm \Theta^l}_{\mathrm{F}}$, $\norm{\bm w^l}$, $\norm{\bm Y^{l-1}}$ and $\norm{\bm y^{l-1}}$.
Thus, $\left\lbrace \left( \bm \Theta^l, \bm w^l, \bm V^l, \bm Y^l, \bm y^l \right) \right\rbrace$ is bounded, completing the induction. Therefore, we can conclude that the sequence $\left\lbrace \left( \bm \Theta^l, \bm w^l, \bm V^l, \bm Y^l, \bm y^l \right) \right\rbrace$ is bounded.
\end{proof}

\begin{lemma} \label{Lem: bounded-L}
The sequence $ L_\rho \left( \bm \Theta^l, \bm w^l, \bm V^l, \bm Y^l, \bm y^l \right)$ generated by Algorithm \ref{alg:k-component-ky-fan} is lower bounded, and
\begin{equation}
L_\rho \left( \bm \Theta^{l+1}, \bm w^{l+1}, \bm V^{l+1}, \bm Y^{l+1}, \bm y^{l+1} \right) \leq  L_\rho \left( \bm \Theta^l, \bm w^l, \bm V^l, \bm Y^l, \bm y^l \right), \quad \forall \, l \in \mathbb{N}_+,
\end{equation}
holds for any sufficiently large $\rho$.
\end{lemma}

\begin{proof}
According to \eqref{eq:k-component-Lagrangian}, we have
\begin{align}\label{eq:k-component-Lagrangian-l}
    L_\rho(\bT^{l}, \ww^{l}, \bm{V}^{l}, \bm{Y}^{l}, \bm{y}^{l} ) = & ~ \mathsf{tr}\left(\sL\ww^{l}  \left(\bm{S} + \eta\bm{V}^{l} \left( \bm{V}^{l} \right)^\top  \right)\right)
    - \log\mathrm{det^*}\left(\boldsymbol{\Theta}^{l}\right)
     + \left\langle \bm{y}^{l}, \mathfrak{d}\bm{w}^{l} - \bm{d} \right\rangle  \nonumber \\
    & + \frac{\rho}{2}\norm{\mathfrak{d}\bm{w}^{l} - \bm{d}}^2_{2} + \left\langle \bm{Y}^{l}, \boldsymbol{\Theta}^{l} - \mathcal{L}\bm{w}^{l} \right\rangle  + \frac{\rho}{2}\norm{\bT^{l} - \mathcal{L}\bm{w}^{l}}^2_{\mathrm{F}}.
\end{align}
We can see that the lower boundedness of the sequence $ L_\rho \left( \bm \Theta^l, \bm w^l, \bm V^l, \bm Y^l, \bm y^l \right)$ can be established by the boundedness of $\left\lbrace \left( \bm \Theta^l, \bm w^l, \bm V^l, \bm Y^l, \bm y^l \right) \right\rbrace$ in Lemma \ref{Lem: bounded-T}.

Now we first establish that 
 \begin{equation}
L_\rho \left( \bm \Theta^{l+1}, \bm w^{l}, \bm V^{l}, \bm Y^{l}, \bm y^{l} \right) \leq  L_\rho \left( \bm \Theta^l, \bm w^l, \bm V^l, \bm Y^l, \bm y^l \right), \quad \forall \, l \in \mathbb{N}_+.
\end{equation}
We have
\begin{align*}
    L_\rho(\bT^{l+1}, \ww^{l}, \bm{V}^{l}, \bm{Y}^{l}, \bm{y}^{l} ) = & ~ \mathsf{tr}\left(\sL\ww^{l} \left(\bm{S} + \eta\bm{V}^{l} \left( \bm{V}^{l} \right)^\top \right)\right)
    - \log\mathrm{det^*}\left(\boldsymbol{\Theta}^{l+1}\right)
     + \left\langle \bm{y}^{l}, \mathfrak{d}\bm{w}^{l} - \bm{d} \right\rangle \nonumber \\
    & + \frac{\rho}{2}\norm{\mathfrak{d}\bm{w}^{l} - \bm{d}}^2_{2} + \left\langle \bm{Y}^{l}, \boldsymbol{\Theta}^{l+1} - \mathcal{L}\bm{w}^{l} \right\rangle  + \frac{\rho}{2}\norm{\bT^{l+1} - \mathcal{L}\bm{w}^{l}}^2_{\mathrm{F}}. \nonumber
\end{align*}
Then we obtain
\begin{align*}
    &L_\rho(\bT^{l+1}, \ww^{l}, \bm{V}^{l}, \bm{Y}^{l}, \bm{y}^{l} ) -  L_\rho(\bT^{l}, \ww^{l}, \bm{V}^{l}, \bm{Y}^{l}, \bm{y}^{l} ) = ~ 
    - \log\mathrm{det^*}\left(\boldsymbol{\Theta}^{l+1}\right) + \left\langle \bm{Y}^{l}, \boldsymbol{\Theta}^{l+1} \right\rangle 
      \\
    &  \quad + \frac{\rho}{2}\norm{\bT^{l+1} - \mathcal{L}\bm{w}^{l}}^2_{\mathrm{F}} - \left(  - \log\mathrm{det^*}\left(\boldsymbol{\Theta}^{l}\right) + \left\langle \bm{Y}^{l}, \boldsymbol{\Theta}^{l} \right\rangle + \frac{\rho}{2}\norm{\bT^{l} - \mathcal{L}\bm{w}^{l}}^2_{\mathrm{F}}  \right). \nonumber
\end{align*}
Note that $\bT^{l+1}$ minimizes the objective function
\begin{equation}
      \bT^{l+1} = \underset{\substack{\mathsf{rank}(\bT) = p-k \\ \bT \succeq \mathbf{0}}}{\mathsf{arg~min}}  ~ - \log\mathrm{det^*}(\bT) + \langle\bT, \bm{Y}^{l}\rangle +
      \frac{\rho}{2}\norm{\bT - \mathcal{L}\bm{w}^{l}}^2_{\mathrm{F}}.
    \label{eq:lw-admm-regular-theta-k-comp-l1}
\end{equation}
Therefore 
\begin{equation} \label{eq:Lag-l1}
L_\rho(\bT^{l+1}, \ww^{l}, \bm{V}^{l}, \bm{Y}^{l}, \bm{y}^{l} ) -  L_\rho(\bT^{l}, \ww^{l}, \bm{V}^{l}, \bm{Y}^{l}, \bm{y}^{l} ) \leq 0
\end{equation}
holds for any $l \in \mathbb{N}_+$.

One has
\begin{align}
    &L_\rho(\bT^{l+1}, \ww^{l}, \bm{V}^{l}, \bm{Y}^{l}, \bm{y}^{l} ) -  L_\rho(\bT^{l+1}, \ww^{l+1}, \bm{V}^{l+1}, \bm{Y}^{l+1}, \bm{y}^{l+1} ) \nonumber \\= & \underbrace{\mathsf{tr}\left( \eta \sL\ww^{l} \left(\bm{V}^{l} \left( \bm{V}^{l} \right)^\top \right)\right) - \mathsf{tr} \left( \eta\sL\ww^{l+1} \left(\bm{V}^{l+1} \left( \bm{V}^{l+1} \right)^\top \right)\right) }_{I_1} + \left \langle \mathcal{L}^\ast \bm S, \ww^l - \ww^{l+1} \right \rangle  \nonumber \\ 
    & + \underbrace{\left\langle \bm{y}^{l}, \mathfrak{d}\bm{w}^{l} - \bm{d} \right\rangle - \left\langle \bm{y}^{l+1}, \mathfrak{d}\bm{w}^{l+1} - \bm{d} \right\rangle}_{I_2} +\underbrace{\left \langle \bm{Y}^{l}, \boldsymbol{\Theta}^{l+1} - \mathcal{L}\bm{w}^{l} \right\rangle - \left \langle \bm{Y}^{l+1}, \boldsymbol{\Theta}^{l+1} - \mathcal{L}\bm{w}^{l+1} \right\rangle}_{I_3}  \nonumber  \\
    &+ \frac{\rho}{2}\norm{\bT^{l+1} - \mathcal{L}\bm{w}^{l}}^2_{\mathrm{F}} -\frac{\rho}{2} \norm{\bT^{l+1} - \mathcal{L}\bm{w}^{l+1}}^2_{\mathrm{F}}  + \frac{\rho}{2}\norm{\mathfrak{d}\bm{w}^{l} - \bm{d}}^2_{2} - \frac{\rho}{2}\norm{\mathfrak{d}\bm{w}^{l+1} - \bm{d}}^2_{2}. \label{eq:Lag-dif}
\end{align}
The term $I_1$ can be written as
\begin{align*}
I_1  = &~ \mathsf{tr}\left( \eta \sL\ww^{l} \left(\bm{V}^{l} \left( \bm{V}^l \right)^{\top} \right)\right) - \mathsf{tr} \left( \eta\sL\ww^{l+1} \left(\bm{V}^{l} \left(\bm{V}^l \right)^{\top} \right)\right)  \\
 & + \underbrace{\mathsf{tr} \left( \eta\sL\ww^{l+1} \left(\bm{V}^{l} \left( \bm{V}^l \right)^{\top} \right)\right) - \mathsf{tr} \left( \eta\sL\ww^{l+1} \left(\bm{V}^{l+1} \left( \bm{V}^{l+1}\right)^{\top} \right)\right)}_{I_{1a}}.
\end{align*}
Note that $\bm V^{l+1}$ is the optimal solution of the problem
\begin{equation}
   \underset{\bm{V} \in \mathbb{R}^{p\times k}}{\mathsf{min}}  ~
    \mathsf{tr}\left(\bm{V}^\top \sL\ww^{l+1} \bm{V}\right), \quad \mathsf{subject~to} \  \bm{V}^\top\bm{V} = \bm{I}.
  \label{eq:k-comp-proposed-relaxed-sub-V-l1}
\end{equation}
Thus the term $I_{1a} \geq 0$, and we can obtain
\begin{equation} \label{eq:I1}
I_1 \geq \mathsf{tr}\left( \eta \sL\ww^{l} \left(\bm{V}^{l} \left( \bm{V}^l \right)^{\top} \right)\right) - \mathsf{tr} \left( \eta\sL\ww^{l+1} \left(\bm{V}^{l} \left( \bm{V}^l \right)^{\top} \right)\right).  
\end{equation}
For the term $I_2$, we have
\begin{equation}\label{eq:I2}
\begin{split}
I_2 &= \left\langle \bm{y}^{l}, \mathfrak{d}\bm{w}^{l} - \bm{d} \right\rangle  -\left\langle \bm{y}^{l}, \mathfrak{d}\bm{w}^{l+1} - \bm{d} \right\rangle - \rho \left\langle \mathfrak{d}\bm{w}^{l+1} - \bm{d}, \mathfrak{d}\bm{w}^{l+1} - \bm{d} \right\rangle  \\
&=\left\langle \mathfrak{d}^\ast \bm{y}^{l}, \bm{w}^{l} - \bm{w}^{l+1} \right\rangle - \rho \norm{\mathfrak{d}\bm{w}^{l+1} - \bm{d}}_2^2,
\end{split}
\end{equation}
where the first equality is due to the updating of $\bm y^{l+1}$ as below
\begin{equation}\label{eq:update-y}
\bm y^{l+1} = \bm y^{l} + \rho \left( \mathfrak{d} \bm w^{l+1} - \bm d  \right).
\end{equation}
For the term $I_3$, similarly, we have
\begin{equation}\label{eq:I3}
\begin{split}
I_3 &= \left \langle \bm{Y}^{l}, \boldsymbol{\Theta}^{l+1} - \mathcal{L}\bm{w}^{l} \right\rangle - \left \langle \bm{Y}^{l}, \boldsymbol{\Theta}^{l+1} - \mathcal{L}\bm{w}^{l+1} \right\rangle  - \rho \left \langle \boldsymbol{\Theta}^{l+1} - \mathcal{L}\bm{w}^{l+1}, \boldsymbol{\Theta}^{l+1} - \mathcal{L}\bm{w}^{l+1} \right\rangle\\
&=\left \langle \mathcal{L}^\ast \bm{Y}^{l}, \bm{w}^{l+1}- \bm{w}^{l} \right\rangle - \rho \norm{\boldsymbol{\Theta}^{l+1} - \mathcal{L}\bm{w}^{l+1}}_{\mathrm{F}}^2,
\end{split}
\end{equation}
where the first equality follows from
\begin{equation}\label{eq:update-Y}
\bm Y^{l+1} = \bm Y^{l} + \rho \left( \boldsymbol{\Theta}^{l+1} - \mathcal{L}\bm{w}^{l+1} \right).
\end{equation}
Therefore, we can obtain
\begin{equation}\label{eq:I-23}
I_2 + I_3 = \left\langle \mathfrak{d}^\ast \bm{y}^{l} -\mathcal{L}^\ast \bm{Y}^{l}, \bm{w}^{l} - \bm{w}^{l+1} \right\rangle - \rho \norm{\mathfrak{d}\bm{w}^{l+1} - \bm{d}}_2^2 - \rho \norm{\boldsymbol{\Theta}^{l+1} - \mathcal{L}\bm{w}^{l+1}}_{\mathrm{F}}^2.
\end{equation}
Recall that $\ww^{l+1}$ is the optimal solution of the problem
\begin{equation}
\begin{split}
    \ww^{l+1} &=
    \underset{\ww \geq \mathbf{0}}{\mathsf{arg~min}} ~ \frac{\rho}{2}\ww^\top \left(\mathfrak{d}^*\mathfrak{d} + \mathcal{L}^*\mathcal{L}\right) \ww  \\
    &\quad \quad + \left\langle \ww, \sL^*\left(\bm{S} + \eta\bm{V}^{l} \left(\bm{V}^l \right)^{\top}- \bm{Y}^{l} - \rho \bT^l \right) + \mathfrak{d}^*\left(\bm{y}^{l} - \rho\bm{d}\right) \right\rangle.
    \label{eq:w-update-new-l1}
\end{split}
\end{equation}
Thus, $\ww^{l+1}$ satisfies the KKT system of \eqref{eq:w-update-new-l1} as below
\begin{align}
\rho \left(\mathfrak{d}^*\mathfrak{d} + \mathcal{L}^*\mathcal{L}\right) \ww^{l+1}  +  \sL^*\left(\bm{S} + \eta\bm{V}^{l} \left( \bm{V}^l \right)^{\top} - \bm{Y}^{l} - \rho \bT^l \right) + \mathfrak{d}^*\left(\bm{y}^{l} - \rho\bm{d}\right) - \bm \nu  = \bm 0;& \label{eq:kkt1} \\
w^{l+1}_i \nu_i = 0, \quad \mathrm{for} \ i =1, \ldots, p(p-1)/2;&  \label{eq:kkt2} \\
\ww^{l+1} \geq \bm  0, \quad \bm \nu \geq \bm 0;  \label{eq:kkt3} &
\end{align}
According to \eqref{eq:kkt1}, we have
\begin{equation} \label{eq:L-kkt}
\mathfrak{d}^\ast \bm{y}^{l} -\mathcal{L}^\ast \bm{Y}^{l} = - \rho \left(\mathfrak{d}^*\mathfrak{d} + \mathcal{L}^*\mathcal{L}\right) \ww^{l+1}  -  \sL^*\left(\bm{S} + \eta\bm{V}^{l} \left( \bm{V}^l \right)^{\top} - \rho \bT^{l+1} \right) + \rho\mathfrak{d}^* \bm{d} + \bm \nu.
\end{equation}
Together with \eqref{eq:I-23} and \eqref{eq:L-kkt}, we obtain
\begin{equation}\label{eq:I-23n}
\begin{split}
I_2 + I_3 &= \left\langle - \rho \left(\mathfrak{d}^*\mathfrak{d} +\mathcal{L}^*\mathcal{L}\right) \ww^{l+1}  -  \sL^*\left(\bm{S} + \eta\bm{V}^{l} \left( \bm{V}^l \right)^{\top} - \rho \bT^{l+1} \right) + \rho\mathfrak{d}^* \bm{d} , \bm{w}^{l} - \bm{w}^{l+1} \right\rangle \\
 & \quad \quad - \rho \norm{\mathfrak{d}\bm{w}^{l+1} - \bm{d}}_2^2 - \rho \norm{\boldsymbol{\Theta}^{l+1} - \mathcal{L}\bm{w}^{l+1}}_{\mathrm{F}}^2 + \left\langle \bm \nu, \bm{w}^{l} - \bm{w}^{l+1} \right\rangle \\
 & \geq \left\langle - \rho \left(\mathfrak{d}^*\mathfrak{d} + \mathcal{L}^*\mathcal{L}\right) \ww^{l+1}  -  \sL^*\left(\bm{S} + \eta\bm{V}^{l} \left( \bm{V}^l \right)^{\top} - \rho \bT^{l+1} \right) + \rho\mathfrak{d}^* \bm{d} , \bm{w}^{l} - \bm{w}^{l+1} \right\rangle \\
 & \quad \quad  - \rho \norm{\mathfrak{d}\bm{w}^{l+1} - \bm{d}}_2^2 - \rho \norm{\boldsymbol{\Theta}^{l+1} - \mathcal{L}\bm{w}^{l+1}}_{\mathrm{F}}^2,
\end{split}
\end{equation}
where the inequality is established by $\left\langle \bm \nu, \bm{w}^{l} - \bm{w}^{l+1} \right\rangle \geq 0$, which follows from the fact that $\left\langle \bm \nu, \bm{w}^{l+1} \right\rangle = 0$ according to \eqref{eq:kkt2}, and $\left\langle \bm \nu, \bm{w}^{l} \right\rangle \geq 0$ because $\bm \nu \geq \bm 0$ by \eqref{eq:kkt3} and $\bm w^{l} \geq \bm 0$.
Plugging \eqref{eq:I1} and \eqref{eq:I-23n} into \eqref{eq:Lag-dif}, by calculation we can obtain
\begin{equation}\label{eq:lag-dif-l}
\begin{split}
    &L_\rho(\bT^{l+1}, \ww^{l}, \bm{V}^{l}, \bm{Y}^{l}, \bm{y}^{l} ) -  L_\rho(\bT^{l+1}, \ww^{l+1}, \bm{V}^{l+1}, \bm{Y}^{l+1}, \bm{y}^{l+1} )   \\
  \geq & \frac{\rho}{2} \norm{\mathfrak{d} \ww^{l+1} - \mathfrak{d} \ww^l}_2^2 + \frac{\rho}{2} \norm{\mathcal{L}\ww^{l+1} - \mathcal{L} \ww^l}_2^2 - \rho \norm{\mathfrak{d} \ww^{l+1} - \bm d}_2^2 - \rho \norm{\mathcal{L}\ww^{l+1} -\boldsymbol{\Theta}^{l+1} }_{\mathrm{F}}^2 \\
  = & \frac{\rho}{2} \norm{\mathfrak{d} \ww^{l+1} - \mathfrak{d} \ww^l }_2^2 + \frac{\rho}{2} \norm{\mathcal{L}\ww^{l+1} - \mathcal{L} \ww^l}_{\mathrm{F}}^2- \frac{1}{\rho} \norm{\bm y^{l+1} - \bm y^l}_2^2 - \frac{1}{\rho} \norm{\bm Y^{l+1} - \bm Y^{l}}_{\mathrm{F}}^2,
\end{split}
\end{equation}
where the equality follows from \eqref{eq:update-y} and \eqref{eq:update-Y}.
If $\rho$ is sufficiently large such that
\begin{equation}\label{eq:rho_bound}
\rho \geq \max_l \left( \frac{2 c\left(  \norm{\bm y^{l+1} - \bm y^l}_2^2 + \norm{\bm Y^{l+1} - \bm Y^{l}}_{\mathrm{F}}^2\right)}{\norm{\mathfrak{d} \ww^{l+1} - \mathfrak{d} \ww^l }_2^2 + \norm{\mathcal{L}\ww^{l+1} - \mathcal{L} \ww^l}_{\mathrm{F}}^2} \right)^{\frac{1}{2}}
\end{equation}
holds with some constant $c>1$, together with \eqref{eq:Lag-l1}, we can conclude that
\begin{equation}\label{eq:Lag-descreasing}
L_\rho(\bT^{l}, \ww^{l}, \bm{V}^{l}, \bm{Y}^{l}, \bm{y}^{l} ) \geq  L_\rho(\bT^{l+1}, \ww^{l}, \bm{V}^{l}, \bm{Y}^{l}, \bm{y}^{l} ) \geq  L_\rho(\bT^{l+1}, \ww^{l+1}, \bm{V}^{l+1}, \bm{Y}^{l+1}, \bm{y}^{l+1} ) ,  \nonumber
\end{equation}
for any $l \in \mathbb{N}_+$. Note that besides a fixed parameter $\rho$, an alternative strategy is to increase the $\rho$ iteratively \citep{ying2017hankel} until the condition \eqref{eq:rho_bound} could be satisfied well.
\end{proof}

Now we are ready to prove Theorem \ref{k-component-convergence}.
By Lemma \ref{Lem: bounded-T}, the sequence $ \left\lbrace \left( \bm \Theta^l, \bm w^l, \bm V^l, \bm Y^l, \bm y^l \right) \right\rbrace$ is bounded. Therefore, there exists at least one convergent subsequence $ \left\lbrace \left( \bm \Theta^{l_s}, \bm w^{l_s}, \bm V^{l_s}, \bm Y^{l_s}, \bm y^{l_s} \right) \right\rbrace_{s \in \mathbb{N}}$, which converges to a limit point denoted by $\left\lbrace \left( \bm \Theta^{l_\infty}, \bm w^{l_\infty}, \bm V^{l_\infty}, \bm Y^{l_\infty}, \bm y^{l_\infty} \right) \right\rbrace$. By Lemma \ref{Lem: bounded-L}, we obtain that $L_\rho \left( \bm \Theta^l, \bm w^l, \bm V^l, \bm Y^l, \bm y^l \right)$ is monotonically decreasing and lower bounded, and thus is convergent. Note that the function $\log\mathrm{det^*} (\bT)$ is continuous over the set $\mathcal{S}=\left\lbrace \bT \in \mathcal{S}_+^p | \mathrm{rank}(\bT) = p-k \right\rbrace$. We can get 
\begin{equation*}
\lim_{l \to + \infty} L_\rho \left( \bm \Theta^{l}, \bm w^{l}, \bm V^{l}, \bm Y^{l}, \bm y^{l} \right) = L_\rho \left( \bm \Theta^{\infty}, \bm w^{\infty}, \bm V^{\infty}, \bm Y^{\infty}, \bm y^{\infty} \right) = L_\rho \left( \bm \Theta^{l_\infty}, \bm w^{l_\infty}, \bm V^{l_\infty}, \bm Y^{l_\infty}, \bm y^{l_\infty} \right).
\end{equation*}
The \eqref{eq:lag-dif-l}, \eqref{eq:rho_bound} and \eqref{eq:Lag-descreasing} together yields
\begin{equation}
\begin{split}
&L_\rho(\bT^{l}, \ww^{l}, \bm{V}^{l}, \bm{Y}^{l}, \bm{y}^{l} ) -  L_\rho(\bT^{l+1}, \ww^{l+1}, \bm{V}^{l+1}, \bm{Y}^{l+1}, \bm{y}^{l+1} ) \\
& \qquad \qquad \qquad \qquad \qquad \geq (c-1)\rho\left( \norm{\mathcal{L}\ww^{l+1} -\boldsymbol{\Theta}^{l+1} }_{\mathrm{F}}^2 + \norm{\mathfrak{d} \ww^{l+1} - \bm d}_2^2 \right).
\end{split}
\end{equation}
Thus, we obtain 
\begin{equation}\label{eq:limit_w}
\lim_{l \to + \infty} \norm{\mathcal{L}\ww^{l} -\boldsymbol{\Theta}^{l} }_{\mathrm{F}} = 0, \quad \mathrm{and} \quad \lim_{l \to + \infty}   \norm{\mathfrak{d} \ww^{l} - \bm d}_2 = 0.
\end{equation}
Obviously, $\norm{\mathcal{L}\ww^{l_s} -\boldsymbol{\Theta}^{l_s} }_{\mathrm{F}} \to 0$ and $\norm{\mathfrak{d} \ww^{l_s} - \bm d}_2 \to 0$ also hold for any subsequence as $s \to + \infty$, which implies that $\bm Y^{l_\infty}$ and $\bm y^{l_\infty}$ satisfy the condition of stationary point of $L_\rho(\bT, \ww, \bm{V}, \bm{Y}, \bm{y} )$ with respect to $\bm Y$ and $\bm y$, respectively. By \eqref{eq:update-Y} and \eqref{eq:update-y}, we also have
\begin{equation}\label{eq:limit_Y}
\lim_{l \to + \infty} \norm{ \bm Y^{l+1} - \bm Y^l}_{\mathrm{F}} = 0, \quad \mathrm{and} \quad \lim_{l \to + \infty}   \norm{ \bm y^{l+1} - \bm y^l}_2 = 0.
\end{equation}
Together with \eqref{eq:lag-dif-l}, we obtain
\begin{equation}\label{eq:lim_ww}
\lim_{l \to + \infty} \norm{\mathfrak{d} \ww^{l+1} - \mathfrak{d} \ww^l }_2 =0 \quad \mathrm{and} \quad  \norm{\mathcal{L}\ww^{l+1} - \mathcal{L} \ww^l}_{\mathrm{F}} =0.
\end{equation}
Recall that $\bm V^l$ contains the $k$ eigenvectors associated with the $k$ smallest eigenvalues of $\mathcal{L} \ww^l$, and thus it is easy to check that
\begin{equation}\label{eq:lim_V}
\lim_{l \to + \infty} \norm{\bm V^{l+1} - \bm V^l }_{\mathrm{F}} =0.
\end{equation}

For the limit point $\left\lbrace \left( \bm \Theta^{l_\infty}, \bm w^{l_\infty}, \bm V^{l_\infty}, \bm Y^{l_\infty}, \bm y^{l_\infty} \right) \right\rbrace$ of any subsequence $ \left\lbrace \left( \bm \Theta^{l_s}, \bm w^{l_s}, \bm V^{l_s}, \bm Y^{l_s}, \bm y^{l_s} \right) \right\rbrace_{s \in \mathbb{N}}$, $\bm \Theta^{l_\infty}$ minimizes the following subproblem
\begin{equation*}
\begin{split}
      \bT^{l_\infty} &= \underset{\substack{\mathsf{rank}(\bT) = p-k \\ \bT \succeq \mathbf{0}}}{\mathsf{arg~min}}  ~ - \log\mathrm{det^*}(\bT) + \left\langle\bT, \bm{Y}^{l_\infty-1} \right\rangle +
      \frac{\rho}{2}\norm{\bT - \mathcal{L}\bm{w}^{l_\infty-1}}^2_{\mathrm{F}} \\
      &= \underset{\substack{\mathsf{rank}(\bT) = p-k \\ \bT \succeq \mathbf{0}}}{\mathsf{arg~min}}  ~ - \log\mathrm{det^*}(\bT) + \left\langle\bT, \bm{Y}^{l_\infty} \right\rangle +  \frac{\rho}{2}\norm{\bT - \mathcal{L}\bm{w}^{l_\infty}}^2_{\mathrm{F}} \\
      & \qquad \qquad \qquad - \left\langle\bT, \bm{Y}^{l_\infty} - \bm{Y}^{l_\infty -1} \right\rangle  + \rho \left\langle \bT, \mathcal{L}\bm{w}^{l_\infty} - \mathcal{L}\bm{w}^{l_\infty -1} \right\rangle.  
      \end{split}
\end{equation*}
By \eqref{eq:limit_Y} and \eqref{eq:lim_ww}, we conclude that $\bT^{l_\infty} $ satisfies the condition of stationary point of $L_\rho(\bT, \ww, \bm{V}, \bm{Y}, \bm{y} )$ with respect to $\bT$. Similarly, $ \ww^{l_\infty} $ minimizes the subproblem
\begin{equation*}
\begin{split}
    \ww^{l_\infty} &=
    \underset{\ww \geq \mathbf{0}}{\mathsf{arg~min}} ~ \frac{\rho}{2}\ww^\top \left(\mathfrak{d}^*\mathfrak{d} + \mathcal{L}^\ast\mathcal{L}\right) \ww + \left\langle \ww,  \mathfrak{d}^\ast\left(\bm{y}^{l_\infty-1} - \rho\bm{d}\right) \right\rangle\\
    &\qquad \qquad+ \left\langle \ww, \sL^\ast\left(\bm{S} + \eta\bm{V}^{l_\infty - 1} \left( \bm{V}^{l_\infty -1} \right)^{\top}- \bm{Y}^{l_ \infty-1} - \rho \bT^{l_\infty} \right) \right\rangle \\
    &=\underset{\ww \geq \mathbf{0}}{\mathsf{arg~min}} ~ \frac{\rho}{2}\ww^\top \left(\mathfrak{d}^\ast\mathfrak{d} + \mathcal{L}^\ast\mathcal{L}\right) \ww  + \left\langle \ww, \mathfrak{d}^\ast\left(\bm{y}^{l_\infty} - \rho\bm{d}\right) \right\rangle\\
    &\qquad \qquad + \left\langle \ww, \sL^\ast\left(\bm{S} + \eta\bm{V}^{l_\infty} \left( \bm{V}^{l_\infty} \right)^{\top}- \bm{Y}^{l_ \infty} - \rho \bT^{l_\infty} \right)  \right\rangle  +   \left\langle \mathcal{L} \ww, \bm{Y}^{l_ \infty} - \bm{Y}^{l_ \infty -1}\right\rangle\\
    &\qquad \qquad  +  \left\langle \mathfrak{d} \ww, \bm{y}^{l_\infty-1}  - \bm{y}^{l_\infty}  \right\rangle  + \eta\left\langle \mathcal{L}\ww,  \bm{V}^{l_\infty-1} \left( \bm{V}^{l_\infty-1} \right)^{\top} - \bm{V}^{l_\infty} \left( \bm{V}^{l_\infty} \right)^{\top} \right\rangle.   
\end{split}
\end{equation*}
By \eqref{eq:limit_Y}, \eqref{eq:lim_ww} and \eqref{eq:lim_V}, $\ww^{l_\infty} $ satisfies the condition of stationary point of $L_\rho(\bT, \ww, \bm{V}, \bm{Y}, \bm{y} )$ with respect to $\ww$. $\bm V^\infty$ minimizes the subproblem
\begin{equation*}
   \bm V^{l_\infty} = \underset{\bm{V} \in \mathbb{R}^{p\times k}}{\mathsf{arg~min}}  ~
    \mathsf{tr}\left(\bm{V}^\top \sL\ww^{l_\infty} \bm{V}\right), \quad \mathsf{subject~to} \  \bm{V}^\top\bm{V} = \bm{I},
\end{equation*}
which implies that $\bm V^{l_\infty} $ satisfies the condition of stationary point of $L_\rho(\bT, \ww, \bm{V}, \bm{Y}, \bm{y} )$ with respect to $\bm V$. To sum up, we can conclude that any limit point $\left\lbrace \left( \bm \Theta^{l_\infty}, \bm w^{l_\infty}, \bm V^{l_\infty}, \bm Y^{l_\infty}, \bm y^{l_\infty} \right) \right\rbrace$ of the sequence generated by Algorithm \ref{alg:k-component-ky-fan} is a stationary point of $L_\rho(\bT, \ww, \bm{V}, \bm{Y}, \bm{y} )$.
\end{proof}

\subsection{Proof of Theorem \ref{Connected heavy-tailed graphs}}\label{sec:proof-connected-heavy-tail}
\begin{proof}
Similar to the proof of Theorem \ref{k-component-convergence}, we establish the boundedness of the sequence $ \left\lbrace \left( \bm \Theta^l, \bm w^l, \bm Y^l, \bm y^l \right) \right\rbrace$ generated by Algorithm \ref{alg:heavy-tail} in Lemma \ref{Lem: bounded-HT}, and the monotonicity of $ L_\rho \left( \bm \Theta^l, \bm w^l, \bm Y^l, \bm y^l \right)$ in Lemma \ref{Lem: bounded-HL}.

\begin{lemma} \label{Lem: bounded-HT}
The sequence $ \left\lbrace \left( \bm \Theta^l, \bm w^l, \bm Y^l, \bm y^l \right) \right\rbrace$ generated by Algorithm \ref{alg:heavy-tail} is bounded.
\end{lemma}

\begin{proof}
Let $\bm w^0$, $\bm Y^0$ and $\bm y^0$ be the initialization of the sequences $\left\lbrace \bm w^l \right\rbrace$, $\left\lbrace \bm Y^l\right\rbrace$ and $\left\lbrace \bm y^l \right\rbrace$, respectively, and $\norm{\bm w^0}$, $\norm{\bm Y^0}_{\mathrm{F}}$ and $\norm{\bm y^0}$ are bounded. 

We prove the boundedness of the sequence by induction. Notice that the subproblem for $\bm \Theta$ is the same with that in \eqref{eq:bt-update-k-comp-l}, and thus we can directly get the boundedness of $\norm{\bT^l}_{\mathrm{F}}$ for $l=1$.
The sequence $\left\lbrace \ww^l \right\rbrace$ is established by solving the subproblems
\begin{equation}
\begin{split}
    \underset{\ww \geq \mathbf{0}}{\mathsf{min}} & ~~ \frac{\rho}{2}\ww^\top\left(\mathfrak{d}^*\mathfrak{d} + \mathcal{L}^*\mathcal{L}\right)\ww
    - \left\langle \ww, \sL^*\left(\bm{Y}^{l-1} + \rho \bT^l\right) - \mathfrak{d}^*\left(\bm{y}^{l-1} - \rho\bm{d}\right) \right\rangle \\
     & ~ +
     \dfrac{p+\nu}{n}\sum_{i=1}^{n}\log\left(1 + \dfrac{\bm{x}^\top_{i,*}\sL\ww{\bm{x}_{i,*}}}{\nu}\right).
    \label{eq:w-update-heavy-tail-l}
    \end{split}
\end{equation}
Let 
\begin{equation}
g_l(\ww) = \frac{\rho}{2}\ww^\top\left(\mathfrak{d}^*\mathfrak{d} + \mathcal{L}^*\mathcal{L}\right)\ww + \left\langle \ww, \bm a^l \right\rangle +
\dfrac{p+\nu}{n}\sum_{i=1}^{n}\log\left(1 + \dfrac{\bm{x}^\top_{i,*}\sL\ww{\bm{x}_{i,*}}}{\nu}\right),
\end{equation}
where $\bm a^l =\sL^*\left(\bm{S} - \bm{Y}^{l-1} - \rho \bT^l \right) + \mathfrak{d}^*\left(\bm{y}^{l-1} - \rho\bm{d}\right)$.
Note that $\norm{\bm a^1}$ is bounded because $\norm{\bm{Y}^{0}}_{\mathrm{F}}$, $\norm{\bm{y}^{0}}$, and $\norm{\bT^1}_{\mathrm{F}}$ are bounded. In the proof of Lemma \ref{Lem: bounded-T}, we have shown that $\mathcal{L}^*\mathcal{L}$ is a positive definite matrix with the minimum eigenvalue $\lambda_{\min} \left( \mathcal{L}^*\mathcal{L}\right) = 2$, and $\mathfrak{d}^*\mathfrak{d}$ is a positive semi-definite matrix.
Since
$\log\left(1 + \dfrac{\bm{x}^\top_{i,*}\sL\ww{\bm{x}_{i,*}}}{\nu} \right) \geq 0$ for any $\ww \geq \bm 0$, we have
\begin{equation}
\lim_{\norm{\ww} \to +\infty} g_1(\ww) \geq \lim_{\norm{\ww} \to + \infty} \rho \ww^\top \ww + \langle \ww, \bm a^1\rangle = + \infty.
\end{equation}
Thus $g_1(\ww)$ is coercive. Recall that we solve the optimization \eqref{eq:w-update-heavy-tail-l} by the MM framework.
Hence, the objective function value is monotonically decreasing as a function of the iterations,
and $\ww^l$ is a stationary point of \eqref{eq:w-update-heavy-tail-l}. Then the coercivity of $g_1(\ww)$ yields the boundedness of $\norm{ \ww^1}$. Finally, $\norm{\bm{Y}^1}_{\mathrm{F}}$ and $\norm{\bm{y}^1}$ are also bounded, because $\bm{Y}^1$ and $\bm{y}^1$ are updated as done in \eqref{eq:Y-update-1} and \eqref{eq:y-update-1}, respectively, and thus the proof is the same.

Now we assume that $ \left\lbrace \left( \bm \Theta^{l-1}, \bm w^{l-1}, \bm Y^{l-1}, \bm y^{l-1} \right) \right\rbrace$ is
bounded for some $l \geq 1$, and check the boundedness of $\left\lbrace \left( \bm \Theta^l, \bm w^l, \bm Y^l, \bm y^l \right) \right\rbrace$.
Similar to the proof in \eqref{eq:bt-update-k-comp-l}, we can prove that $\norm{\bm \Theta^l}_{\mathrm{F}}$ is bounded.
By \eqref{eq:w-update-heavy-tail-l}, we can also obtain the boundedness of $\norm{\bm w^l}$.
We can also obtain that $\norm{\bm Y^l}$ and $\norm{\bm y}^l$ are bounded according to the boundedness
of $\norm{\bm \Theta^l}_{\mathrm{F}}$, $\norm{\bm w^l}$, $\norm{\bm Y^{l-1}}$ and $\norm{\bm y^{l-1}}$.
Thus, $\left\lbrace \left( \bm \Theta^l, \bm w^l, \bm Y^l, \bm y^l \right) \right\rbrace$ is bounded, completing the induction.
Therefore, we establish the boundedness of the sequence $\left\lbrace \left( \bm \Theta^l, \bm w^l, \bm Y^l, \bm y^l \right) \right\rbrace$.
\end{proof}

\begin{lemma} \label{Lem: bounded-HL}
The sequence $ L_\rho \left( \bm \Theta^l, \bm w^l, \bm Y^l, \bm y^l \right)$ generated by Algorithm \ref{alg:heavy-tail} is lower bounded, and
\begin{equation}
L_\rho \left( \bm \Theta^{l+1}, \bm w^{l+1}, \bm Y^{l+1}, \bm y^{l+1} \right) \leq  L_\rho \left( \bm \Theta^l, \bm w^l, \bm Y^l, \bm y^l \right), \quad \forall \, l \in \mathbb{N}_+,
\end{equation}
holds for any sufficiently large $\rho$.
\end{lemma}

\begin{proof}
According to \eqref{eq:student-t}, we have
\begin{align}\label{eq:tyler-l}
  L_\rho(\bT^l, \ww^l, \bm{Y}^l, \bm{y}^l) = &~ \dfrac{p+\nu}{n}\sum_{i=1}^{n}\log\left(1 + \dfrac{\bm{x}^\top_{i,*}\sL\ww^l{\bm{x}_{i,*}}}{\nu}\right)
  - \log\mathrm{det}\left(\boldsymbol{\Theta}^l + \bm{J}\right)
   + \left\langle \bm{y}^l, \mathfrak{d}\bm{w}^l - \bm{d} \right\rangle 
   \nonumber\\
  & + \frac{\rho}{2}\norm{\mathfrak{d}\bm{w}^l - \bm{d}}^2_{2} + \left\langle \bm{Y}^l, \boldsymbol{\Theta}^l - \mathcal{L}\bm{w}^l \right\rangle
  +\frac{\rho}{2}\norm{\bT^l - \mathcal{L}\bm{w}^l }^2_{\mathrm{F}}.
\end{align}
We can see that the lower boundedness of the sequence $ L_\rho \left( \bm \Theta^l, \bm w^l, \bm Y^l, \bm y^l \right)$ can be established by the boundedness of $\left\lbrace \left( \bm \Theta^l, \bm w^l, \bm Y^l, \bm y^l \right) \right\rbrace$ in Lemma \ref{Lem: bounded-HT}.

By a similar argument in the proof of Lemma \ref{Lem: bounded-L}, we can also establish that 
 \begin{equation} \label{eq:T_l1}
L_\rho \left( \bm \Theta^{l+1}, \bm w^{l},  \bm Y^{l}, \bm y^{l} \right) \leq  L_\rho \left( \bm \Theta^l, \bm w^l, \bm Y^l, \bm y^l \right), \quad \forall \, l \in \mathbb{N}_+.
\end{equation}

One has
\begin{align}
    &L_\rho(\bT^{l+1}, \ww^{l}, \bm{Y}^{l}, \bm{y}^{l} ) -  L_\rho(\bT^{l+1}, \ww^{l+1}, \bm{Y}^{l+1}, \bm{y}^{l+1} ) \nonumber \\
    =&\underbrace{\left\langle \bm{y}^{l}, \mathfrak{d}\bm{w}^{l} - \bm{d} \right\rangle - \left\langle \bm{y}^{l+1}, \mathfrak{d}\bm{w}^{l+1} - \bm{d} \right\rangle}_{I_1} +\underbrace{\left \langle \bm{Y}^{l}, \boldsymbol{\Theta}^{l+1} - \mathcal{L}\bm{w}^{l} \right\rangle - \left \langle \bm{Y}^{l+1}, \boldsymbol{\Theta}^{l+1} - \mathcal{L}\bm{w}^{l+1} \right\rangle}_{I_2} \nonumber  \\
     & + r\left( \mathcal{L}\bm w^l \right) - r \left( \mathcal{L} \bm w^{l+1} \right)  + \frac{\rho}{2}\norm{\mathfrak{d}\bm{w}^{l} - \bm{d}}^2_{2} - \frac{\rho}{2}\norm{\mathfrak{d}\bm{w}^{l+1} - \bm{d}}^2_{2} \nonumber \\ 
    &+ \frac{\rho}{2}\norm{\bT^{l+1} - \mathcal{L}\bm{w}^{l}}^2_{\mathrm{F}} -\frac{\rho}{2} \norm{\bT^{l+1} - \mathcal{L}\bm{w}^{l+1}}^2_{\mathrm{F}}, \label{eq:Lag-dif-k}
\end{align}
where $r(\bm L) = \dfrac{p+\nu}{n}\sum_{i=1}^{n}\log\left(1 + \dfrac{\bm{x}^\top_{i,*}\bm{L}{\bm{x}_{i,*}}}{\nu}\right)$. According to \eqref{eq:I2} and \eqref{eq:I3}, we obtain
\begin{equation}\label{eq:I-12K}
I_1 + I_2 = \left\langle \mathfrak{d}^\ast \bm{y}^{l} -\mathcal{L}^\ast \bm{Y}^{l}, \bm{w}^{l} - \bm{w}^{l+1} \right\rangle - \rho \norm{\mathfrak{d}\bm{w}^{l+1} - \bm{d}}_2^2 - \rho \norm{\boldsymbol{\Theta}^{l+1} - \mathcal{L}\bm{w}^{l+1}}_{\mathrm{F}}^2.
\end{equation}
According to the convergence result of the majorization-minimization framework \citep{sun2017}, we know that any limit point of the sequence is a stationary point of the following problem
\begin{equation}
    \underset{\ww \geq \mathbf{0}}{\mathsf{min}} ~ 
    \dfrac{p+\nu}{n}\sum_{i=1}^{n}\log\left(1 + \dfrac{\bm{x}^\top_{i,*}\sL\ww^l{\bm{x}_{i,*}}}{\nu}\right)
    + \frac{\rho}{2}\ww^\top \left(\mathfrak{d}^*\mathfrak{d} + \mathcal{L}^*\mathcal{L}\right) \ww  - \left\langle \ww, \sL^*\left(  \bm{Y}^{l} + \rho \bT^l \right) - \mathfrak{d}^*\left(\bm{y}^{l} - \rho\bm{d}\right) \right\rangle.
    \label{eq:w-update-new-l1-k}
\end{equation}
The set of the stationary points for the optimization \eqref{eq:w-update-new-l1-k} is defined by
\begin{equation} \label{eq:stationary}
\mathcal{X} = \left\lbrace \bm w | \nabla g_l (\bm w)^\top (\bm z - \bm w) \geq 0, \ \forall \bm z \geq \bm 0 \right\rbrace,
\end{equation}
where $g_l(\bm w) $ is the objective function in \eqref{eq:w-update-new-l1-k}. The existence of the limit point can be guaranteed by the the coercivity of $g_l(\ww)$, which has been established in the proof of Lemma \ref{Lem: bounded-HT}. Therefore, $\bm w^{l+1}$ is a stationary point. By taking $\bm z = \bm w^l$ and $\bm w = \bm w^{l+1}$ in \eqref{eq:stationary}, we obtain
\begin{equation*}
\left ( \mathcal{L}^\ast \left( \nabla r \left ( \mathcal{L}\bm w^{l+1} \right) \right) + \rho \left(\mathfrak{d}^*\mathfrak{d} + \mathcal{L}^\ast\mathcal{L}\right) \ww^{l+1}  -  \sL^\ast \left(  \bm{Y}^{l} + \rho \bT^l \right) + \mathfrak{d}^\ast \left(\bm{y}^{l} - \rho\bm{d}\right)  \right) ^\top \left (\bm w^l - \bm w^{l+1} \right) \geq 0.
\end{equation*}
Thus, we have
\begin{equation} \label{eq:stat-ineq}
\begin{split}
& \left\langle \mathfrak{d}^\ast \bm{y}^{l} -\mathcal{L}^\ast \bm{Y}^{l}, \bm{w}^{l} - \bm{w}^{l+1} \right\rangle \geq -\left\langle \nabla r \left (\sL\bm w^{l+1} \right) , \mathcal{L}\bm w^l - \mathcal{L}\bm w^{l+1} \right\rangle \\
& \qquad \qquad \qquad  + \rho \left\langle -\left(\mathfrak{d}^\ast \mathfrak{d} + \mathcal{L}^*\mathcal{L}\right) \ww^{l+1}  + \sL^\ast \bT^l + \mathfrak{d}^*\bm{d},  \bm w^l - \bm w^{l+1} \right\rangle,
\end{split}
\end{equation}
Plugging \eqref{eq:I-12K} and \eqref{eq:stat-ineq} into \eqref{eq:Lag-dif-k}, we obtain
\begin{equation}\label{eq:Lag-bound-T-l1}
\begin{split}
&L_\rho(\bT^{l+1}, \ww^{l}, \bm{Y}^{l}, \bm{y}^{l} ) -  L_\rho(\bT^{l+1}, \ww^{l+1}, \bm{Y}^{l+1}, \bm{y}^{l+1} ) \\
  \geq & \frac{\rho}{2} \norm{\mathfrak{d} \ww^{l+1} - \mathfrak{d} \ww^l}_2^2 + \frac{\rho}{2} \norm{\mathcal{L}\ww^{l+1} - \mathcal{L} \ww^l}_2^2 - \rho \norm{\mathfrak{d} \ww^{l+1} - \bm d}_2^2 - \rho \norm{\mathcal{L}\ww^{l+1} -\boldsymbol{\Theta}^{l+1} }_{\mathrm{F}}^2 \\
  &+ r\left( \mathcal{L}\bm w^l \right) - r \left( \mathcal{L} \bm w^{l+1} \right) -\left\langle \nabla r \left (\sL\bm w^{l+1} \right) , \mathcal{L}\bm w^l - \mathcal{L}\bm w^{l+1} \right\rangle  \\
 \geq & \frac{\rho}{2} \norm{\mathfrak{d} \ww^{l+1} - \mathfrak{d} \ww^l }_2^2 + \frac{\rho - L_r}{2} \norm{\mathcal{L}\ww^{l+1} - \mathcal{L} \ww^l}_{\mathrm{F}}^2- \frac{1}{\rho} \norm{\bm y^{l+1} - \bm y^l}_2^2 - \frac{1}{\rho} \norm{\bm Y^{l+1} - \bm Y^{l}}_{\mathrm{F}}^2,
\end{split}
\end{equation}
where the last inequality is due to the fact that $r(\bm L)$ is a concave function and has $L_r$-Lipschitz continuous gradient, in which $L_r>0$ is a constant, thus we have
\begin{equation}
r\left( \mathcal{L}\bm w^l \right) - r \left( \mathcal{L} \bm w^{l+1} \right) -\left\langle \nabla r \left (\sL\bm w^{l+1} \right) , \mathcal{L}\bm w^l - \mathcal{L}\bm w^{l+1} \right\rangle \geq -\frac{L_r}{2} \norm{\mathcal{L}\ww^{l+1} - \mathcal{L} \ww^l}_{\mathrm{F}}^2.
\end{equation}

By calculation, we obtain that if $\rho$ is sufficiently large such that
\begin{equation}\label{eq:rho-ht}
\rho \geq \max \left( L_r, \max_l \frac{L_r\norm{\mathcal{L}\ww^{l+1} - \mathcal{L} \ww^l}_{\mathrm{F}}^2 + \left( L_r^2 \norm{\mathcal{L}\ww^{l+1} - \mathcal{L} \ww^l}_{\mathrm{F}}^4 + 8 a b c  \right)^{\frac{1}{2}} }{2a} \right)
\end{equation}
holds with some constant $c>1$, where $a = \norm{\mathfrak{d} \ww^{l+1} - \mathfrak{d} \ww^l }_2^2 + \norm{\mathcal{L}\ww^{l+1} - \mathcal{L} \ww^l}_{\mathrm{F}}^2$
and $b = \norm{\bm y^{l+1} - \bm y^l}_2^2 + \norm{\bm Y^{l+1} - \bm Y^{l}}_{\mathrm{F}}^2 $, then, together with \eqref{eq:T_l1}, we conclude that
\begin{equation}\label{eq:Lag-descreasing-H}
L_\rho(\bT^{l}, \ww^{l}, \bm{Y}^{l}, \bm{y}^{l} ) \geq  L_\rho(\bT^{l+1}, \ww^{l}, \bm{Y}^{l}, \bm{y}^{l} ) \geq  L_\rho(\bT^{l+1}, \ww^{l+1}, \bm{Y}^{l+1}, \bm{y}^{l+1} ) ,  \nonumber
\end{equation}
for any $l \in \mathbb{N}_+$. Note that for the case of Gaussian distribution in Section \ref{subsec:connected-graph}, the constant $L_r$ will be zero, and $\rho$ in \eqref{eq:rho-ht} will be consistent with that in \eqref{eq:rho_bound}.
\end{proof}

Now we are ready to prove Theorem \ref{Connected heavy-tailed graphs}. By Lemma \ref{Lem: bounded-HT}, the sequence $ \left\lbrace \left( \bm \Theta^l, \bm w^l, \bm Y^l, \bm y^l \right) \right\rbrace$ generated by Algorithm \ref{alg:heavy-tail} is bounded. Therefore, there exists at least one convergent subsequence $ \left\lbrace \left( \bm \Theta^{l_s}, \bm w^{l_s}, \bm Y^{l_s}, \bm y^{l_s} \right) \right\rbrace_{s \in \mathbb{N}}$, which converges to the limit point denoted by $\left\lbrace \left( \bm \Theta^{l_\infty}, \bm w^{l_\infty}, \bm Y^{l_\infty}, \bm y^{l_\infty} \right) \right\rbrace$. By Lemma \ref{Lem: bounded-HL}, we obtain that the sequence $L_\rho \left( \bm \Theta^l, \bm w^l, \bm Y^l, \bm y^l \right)$ dedfined in \eqref{eq:tyler-l} is monotonically decreasing and lower bounded, implying that $L_\rho \left( \bm \Theta^l, \bm w^l, \bm Y^l, \bm y^l \right)$ is convergent. We can get $\lim_{l \to + \infty} L_\rho \left( \bm \Theta^{l}, \bm w^{l}, \bm Y^{l}, \bm y^{l} \right) = L_\rho \left( \bm \Theta^{l_\infty}, \bm w^{l_\infty}, \bm Y^{l_\infty}, \bm y^{l_\infty} \right)$. The \eqref{eq:Lag-bound-T-l1}, \eqref{eq:rho-ht} and \eqref{eq:Lag-descreasing-H} together yields
\begin{equation}
\begin{split}
&L_\rho(\bT^{l}, \ww^{l}, \bm{Y}^{l}, \bm{y}^{l} ) -  L_\rho(\bT^{l+1}, \ww^{l+1}, \bm{Y}^{l+1}, \bm{y}^{l+1} ) \\
& \qquad \qquad \qquad \qquad \qquad \geq (c-1)\rho\left( \norm{\mathcal{L}\ww^{l+1} -\boldsymbol{\Theta}^{l+1} }_{\mathrm{F}}^2 + \norm{\mathfrak{d} \ww^{l+1} - \bm d}_2^2 \right).
\end{split}
\end{equation}
The convergence of $L_\rho \left( \bm \Theta^l, \bm w^l, \bm Y^l, \bm y^l \right)$ yields 
\begin{equation}\label{eq:limit_w-h}
\lim_{l \to + \infty} \norm{\mathcal{L}\ww^{l} -\boldsymbol{\Theta}^{l} }_{\mathrm{F}} = 0, \quad \mathrm{and} \quad \lim_{l \to + \infty}   \norm{\mathfrak{d} \ww^{l} - \bm d}_2 = 0.
\end{equation}
By the updating of $\bm Y^{l+1}$ and $\bm y^{l+1}$, we can get
\begin{equation}\label{eq:limit_Y-h}
\lim_{l \to + \infty} \norm{ \bm Y^{l+1} - \bm Y^l}_{\mathrm{F}} = 0, \quad \mathrm{and} \quad \lim_{l \to + \infty}   \norm{ \bm y^{l+1} - \bm y^l}_2 = 0.
\end{equation}
Together with \eqref{eq:Lag-bound-T-l1}, we obtain
\begin{equation}\label{eq:lim_ww-h}
\lim_{l \to + \infty} \norm{\mathfrak{d} \ww^{l+1} - \mathfrak{d} \ww^l }_2 =0 \quad \mathrm{and} \quad  \norm{\mathcal{L}\ww^{l+1} - \mathcal{L} \ww^l}_{\mathrm{F}} =0.
\end{equation}

Similar to the proof of Theorem~\ref{k-component-convergence}, by \eqref{eq:limit_w-h}, \eqref{eq:limit_Y-h} and \eqref{eq:lim_ww-h}, we can prove that $\bT^{l_\infty}$, $\ww^{l_\infty}$, $\bm Y^{l_\infty}$ and $\bm y^{l_\infty}$ satisfy the condition of stationary point of $L_\rho(\bT, \ww, \bm{V}, \bm{Y}, \bm{y} )$ with respect to $\bT$, $\ww$, $\bm Y$ and $\bm y$, respectively. To sum up, we can conclude that any limit point $\left\lbrace \left( \bm \Theta^{l_\infty}, \bm w^{l_\infty}, \bm Y^{l_\infty}, \bm y^{l_\infty} \right) \right\rbrace$ of the sequence is a stationary point of $L_\rho(\bT, \ww, \bm{Y}, \bm{y} )$.
\end{proof}

\subsection{Proof of Theorem~\ref{heavy-tail-k-comp-convergence}}\label{sec:proof-heavy-tail-k-comp}
\begin{proof}
The proof of Theorem~\ref{heavy-tail-k-comp-convergence} is similar to the proof of Theorems \ref{k-component-convergence} and
\ref{Connected heavy-tailed graphs}. We can establish the boundedness of the sequence
$ \left\lbrace \left( \bm \Theta^l, \bm w^l, \bm V^l, \bm Y^l, \bm y^l \right) \right\rbrace$ generated by Algorithm \ref{alg:heavy-tail-k-comp} as done in Lemma \ref{Lem: bounded-T} and \ref{Lem: bounded-HT}. Similar to Lemmas \ref{Lem: bounded-L} and \ref{Lem: bounded-HL}, we can establish the monotonicity and boundedness of $ L_\rho \left( \bm \Theta^l, \bm w^l, \bm V^l, \bm Y^l, \bm y^l \right)$. Therefore, we omit the details of the proof of Theorem \ref{heavy-tail-k-comp-convergence} to avoid redundancy.
\end{proof}

\vskip 0.2in
\newpage
\bibliography{manuscript}

\end{document}